\documentclass[a4paper,twoside]{werner_thesis}
\usepackage[utf8]{inputenc}
\usepackage{hyperref}
\usepackage{graphicx}
\usepackage{epstopdf}
\usepackage{amsmath}
\usepackage{amssymb}
\usepackage{amsthm}
\usepackage{mathabx}
\usepackage{mathtools}
\usepackage{centernot}
\usepackage{cite}
\usepackage{nicefrac}
\usepackage{bbm}
\usepackage{thmtools}
\usepackage{thm-restate}
\usepackage{cleveref}
\usepackage{caption}
\usepackage{subcaption}
\usepackage[framemethod=TikZ]{mdframed}
\mdfsetup{skipabove=0pt,skipbelow=0pt}
\usepackage{multirow}
\usepackage{todonotes}
\usepackage[parfill]{parskip}
\usepackage{longtable}
\usepackage{needspace}
\usepackage{centernot}
\usepackage{thmtools}
\usepackage{thm-restate}
\usepackage[algoruled,boxed,lined,algochapter]{algorithm2e}
\usepackage{stmaryrd}
\usepackage{accents}
\usepackage{pdfpages}
\usepackage{fancyhdr}
\usepackage{eso-pic}

\newcommand\frontmatter{%
\cleardoublepage
\pagenumbering{roman}}

\newcommand\mainmatter{%
\cleardoublepage
\pagenumbering{arabic}}

\graphicspath{ {img/} }

\newcommand*\diff{\mathop{}\!\mathrm{d}}
\renewcommand{\vec}[1]{\mathbf{#1}}

\newcommand{\Pk}{\mathcal{P}}
\newcommand{\E}{\mathrm{E}}
\newcommand{\T}{^\text{T}}
\renewcommand{\H}{\mathcal{H}}
\newcommand{\x}{\vec x}
\newcommand{\w}{{\vec w}}
\newcommand{\y}{\vec y}

\newcommand{\vtheta}{\boldsymbol{\theta}}

\newcommand{\norm}[1]{\left\lVert#1\right\rVert}
\newcommand{\wrt}{w.\,r.\,t.}
\newcommand{\eg}{e.\,g.}
\newcommand{\ie}{i.\,e.}

\newcommand{\phim}{\boldsymbol \phi_m}
\newcommand{\Ms}{\mathcal{M}\!\left([0,1]^d\right)}
\newcommand{\MR}{\mathcal{M}\!\left(\mathbb{R}^d\right)}
\newcommand{\vc}{\mathrm{VC}(\mathcal{F})}

\newcommand\blfootnote[1]{%
  \begingroup
  \renewcommand\thefootnote{}\footnote{#1}%
  \addtocounter{footnote}{-1}%
  \endgroup
}
\DeclarePairedDelimiter\ceil{\lceil}{\rceil}
\DeclarePairedDelimiter\floor{\lfloor}{\rfloor}
\DeclareMathOperator*{\argmax}{arg\,max}
\DeclareMathOperator*{\argmin}{arg\,min}
\DeclareMathOperator{\sign}{sign}
\DeclareUnicodeCharacter{0301}{\'{e}}
\newcommand\BackgroundPic{
\put(0,0){
\parbox[b][\paperheight]{\paperwidth}{%
\vfill
\centering
\includegraphics[width=\paperwidth,height=\paperheight]{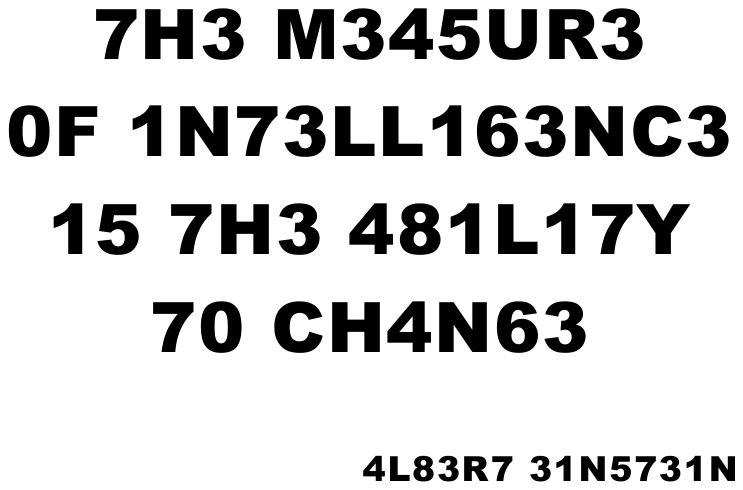}
\vfill
}}}

\renewenvironment{proof}{{\bfseries Proof.}}{\pushQED{\qed}\qedhere\popQED}

\newtheorem{example}{Example}[chapter]
\newtheorem{remark}{Remark}

\declaretheoremstyle[
    headfont=\bfseries,
    notebraces={(}{)},
    bodyfont=\normalfont,
    headpunct={},
    postheadspace=\newline,
    postheadhook={\textcolor{black}{\rule[.6ex]{\linewidth}{0.4pt}}\\},
    spacebelow=6pt,
    spaceabove=6pt,
    mdframed={
        backgroundcolor=white, 
            linecolor=black, 
            innertopmargin=6pt,
            roundcorner=0pt, 
            innerbottommargin=6pt, 
            skipabove=11pt, 
            skipbelow=6pt } 
]{myexamplestyle}

\declaretheoremstyle[
    headfont=\bfseries,
    notebraces={(}{)},
    bodyfont=\normalfont,
    headpunct={},
    postheadspace=\newline,
    postheadhook={\textcolor{black}{\rule[.6ex]{\linewidth}{0.4pt}}\\},
    spacebelow=6pt,
    spaceabove=6pt,
    mdframed={
        backgroundcolor=white, 
            linecolor=black, 
            innertopmargin=6pt,
            roundcorner=0pt, 
            innerbottommargin=6pt, 
            skipabove=11pt, 
            skipbelow=6pt } 
]{myexamplestyle}

\declaretheoremstyle[
    headfont=\bfseries,
    notebraces={(}{)},
    bodyfont=\normalfont,
    headpunct={},
    postheadspace=\newline,
    postheadhook={\textcolor{black}{\rule[.6ex]{\linewidth}{0.4pt}}\\},
    spacebelow=6pt,
    spaceabove=6pt,
    mdframed={
        backgroundcolor=gray!20, 
        linecolor=black, 
        innertopmargin=6pt,
        roundcorner=0pt, 
        innerbottommargin=6pt, 
        skipabove=11pt} 
]{myexamplestyle}

\declaretheoremstyle[
    headfont=\bfseries,
    notebraces={(}{)},
    bodyfont=\normalfont,
    headpunct={},
    postheadspace=\newline,
    postheadhook={\textcolor{black}{\rule[.6ex]{\linewidth}{0.4pt}}\\},
    spacebelow=6pt,
    spaceabove=6pt,
    mdframed={
        backgroundcolor=white, 
            linecolor=black, 
            innertopmargin=6pt,
            roundcorner=0pt, 
            innerbottommargin=6pt, 
            skipabove=11pt, 
            skipbelow=6pt } 
]{myexamplestyle}

\declaretheoremstyle[
    headfont=\bfseries,
    notebraces={(}{)},
    bodyfont=\normalfont,
    headpunct={},
    postheadspace=\newline,
    postheadhook={\textcolor{black}{\rule[.6ex]{\linewidth}{0.4pt}}\\},
    spacebelow=6pt,
    spaceabove=6pt,
    mdframed={
        backgroundcolor=white, 
            linecolor=black, 
            innertopmargin=6pt,
            roundcorner=0pt, 
            innerbottommargin=6pt, 
            skipabove=11pt, 
            skipbelow=6pt } 
]{myexamplestyle}

\begin{document}

\frontmatter


\includepdf[pages=-]{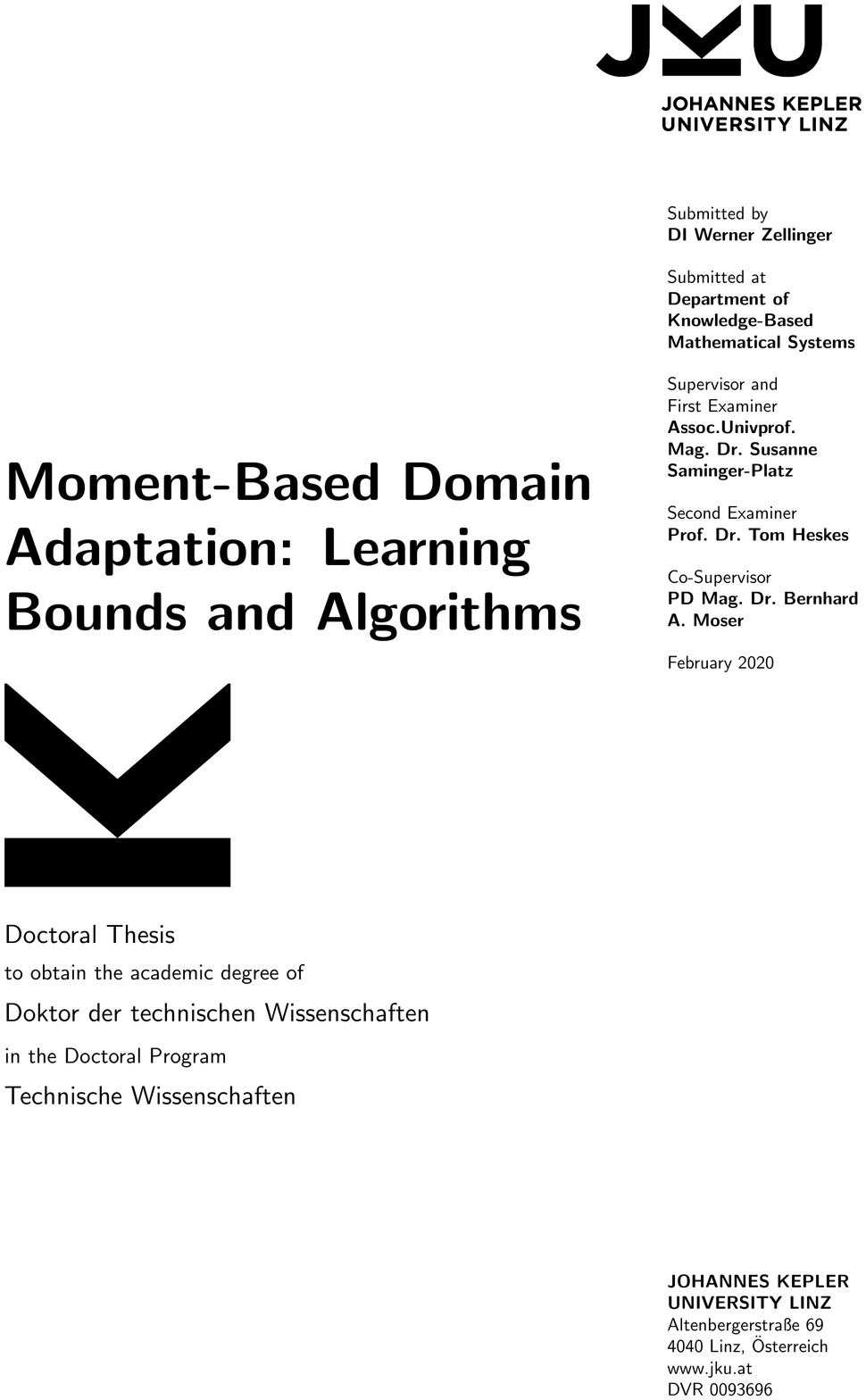}
\newpage

\chapter*{Statutory Declaration}

I hereby declare that the thesis submitted is my own unaided work, that I have not used other than the sources indicated, and that all direct and indirect sources are acknowledged as references.

This is a draft of the thesis which is similar to the final version.
\\\\\\\\
Linz, February 2020\hfill DI Werner Zellinger
\\\\

\newpage

\section*{Abstract}

This thesis contributes to the mathematical foundation of domain adaptation as emerging field in machine learning.
In contrast to classical statistical learning, the framework of domain adaptation takes into account deviations between probability distributions in the training and application setting.
Domain adaptation applies for a wider range of applications as future samples often follow a distribution that differs from the ones of the training samples.
A decisive point is the generality of the assumptions about the similarity of the distributions.
Therefore, in this thesis we study domain adaptation problems under as weak similarity assumptions as can be modelled by finitely many moments.

By examining the generalization ability of discriminative models trained under this relaxed assumption we establish, in the first part, a framework for bounding the misclassification risk based on finitely many moments and additional smoothness conditions.
Our results show that a low misclassification risk of the discriminative models can be expected if a) the misclassification risk on the training sample is small, b) the sample size is large enough, c) finitely many moments of the underlying distributions are similar, and d) the samples' distributions meet an additional entropy condition.

In the second part, we apply our theoretical framework to the design of machine learning algorithms for domain adaptation.
We propose a new moment distance for metric-based regularization of neural networks.
Our methods aim at finding new data representations such that our weak assumptions on the similarity of the distributions are satisfied.
In this context, various relations of the new moment distance to other probability metrics are proven.
Further, a bound on the misclassification risk of our method is derived.
To underpin the relevance of our theoretical framework, we perform empirical experiments on several large-scale benchmark datasets.
The results show that our method, though based on weaker assumptions, often outperforms related alternatives based on stronger assumptions on the similarity of distributions.

In the third part, we apply our framework on two industrial regression problems.
The first problem is settled in the area of industrial manufacturing.
We propose a new algorithm that is based on the similarity of the first moments of multiple different distributions.
Our algorithm enables the modeling of time series from previously unseen distributions and outperforms several standard regression algorithms on real-world data.
The second problem stems from the area of analytical chemistry.
We propose a new moment-based domain adaptation algorithm for the calibration of chemical measurement systems.
In contrast to standard approaches, our algorithm is only based on unlabeled data from the application system.
Theoretical properties of the proposed algorithm are discussed and it is shown to empirically outperform standard alternatives on two real-world datasets.

\newpage
\section*{Kurzfassung}

Diese Dissertation trägt zu den mathematischen Grundlagen des Bereichs "Domain Adaptation" bei, welcher einen aufstrebenden Teilbereich des Maschinellen Lernens bildet. Im Gegensatz zum klassischen Statistischen Lernen berücksichtigt das Framework Domain Adaptation auch Abweichungen zwischen den Wahrscheinlichkeits\-verteilungen der Trainings- und Anwendungsumgebung. Domain Adaptation kann damit in breiteren Bereichen eingesetzt werden, da Stichproben zukünftiger Daten oft einer anderen Wahrscheinlichkeitsverteilung folgen als Stichproben der Trainingsdaten. Ein wichtiger Punkt bei Domain Adaptation ist die Allgemeinheit der Annahmen über die Ähnlichkeit der Wahrscheinlichkeitsverteilungen. Aus diesem Grund studieren wir in dieser Dissertation Probleme von Domain Adaptation unter so schwachen Annahmen wie sie mit endlich vielen Momenten modelliert werden können.

Durch die Untersuchung der Generalisierungsfähigkeit von unterscheidenden Modellen, welche unter diesen verallgemeinerten Annahmen gelernt wurden, entwerfen wir im ersten Teil dieser Arbeit ein neues Framework, um obere Schranken für das Missklassifikationsrisiko zu finden. Diese neu beschriebenen oberen Schranken basieren auf endlich vielen Momenten und zusätzlichen Glattheitseigenschaften. Unsere Resultate zeigen, dass ein kleines Missklassifikationsrisiko von unterscheidenden Modellen erwartet werden kann, wenn a) das Missklassifikationsrisiko bezüglich der Trainingsstichprobe klein ist, b) die Stichprobengröße groß genug ist und c) die Wahrscheinlichkeitsverteilungen der Stichproben eine zusätzliche Entropieeigenschaft erfüllen.

Im zweiten Teil setzen wir unser Framework zur Entwicklung neuer Lernalgorithmen ein. Unter Anderem stellen wir eine neue, auf Momenten basierende Distanz für die Regularisierung von Neuronalen Netzen vor. Die von uns vorgestellten Methoden zielen darauf ab, neue Datenrepräsentationen zu finden, welche die im ersten Teil vorgestellten, schwachen Annahmen an die Ähnlichkeit von Wahrscheinlichkeitsverteilungen erfüllen. In diesem Kontext beweisen wir verschiedene Relationen zwischen der neuen, auf Momenten basierenden Distanz und anderen Distanzen auf Wahrscheinlichkeitsmaßen. Des Weiteren leiten wir mit Hilfe unseres Frameworks eine obere Schranke für das Missklassifikationsrisiko unserer Methode her. Um die Relevanz unseres theoretischen Frameworks zu untermauern, führen wir empirische Experimente auf zahlreichen großen Datenbanken durch. Die Resultate zeigen, dass unsere Methode, obwohl sie auf schwächeren Annahmen basiert, oft ähnliche alternative Methoden übertrifft, welche auf stärkeren Annahmen basieren.

Im dritten Teil wenden wir unser Framework auf zwei industrielle Regressionsprobleme an. Das erste Problem stammt aus dem Bereich der industriellen Produktion. Wir stellen einen neuen Algorithmus vor, der auf der Ähnlichkeit der ersten Momente von mehreren Wahrscheinlichkeitsverteilungen basiert. Unser Algorithmus ermöglicht die Modellierung von neuen, nicht der Wahrscheinlichkeitsverteilung der Trainingsdaten folgenden Zeitreihen und übertrifft, auf Datensätzen realer Problemstellungen, zahlreiche Standardregressionsalgorithmen. Das zweite Problem stammt aus dem Bereich der Analytischen Chemie. Wir stellen einen neuen, auf Momenten basierenden Algorithmus zur Kalibrierung chemischer Messsysteme vor. Im Gegensatz zu Standardalgorithmen basiert unser Algorithmus nur auf ungelabelten Daten des Anwendungsmesssystems. Wir diskutieren theoretische Eigenschaften des vorgestellten Algorithmus und zeigen, dass unser Algorithmus Standardalternativen oft übertrifft.

\newpage

\newpage
\section*{Acknowledgements}

I wish to express my sincere appreciation to my supervisor and first examiner, Susanne Saminger-Platz, who convincingly guided and encouraged me to aim at mathematical excellence and correctness even when the road got tough.
Without her untiring effort of providing detailed reviews, especially during several evenings of work, the goal of this thesis would not have been realized.

I would like to pay my special regards to my co-supervisor, Bernhard Moser, who continuously helped me to identify the core questions guiding my work.

I would like to thank my second examiner, Tom Heskes, for taking the time and effort to review my thesis.

The physical and technical contribution of the Software Competence Center Hagenberg GmbH is truly appreciated. Without their support and funding, this project could not have reached its goal.

I would like to thank my co-workers for all the exciting research projects we have done together.
This thesis is the result of various collaborations and would not have been possible without them.
In particular I would like to thank Thomas Natschläger, Thomas Grubinger, Michael Zwick and Ramin Nikzad-Langerodi from the Software Competence Center Hagenberg GmbH, Edwin Lughofer from the Department of Knowledge-Based Mathematical Systems, and, Hamid Eghbal-zadeh and Gerhard Widmer from the Institute of Computational Perception.

Additionally I would like to thank Sepp Hochreiter, Helmut Gfrerer, Florian Sobieczky, Johannes Himmelbauer, Ciprian Zavoianu, Robert Pollak, Paul Wiesinger and Laura Peham for their valuable feedback on my work.

Last but by no means least, I would like to thank my wife Marion and my son Jakob for pointing me to the most important things.

The research reported in this doctoral thesis has been supported by the Austrian Ministry for Transport, Innovation and Technology, the Federal Ministry for Digital and Economic Affairs, and the Province of Upper Austria in the frame of the COMET center SCCH. I also gratefully acknowledge the support of NVIDIA Corporation with the donation of a Titan X GPU used for this research.
\newpage

\AddToShipoutPicture*{\BackgroundPic}
\blfootnote{explanation on the last page}
\newpage 
\thispagestyle{empty}
\quad 
\newpage

\tableofcontents
\newpage

\thispagestyle{empty}
 
\listoffigures
 
\listoftables

\listofalgorithms
 
\newpage

\mainmatter

\chapter{Introduction}
\label{chap:introduction}

Inductive inference is to observe a phenomenon, to construct a model of that phenomenon and to make predictions using this model.
Indeed, this definition is very general and could roughly be taken as the goal of natural sciences.
\textit{Statistical learning} considers the process of inductive inference as a problem of estimating a desired dependency based on a finite sample.

Most results in statistical learning, both theoretical and empirical, assume an application sample that follows the same distribution as the training sample.
This assumption is violated in typical applications such as natural language processing, computer vision, industrial manufacturing and analytical chemistry.
\textit{Domain adaptation} extends the classical learning framework by allowing training and test samples which follow different distributions.

However, standard approaches study domain adaptation based on empirical estimations of strong similarity concepts between distributions.
It is the aim of this thesis to study domain adaptation under weak assumptions on the similarity of training and application distribution.

We model these assumptions based on moment distances which realize weaker similarity concepts than most other common probability metrics, see Figure~\ref{fig:metrics_with_moments}.

In our study we follow the four main components of statistical learning~\cite{vapnik2013nature}:

\begin{itemize}
	\item[(i)] We study conditions for the convergence of a discriminative learning process with increasing sample size.
	\item[(ii)] We give bounds describing the generalization ability of the learning process.
	\item[(iii)] We perform inductive inference based on the common principle of finding new data representations such that our weak assumptions are satisfied.
	\item[(iv)] We provide algorithms which follow our theoretical framework.
\end{itemize}
In particular, we start by describing the required preliminaries in Chapter~\ref{chap:background}.
The experienced reader is encouraged to skip this chapter and return to it if some background is missing.

In Chapter~\ref{chap:learning_bounds} we give conditions for the convergence of learning processes of discriminative models under the relaxed setting of weaker assumptions.
We provide upper bounds on the misclassification risk based on a moment distance and smoothness conditions on the underlying distributions.
We show that a small misclassification risk can be expected if the misclassification risk on the training sample is small, if the samples are large enough and its distributions have high entropy in the respective classes of densities sharing the same finite collection of moments.

In Chapter~\ref{chap:cmd_algorithm} we study the principle of learning new data representations such that all the samples' distributions have only finitely many moments in common.
We propose a new moment distance for metric-based regularization of neural networks.
Some relations of the new distance to other probability metrics are provided and a bound on the misclassification error of the new method is derived.
To underpin the relevance of our theoretical framework described in Chapter~\ref{chap:learning_bounds}, we perform empirical experiments on several large-scale benchmark datasets.
Results show that our method, though based on weaker assumptions, often outperforms related alternatives which are based on stronger concepts of similarity.

In Chapter~\ref{chap:applications}, we exploit our mathematical framework to come up with algorithms for two industrial regression problems.
The first problem is in the area of industrial manufacturing.
We propose a new algorithm that is based on the similarity of the first moments of multiple different distributions.
In contrast to standard regression methods, our algorithm enables the modeling of time series from previously unseen distributions.
The second problem is in the area of analytical chemistry.
We propose a new moment-based domain adaptation algorithm for the calibration of chemical measurement systems.
In contrast to standard approaches, our algorithm is only based on unlabeled application data.
Theoretical properties of the algorithm are discussed and it is shown to empirically outperform standard alternatives on two real-world datasets.

Chapter~\ref{chap:conclusion} concludes with a positioning of our research results from the point of view of current trends in statistical learning together with an outline of future research lines.

\begin{figure}[t]
	\centering
	\includegraphics[width=\linewidth]{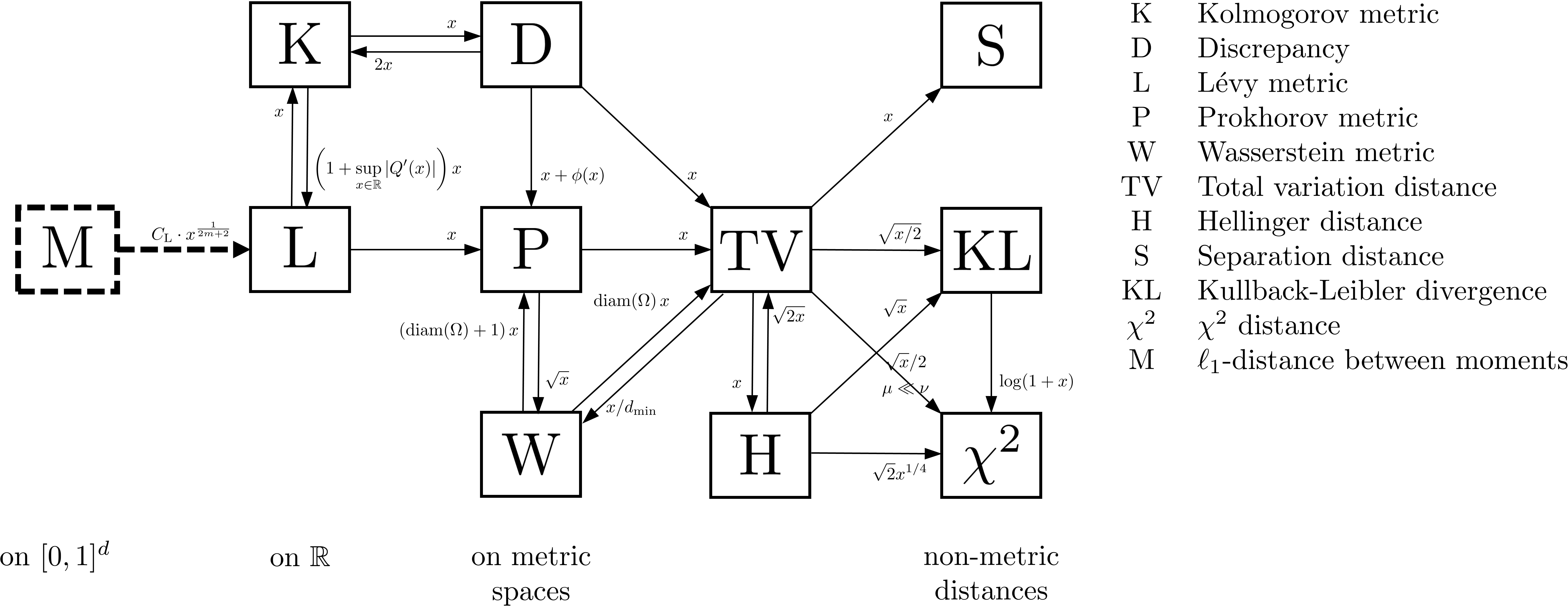}
	\caption[Relationships among probability metrics as illustrated in~\cite{gibbs2002choosing} and supplemented by a moment distance.]{Relationships among probability metrics as illustrated in~\cite{gibbs2002choosing} and supplemented by Lemma~\ref{lemma:moment_distance_bound_by_levy} (dashed).
		A directed arrow from $\text{A}$ to $\text{B}$ annotated by a function $h(x)$ means that $d_\text{A}\leq h(d_\text{B})$.
		For notations, restrictions and applicability see Section~\ref{sec:probability_metrics}.
	}
	\label{fig:metrics_with_moments}
\end{figure}

\section{Original Contribution}
\label{sec:original_contribution}

For the general interest of the reader, we now summarize the novel parts of our research, most of which have already been disseminated in scientific journals and conference proceedings.

The learning bounds for moment-based domain adaptation in Chapter~\ref{chap:learning_bounds} have been initially proposed in~\cite{zellinger2018compstat} and described at length in~\cite{zellinger2019amai}.

The metric for domain adaptation in Chapter~\ref{chap:cmd_algorithm} has been first proposed in~\cite{zellinger2017central} and described at length in~\cite{zellinger2017robust} with exception of the discussed relations to other probability metrics, \ie~Subsection~\ref{subsec:relation_of_cmd_to_other_probability_metrics} and Subsection~\ref{subsec:proof_relation_to_other_probability metrics}, which are completely new.
The source code of all experiments has been made publicly available\footnote{\url{https://github.com/wzell/mann} (accessed October 31, 2019)}.

The industrial applications in Chapter~\ref{chap:applications} have been published in~\cite{Zellinger2019multisource,nikzad2018domain,nikzad2019icmla,nikzad2020kbs}.
In particular, most of the work in Section~\ref{sec:cyclical_manufacturing} has been published in~\cite{Zellinger2019multisource}.
The details of the algorithm in Subsection~\ref{subsec:approach_scitsm} have been discovered through many years of industrial work by many of the included coauthors and the empirical evaluations in Subsection~\ref{subsec:use_case} have been mainly implemented by my coworkers Thomas Grubinger and Michael Zwick.
The algorithm as presented in Subsection~\ref{subsec:algo_dipls} and the implementations of the empirical evaluations as described in Subsection~\ref{subsec:experiments_dipls} are published in~\cite{nikzad2018domain,nikzad2019icmla} and are mainly due to my coworker Ramin Nikzad-Langerodi.
The learning bound in Subsection~\ref{subsec:learning_bound_dipls} together with the parameter heuristic in Subsection~\ref{subsec:parameter_setting} and parts of the discussion in Subsection~\ref{subsec:properties_dipls} are described in~\cite{nikzad2020kbs}.

\newpage

\section{Notation}
\label{sec:notations}

Most notations used in this work are either standard or defined on the spot.
This section provides our main conventions with a summarizing table at its end.
The reader is encouraged to skip this section and return to it if some notations are unclear.

We denote indices, natural numbers and abstract variables by lower case letters, \eg~$x$, $n$ and $\lambda$.

We denote by $\mathbb{N}$ the set of natural numbers including $0$, by $\mathbb{R}$ the set of real numbers and by $\mathbb{R}_+=\{x\in\mathbb{R}\mid x>0\}$ the set of positive real numbers.
We further denote by $\mathbb{R}^d$ the set of $d$-dimensional vectors over $\mathbb{R}$ and by $[0,1]^d\subset\mathbb{R}^d$ the $d$-dimensional unit cube.

Sets of functions are denoted by calligraphic letters, \eg~$\mathcal{F}$ and $\mathcal{G}$.

Finite multisets are denoted by uppercase letters, \eg~$X$ and $Y$.
The empty set is denoted by $\varnothing$, the union of two multisets $X$ and $Y$ is denoted by $X\cup Y$ and the cardinality of $X$ is denoted by $\left|X\right|$.

To emphasize that some objects are column vectors, we use boldface letters, \eg~$\vec x$ and $\boldsymbol{\phi}$.
The $i$-th element of a vector $\x$ is denoted by $x_i$.
We denote by $\x_1,\ldots,\x_n$ a sequence of $n$ vectors and by $x_{i,j}$ the $j$-th element of the $i$-th vector in the sequence.
We denote by $\mathrm{Span}(A)$ the linear span of a set $A$ of vectors.
We use upper case boldface letters for matrices, \eg~$\vec X$ and $\vec Y$, and denote its transpose by means of the letter $\text{T}$, \eg~${\vec X}^\text{T}$ and ${\vec Y}^\text{T}$.
The $i$-th element in the $j$-th column of a matrix $\vec X$ is denoted by $x_{i,j}$.

The element-wise multiplication of two vectors $\x$ and $\y$ is denoted by $\x\odot \y$.
The inner product between two vectors $\x$ and $\y$ on the Euclidean space $\mathbb{R}^d$ is denoted by $\langle\x,\y\rangle$.
The Euclidean norm, or $\ell^2$-norm, is denoted by $\norm{\x}_2=\sqrt{\langle\x,\x\rangle}$.
We denote the $\ell^1$-norm of $\x$ by $\norm{\vec x}_1=\sum_{i=1}^d \left|x_i\right|$.
The Frobenius norm of a matrix $\vec X\in\mathbb{R}^{n\times d}$ is denoted by $\norm{\vec X}_\mathrm{F}=\sqrt{\sum_{i=1}^n\sum_{j=1}^d \left|x_{i,j}\right|^2}$.

We denote by $(\Omega,d)$ a metric space with set $\Omega$ and metric $d$.

We denote by $(\Omega,\mathcal{A})$ a measurable space with set $\Omega$ and $\sigma$-algebra $\mathcal{A}$.
The Borel $\sigma$-algebra on a set $\Omega$ is denoted by $\mathcal{B}(\Omega)$.
For two probability measures $\mu$ and $\nu$ on $(\Omega,\mathcal{A})$ we denote by $\mu\ll\nu$ the property that $\mu$ is dominated by $\nu$, \ie~for all measurable sets $A$ it holds that $\nu(A)=0\implies\mu(A)=0$.
If they exits, we denote by $p$ and $q$ the density functions of $\mu$ and $\nu$, respectively.

We denote by $\int_\Omega f(\x)\diff\x$ the Lebesgue integral of a function $f:\Omega\to\mathbb{R}$ with $\Omega\subseteq\mathbb{R}^d$.
For example we often consider the integral $\int_{[0,1]^d} f(\x)\diff\x=\int_0^1\ldots\int_0^1 f(x_1,\ldots,x_d)\diff x_1\ldots\diff x_d$ on the unit cube $[0,1]^d$.
If the meaning is clear from the context we omit the support and the integration variables, \eg~we use $\int f$ to denote $\int_\Omega f(\x)\diff\x$.

We denote by $\mathcal{M}(\Omega)$ the set of all probability density functions \wrt~the Lebesgue reference measure and support $\Omega\subseteq\mathbb{R}^d$, \ie~the set of all functions $p:\Omega\to [0,\infty)$ with $\int_\Omega p(\x)\diff\x=1$.

Let $\mu$ be a measure on $(\Omega,\mathcal{B}(\Omega))$ with probability density function $p$ and $\Omega\subset\mathbb{R}$.
Let further $0<n<\infty$.
For
\begin{align*}
	\mathcal{L}^n=\left\{f:\Omega\to\mathbb{R}\,\middle|\, f~\text{measurable}, \int_\Omega \left|f(\x)\right|^n p(\x)\diff\x<\infty\right\}
\end{align*}
the $L^n$-norm is defined by
\begin{align*}
	\norm{.}_{L^n(p)}:\mathcal{L}^n&\to \mathbb{R}\\
	f&\mapsto \left(\int_\Omega \left|f(\x)\right|^n p(\x)\diff\x \right)^{1/n}.
\end{align*}
The $\infty$-norm of some function $f\in\left\{f:\Omega\to\mathbb{R}\mid f~\text{measurable}, \norm{f}_\infty<\infty\right\}$ is denoted by ${\norm{f}_\infty=\mathrm{ess}\sup_{\x\in\Omega} |f(\x)|}$.
If $\Omega=[0,1]^d$ is the unit cube and $\mu$ is the Lebesgue measure with uniform weight function $\tilde p$ we denote by $\norm{f}_{L^n}=\norm{f}_{L^n(\tilde p)}$ for simplicity.

$\mathbb{R}[x_1,\ldots,x_d]$ denotes the set of polynomials in the variables $x_1,\ldots,x_d$.
The maximum total degree $m$ of a polynomial
\begin{align*}
	\alpha_1\cdot x_1^{a_{1,1}}\cdots x_d^{a_{1,d}}+\ldots+\alpha_s\cdot x_1^{a_{s,1}}\cdots x_d^{a_{s,d}}\in\mathbb{R}[x_1,\ldots, x_d]
\end{align*}
with $\alpha_1,\ldots,\alpha_s\in\mathbb{R}$ and $a_{1,1},\ldots,a_{s,d}\in\mathbb{N}$ is $m=\max_{i\in\{1,\ldots,s\}} a_{i,1}+\cdots+a_{i,d}$.
Polynomials with maximum total degree $m$ are denoted by $\mathbb{R}_m[x_1,\ldots,x_d]$.
We often consider the vector space $\mathrm{Span}(\mathbb{R}_m[x_1]\cup\ldots\cup \mathbb{R}_m[x_d])$ of polynomials with only univariate terms of maximum total degree $m$.

For some polynomial $\phi\in \mathbb{R}[x_1,\ldots,x_d]$ and some vector $\boldsymbol{\alpha}\in\mathbb{R}^d$ we denote by $\phi(\boldsymbol{\alpha})$ the evaluation of the corresponding polynomial function at $\boldsymbol{\alpha}$.
For some vector $\boldsymbol{\phi}=(\phi_1,\ldots,\phi_n)^\text{T}$ of polynomials $\phi_1,\ldots,\phi_n\in\mathbb{R}[x_1,\ldots,x_d]$ we denote by $\boldsymbol{\phi}(\boldsymbol{\alpha})=\left( \phi_1(\boldsymbol{\alpha}),\ldots,\phi_n(\boldsymbol{\alpha}) \right)^\text{T}$ the vector of evaluations.

For some probability density function $p\in\mathcal{M}(\Omega)$ we denote by $\int_\Omega \boldsymbol{\phi}(\x) p(\x) \diff \x$, or sometimes just $\int\boldsymbol{\phi} p$,
the vector $\left(\int \phi_1 p,\ldots,\int \phi_n p\right)^\text{T}$.

We call a sequence $\phi_1, \ldots, \phi_n$ of polynomials with $\phi_1, \ldots, \phi_n\in \mathbb{R}[x_1,\ldots,x_d]$
\textit{orthonormal \wrt~a probability density $p\in\mathcal{M}(\Omega)$} if $\int_{\Omega}\phi_i(\x) \phi_j(\x) p(\x)\diff\x=0$ for $i\neq j$ and $\int_{\Omega}\phi_i(\x) \phi_j(\x) p(\x)\diff\x=1$ for $i=j$.
For simplicity, we call such a sequence \textit{orthonormal} if it is orthonormal \wrt~the uniform density on $[0,1]^d$, \ie~$p(\x)=1$ for $\x\in [0,1]^d$ and $p(\x)=0$ otherwise.

For a multiset $X=\{\x_1,\ldots,\x_n\}$ with $\x_1,\ldots,\x_n\in\mathbb{R}^d$ and a function $f:\mathbb{R}^d\to\mathbb{R}^s$, we denote by $f(X)=\{f(\x_1),\ldots,f(\x_n)\}$ the multiset consisting of the values of $f$ applied to each element in $X$.

Sometimes we use a probability density function $p$ as index of a $k$-sized multiset $X_p$ to emphasize that its elements are realizations of iid random variables with density $p$.
Such a multiset is called \textit{sample drawn from $p$}.
In this case, we denote by $\mathrm{E}[X_p]=\frac{1}{k}\sum_{\x\in X_p} \x$ the vector of arithmetic means of $X_p$.

For two functions $f:\mathbb{R}\to \mathbb{R}_+$ and $g:\mathbb{R}\to\mathbb{R}_+$ we write $f=O(g)$ if there exist $x_0, \alpha\in \mathbb{R}_+$ such that for all $x>x_0$ we have $f(x)\leq \alpha g(x)$.
Analogously we write $f=\Omega(g)$ if there exist $x_0, \alpha\in \mathbb{R}_+$ such that for all $x>x_0$ we have $f(x)\geq \alpha g(x)$.

The $r$-th derivative of a function $f:\mathbb{R}\to\mathbb{R}$ at $x$ is denoted by $f^{(r)}(x)=\frac{\diff^r\! f(x)}{\diff x^r}$.
We denote by $\partial^r_{x_i} f=\frac{\partial^r f}{\partial x_i^r}$ the $r$-th partial derivative in direction $x_i$ and by $\boldsymbol{D}^{\boldsymbol{\alpha}} f=\frac{\partial^{\alpha_1+\ldots+\alpha_d} f}{\partial x_1^{\alpha_1}\ldots\partial x_d^{\alpha_d}}$ the mixed partial derivative of some function $f:\mathbb{R}^d\to\mathbb{R}$ \wrt~some vector $\boldsymbol{\alpha}=(\alpha_1,\ldots,\alpha_d)^\text{T}\in\mathbb{N}^d$, especially $\boldsymbol{D}^{\boldsymbol{\alpha}} f=f$ for $\boldsymbol{\alpha}=(0,\ldots,0)^\text{T}$.

The factorial of some natural number $n$ is $n!=n\cdot(n-1)\cdots 1$.
The binomial coefficient of some natural number $n$ over some natural number $k$ is $\binom{n}{k}=\frac{n!}{(n-k)!\, k!}$.

We denote the number of monomials of total degree $m$ in $d$ variables by $\zeta(m,d)$. It is equal to the number of weak compositions and therefore $\zeta(m,d)=\binom{d+m-1}{m}$.
The number of monomials of maximum total degree $m$ in $d$ variables, excluding the monomial $1$ of degree $0$, is $\psi(m,d)=\sum_{i=1}^m \zeta(i,d)=\binom{d+m}{m}-1$.

We denote by $\argmax_{\x\in\Omega} f(\x)=\{\x\in \Omega\mid \forall \y\in \Omega: f(\y)\leq f(\x)\}$ the set of values $\x\in \Omega$ achieving the maximum of the function $f:\Omega\to\mathbb{R}$.
Analogously, we denote the set $\argmin_{\x\in\Omega} f(\x)=\{\x\in \Omega\mid \forall \y\in \Omega: f(\y)\geq f(\x)\}$.
If the set $\argmin_{\x\in\Omega} f(\x)$ has only one element, we write $\x=\argmin_{\x\in\Omega} f(\x)$ as abbreviation for $\{\x\}=\argmin_{\x\in\Omega} f(\x)$.

We use the multi-index notations $\x^{\boldsymbol{\alpha}}=x_1^{\alpha_1}\cdots x_d^{\alpha_d}$ and $\boldsymbol{\alpha}!=\alpha_1 !\cdots \alpha_d!$ for some vectors ${\x=(x_1,\ldots,x_d)^\text{T}\in\mathbb{R}^d}$ and $\boldsymbol{\alpha}=(\alpha_1,\ldots,\alpha_d)^\text{T}\in\mathbb{N}^d$.

The $n$-ary Cartesian product is denoted by $\bigtimes_{i=0}^n X_i=\left\{(x_1,\ldots,x_n)\mid x_i\in X_i, i\in \left\{1,\ldots,n\right\}\right\}$.

Let $x_1,x_2,\ldots$ be a sequence of elements in a set $A$ and $d$ be a metric on $A$. We say \textit{$x_1,x_2,\ldots$ converges in $d$ to $y$} iff $\lim_{i\to\infty} x_i=y$.
\\\\
\begin{longtable}{rl}
	Notation & Description\\
	\hline
	$\mathbb{N}$ & set of natural numbers including $0$\\
	$\mathbb{R}$ & set of real numbers\\
	$\mathbb{R}_+$ & $=\{x\in\mathbb{R}\mid x>0\}$, set of positive real numbers\\
	$\mathbb{R}^d$ & set of $d$-dimensional vectors over $\mathbb{R}$\\
	$[0,1]^d$ & unit cube of dimension $d$\\
	$\mathcal{F}$, $\mathcal{G}$ & sets of functions\\
	$X$, $Y$ & finite multisets of vectors\\
	$\varnothing$ & $=\{\}$, empty set\\
	$X\cup Y$ & union of multisets $X$ and $Y$\\
	$\left|X\right|$ & cardinality of set $X$\\
	$\mathrm{Span(A)}$ & linear span of set $A$ of vectors\\
	$\vec x$, $\vec y$, $\vec w$ & column vectors of real numbers\\
	$\vec X, \vec Y\in \mathbb{R}^{k\times d}$ & $k\times d$ matrices over $\mathbb{R}$\\
	${\vec X}^\text{T}$ & transpose of $\vec X$\\
	$x_i$ & $i$-th element of the vector $\vec x$\\
	$\x_1,\ldots, \vec {x}_n$ & sequence of $n$ vectors\\
	$x_{i,j}$ & the $j$-th element of the $i$-th vector in the sequence\\
	$\x\odot\y$ & $=(x_1 y_1,\ldots,x_d y_d)^\text{T}$, element-wise multiplication of vectors $\x$ and $\y$\\
	$\langle \vec x,\vec y\rangle$ & $=\sum_{i=1}^n x_i y_i$, inner product\\
	$\norm{\vec x}_2$ & $=\sqrt{\langle \vec x,\vec x\rangle}$, $\ell_2$-norm of $\vec x$\\
	$\norm{\vec x}_1$ & $=\sum_{i=1}^d \left|x_i\right|$, $\ell_1$-norm of $\vec x\in\mathbb{R}^d$\\
	$(\Omega, d)$ & metric space with set $\Omega$ and metric $d$\\
	$(\Omega,\mathcal{A})$ & measurable space with set $\Omega$ and $\sigma$-algebra $\mathcal{A}$\\
	$\mathcal{B}(\Omega)$ & Borel $\sigma$-algebra on $\Omega$\\
	$\mu\ll\nu$ & probability measure $\nu$ dominates probability measure $\mu$\\
	$p$, $q$ & density functions\\
	$\int f$ & $=\int_\Omega f(\x)\diff\x$, Lebesgue integral of function $f:\Omega\to\mathbb{R}$\\
	$\norm{f}_{L^2(p)}$ & $=\sqrt{\int_\Omega \left|f(\x)\right|^2 p(\x)\diff\x}$, $L^2(p)$-norm \wrt~density $p$\\
	$\norm{f}_{L^2}$ & $=\sqrt{\int_\Omega \left|f(\x)\right|^2\diff \x}$, $L^2$-norm \wrt~Lebesgue measure\\
	$\norm{f}_\infty$ & $=\mathrm{ess}\sup_{\vec x\in\mathbb{R}^d} f(\vec x)$, $\infty$-norm\\
	$\mathcal{M}(A)$ & set of probability density functions on $A\subset \mathbb{R}^d$\\
	$X_p$ & sample of $p\in\mathcal{M}(A)$ (see text)\\
	s.t. & abbreviation for \textit{subject to}\\
	a.e. & abbreviation for \textit{almost everywhere}\\
	iff & abbreviation for \textit{if and only if}\\
	iid & abbreviation for \textit{independent and identically distributed}\\
	$f(n)\to\alpha$ & abbreviation for $\lim_{n\to\infty} f(n)=\alpha$, pointwise convergence\\
	$f(X)$ & $=\{f(\x_1),\ldots,f(\x_n)\}$, function $f$ applied to multiset $X$\\
	$\mathrm{E}[X]$ & $=\frac{1}{k}\sum_{\x\in X} \x$, arithmetic mean of $k$-sized sample $X$\\
	$\mathbb{R}[x_1,\ldots,x_d]$ & set of polynomials in the variables $x_1,\ldots,x_d$\\
	$\left(\mathbb{R}[x_1,\ldots,x_d]\right)^n$ & set of $n$-dimensional vectors over $\mathbb{R}[x_1,\ldots,x_d]$\\
	$\int \boldsymbol{\phi}$ & $=(\int\phi_1,\ldots,\int\phi_n)^\text{T}$, vector of Lebesgue integrals\\
	$\mathbb{R}_m[x_1,\ldots,x_d]$ & set of polynomials with maximum total degree $m$\\
	$\boldsymbol{\phi}$, $\boldsymbol{\phi}_m$ & column vectors of polynomials\\
	$e$ & $=2.71828\ldots$, Euler's number\\
	$\log(x)$ & natural logarithm\\
	$\sign(x)$ & signum, equals $1$ iff $x>0$, $0$ iff $x=0$ and $-1$ iff $x<0$\\
	$O$ & asymptotic notation (see text)\\
	$\mathbbm{1}_A(\vec x)$ & function that equals $1$ iff $\vec x$ is in the set $A$ and $0$ otherwise\\
	$f^{(r)}(x)$ & $=\frac{\diff^r\! f(x)}{\diff x^r}$, $r$-th derivative of $f:\mathbb{R}\to\mathbb{R}$ at $x$\\
	$\partial^r_{x_i} f$ & $=\frac{\partial^r f}{\partial x_i^r}$, $r$-th partial derivative in direction $x_i$\\
	$\boldsymbol{D}^{\boldsymbol{\alpha}}$ & $=\frac{\partial^{\alpha_1+\ldots+\alpha_d} f}{\partial x_1^{\alpha_1}\ldots\partial x_d^{\alpha_d}}$, mixed partial derivative\\
	$n!$ & $=n\cdot(n-1)\cdots 1$, factorial of $n$\\
	$\binom{n}{k}$ & $=\frac{n!}{(n-k)!\, k!}$, binomial coefficient\\
	$\zeta(m,d)$ & $=\binom{d+m-1}{m}$, number of monomials of total degree $m$ in $d$ variables\\
	$\psi(m,d)$ & $=\binom{d+m}{m}-1$, number of monomials of maximum total degree $m$\\
	$\argmax_{\x\in \Omega} f(\x)$ & $=\{\x\in \Omega\mid \forall \y\in \Omega: f(\y)\leq f(\x)\}$\\
	$\argmin_{\x\in \Omega} f(\x)$ & $=\{\x\in \Omega\mid \forall \y\in \Omega: f(\y)\geq f(\x)\}$\\
	$\x^{\boldsymbol{\alpha}}$ & $=x_1^{\alpha_1}\cdots x_d^{\alpha_d}$, multi-index notation\\
	$\boldsymbol{\alpha}!$ & $=\alpha_1 !\cdots \alpha_d!$, multi-index notation\\
	$\bigtimes_{i=0}^n X_i$ & $=\left\{(x_1,\ldots,x_n)\mid x_i\in X_i, i\in \left\{1,\ldots,n\right\}\right\}$, Cartesian product\\
	\hline
\end{longtable}

\newpage

\chapter{Background}
\label{chap:background}

In this chapter we summarize the related work required for all the results and proofs of this thesis.
The experienced reader is encouraged to skip this chapter and return to it if some background is missing.

This chapter is structured as follows:
Section~\ref{sec:probability_metrics} describes related work about probability metrics.
Section~\ref{sec:stat_learn_th} reviews related work in statistical learning theory.
Section~\ref{sec:domain_adaptation} summarizes recent related work from the field of domain adaptation.
Section~\ref{sec:maximum_entropy_distribution} summarizes related work about the principle of maximum entropy applied on probability densities.
Section~\ref{sec:neural_networks} finalizes this chapter with related work on neural networks.

\section{Probability Metrics}
\label{sec:probability_metrics}

A central topic of this work is to quantify the distance between random elements.
Such distance concepts are called probability metrics~\cite{rachev2013methods}.
In this section, we discuss some examples of probability metrics, important properties and relationships among them.

This section is structured as follows:
Subsection~\ref{subsec:ten_probability_metrics} follows the work of Gibbs and Su~\cite{gibbs2002choosing} and describes ten important probability metrics.
Subsection~\ref{subsec:bounding_probability_metrics} reviews some important relationships among them and gives a summary in Figure~\ref{fig:metrics}.
Subsection~\ref{subsec:moment_distances} gives the notion of moment distances and some of its basic properties.

\subsection{Some Important Probability Metrics}
\label{subsec:ten_probability_metrics}

In this subsection, we follow Gibbs and Su~\cite{gibbs2002choosing} to review ten important probability metrics which have been proven to be useful and are depicted in Figure~\ref{fig:metrics}.
In these examples, we focus on distances between probability measures, \ie~\textit{simple probability metrics} rather than the broader class of probability metrics between random variables, \ie~\textit{compound probability metrics}~\cite{rachev2013methods}.
Note that many probability metrics are not metrics in the strict sense, but are simply notions of the dissimilarity between random elements.

In the following, let $(\Omega,\mathcal{A})$ denote a measurable space with state space $\Omega$ and $\sigma$-algebra $\mathcal{A}$.
Let $\mu$ and $\nu$ be two probability measures on $(\Omega,\mathcal{A})$ and $p, q$ be two corresponding density functions \wrt~some $\sigma$-finite dominating measure $\rho$.
For simplicity, we call $\mu$ and $\nu$ measures on $\Omega$ iff they are measures on $(\Omega,\mathcal{B}(\Omega))$ with Borel $\sigma$-algebra $\mathcal{B}(\Omega)$.
If $\Omega=\mathbb{R}$, let $P$ and $Q$ denote the corresponding cumulative distribution functions.
If $\Omega$ is a metric space with metric $d:\Omega\times\Omega\to[0,\infty)$, it will be understood as measurable space with Borel $\sigma$-algebra $\mathcal{B}(\Omega)$.
Recall that $d$ is a metric on $\Omega$ iff for all $x,y,z\in\Omega$ it holds that
\begin{align*}
	d(x,y)=0 \iff x=y,~d(x,y)=d(y,x)~\text{and}~d(x,z)\leq d(x,y)+d(y,z).
\end{align*}
If $\Omega$ is a bounded metric space, we denote by $\mathrm{diam}(\Omega)=\sup_{x,y\in\Omega}d(x,y)$ its diameter.

\begin{figure}[t]
	\centering
	\includegraphics[width=\linewidth]{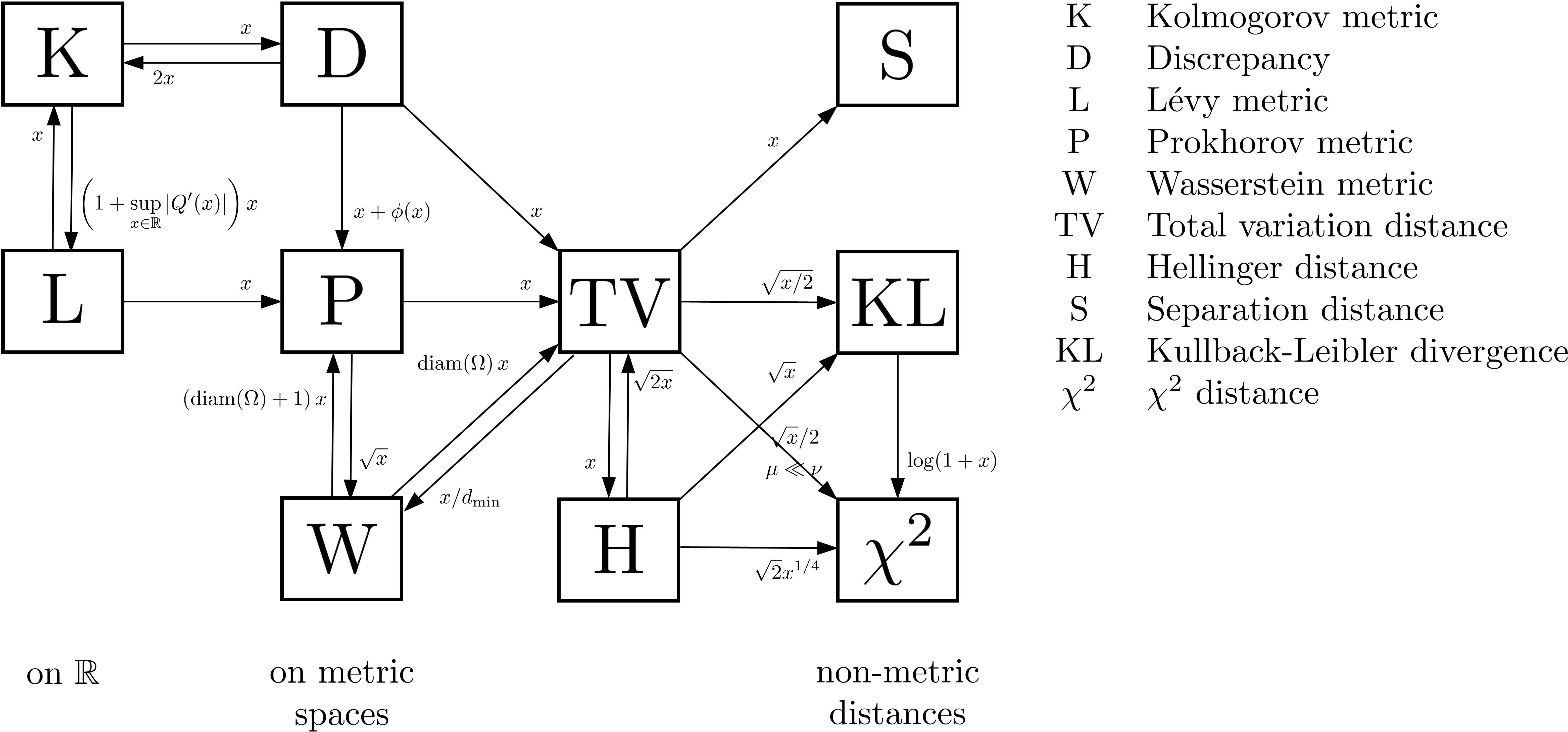}
	\caption[Relationships among probability metrics as illustrated in~\cite{gibbs2002choosing}.]{Relationships among probability metrics as illustrated in~\cite{gibbs2002choosing}.
		A directed arrow from $\text{A}$ to $\text{B}$ annotated by a function $h(x)$ means that $d_\text{A}\leq h(d_\text{B})$.
		For notations, restrictions and applicability see Section~\ref{sec:probability_metrics}.}
	\label{fig:metrics}
\end{figure}

\needspace{10\baselineskip}
\begin{restatable}[Discrepancy~\cite{weyl1916gleichverteilung,diaconis1988group}]{defrep}{discrepancy}%
	\label{def:discrepancy}%
	The discrepancy between two probability measures $\mu$ and $\nu$ on a metric space $\Omega$ is
	\begin{align}
		\label{eq:metric:discrepancy}
		d_\mathrm{D}(\mu,\nu)=\sup_{B_r(y)\in\mathcal{Q}}\left|\mu(B)-\nu(B)\right|,
	\end{align}
	where $\mathcal{Q}$ is the set of all closed balls $B=\{x\in\Omega\mid d(x,y)\leq r\}$ with $y\in\Omega$ and $r>0$.
\end{restatable}
The discrepancy assumes values in $[0,1]$ and is scale-invariant, \ie~multiplication with a positive constant does not affect the discrepancy.
The discrepancy has important applications in the study of random walks on groups~\cite{moser2014range}, as similarity measure in computer vision~\cite{moser2009similarity} and, as recently shown, in the foundation of bio-inspired threshold-based sampling~\cite{moser2017similarity,moser2019quasi}.
\Needspace{10\baselineskip}
\begin{restatable}[Hellinger Distance~\cite{hellinger1907orthogonalinvarianten}]{defrep}{hellingerdistance}%
	\label{def:Hellinger_distance}%
	The Hellinger distance between two probability measures $\mu$ and $\nu$ on a measurable space $\Omega$ is defined by
	\begin{align}
		\label{eq:metric:Hellinger_distance}
		d_\mathrm{H}(\mu,\nu)=\sqrt{\int_\Omega \left(p-q\right)^2\diff\rho}.
	\end{align}
\end{restatable}
The Hellinger distance does not depend on the choice of the dominating measure $\rho$.

\needspace{10\baselineskip}
\begin{restatable}[Kullback-Leibler Divergence~\cite{kullback1951information}]{defrep}{kldivergence}%
	\label{def:KL-divergence}%
	The Kullback-Leibler divergence (KL-divergence) between two probability measures $\mu$ and $\nu$ on a measurable space $\Omega$ is
	\begin{align}
		\label{eq:metric:KL_divergence}
		d_\mathrm{KL}(\mu,\nu)=\int_{S(\mu)} p\log\frac{p}{q}\diff\rho,
	\end{align}
	where $S(\mu)$ denotes the support of $\mu$.
\end{restatable}
The definition of the KL-divergence is independent of the choice of the dominating measure $\rho$.
The KL-divergence is not a metric as it is not symmetric and does not satisfy the triangle inequality.
However, it has many useful properties such as additivity over marginals, \ie~if $\mu=\mu_1\times\mu_2$ and $\nu=\nu_1\times\nu_2$ are measures on a product space $(\Omega_1\times\Omega_2,\mathcal{A}_1\otimes\mathcal{A}_2)$, then $d_\mathrm{KL}(\mu,\nu)=d_\mathrm{KL}(\mu_1,\nu_1)+d_\mathrm{KL}(\mu_2,\nu_2)$~\cite{cover2012elements}.
The KL-divergence is sometimes called \textit{relative entropy} and it was first introduced by Kullback and Leibler in~\cite{kullback1951information} as a measure of entropy.
It can be interpreted as the amount of information lost when identifying $\mu$ with the measure $\nu$~\cite{cover2012elements}.
The KL-divergence plays a central role in Chapter~\ref{chap:learning_bounds} of this work.

\Needspace{10\baselineskip}
\begin{restatable}[Kolmogorov Metric~\cite{kolmogorov1933sulla}]{defrep}{kolmogorovmetric}%
	\label{def:Kolmogorov_metric}%
	The Kolmogorov metric between two probability measures $\mu$ and $\nu$ on $\mathbb{R}$ is defined by
	\begin{align}
		\label{eq:metric:Kolmogorov}
		d_\mathrm{K}(\mu,\nu)=\sup_{x\in\mathbb{R}}\left|P(x)-Q(x)\right|,
	\end{align}
	where $P$ and $Q$ are the cumulative distribution functions of $\mu$ and $\nu$, respectively.
\end{restatable}%
The Kolmogorov metric assumes values in $[0,1]$, is invariant under all increasing one-to-one transformations of the real line and is sometimes called \textit{uniform metric}.
\Needspace{10\baselineskip}
\begin{restatable}[L\'evy Metric~\cite{levy1925probability}]{defrep}{levymetric}%
	\label{def:Levy_metric}%
	The L\'evy metric between two probability measures $\mu$ and $\nu$ on $\mathbb{R}$ is defined by
	\begin{align}
		\label{eq:metric:Levy}
		d_\mathrm{L}(\mu,\nu)=\inf\left\{\epsilon>0\,\middle|\, \forall x\in\mathbb{R}: P(x-\epsilon)-\epsilon\leq Q(x)\leq P(x+\epsilon)+\epsilon\right\}.
	\end{align}%
\end{restatable}
The L\'evy metric is shift-invariant and metrizes weak convergence of measures on $\mathbb{R}$.

\begin{restatable}[Prokhorov Metric~\cite{prokhorov1956convergence}]{defrep}{prokhorovmetric}%
	\label{def:Prokhorov_metric}%
	The Prokhorov metric between two probability measures $\mu$ and $\nu$ on a metric space $\Omega$ is defined by
	\begin{align}
		\label{eq:metric:Prokhorov}
		d_\mathrm{P}(\mu,\nu)=\inf\left\{\epsilon>0\,\middle|\, \forall B\in\mathcal{B}(\Omega): \mu(B)\leq \nu(B^\epsilon) +\epsilon\right\},
	\end{align}%
	where $B^\epsilon=\left\{x\in\Omega\,\middle|\, \inf_{y\in B} d(x,y)\leq \epsilon\right\}$.
\end{restatable}
The Prokhorov metric was introduced as the analogue to the L\'evy metric for more general spaces.
This metric is theoretically important because it metrizes weak convergence of measures on any separable metric space $(\Omega,d)$, \ie~any metric space that contains a countable and dense subset.

\begin{restatable}[Separation Distance~\cite{aldous1987strong}]{defrep}{separationdistance}%
	\label{def:separation_distance}%
	The separation distance between two probability measures $\mu$ and $\nu$ on a countable measurable space $\Omega$ is defined by
	\begin{align}
		\label{eq:metric:Separation}
		d_\mathrm{S}(\mu,\nu)=\max_{i\in\Omega}\left(1-\frac{\mu(i)}{\nu(i)}\right).
	\end{align}%
\end{restatable}
The separation distance is not a metric.
However, it is important in the study of Markov chains.
\Needspace{10\baselineskip}
\begin{restatable}[Total Variation Distance~\cite{aldous1987strong}]{defrep}{totalvariationdistance}%
	\label{def:TV_distance}%
	The total variation distance between two probability measures $\mu$ and $\nu$ on a measurable space $\Omega$ is defined by
	\begin{align}
		\label{eq:metric:total_variation}
		d_\mathrm{TV}(\mu,\nu)=\sup_{A\in\mathcal{A}} \left|\mu(A)-\nu(A)\right|,
	\end{align}%
	where $\mathcal{A}$ is the $\sigma$-algebra on $\Omega$.
\end{restatable}
The total variation distance assumes values in $[0,1]$.
The following theorem is useful for this work as it allows to focus on the $L^1$-difference between probability density functions when applying the total variation distance.
\Needspace{10\baselineskip}
\begin{restatable}[Total Variation Distance, see \eg~\cite{Tsybakov:2008:INE:1522486}]{thmrep}{thmtotalvariationdistance}%
	\label{thm:TV_distance}%
	Let $\mu$ and $\nu$ be probability measures on $\mathbb{R}^d$ with respective probability density functions $p$ and $q$ \wrt~the Lebesgue reference measure. Then the following holds:
	\begin{align}
		\label{eq:TV_and_L1_equality}
		d_\mathrm{TV}(\mu,\nu)=\frac{1}{2}\max_{h:\Omega\to [-1,1]} \left|\int h\diff\mu-\int h\diff\nu\right|=\frac{1}{2}\norm{p-q}_{L^1}.
	\end{align}
\end{restatable}

\begin{restatable}[Wasserstein Metric~\cite{dudley2002real}]{defrep}{wassersteinmetric}%
	\label{def:Wasserstein_metric}%
	The Wasserstein metric between two probability measures $\mu$ and $\nu$ on a separable metric space $\Omega$ with metric $d$ is defined by
	\begin{align}
		\label{eq:metric:Wasserstein}
		d_\mathrm{W}(\mu,\nu)=\sup_{h\in\mathcal{W}}\left| \int h\diff\mu-\int h\diff\nu \right|,
	\end{align}%
	where $\mathcal{W}=\left\{h:\Omega\to\mathbb{R}\mid \norm{h}_\mathrm{L}\leq 1 \right\}$ and $\norm{h}_\mathrm{L}=\sup_{x,y\in\Omega, x\neq y} \frac{\left|h(x)-h(y)\right|}{d(x,y)}$.
\end{restatable}
The Wasserstein distance has found applications in information theory, mathematical statistics, mass transportation problems and is also called as the \textit{earth mover’s distance} in engineering applications, see \eg~\cite{sriperumbudur2009integral} for further references.

\begin{restatable}[$\chi^2$-Distance~\cite{csiszar1967information}]{defrep}{chisquaredistance}%
	\label{def:Chisquare_distance}%
	The $\chi^2$-distance between two probability measures $\mu$ and $\nu$ on a measurable space $\Omega$ is defined by
	\begin{align}
		\label{eq:metric:chisquare}
		d_{\chi^2}(\mu,\nu)=\int_{S(\mu)\cup S(\nu)} \frac{(p-q)^2}{q}\diff\rho,
	\end{align}
	where $S(\mu)$ and $S(\nu)$ denote the supports of $\mu$ and $\nu$.
\end{restatable}
Definition~\ref{def:Chisquare_distance} is independent of the choice of the dominating measure $\rho$.
The $\chi^2$-distance is not symmetric in $\mu$ and $\nu$.
The $\chi^2$-distance has origins in mathematical statistics dating back to Pearson.

It is interesting to observe that several distances in this subsection are instances of a larger class of probability metrics called $f$-divergences~\cite{csiszar1967information}.
For any convex function $f$ with $f(1)=0$, define
\begin{align}
	d_f(\mu,\nu)=\int_\Omega  f\!\left(\frac{p}{q}\right) q\diff\rho.
\end{align}
Choosing $f(t)=(t-1)^2$ yields the $\chi^2$-distance, $f(t)=(\sqrt{t}-1)^2$ the squared Hellinger distance, $f(t)=t\log t$ the KL-divergence and $f(t)=\frac{1}{2}\left|t-1\right|$ the total variation distance.

Another important class of probability metrics are \textit{integral probability metrics}~\cite{muller1997integral}.
For any set $\mathcal{F}$ of real-valued bounded measurable functions on $\Omega$, define
\begin{align}
	\label{eq:integral_probability_metrics}
	d_\mathcal{F}(\mu,\nu)=\sup_{f\in\mathcal{F}}\left|\int f \diff\mu-\int f\diff\nu \right|.
\end{align}
Choosing $\mathcal{F}=\left\{h:\Omega\to\mathbb{R}\mid \norm{h}_\mathrm{L}\leq 1\right\}$ yields the Wasserstein metric and $\mathcal{F}=\{h:\Omega\to\mathbb{R}\mid \norm{f}_\infty\leq 1\}$ yields the total variation distance, see Theorem~\ref{thm:TV_distance}.
The total variation distance is the only non-trivial $f$-divergence that is also an integral probability metric~\cite{sriperumbudur2009integral}.
In statistics, integral probability metrics are called \textit{maximum mean discrepancy} if $\mathcal{F}$ is the unit ball of a reproducing kernel Hilbert space (RKHS)~\cite{gretton2012kernel,sriperumbudur2009integral}.
A Hilbert space $\mathcal{H}$ is called RKHS, iff there exists a function $\kappa:\Omega\times\Omega\to\mathbb{R}$ satisfying
\begin{align*}
	\forall y\in\Omega:\kappa(.,y)\in\mathcal{H}\quad\text{and}\quad\forall y\in\Omega,\forall f\in\mathcal{H}:\langle f,\kappa(.,y)\rangle_\mathcal{H}=f(y).
\end{align*}
The function $\kappa$ is called \textit{reproducing kernel} of $\mathcal{H}$.
Important examples of kernels on $\mathbb{R}^d$ are linear kernels $\kappa(\x,\y)=\langle\x,\y\rangle + b$ with bias $b\in\mathbb{R}$, polynomial kernels $\kappa(\x,\y)=(\langle\x,\y\rangle + b)^m$ of order $2\leq m\in\mathbb{N}$ and Gaussian kernels $\kappa(\x,\y)=\exp\!\left(\frac{\norm{\x-\y}^2}{2\sigma^2}\right)$ with bandwidth $\sigma\in\mathbb{R}$.

The main results proposed in this work focus on $\Omega=\mathbb{R}^d$, the Borel $\sigma$-algebra $\mathcal{B}(\mathbb{R}^d)$ and probability measures $\mu,\nu$ which admit probability density functions $p,q\in\MR$ \wrt~the Lebesgue reference measure $\rho$.
Throughout this work, whenever possible, we therefore express probability metrics as distances between densities.
For example, we denote $d_\mathrm{KL}(p,q)=d_\mathrm{KL}(\mu,\nu)$ and $d_\mathrm{TV}(p,q)=d_\mathrm{TV}(\mu,\nu)$.

\subsection{Bounds}
\label{subsec:bounding_probability_metrics}

Our goal of weak similarity assumptions between probability measures requires some intuition about the strength of probability metrics.
In this subsection, we review some relationships among the probability metrics proposed in Subsection~\ref{subsec:ten_probability_metrics}.
We follow Gibbs and Su~\cite{gibbs2002choosing} who give an illustrative summary of these relationships which we summarize in the following three theorems and Figure~\ref{fig:metrics}.

\Needspace{15\baselineskip}
\begin{restatable}[Relationships on Measurable Spaces, see \eg~\cite{gibbs2002choosing}]{thmrep}{relationshipsmeasurable}%
	\label{thm:relationships_measurable}%
	Let $\mu$ and $\nu$ be probability measures on a measurable space $\Omega$ with densities $p$ and $q$, respectively. Then the following holds:
	\begin{align}
		\label{eq:relationship:Pinsker}
		d_\mathrm{TV}(\mu,\nu) &\leq \sqrt{\frac{1}{2} d_\mathrm{KL}(\mu,\nu)}\\
		\label{eq:hellinger_and_tv_relation}
		d_\mathrm{TV}(\mu,\nu) &\leq d_\mathrm{H}(\mu,\nu) \leq \sqrt{2 d_\mathrm{TV}(\mu,\nu)}\\
		d_\mathrm{H}(\mu,\nu) &\leq \sqrt{d_\mathrm{KL}(\mu,\nu)}\\
		d_\mathrm{KL}(\mu,\nu) &\leq \log(1+d_{\chi^2}(\mu,\nu))\\
		d_\mathrm{H}(\mu,\nu) &\leq \sqrt{2}d_{\chi^2}(\mu,\nu)^{1/4}.
	\end{align}%
	If $\mu$ is dominated by $\nu$ it further holds that:
	\begin{align}
		d_\mathrm{TV}(\mu,\nu) &\leq \frac{1}{2}\sqrt{d_{\chi^2}(\mu,\nu)}.
	\end{align}
\end{restatable}
Recall that $\mu$ is dominated by $\nu$, denoted by $\mu\ll\nu$, iff $\nu(A)=0$ implies $\mu(A)=0$ for all measurable sets $A$.

\begin{restatable}[Relationships on Metric Spaces, see \eg~\cite{gibbs2002choosing}]{thmrep}{relationshipsmetricspaces}%
	\label{thm:relationships_metric_spaces}%
	Let $\mu$ and $\nu$ be probability measures on a metric space $\Omega$ with metric $d$. Then the following holds:
	\begin{align}
		d_\mathrm{P}(\mu,\nu)^2 &\leq d_\mathrm{W}(\mu,\nu) \leq (\mathrm{diam}(\Omega) + 1)\, d_\mathrm{P}(\mu,\nu)\\
		d_\mathrm{D}(\mu,\nu) &\leq d_\mathrm{TV}(\mu,\nu)\\
		d_\mathrm{P}(\mu,\nu) &\leq d_\mathrm{TV}(\mu,\nu)\\
		d_\mathrm{W}(\mu,\nu) &\leq \mathrm{diam}(\Omega)\, d_\mathrm{TV}(\mu,\nu),
	\end{align}%
	where $\mathrm{diam}(\Omega)=\sup_{x,y\in\Omega} d(x,y)$. If $\nu$ satisfies $\nu(B^\epsilon)\leq \nu(B)+\phi(\epsilon)$ for all $B\in\mathcal{B}(\Omega)$, $B^\epsilon=\left\{x\in\Omega\,\middle|\, \inf_{y\in B} d(x,y)\leq \epsilon\right\}$ and some right-continuous function $\phi$ then
	\begin{align}
		d_\mathrm{D}(\mu,\nu) &\leq d_\mathrm{P}(\mu,\nu) + \phi(d_\mathrm{P}(\mu,\nu)).
	\end{align}
	If $\Omega$ is finite then the following holds:
	\begin{align}
		d_{\min{}}\, d_\mathrm{TV}(\mu,\nu) &\leq d_\mathrm{W}(\mu,\nu),
	\end{align}%
	where $d_{\min{}} = \min_{x,y\in\Omega, x\neq y} d(x,y)$.
\end{restatable}

\Needspace{10\baselineskip}
\begin{restatable}[Relationships on $\mathbb{R}$, see \eg~\cite{gibbs2002choosing}]{thmrep}{relationshipsR}%
	\label{thm:relationships_R}%
	Let $\mu$ and $\nu$ be probability measures on $\mathbb{R}$ with cumulative distribution functions $P$ and $Q$ respectively. Then the following holds:
	\begin{align}
		d_\mathrm{K}(\mu,\nu) &\leq d_\mathrm{D}(\mu,\nu) \leq 2 d_\mathrm{K}(\mu,\nu)\\
		d_\mathrm{L}(\mu,\nu) &\leq d_\mathrm{P}(\mu,\nu)\\
		d_\mathrm{L}(\mu,\nu) &\leq d_\mathrm{K}(\mu,\nu).
	\end{align}%
	If $\nu$ is dominated by the Lebesgue measure it further holds that:
	\begin{align}
		d_\mathrm{K}(\mu,\nu) &\leq \left(1+\sup_{x\in\mathbb{R}}\left| Q'(x)\right|\right) d_\mathrm{L}(\mu,\nu).
	\end{align}%
\end{restatable}

Theorem~\ref{thm:relationships_measurable}, Theorem~\ref{thm:relationships_metric_spaces} and Theorem~\ref{thm:relationships_R} provide several interesting relationships between topologies on the space of measures. For example, Eq.~\eqref{eq:hellinger_and_tv_relation} shows that the total variation distance and the Hellinger distance induce equivalent topologies. Other inequalities induce other topologies.
Moreover, the following interesting statements follow immediately.

\Needspace{10\baselineskip}
\begin{restatable}[Weak Convergence, see \eg~\cite{gibbs2002choosing}]{correp}{wekconvcor}%
	\label{cor:weak_convergence}%
	For measures on $\mathbb{R}$, the L\'evy metric metrizes weak convergence.
	Convergence under the discrepancy and Kolmogorov metric imply weak convergence.
	The discrepancy and Kolmogorov metric metrize weak convergence of a sequence $\mu_1,\mu_2,\ldots$ towards $\nu$ if $\nu$ is dominated by the Lebesgue reference measure.
	
	For measures on a measurable space $\Omega$, the Prokhorov metric metrizes weak convergence.
	Convergence under the Wasserstein metric implies weak convergence.
	
	Furthermore, if $\Omega$ is bounded, the Wasserstein metric metrizes weak convergence and convergence under any of the following metrics implies weak convergence: total variation, Hellinger distance, separation distance, KL-divergence and the $\chi^2$-divergence.
	
	If $\Omega$ is both bounded and finite, the total variation and Hellinger distance metrize weak convergence.
\end{restatable}

\subsection{Moment Distances}
\label{subsec:moment_distances}

In this work, we analyze domain adaptation problems under weak assumptions on the similarity of the underlying probability measures.
Our assumptions are based on \textit{moment distances} which imply a weak form of similarity of probability measures~\cite{rachev2013methods}.

Simple probability metrics as proposed in Subsection~\ref{subsec:ten_probability_metrics} satisfy the identity of indiscernibles, \ie~for all probability measures $\mu$ and $\nu$ on the measurable space $(\Omega,\mathcal{A})$ it holds that
\begin{align}
	\label{eq:identity_of_indiscernibles}
	d(\mu,\nu)=0\iff \mu=\nu.
\end{align}
In contrast, a moment distance between probability measures $\mu$ and $\nu$ on $\mathbb{R}^d$ satisfies
\begin{align}
	\label{eq:moment_distance_definition}
	d(\mu,\nu)=0 \iff h(\mu)=h(\nu)
\end{align}
where $h(\mu)=\int \boldsymbol{\phi}\diff\mu$ is a vector of moments corresponding to some vector $\boldsymbol{\phi}=(\phi_1,\ldots,\phi_n)^\text{T}$ of polynomials $\phi_1,\ldots,\phi_n\in\mathbb{R}[x_1,\ldots,x_d]$.
Moment distances can be extended to more general measurable spaces and more general functionals $h$.
Such metrics are called \textit{primary probability metrics}.
However, in this work, we are only interested in the real case and functionals $h$ as given above.

One important example of a moment distance is the following extension of the \textit{Engineer's metric}~\cite{rachev2013methods}.
\Needspace{10\baselineskip}
\begin{restatable}[$\ell_1$-Distance Between Moments]{defrep}{dm}%
	\label{def:ell1_moment_distance}%
	The $\ell_1$-distance between moments \wrt~some $\boldsymbol{\phi}\in\left(\mathbb{R}[x_1,\ldots,x_d]\right)^n$ between two probability measures $\mu$ and $\nu$ on the unit cube $[0,1]^d$ is defined by
	\begin{align}
		\label{eq:ellone_moment_distance}
		d_\mathrm{M}(\mu,\nu)=\norm{\int\boldsymbol{\phi}\diff\mu-\int\boldsymbol{\phi}\diff\nu}_1.
	\end{align}
\end{restatable}
Note that the focus on probability measures $\mu$ and $\nu$ on the unit cube implies that the vectors $\int\boldsymbol{\phi}\diff\mu$ and $\int\boldsymbol{\phi}\diff\mu$ are finite.

Given the moment distance above, questions about its relation to the probability metrics described in Subsection~\ref{subsec:ten_probability_metrics} arise.
The following theorem gives some intuition.
\needspace{25\baselineskip}
\begin{restatable}[Rachev et al.~\cite{rachev2013methods}]{thmrep}{thmrech}%
	\label{thm:rachev}%
	Let $\mu$ and $\nu$ be probability measures on $[0,1]$ with characteristic functions $f$ and $g$, respectively, fulfilling
	\begin{align}
		\label{eq:bounded_zolotarev}
		\sup_{\left|t\right|\leq T_0} |f(t)-g(t)|\leq \varepsilon
	\end{align}
	for some real constants $T_0$ and $\varepsilon$. Then there exists an absolute constant $C_\text{Z}$ such that for all $n\in\mathbb{N}$ with
	\begin{align}
		n^3 C_\text{Z}^{\frac{1}{n+1}} \varepsilon^{\frac{1}{n+1}}\leq T_0/2
	\end{align}
	we have
	\begin{align}
		\left| \int x^n \diff\mu - \int x^n \diff\nu\right|\leq C_\text{Z} n^3 \varepsilon^{\frac{1}{n+1}}.
	\end{align}
\end{restatable}
Theorem~\ref{thm:rachev} gives a bound on the differences between moments based on a local bound on the underlying probability measures.

In Subsection~\ref{subsec:bound_on_moment_distance_by_levy_metric} we extend this theorem to an upper bound on $d_\mathrm{M}$ in terms of the L\'evy metric.
This implies that $d_\mathrm{M}$ can be bounded from above by all probability metrics described in Subsection~\ref{subsec:ten_probability_metrics}.
One consequence of Theorem~\ref{thm:rachev} is that weak convergence on compact intervals implies convergence of finitely many moments.
This result also follows from Portmanteau's theorem, see \eg~\cite{dudley2002real}.
\begin{restatable}[]{lemmarep}{weakconvtomomentconv}%
	\label{lemma:from_weak_convergence_to_moment_convergence}%
	The weak convergence of a sequence $\mu_1,\mu_2,\ldots$ of probability measures on $[0,1]^d$ to some probability measure $\nu$ on $[0,1]^d$ implies the convergence of $\mu_1,\mu_2,\ldots$ to $\nu$ in $d_\mathrm{M}$.
\end{restatable}%
However, a zero moment distance does not imply identical probability measures and convergence in moments does not imply weak convergence for general probability measures.
Therefore, questions about the difference of two probability measures based on finitely many moments arise.

The literature about \textit{moment problems}~\cite{akhiezer1965classical,tardella2001note,kleiber2013multivariate,schmudgen2017moment} provides bounds on the difference between two one-dimensional probability measures on $\mathbb{R}$ with finitely many coinciding moments.
However, bounds in the multivariate case remain scarce~\cite{laurent2009sums,di2018multidimensional}.

Lindsay and Basak show~\cite{lindsay2000moments} that the Kolmogorov metric between two probability measures with finitely many coinciding moments can be very large.

Tagliani et al.~\cite{tagliani2003note,tagliani2002entropy,tagliani2001numerical,milev2012moment} show that, in the case of compactly supported probability measures, this difference can be bounded by means of the KL-divergence between the probability density function and the maximum entropy density sharing the same finite collection of moments.

Barron and Sheu~\cite{barron1991approximation} give bounds on the KL-divergence between a compactly supported probability density function and its approximation by estimators of maximum entropy densities.
They establish rates of convergence for log-density functions assumed to have square integrable derivatives.
Their analysis involves moment-based bounds which we will review in more detail in Subsection~\ref{subsec:approximation_theory}.

\section{Statistical Learning Theory}
\label{sec:stat_learn_th}

The process of inductive inference which can roughly be summarized as follows~\cite{bousquet2003introduction}: (1) observe a phenomenon, (2) construct a model of that phenomenon and (3) make predictions using this model.
It is the goal of learning theory to formalize this process.
In this thesis, we rely on a classical part of learning theory which is the statistical learning framework for binary classification.
Most results in binary classification can be readily extended to more general settings as \eg~multi-class classification and regression~\cite{shalev2014understanding}.

In this section we follow Vapnik~\cite{vapnik2013nature} and Ben-David et al.~\cite{ben2010theory}.

For simplicity, we focus on distributions which are represented by probability density functions $p\in\MR$ \wrt~the Lebesgue reference measure.

This section is structured as follows: Subsection~\ref{subsec:binary_classification} formalizes the problem of binary classification and the principle of empirical risk minimization.
Subsection~\ref{subsec:learning_bounds_in_stat_learn_th} summarizes related results.

\subsection{Binary Classification}
\label{subsec:binary_classification}

In the framework of binary classification observations are considered in the form of instance-label pairs.
The instances are vectors in $\mathbb{R}^d$.
We follow~\cite{ben2007analysis,ben2010theory} and assume labels in $[0,1]$, where intermediate values are used to model non-deterministic, \eg~expected, behaviour.
The goal of binary-classification is the estimation of some unknown function ${l:\mathbb{R}^d\to [0,1]}$ based on finitely many such instance-label pairs.

\Needspace{10\baselineskip}
\begin{restatable}[Binary Classification, see \eg~\cite{vapnik2013nature}]{problemrep}{binaryclassification}%
	\label{problem:binary_classification}%
	Consider some probability density $p\in\MR$ and a labeling function $l:\mathbb{R}^d\to[0,1]$.
	
	Given a \textit{training} sample $X_p=\{\x_1,\ldots,\x_k\}$ drawn from $p$ and a corresponding multiset $Y=\{l(\x_1),\ldots,l(\x_k)\}$ of labels,
	find some function $f:\mathbb{R}^d\to\{0,1\}$ with a small {\it misclassification risk}
	\begin{align}
		\label{eq:misclassification_risk}
		\int_{\mathbb{R}^d} \left|f(\x)-l(\x)\right| p(\x)\diff\x.
	\end{align}
\end{restatable}
~
\begin{remark}[A Note on Integrability]
	For the misclassification risk in Eq.~\eqref{eq:misclassification_risk} to exist, the Lebsgue integral of $\left|f-l\right|$ need to exist.
	We therefore assume that all labeling functions are Lebesgue integrable and focus on functions $f$ which are integrable.
\end{remark}

In the following let $\mathcal{F}$ be a set of integrable binary classifiers, \ie
\begin{align}
	\label{eq:setF}
	\mathcal{F}\subset\left\{f:\mathbb{R}^d\to\{0,1\}\,\middle|\, f~\text{integrable}\right\}.
\end{align}

Note that the probability density $p$ and the labeling function $l$ in Problem~\ref{problem:binary_classification} are typically unknown in practical applications.
Therefore, different principles have been proposed to solve Problem~\ref{problem:binary_classification} based on the samples $X_p$ and $Y$.
The principle of {\it empirical risk minimization} is to choose a function $f\in\mathcal{F}$ with small {\it empirical misclassification risk} which is given by:
\begin{align}
	\label{eq:empirical_misclassification_risk}
	\frac{1}{k}\sum_{\x\in X_p} \left|f(\x)-l(\x)\right|.
\end{align}
Other principles which extend empirical risk minimization are {\it structural risk minimization}, {\it regularization} and \textit{normalized regularization}.
For an overview, we refer to~\cite{bousquet2003introduction}.

All these principles require an a priori choice of a function class $\mathcal{F}$ from which $f$ is chosen.
Another common feature is that the empirical misclassification risk as given by Eq.~\eqref{eq:empirical_misclassification_risk} is still considered as part of the optimization procedure, \ie~as a term of the corresponding objective.
One important question is therefore:
\begin{center}
	\textit{Under which conditions can we expect empirical risk minimization to solve Problem~\ref{problem:binary_classification}?}
\end{center}
It turns out that the success can be expected for large samples $X_p$ and function classes $\mathcal{F}$ of finite complexity.
This answer is formalized in the next subsection.

\subsection{Learning Bounds}
\label{subsec:learning_bounds_in_stat_learn_th}

This subsection provides results regarding the success of the empirical risk minimization principle for solving Problem~\ref{problem:binary_classification}.
First proofs are given by Vapnik and Chervonenkis in 1966, see~\cite{vapnik2015uniform} for an English translation.

These results take the form of probabilistic upper bounds on the absolute difference between the misclassification risk in Eq.~\eqref{eq:misclassification_risk} and the empirical misclassification risk in Eq.~\eqref{eq:empirical_misclassification_risk}.
The bounds are based on the sample size of $X_p$ and a measure of the complexity of set $\mathcal{F}$.

To obtain the complexity measure, the idea is to look at the function class 'projected' on a sample.

\Needspace{10\baselineskip}
\begin{restatable}[Growth Function~\cite{vapnik2015uniform}]{defrep}{growthfunction}%
	\label{def:growth_function}%
	The growth function of $\mathcal{F}$ is defined by
	\begin{align}
		\label{eq:growth_function}
		\begin{split}
			S_\mathcal{F}:\mathbb{N} &\to \mathbb{N}\\
			k &\mapsto \left| \left\{\left(f\left(\x_1\right),\ldots,f\left(\x_k\right)\right)\,\middle|\, f\in\mathcal{F},\, \x_1,\ldots,\x_k\in\mathbb{R}^d\right\} \right|.
		\end{split}
	\end{align}
\end{restatable}
The growth function value $S_\mathcal{F}(k)$ is the maximum number of ways into which $k$ points can be classified by the function class $\mathcal{F}$.
As shown in~\cite{vapnik2015uniform}, the value of the growth function can be used to bound the absolute difference between the misclassification risk and the empirical misclassification risk.
\begin{restatable}[Vapnik and Chervonenkis~\cite{vapnik2015uniform}]{thmrep}{growthfunctionbound}%
	\label{thm:growth_function}%
	Consider some probability density $p\in\MR$ and a labeling function $l:\mathbb{R}^d\to[0,1]$.
	
	For any $\delta\in (0,1)$ and any $f\in\mathcal{F}$ the following holds with probability at least $1-\delta$ over the choice of a $k$-sized sample $X_p$ drawn from $p$:
	\begin{align}
		\int\left| f-l \right| p \leq \frac{1}{k}\sum_{\x\in X_p} \left|f(\x)-l(\x)\right| + \sqrt{8 \frac{\log S_\mathcal{F}(2 k)+\log\frac{2}{\delta}}{k}}.
	\end{align}
\end{restatable}
The question remains how to compute the growth function. Therefore the following quantity is of special importance.
\begin{restatable}[VC-dimension~\cite{vapnik2015uniform}]{defrep}{vcdimension}%
	\label{def:vc_dimension}%
	The Vapnik Chervonenkis dimension (VC-dimension) $\mathrm{VC}(\mathcal{F})$ of $\mathcal{F}$ is the largest $k\in\mathbb{N}$ such that
	\begin{align}
		\label{eq:vc_def_eqution}
		S_\mathcal{F}(k)=2^k.
	\end{align}
\end{restatable}
One interpretation of the VC-dimension is that it measures the size of the projections of a function class onto finite samples~\cite{bousquet2003introduction}.
The VC-dimension does not just 'count' the number of functions in the class but depends on the geometry of the class.
For example consider the VC-dimension of linear and affine functions:
\begin{align}
	&\mathrm{VC}\!\left(\left\{f:\mathbb{R}^d\to\left\{0,1\right\}\,\middle|\,
	f(\x)=\mathbbm{1}_{\mathbb{R}_+}(\langle{\vec w}, \x\rangle),
	{\vec w}\in\mathbb{R}^d \right\}\right)=d\\
	\label{eq:vc_dim__of_affine_functions}
	&\mathrm{VC}\!\left(\left\{f:\mathbb{R}^d\to\left\{0,1\right\}\,\middle|\, f(\x)=\mathbbm{1}_{\mathbb{R}_+}(\langle{\vec w}, \x\rangle + b), {\vec w}\in\mathbb{R}^d, b\in\mathbb{R} \right\}\right)=d+1
\end{align}
where $\mathbbm{1}_{\mathbb{R}_+}(\x)$ is one iff $\x$ is a positive real number and it is zero otherwise.

The following lemma serves as a key to upper bound the growth function.
It was independently discovered by Sauer in combinatorics, Shelah in model theory and Vapnik and Chervonenkis in statistics.
\needspace{25\baselineskip}
\begin{restatable}[Vapnik and Chervonenkis, Sauer, Shelah, see \eg~\cite{vapnik2013nature}]{lemmarep}{sauershelahlemma}%
	\label{lemma:sauer_shelah}%
	If $\vc<\infty$ then the following holds for all $k\in\mathbb{N}$:
	\begin{align}
		S_\mathcal{F}(k)\leq \sum_{i=1}^{\vc} \binom{k}{i}
	\end{align}
	and for all $k\geq\vc$ the following holds:
	\begin{align}
		\label{eq:important_sauer_shelah}
		S_\mathcal{F}(k)\leq \left(\frac{e k}{\vc}\right)^{\vc}.
	\end{align}
\end{restatable}
Combining Lemma~\ref{lemma:sauer_shelah} with Theorem~\ref{thm:growth_function} yields the following learning bound.
\begin{restatable}[Learning Bound, Vapnik and Chervonenkis~\cite{vapnik2015uniform}]{thmrep}{vcbound}%
	\label{thm:vc_bound}%
	Consider some probability density $p\in\MR$ and a labeling function $l:\mathbb{R}^d\to[0,1]$.
	
	If $\vc\leq k<\infty$, then, for any $\delta\in (0,1)$ and any $f\in\mathcal{F}$, the following holds with probability at least $1-\delta$ over the choice of a $k$-sized sample $X_p$ drawn from $p$:
	\begin{align}
		\int\left| f-l \right| p \leq \frac{1}{k}\sum_{\x\in X_p} \left|f(\x)-l(\x)\right| + \sqrt{8 \frac{\vc \log \frac{2 e k}{\vc}+\log\frac{2}{\delta}}{k}}.
	\end{align}
\end{restatable}
Theorem~\ref{thm:vc_bound} shows that the empirical risk minimization principle solves Problem~\ref{problem:binary_classification} of binary classification if the sample size $k$ is large enough and the function class $\mathcal{F}$ has small VC-dimension $\vc$.
Moreover, it leads to upper bounds on the following important quantity.

\begin{restatable}[Sample Complexity, see \eg~\cite{shalev2014understanding}]{defrep}{samplecomplexity}%
	\label{def:sample_complexity}%
	Consider some probability density $p\in\MR$ and a labeling function $l:\mathbb{R}^d\to[0,1]$.
	
	The sample complexity $k_\mathcal{F}(\varepsilon,\delta)$ of $\mathcal{F}$ is the minimum $k\in\mathbb{N}$ such that the following holds for all $f\in\mathcal{F}$ with probability at least $1-\delta$ over the choice of a $k$-sized sample $X_p$ drawn from $p$:
	\begin{align}
		\label{eq:generalization_error_inequality}
		\left|\int\left| f-l \right| p - \frac{1}{k}\sum_{\x\in X_p} \left|f(\x)-l(\x)\right|\right|\leq \varepsilon.
	\end{align}
\end{restatable}
It follows from Theorem~\ref{thm:vc_bound}, see \eg~\cite{shalev2014understanding}, that there exists some constant $C>0$ with
\begin{align}
	\label{eq:nearly_fundamental_bound}
	k_\mathcal{F}(\varepsilon,\delta)\leq C\frac{\vc\log\frac{\vc}{\varepsilon}+\log\frac{1}{\delta}}{\varepsilon^2}.
\end{align}
This result is one of the biggest breakthroughs in machine learning.
It shows that the sample size required for accurately estimating the true misclassification risk does often grow slower than exponentially with the dimension $d$.
This result seems to be counter intuitive in the light of the exponential rate of convergence in the Weierstrass approximation theorem for non-smooth functions~\cite{vapnik2013nature}.

It turns out that Eq.~\eqref{eq:nearly_fundamental_bound} can be even improved based on a careful analysis of the so-called Rademacher complexity using a technique called \textit{chaining}.
This leads to the following result often called the quantitative version of the fundamental theorem of statistical learning.

\begin{restatable}[Fundamental Theorem of Statistical Learning, see \eg~\cite{shalev2014understanding}]{thmrep}{fundamentaltheorem}%
	\label{thm:fundamental}%
	Consider some probability density $p\in\MR$ and a labeling function $l:\mathbb{R}^d\to[0,1]$.
	
	If $\vc<\infty$ then there exist constants $C_1,C_2\in\mathbb{N}$ such that
	\begin{align}
		C_1 \frac{\vc+\log 1/\delta}{\varepsilon^2}\leq k_{\mathcal{F}}(\epsilon,\delta)\leq C_2 \frac{\vc+\log 1/\delta}{\varepsilon^2}.
	\end{align}
	
\end{restatable}

The results above assume one unique labeling function $l$ and samples $X_p$ with elements being realizations of random variables with the same probability density function $p$.
However, these assumptions are violated in many practical tasks.
In the next section we give a short overview of the field of domain adaptation which is concerned with the generalization of these assumptions.

\section{Domain Adaptation}
\label{sec:domain_adaptation}

One motivating question for the framework of domain adaptation is the following:
\begin{center}
	\it
	Under which conditions can we expect a classifier to perform well on some target data from a situation different from the training one?
\end{center}
To answer this question, the classical statistical learning theory described in Section~\ref{sec:stat_learn_th} must be extended.
One such extension is the framework of domain adaptation.

In this section, we describe the problem of domain adaptation for binary classification following Ben-David et al.~\cite{ben2010theory}, Mansour, Mohri and Rostamizadeh~\cite{mansour2009domain} and Cortes and Mohri~\cite{cortes2014domain}.
We also briefly summarize related work on learning bounds for domain adaptation and algorithms for solving practical domain adaptation problems.

For simplicity and consistency, we focus on distributions represented by probability density functions $p\in\MR$ \wrt~the Lebesgue reference measure.

This section is structured as follows:
Subsection~\ref{subsec:motivation_for_domain_adaptation} motivates the generalization of classical statistical learning theory, Subsection~\ref{subsec:domain_adaptation_for_binary_classification} formalizes the problem of domain adaptation for binary classification, Subsection~\ref{subsec:learning_bounds_for_domain_adaptation} summarizes important results and Subsection~\ref{subsec:da_algorithms} reviews different algorithms for solving domain adaptation problems.

\subsection{Motivation}
\label{subsec:motivation_for_domain_adaptation}

\begin{figure}[ht]
	\includegraphics[width=\linewidth]{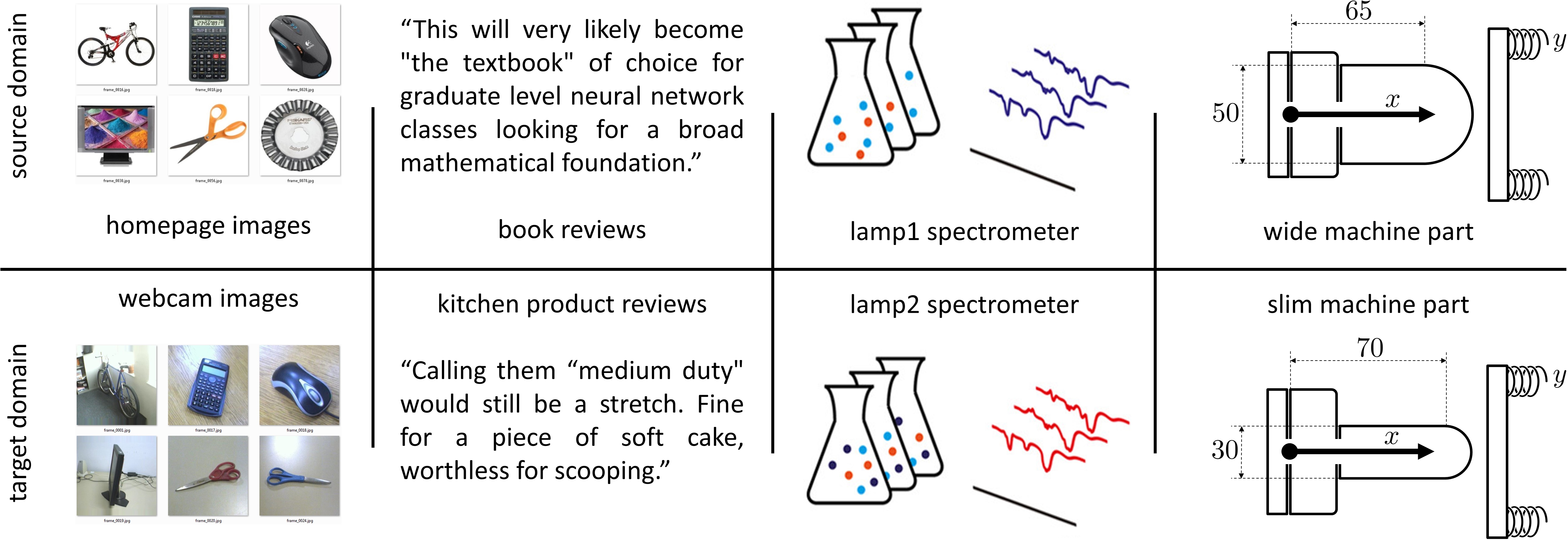}
	\caption[Practical examples violating assumptions of statistical learning theory.]{Practical examples violating assumptions of statistical learning theory. The goal is to learn a statistical model on a source domain which performs well on a target domain.
	}
	\label{fig:da_examples}
\end{figure}

Most results in statistical learning, both practical and theoretical, assume that the underlying data follow one fixed distribution and one labeling function~\cite{ben2010theory,vapnik2013nature,shalev2014understanding}.
For example consider Problem~\ref{problem:binary_classification} of binary classification, which assumes one unique labeling function $l$ and a sample $X_p$ with elements being realizations of random variables with the same probability density function $p$.
However, these assumptions are violated in many practical tasks.

Figure~\ref{fig:da_examples} shows different practical examples violating the assumptions made in statistical learning theory.
The goal is to learn a statistical model on some \textit{source domain} such that it performs well on some \textit{target domain}.
One example is the training of statistical classifiers on images from a homepage showing clear white backgrounds with the goal of a small misclassification risk on images captured by a webcam~\cite{saenko2010adapting}.
Another important example is sentiment analysis of product reviews, where a model is trained on data of a source product category, \eg~book reviews, and it is tested on data of a related category, \eg~kitchen product reviews~\cite{glorot2011domain}.
A third example is the regression of spectroscopic measurements where different instrumental responses, environmental conditions, or sample matrices can lead to different source and target measurements~\cite{malli2017standard}.
A fourth example is the drilling of steel components where different machine settings can lead to different torque curves during time~\cite{pena2005monitoring,ferreiro2012bayesian}.
As a last example consider the content-based depth range adaptation of unlabeled stereoscopic videos by means of labeled data from movies~\cite{zellinger2016linear,zellinger2016improving,seitner2015trifocal}.

The examples above are discussed in the general area of transfer learning~\cite{pan2010survey}.
General transfer learning problems can have the goal of adapting functions such that they solve new learning tasks, \eg~using a binary classifier to find a function separating three classes.
Such problems are too general for our purpose and we restrict ourselves to the more specific sub-area of domain adaptation.
Domain adaptation is concerned with the learning of statistical models that perform well on some target data with a labeling function and a distribution different from some source data.

Subsection~\ref{subsec:domain_adaptation_for_binary_classification} extends the Problem~\ref{problem:binary_classification} of binary classification to the setting of domain adaptation.

\subsection{Binary Classification}
\label{subsec:domain_adaptation_for_binary_classification}

We start with a formal definition of a domain.
\needspace{10\baselineskip}
\begin{restatable}[Domain~\cite{ben2007analysis,blitzer2007domain,ben2010theory}]{defrep}{domain}%
	\label{def:domain}%
	A domain is a pair $\left(p,l\right)$ of a probability density function $p\in\MR$ and a labeling function $l:\mathbb{R}^d\to [0,1]$.
\end{restatable}
Recall that a labeling function is assumed to be integrable.
The problem of domain adaptation for binary classification can now be defined as follows.
\needspace{10\baselineskip}
\begin{restatable}[Domain Adaptation for Binary Classification~\cite{ben2007analysis,blitzer2007domain,ben2010theory}]{problemrep}{domain_adaptation_for_binary_classification}%
	\label{problem:da_for_binary_classification}%
	Consider two domains, a \textit{source} domain $\left(p,l_p\right)$ and a \textit{target} domain $\left(q,l_q\right)$.
	
	Given a source sample $X_p=\{\x_1,\ldots,\x_k\}$ drawn from $p$ with corresponding labels $Y_p=\{l_p(\x_1),\ldots,l_p(\x_k)\}$ and a target sample $X_q=\{\x_1',\ldots,\x_s'\}$ drawn from $q$ with corresponding labels $Y_q\subseteq\{l_q(\x_1'),\ldots,l_q(\x_s')\}$, find some function $f:\mathbb{R}^d\to\{0,1\}$ with a small target misclassification risk
	\begin{align}
		\label{eq:misclassification_risk_domain_adaptation}
		\int_{\mathbb{R}^d} \left|f(\x)-l_q(\x)\right| q(\x)\diff\x.
	\end{align}%
\end{restatable}%
Problem~\ref{problem:da_for_binary_classification} is based on two domains.
However, it can be easily extended to multiple domains, see \eg~\cite{zhao2018multiple}.

Sometimes, equality of the two labeling functions is assumed, \ie~$l_p=l_q$.
This assumption is called \textit{covariate shift assumption}~\cite{sugiyama2005generalization,sugiyama2012machine,ben2014domain}.
Covariate shift problems have seen significant work~\cite{huang2007correcting,sugiyama2008direct,cortes2008sample} in the area of sample selection bias~\cite{heckman1979sample} which can be seen as a sub-area of domain adaptation.

Note that the target sample $X_q$ in Problem~\ref{problem:da_for_binary_classification} can be empty and the label multiset $Y_q$ can be a strict subset of the full label set $\{l_q(\x_1'),\ldots,l_q(\x_s')\}$.
According to the size of $X_q$ and $Y_q$, different variants of Problem~\ref{problem:da_for_binary_classification} are considered:
If $s>0$ and $Y_q=\{l_q(\x_1'),\ldots,l_q(\x_s')\}$, the problem is called \textit{supervised}. The setting of supervised domain adaptation is similar to the one of \textit{multi-task learning}~\cite{caruana1997multitask}.
However, in contrast to domain adaptation, multi-task learning aims at finding functions with a high performance on both domains, source and target.
If $s>0$ and $\varnothing\neq Y_q\subset\{l_q(\x_1'),\ldots,l_q(\x_s')\}$, Problem~\ref{problem:da_for_binary_classification} is called \textit{semi-supervised}.
If $s>0$ and no target labels are given, \ie~$Y_q=\varnothing$, Problem~\ref{problem:da_for_binary_classification} is called \textit{unsupervised}.
Unsupervised domain adaptation is a problem which often arises in practice when collecting labels is expensive~\cite{blitzer2007biographies,saenko2010adapting,ganin2016domain}.
In Chapter~\ref{chap:cmd_algorithm}, we propose a new algorithm for unsupervised domain adaptation and test it on benchmark datasets.
If $s=0$, Problem~\ref{problem:da_for_binary_classification} is called \textit{domain generalization}.
This problem often arises in industrial applications where application data has a distribution that is different from the one of the training data~\cite{malli2017standard}.
A problem of domain generalization from industrial manufacturing is discussed in more detail in Section~\ref{sec:cyclical_manufacturing}.

It is important to note that Problem~\ref{problem:da_for_binary_classification} is not solvable if the distance between the two domains is large.
This intuition is formalized by Ben-David in~\cite{ben2010impossibility} and Ben-David and Urner in~\cite{ben2014domain}.
As the authors point out, the unsupervised domain adaptation problem becomes intractable when the labeling functions are too different.
The same holds if the input distributions largely differ.

However, when the domains are similar, it has been empirically and theoretically shown that Problem~\ref{problem:da_for_binary_classification} can be solved.
Most of these theoretical results take the form of learning bounds as described in Subsection~\ref{subsec:learning_bounds_for_domain_adaptation}.

\subsection{Learning Bounds}
\label{subsec:learning_bounds_for_domain_adaptation}

In the following, let $\mathcal{F}\subset\left\{f:\mathbb{R}^d\to\{0,1\}\,\middle|\, f~\text{integrable}\right\}$ be a class of binary classifiers.
Following~\cite{ben2010theory}, we may state the following result.
\begin{restatable}[Ben-David et al.~\cite{ben2010theory}]{thmrep}{lonedomainadaptationbound}%
	\label{thm:lone_domain_adaptation_bound}%
	Let $(p,l_p)$ and $(q,l_q)$ be two domains.
	Then the following holds for all $f\in\mathcal{F}$:
	\begin{align}
		\label{eq:simple_da_equation}
		\int\left|f-l_q\right| q\leq \int\left|f-l_p\right| p + \norm{p-q}_{L^1} + \lambda^*
	\end{align}
	where
	\begin{align}
		\label{eq:minimal_combined_error}
		\lambda^* = \inf_{h\in\mathcal{F}}\left(\int\left|h-l_q\right| q+\int\left|h-l_p\right| p\right).
	\end{align}
\end{restatable}
\begin{proof}
	Following the proof of Theorem~2 in~\cite{ben2010theory}, we obtain for any $h\in\mathcal{F}$,
	\begin{align*}
		\int\left|f-l_q\right| q &= \int\left|f-l_q\right| q + \int\left|f-h\right| p - \int\left|f-h\right| p\\
		&\leq \int\left|f-h\right| q + \int\left|h-l_q\right| q + \int\left|f-l_p\right| p + \int\left|l_p-h\right| p - \int\left|f-h\right| p,
	\end{align*}
	where the last inequality follows from the triangle inequality. Note that
	\begin{align}
		\int\left|f-h\right| q - \int\left|f-h\right| p &\leq \sup_{h\in\mathcal{F}} \left|\int \left|f-h\right| q-\int \left|f-h\right| p\right|\nonumber\\
		&\leq \sup_{g:\mathbb{R}^d\to [-1,1]} \left|\int g q-\int g p\right|\nonumber\\
		&= \norm{p-q}_{L^1}\nonumber,
	\end{align}
	where the last equality is due to Theorem~\ref{thm:TV_distance} in Section~\ref{sec:probability_metrics}.
	Combining the two inequalities yields
	\begin{align*}
		\int\left|f-l_q\right| q &\leq \int\left|f-l_p\right| p + \norm{p-q}_{L^1} + \int\left|h-l_q\right| q + \int\left|l_p-h\right| p
	\end{align*}
	and the theorem follows by taking the infimum over all $h\in\mathcal{F}$.
\end{proof}
\\\\
The term $\lambda^*$ is called the {\it minimum combined misclassification risk} and it embodies a notion of \textit{adaptability} of a classifier~\cite{ben2010theory}.
If $\lambda^*$ is large, then Problem~\ref{problem:da_for_binary_classification} cannot be solved by focusing on functions with a small source misclassification risk.
On the other hand, if $\lambda^*$ is small, the $L^1$-difference between the densities can be used to measure \textit{adaptability}.
This can be seen from Theorem~\ref{thm:lone_domain_adaptation_bound}.
If $\norm{p-q}_{L^1}$ is small, a classifier $f\in\mathcal{F}$ with a small source misclassification risk shows also a small target misclassification risk.

Theorem~\ref{thm:lone_domain_adaptation_bound} has important implications for the problem of domain generalization where no target data is available.
Together with Theorem~\ref{thm:vc_bound} it shows that empirical risk minimization in the source domain using an appropriately large function class $\mathcal{F}$ solves this problem in settings where the domains are similar.

Theorem~\ref{thm:lone_domain_adaptation_bound} has also implications for the problem of unsupervised domain adaptation.
In unsupervised domain adaptation no labels of the target domain are available and the left-hand side of Eq.~\eqref{eq:simple_da_equation} cannot be sampled.
Theorem~\ref{thm:lone_domain_adaptation_bound} motivates a large class of algorithms which aim at minimizing the right-hand side, see Subsection~\ref{subsec:da_algorithms}.

Unfortunately, the $L^1$-norm cannot be accurately sampled~\cite{batu2000testing}.
Different approaches have been proposed to overcome this problem.

One approach extends the statistical learning theory described in Section~\ref{sec:stat_learn_th}.
It is based on the \textit{empirical $\mathcal{F}$-divergence} $\widehat{d}_\mathcal{F}(X_p,X_q)$ between two $k$-sized samples $X_p$ and $X_q$.
The empirical $\mathcal{F}$-divergence of a symmetric function class $\mathcal{F}$, \ie~a function class $\mathcal{F}$ such that for all~$h\in\mathcal{F}$ also $1-h\in\mathcal{F}$, is defined by~\cite{kifer2004detecting,ben2007analysis}
\begin{align}
	\label{eq:h_divergence}
	\widehat{d}_\mathcal{F}(X_p,X_q)=2 \left( 1-\inf_{f\in\mathcal{F}}\left( \frac{1}{k}\sum_{\x:f(\x)=0}\mathbbm{1}_{X_p}(\x) + \frac{1}{k}\sum_{\x:f(\x)=0}\mathbbm{1}_{X_q}(\x) \right) \right).
\end{align}
It is shown in~\cite{ben2010theory} that the empirical $\mathcal{F}$-divergence can be efficiently approximated.
\needspace{25\baselineskip}
\begin{restatable}[Ben-David et al.~\cite{ben2010theory,ben2007analysis,blitzer2006domain}]{thmrep}{hdivergencebound}%
	\label{thm:h_divergence}%
	Consider two domains $(p,l_p)$ and $(q,l_q)$ and let $\mathcal{F}$ be symmetric.
	
	If $\vc\leq k<\infty$, then, for any $\delta\in (0,1)$ and any $f\in\mathcal{F}$, the following holds with probability at least $1-\delta$ over the choice of two $k$-sized samples $X_p$ drawn from $p$ and $X_q$ drawn from $q$:
	\begin{align}
		\begin{split}
			\int\left| f-l_q \right| q &\leq \frac{1}{k}\sum_{\x\in X_p} \left|f(\x)-l_p(\x)\right| + \sqrt{8 \frac{\vc \log \frac{2 e k}{\vc}+\log\frac{2}{\delta}}{k}}\\
			&\phantom{\leq} + \frac{1}{2} \widehat{d}_\mathcal{F}(X_p,X_q) + 4\sqrt{\frac{2\vc\log(2 k)+\log\frac{2}{\delta}}{k}} + \lambda^*.
		\end{split}
	\end{align}
	where $\lambda^*$ is defined as in Eq.~\eqref{eq:minimal_combined_error}. 
\end{restatable}
For large samples, symmetric classes $\mathcal{F}$ with small VC-dimension and well solvable domain adaptation problems, \ie~$\lambda^*\approx 0$, Theorem~\ref{thm:h_divergence} shows that the empirical source error and the empirical $\mathcal{F}$-divergence can be used to estimate an upper bound on the target error.

The proofs of Theorem~\ref{thm:lone_domain_adaptation_bound} and Theorem~\ref{thm:h_divergence} are based on the triangle inequality for the binary misclassification error.
Other types of errors lead to other forms of these bounds~\cite{crammer2008learning}.

Mansour et al.~\cite{mansour2009domain,mansour2009multiple,mansour2014robust} extend the arguments of Ben-David et al.~by more general distance measures~\cite{mansour2009domain}, robustness concepts of algorithms~\cite{mansour2014robust} and tighter error bounds based on the Rademacher complexity.

Recently, Vural considered the problem of transforming two differently distributed samples by means of two different functions in a common latent space and subsequently learn a discriminative model~\cite{vural2018}.
Her assumptions imply that the two different functions do not map differently labeled sample points onto the same point in the latent space.

In Chapter~\ref{chap:learning_bounds} we provide learning bounds for domain adaptation based on moment distances in order to provide learning guarantees under weak similarity assumptions on the source and target density.

\subsection{Algorithms}
\label{subsec:da_algorithms}

Theorem~\ref{thm:lone_domain_adaptation_bound} suggests various algorithms for domain adaptation based on empirical risk minimization and the minimization of distances between the transformed source and target distributions.
The large majority of them follow one of the two principles~\cite{cortes2019adaptation}: (a) to reweight the source and target sample or (b) to learn new feature representations.

Algorithms following principle (a) aim at correcting the domain difference by multiplying the loss at each training example by a positive weight.
Many of these algorithms are based on the minimization of probability metrics as discussed in Section~\ref{sec:probability_metrics}.
For example, the algorithm proposed in~\cite{huang2007correcting} is based on minimizing the maximum mean discrepancy, \ie~an integral probability metric based on a reproducing kernel Hilbert space.
The KL-divergence is minimized in~\cite{sugiyama2008direct}.
The generalization bounds proposed in~\cite{mansour2009domain} motivate an algorithm that minimizes a new distance between empirical distribution functions which is based on a function space and a distance between two functions from this space.
This algorithm has been further extended in~\cite{cortes2019adaptation}.

Principle (b) of learning new data representations is illustrated in Figure~\ref{fig:grafical_abstract} for the problem of unsupervised domain adaptation.
Consider two domains $(p,l_p)$ and $(q,l_q)$, a source sample $X_p$ drawn from $p\in\MR$ and a target sample $X_q$ drawn from $q\in\MR$.
Algorithms which follow principle (b) aim at finding some functions $g:\mathbb{R}^d\to\mathbb{R}^s$ and $f:\mathbb{R}^s\to\{0,1\}$ such that $f\circ g:\mathbb{R}^d\to\{0,1\}$ has a small source risk and such that the probability density functions $\tilde p$ and $\tilde q$ of the sample representations $g(X_p)$ and $g(X_q)$ are similar.
This is, in the case of binary classification, often done by minimizing an approximation of the following objective function:
\begin{align}
	\label{eq:objective_of_principle_of_representation_learning_for_da}
	\frac{1}{k}\sum_{\x\in X_p}\left|f(g(\x))-l_p(\x)\right| + \lambda\cdot \hat{d}\left(g(X_p),g(X_q)\right)
\end{align}
where $\lambda>0$ is a parameter and $\hat d$ is a distance between the source and target sample representation $g(X_p)$ and $g(X_q)$, \eg~an empirical estimation of some probability metric $d:\MR\times\MR\to [0,\infty)$.

For example, some algorithms focus on the minimization of empirical estimations of the maximum mean discrepancy with linear kernel~\cite{tzeng2014deep, csurka2016unsupervised} or the maximum mean discrepancy with Gaussian kernel~\cite{bousmalis2016domain,long2015learning,long2016joint}.
The Wasserstein distance is applied in~\cite{courty2017joint,shalit2017estimating}.
An empirical estimation of the KL-divergence is minimized in~\cite{zhuang2015supervised}.
Moment distances based on first and second moments are applied in~\cite{sun2016deep,li2016revisiting}.
An empirical estimator of the $\mathcal{F}$-divergence as defined in Eq.~\eqref{eq:h_divergence} is applied in~\cite{ganin2016domain,tzeng2017adversarial,eghbal2019mixture}.
Other divergences are used in~\cite{si2009bregman,berisha2015empirically,muandet2013domain}.

In Chapter~\ref{chap:cmd_algorithm}, we propose a new moment distance based on higher-order moments and apply it to the representation learning principle.

The parameter $\lambda$ in Eq.~\eqref{eq:objective_of_principle_of_representation_learning_for_da} is sometimes called domain regularization parameter.
Its selection is a hard problem in unsupervised domain adaptation and domain generalization due to missing target labels~\cite{you2019towards}.
This problem is discussed in more detail in Subsection~\ref{subsec:domain_adaptation_by_nns}.

It is important to note that under the covariate shift assumption, \ie~$l_p=l_q$, two new domains $(\tilde p,l_{\tilde p})$ and $(\tilde q,l_{\tilde q})$ with $l_{\tilde p},l_{\tilde q}:\mathbb{R}^s \to [0,1]$ are defined by
\begin{align}
	\label{eq:ben_david_labeling_function}
	l_{\tilde p}(\vec a)=\frac{\int_{\{\x\mid g(\x) =\vec a\}} l_p(\x) p(\x) \diff\x }{\int_{\{\x\mid g(\x) =\vec a\}} p(\x)\diff\x}
\end{align}
and $l_{\tilde q}$ analogously~\cite{ben2007analysis}.
Based on this definition, Theorem~\ref{thm:lone_domain_adaptation_bound} and Theorem~\ref{thm:h_divergence} can be used to provide learning bounds for algorithms following principle (b).
This is done in Chapter~\ref{chap:learning_bounds} for moment distances.

\begin{figure}[t]
	\begin{center}
		\includegraphics[width=0.8\linewidth]{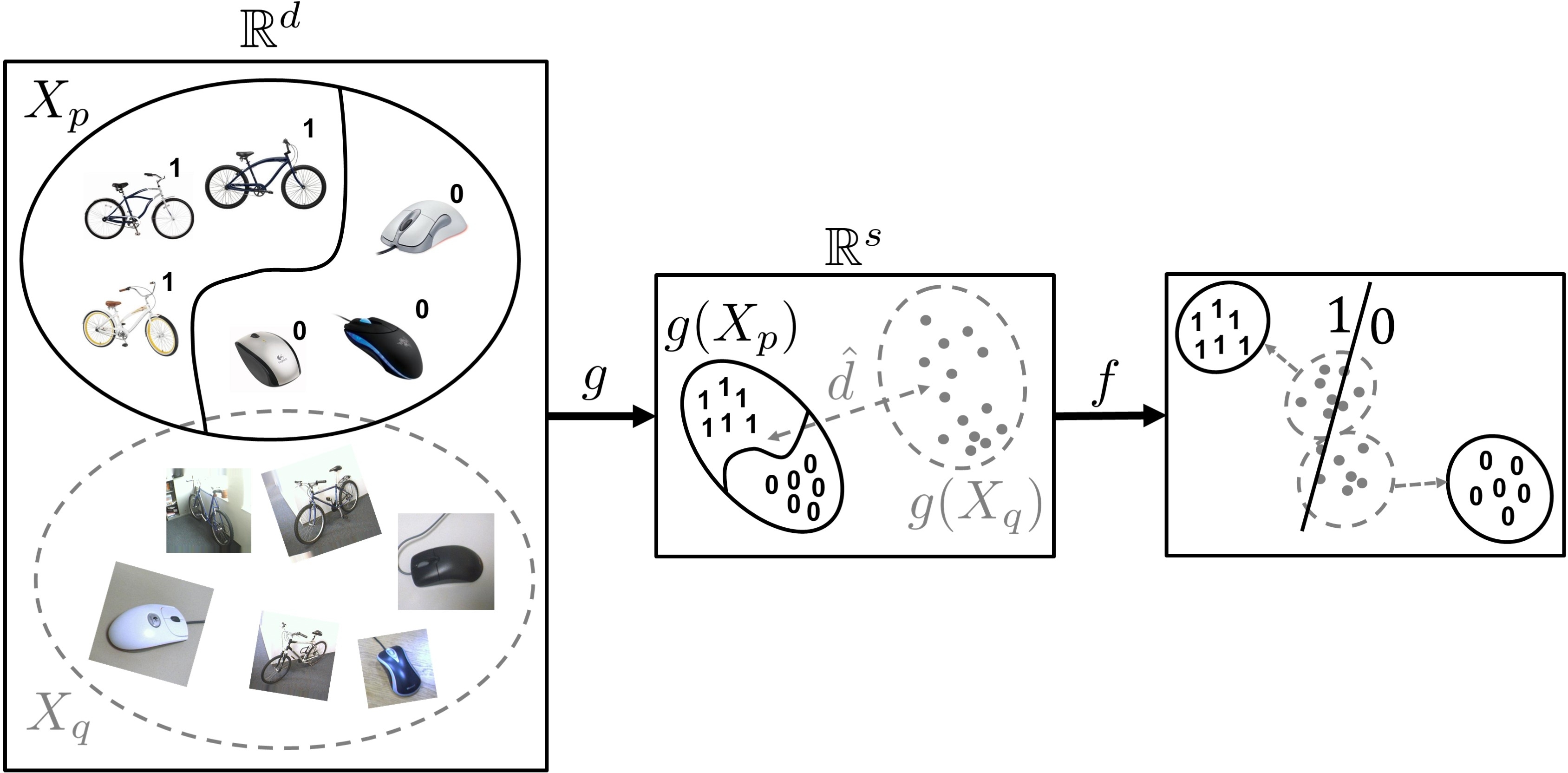}
	\end{center}
	\caption[Principle of learning representations for unsupervised domain adaptation.]{Principle of learning representations for unsupervised domain adaptation. Given: Source sample $X_p$ with labels and \textit{unlabeled} target sample $X_q$; Goal: Find a function $f\circ g:\mathbb{R}^d\to \{0,1\}$ with a small misclassification risk on target density $q$; Method: Minimizing misclassification risk on source sample $X_p$ and distance $\hat d$ between the sample representations $g(X_p)$ and $g(X_q)$.
	}
	\label{fig:grafical_abstract}
\end{figure}

\section{Maximum Entropy Distribution}
\label{sec:maximum_entropy_distribution}

In this work, we often choose some specific probability distribution from a broader class of distributions having finitely many moments in common.
We take these decisions based on the {\it principle of maximum entropy}.
This principle states that the probability distribution which best represents the current state of knowledge is the one with the largest entropy, in the context of precisely stated prior data.
In this work, the prior data is given by a finite set of (sample) moments.

In this section, we describe the concept of maximum entropy distributions following mainly Cover and Thomas~\cite{cover2012elements}, and, Wainwright and Jordan~\cite{wainwright2008graphical}.
We also review some approximation results following Barron and Sheu~\cite{barron1991approximation}.

This section is structured as follows:
Subsection~\ref{subsec:maximum_entropy_density} gives basic definitions.
Subsection~\ref{subsec:existence_of_maxent} and subsection~\ref{subsec:properties} review some properties of maximum entropy distributions.
Finally, Subsection~\ref{subsec:approximation_theory} describes some approximation properties of maximum entropy distributions.

\subsection{Maximum Entropy}
\label{subsec:maximum_entropy_density}

In this subsection, we focus on distributions which are represented by probability density functions $p\in\Ms$ on the $d$-dimensional unit cube $[0,1]^d$ \wrt~the Lebesgue reference measure.

We rely on the following measure of the entropy of a density $p\in\Ms$.
\needspace{25\baselineskip}
\begin{restatable}[Differential Entropy, see \eg~\cite{cover2012elements}]{defrep}{differentialentropy}%
	\label{def:differential_entropy}%
	Shannon's differential entropy $h(p)$ of a probability density $p\in\Ms$ is given by
	\begin{align}
		\label{eq:entropy}
		h(p)=-\int_{[0,1]^d} p(\x) \log p(\x)\diff \x.
	\end{align}
\end{restatable}%
The differential entropy is concave, may be negative, and may be potentially infinite if the integral in Eq.~\eqref{eq:entropy} diverges~\cite{cover2012elements}.

\Needspace{10\baselineskip}
\begin{restatable}[Maximum Entropy Density, see \eg~\cite{cover2012elements}]{defrep}{maxent}%
	\label{def:maxent_density}%
	Let $\boldsymbol{\phi}\in\left(\mathbb{R}[x_1,\ldots,x_d]\right)^n$ be a vector of polynomials and let $\boldsymbol{\mu}\in\mathbb{R}^n$.
	The maximum entropy density $p$ satisfying the moment constraint $\int\boldsymbol{\phi} p=\boldsymbol{\mu}$ is the probability density function with maximum differential entropy $h(q)$ among all probability density functions in the set
	\begin{align}
		\label{eq:maxent_set}
		\left\{q\in\Ms\,\middle|\,\int_{[0,1]^d} \boldsymbol{\phi}(\x) q(\x)\diff\x =\boldsymbol{\mu}\right\}.
	\end{align}
\end{restatable}
Note that our focus on probability densities $q\in\Ms$ on the unit cube implies that the elements of the vector $\int\boldsymbol{\phi} q$ are finite.
However, the set in Eq.~\eqref{eq:maxent_set} can be empty and the maximum entropy density might not exist~\cite{cover2012elements}.

\subsection{Existence and Uniqueness}
\label{subsec:existence_of_maxent}

If the maximum entropy density as specified in Definition~\ref{def:maxent_density} exists, there is a special relation to the following class of probability density functions.

\begin{restatable}[Polynomial Exponential Family]{defrep}{polyexpfamily}%
	\label{def:poly_exp_family}%
	Let $\boldsymbol{\phi}\in\left(\mathbb{R}[x_1,\ldots,x_d]\right)^n$ be a vector of polynomials.
	The polynomial exponential family $\mathcal{E}_{\boldsymbol{\phi}}$ corresponding to $\boldsymbol{\phi}$ is the set all probability density functions $q\in\Ms$ of the form
	\begin{align}
		\label{eq:maxent_distr_formula}
		q(\x) = c(\boldsymbol\lambda) \exp\!\left(-\langle \boldsymbol\lambda,\boldsymbol\phi(\x)\rangle\right)
	\end{align}
	where $\boldsymbol\lambda\in\mathbb{R}^{n}$ is a parameter vector and
	\begin{align}
		\label{eq:constant_of_normalization}
		c(\boldsymbol\lambda)=\left(\int_{[0,1]^n} \exp\!\left(-\langle \boldsymbol\lambda,\boldsymbol\phi(\x)\rangle\right) d\x\right)^{-1}
	\end{align}
	is the constant of normalization.
\end{restatable}
The function $f:\mathbb{R}^{n}\to\mathbb{R}, \boldsymbol\lambda \mapsto -\log c(\boldsymbol\lambda)$ is sometimes called \textit{cumulant function} and is continuous, see e.g.~\cite[Proposition~3.1]{wainwright2008graphical}.
The following Lemma~\ref{lemma:uniqueness_of_maximum_entropy} shows the uniqueness of maximum entropy distributions and its special relation to polynomial exponential families.

\Needspace{10\baselineskip}
\begin{restatable}[Uniqueness of Maximum Entropy Density, see \eg~\cite{cover2012elements}]{lemmarep}{uniquemaxent}%
	\label{lemma:uniqueness_of_maximum_entropy}%
	Let $\boldsymbol{\phi}\in\left(\mathbb{R}[x_1,\ldots,x_d]\right)^n$ be a vector of polynomials and let $\boldsymbol{\mu}\in\mathbb{R}^n$.
	If there exists some $q\in\mathcal{E}_{\boldsymbol{\phi}}$ with $\int \boldsymbol{\phi} q=\boldsymbol{\mu}$ then $q$ is the unique maximum entropy density satisfying the moment constraint $\int \boldsymbol{\phi} q=\boldsymbol{\mu}$.
\end{restatable}
Definition~\ref{def:poly_exp_family} of polynomial exponential families is based on an arbitrary vector $\boldsymbol{\phi}$ of polynomials.
More guarantees can be given for specific polynomials.
Let therefore $\psi(m,d)$ denote the number of monomials of maximum degree $m$ in $d$ variables without the zero-degree monomial $1$.
\Needspace{10\baselineskip}
\begin{restatable}[Existence of Maximum Entropy Density, see e.g.~\cite{wainwright2008graphical}]{lemmarep}{existmaxent}%
	\label{lemma:existence_of_maximum_entropy}%
	Consider some vector $\phim=(\phi_1,\ldots,\phi_{\psi(m,d)})^\text{T}$ such that $1,\phi_1,\ldots,\phi_{\psi(m,d)}$ is a basis of the space $\mathbb{R}_m[x_1,\ldots,x_d]$ of polynomials with maximum degree $m$.
	Then for each $\boldsymbol{\mu}$ in the interior of the set
	\begin{align}
		\label{eq:moment_space}
		\mathcal{P}:=\left\{\boldsymbol{\mu}\in\mathbb{R}^{\psi(m,d)}\,\middle|\, \exists\, p\in\Ms~\text{s.t.}~\int\phim p=\boldsymbol{\mu}\right\}
	\end{align}
	there exists some probability density $q\in\mathcal{E}_{\phim}$ satisfying $\int\phim q=\boldsymbol{\mu}$.
\end{restatable}
It is interesting to note that
\begin{align*}
	\psi(m,d)=\sum_{i=1}^m \zeta(i,d)=\sum_{i=1}^m \binom{d+i-1}{i}=\binom{d+m}{m}-1,
\end{align*}
where $\zeta(m,d)$ denotes the number of monomials of total degree $m$ in $d$ variables.
For a proof of the equality $\zeta(i,d)=\binom{d+i-1}{i}$ see \eg~\cite{heubach2004compositions}.

The following serves as a key observation.

\begin{restatable}[Maximum Entropy Density]{thmrep}{maxentdensity}%
	\label{thm:uniqueness_and_existence_of_maxent}%
	Let $p\in\Ms$ and $\phim=(\phi_1,\ldots,\phi_{\psi(m,d)})^\text{T}$ be such that $1,\phi_1,\ldots,\phi_{\psi(m,d)}$ is a basis of the space $\mathbb{R}_m[x_1,\ldots,x_d]$ of polynomials with maximum degree $m$.
	Then the maximum entropy probability density $p^*$ which satisfies $\int\phim p^*=\int\phim p$ exists, is unique and belongs to the exponential family $\mathcal{E}_{\phim}$.
\end{restatable}

\begin{proof}
	Our restriction of probability densities \wrt~the Lebesgue reference measure excludes convex combinations of Dirac delta functions.
	It follows that the moment vector $\int\phim p$ lies in the interior of the set $\mathcal{P}$ in Eq.~\eqref{eq:moment_space} as it is shown \eg~in~\cite{frontini2011hausdorff}).
	Lemma~\ref{lemma:existence_of_maximum_entropy} gives the existence and the form of the solution. Lemma~\ref{lemma:uniqueness_of_maximum_entropy} gives the uniqueness of the solution.
\end{proof}
\\\\
For the rest of this work and some given $p\in\Ms$ and $\phim$, we denote by $p^*$ the maximum entropy density satisfying the constraint $\int\phim p^*=\int\phim p$.
We further denote by $h_{\phim}(p)=h(p^*)$ the entropy of $p^*$.

\subsection{Further Properties}
\label{subsec:properties}

The following Lemma~\ref{lemma:properties_of_maxent} summarizes some important properties of maximum entropy densities as characterized by Theorem~\ref{thm:uniqueness_and_existence_of_maxent}.
\Needspace{10\baselineskip}
\begin{restatable}[Properties of Maximum Entropy Densities, see~\cite{csiszar1975, tagliani1999hausdorff, borwein1991convergence, wainwright2008graphical}]{lemmarep}{maxentproperties}%
	\label{lemma:properties_of_maxent}%
	Let $p\in\Ms$ and $\phim=(\phi_1,\ldots,\phi_{\psi(m,d)})^\text{T}$ be such that $1,\phi_1,\ldots,\phi_{\psi(m,d)}$ is a basis of the space $\mathbb{R}_m[x_1,\ldots,x_d]$ of polynomials with maximum total degree $m$.
	Then the following holds:
	\begin{enumerate}
		\item $p^*=\argmin_{q\in\mathcal{E}_{\phim}} d_{\mathrm{KL}}(p,q)$
		\item $d_{\mathrm{KL}}(p,p^*)=h_{\phim}(p)-h(p)$
		\item $h_{\phim}(p)\to h(p)$ as $m\to\infty$
		\item $p^*=c(\boldsymbol{\lambda}^*) \exp\!\left(-\langle \boldsymbol{\lambda}^*,\phim\rangle\right)$ iff $\boldsymbol{\lambda}^*=\min_{\boldsymbol{\lambda}\in\mathbb{R}^{\psi(m,d)}} \langle \boldsymbol{\lambda}, \int\phim p\rangle - \log(c(\boldsymbol{\lambda}))$
	\end{enumerate}
\end{restatable}

Property~1 shows that $p^*$ is the best approximation of $p$ by exponential families in $\mathcal{E}_{\phim}$ \wrt~the KL-divergence.
This fact gives reason to call $p^*$ {\it information projection} of $p$ onto the space $\mathcal{E}_{\phim}$~\cite{csiszar1975}.
Applying Property~2, Property~3 and Eq.~\eqref{eq:relationship:Pinsker} shows that ${\norm{p^*-p}_{L^1}\to 0}$ as $m\to\infty$.
Property~4 is often used in optimization algorithms to compute approximations of $p^*$, see \eg~\cite{batou2013calculation}.

\subsection{Approximation by Maximum Entropy Densities}
\label{subsec:approximation_theory}

In this subsection, we recall some results from the theory of approximation by sequences of maximum entropy densities as proposed by Barron, Sheu and Cox mainly in~\cite{barron1991approximation} and~\cite{cox1988approximation}.
Before that, we need the following definition.

\begin{restatable}[Sobolev Space, see \eg~\cite{adams2003sobolev}]{defrep}{sobolevspace}%
	\label{def:sobolev_space}%
	The Sobolev space of order $r$ \wrt~the $L^2$-norm is defined by
	\begin{align}
		\label{eq:sobolev_space_defi}
		W_2^r=\left\{f:\mathbb{R}^d\to\mathbb{R}\,\middle|\, 
		\norm{\boldsymbol{D}^{\boldsymbol{\alpha}} f}_{L^2}<\infty~\forall\boldsymbol{\alpha}\in\mathbb{N}^d: \norm{\boldsymbol{\alpha}}_1\leq r\right\}.
	\end{align}
\end{restatable}
Note that $f\in W_2^r$ implies that $\norm{f}_{L^2}<\infty$ since $\boldsymbol{D}^{\boldsymbol{\alpha}} f=f$ for $\boldsymbol{\alpha}=(0,\ldots,0)^\text{T}\in\mathbb{N}^d$.

\Needspace{10\baselineskip}
\begin{restatable}[Barron and Sheu~\cite{barron1991approximation}]{lemmarep}{barronandshoulemmafive}%
	\label{lemma:barron_and_sheu_lemma5}%
	Consider some $\phim=(\phi_1,\ldots,\phi_m)^\text{T}$ such that $1,\phi_1,\ldots,\phi_m$ is a basis of $\mathbb{R}_m[x]$ orthonormal with respect to some probability density $q$ with $\norm{\log q}_\infty<\infty$.
	Further consider some $A_q\in\mathbb{R}$ such that $\norm{f_m}_\infty\leq A_q \norm{f_m}_{L^2(q)}$ for all $f_m\in \mathbb{R}_m[x]$.
	
	Let $\boldsymbol{\mu}\in [0,1]^m$, $p_0\in\mathcal{M}([0,1])$ and denote by $\boldsymbol{\mu}_0=\int \phim p_0$ and $b=e^{\norm{\log q/p_0^*}_\infty}$.
	
	If
	\begin{align}
		\label{eq:valid_moment_diff}
		\norm{\boldsymbol{\mu}-\boldsymbol{\mu}_0}_2\leq \frac{1}{4 A_q e b}
	\end{align}
	then the maximum entropy probability density $p^*\in\mathcal{M}([0,1])$ fulfilling $\int\phi_m p^*=\boldsymbol{\mu}$ exists and satisfies
	\begin{align}
		\label{eq:log_ration}
		\norm{\log{p_0^*}/{p^*}}_\infty &\leq 4 e^t b A_q \norm{\boldsymbol{\mu}-\boldsymbol{\mu}_0}_2 \leq t\\
		\label{eq:barron_inequ}
		d_{\mathrm{KL}}(p_0^*,p^*) &\leq 2 e^t b \norm{\boldsymbol{\mu}-\boldsymbol{\mu}_0}_2^2
	\end{align}
	for $t$ satisfying $4 e b A_q \norm{\boldsymbol{\mu}-\boldsymbol{\mu}_0}_2\leq t\leq 1$.
\end{restatable}
The following Corollary~\ref{lemma:barron_proof_thm3} follows from Lemma~\ref{lemma:barron_and_sheu_lemma5} and shows the relation between results on the approximation by exponential families and results on the approximation by polynomials.
\needspace{10\baselineskip}
\begin{restatable}[Barron and Sheu~\cite{barron1991approximation}]{correp}{barronproofthmthree}%
	\label{lemma:barron_proof_thm3}%
	Consider some vector $\phim=(\phi_1,\ldots,\phi_m)^\text{T}$ of polynomials such that $1,\phi_1,\ldots,\phi_m$ is an orthonormal basis of $\mathbb{R}_m[x]$.
	
	Let $p\in\mathcal{M}([0,1])$ such that $\log p\in W_2^r$ and $A_p\in\mathbb{R}$ such that $\norm{f_m}_\infty\leq A_p \norm{f_m}_{L^2(p)}$ for all $f_m\in \mathbb{R}_m[x]$.
	
	Denote by $f=\log p$. Further denote by $\gamma=\min_{f_m\in\mathbb{R}_m[x]} \norm{f-f_m}_\infty$ and $\xi=\min_{f_m\in\mathbb{R}_m[x]} \norm{f-f_m}_{L^2(p)}$ minimal errors of approximating $f$ by polynomials $f_m\in \mathbb{R}_m[x]$.
	Then the following holds:
	\begin{align*}
		4 e^{4\gamma + 1} A_p \xi \leq 1\quad\implies\quad
		\norm{\log{p/p^*}}_\infty \leq 2 \gamma + 4 e^{4 \gamma + 1} \xi A_p.
	\end{align*}
\end{restatable}
The following Corollary gives some insights in the case of maximum entropy densities constrained at sample moments.

\Needspace{10\baselineskip}
\begin{restatable}{correp}{barronsample}%
	\label{cor:barron_sample}%
	Let $p$, $\phim$, $A_p$, $\gamma$ and $\xi$ as in Corollary~\ref{lemma:barron_proof_thm3}. Denote by $b=e^{2 \gamma + 4 e^{4 \gamma + 1} \xi A_p}$ and by $\widehat{\boldsymbol{\mu}}_p=\frac{1}{k}\sum_{\x\in X_p}\phim(\x)$ the sample moments of a $k$-sized sample $X_p$ drawn from $p$.
	
	If $4 e^{4\gamma + 1} A_p \xi \leq 1$ then for all $\delta\in (0,1)$ such that $(4 e b A_p)^2 m \leq \delta k$ with probability at least $1-\delta$ the maximum entropy probability density $\widehat p$ satisfying the constraint $\int\phim\widehat{p}=\widehat{\boldsymbol{\mu}}_p$ exists and the following holds:
	\begin{align}
		\label{eq:barron_sample1}
		d_{\mathrm{KL}}(p^*,\widehat{p})\leq 2 e b \frac{m}{k\delta}\\
		\label{eq:barron_sample2}
		\norm{\log p/\widehat{p}}_\infty\leq 1.
	\end{align}
\end{restatable}

\begin{proof}
	For the proof of Eq.~\eqref{eq:barron_sample1} see the second part of the proof of Theorem~3 in~\cite{barron1991approximation}.
	The proof of Eq.~\eqref{eq:barron_sample2} follows from the application of Eq.~(5.7) of~\cite{barron1991approximation} subsequently to the application of Lemma~5 of~\cite{barron1991approximation} in the proof of Theorem~3 in~\cite{barron1991approximation}.
\end{proof}
\\\\
Note that the approximation error $\xi$ in Corollary~\ref{lemma:barron_proof_thm3} is in terms of $L^2(p)$-norm instead of $L^2(\nu)$ with uniform weight function $\nu$.
To obtain concrete values for the constant $A_p$ in Corollary~\ref{lemma:barron_proof_thm3}, the following result can be applied.

\begin{restatable}[Barron and Sheu~\cite{barron1991approximation}]{lemmarep}{barronAbound}%
	\label{lemma:barron_A_bound}%
	For some $f_m\in\mathbb{R}_m[x]$ with degree less than or equal to $m$ on $[0,1]$ it holds that
	\begin{align}
		\norm{f_m}_\infty \leq (m+1) \norm{f_m}_{L^2}.
	\end{align}
\end{restatable}
The following result from the theory of approximation by orthonormal polynomials can be used to obtain concrete values for the approximation errors $\gamma$ and $\xi$ in Corollary~\ref{lemma:barron_proof_thm3}.

\begin{restatable}[Cox~\cite{cox1988approximation}]{lemmarep}{lemmacox}%
	\label{lemma:cox}%
	For $m\geq r\geq 2$ and $f\in W_2^r$ the following holds:
	\begin{align}
		\min_{f_m\in\mathbb{R}_m[x]}\norm{f-f_m}_\infty &\leq \frac{e^r}{(r-1)^{1/2} (m+r)^{r-1}} \left(\frac{1}{2}\right)^r \norm{f^{(r)}}_{L^2}\\
		\min_{f_m\in\mathbb{R}_m[x]}\norm{f-f_m}_{L^2}^2 &\leq \frac{1}{(m+r+1)\cdots (m-r+2)}\left( \frac{1}{4} \right)^r \norm{f^{(r)}}_{L^2}^2
	\end{align}
\end{restatable}

\section{Neural Networks}
\label{sec:neural_networks}

Empirical risk minimization based on function classes of \textit{neural networks} has improved the state-of-the-art in speech recognition~\cite{hinton2012deep}, visual object recognition~\cite{krizhevsky2012imagenet}, object detection~\cite{ren2015faster} and many other practical areas such as drug discovery~\cite{ma2015deep} and genomics~\cite{xiong2015human}, see \eg~\cite{lecun2015deep} for further references.

In this section, we describe typically used classes of neural networks following mainly Goodfellow, Bengio and Courville~\cite{Goodfellow-et-al-2016} and Shalev and Ben-David~\cite{shalev2014understanding}.
We also briefly summarize related works regarding the expressive power, optimization, generalization properties and principles for the domain adaptation of neural networks.

This section is structured as follows:
Subsection~\ref{subsec:class_nn} gives definitions and some examples of neural networks.
Subsection~\ref{subsec:approximation} discusses the power of neural networks to express functions of different kinds.
Subsection~\ref{subsec:optimization} describes a standard heuristic for finding well performing neural networks and Subsection~\ref{subsec:learning_bounds_nns} provides some generalization properties of neural networks.
Finally, Subsection~\ref{subsec:domain_adaptation_by_nns} reviews recent works for solving practical domain adaptation problems with neural networks.

\subsection{Definition}
\label{subsec:class_nn}

In this work, we focus on the following class of neural networks, sometimes called \textit{fully connected feed-forward neural networks}.

\Needspace{30\baselineskip}
\begin{restatable}[Neural Network, see \eg~\cite{Goodfellow-et-al-2016,berner2018analysis}]{defrep}{neuralnetwork}%
	\label{def:neural_network}%
	A neural network $f_{{\vec a},\boldsymbol{\sigma}, \xi}$ is a function
	\begin{align}
		\begin{split}
			f_{{\vec a},\boldsymbol{\sigma},\xi}:\mathbb{R}^d\times \bigtimes_{i=0}^h \left(\mathbb{R}^{a_{i+1}\times a_i}\times \mathbb{R}^{a_{i+1}}\right) &\to \mathbb{R}^{a_{h+1}}\\
			(\x,\boldsymbol{\theta}) &\mapsto \xi \circ g_{h} \circ \rho_{h}\circ g_{h-1} \circ \ldots \circ \rho_1\circ g_0(\x)
		\end{split}
	\end{align}
	where $h\in\mathbb{N}$ is the number of (hidden) layers, ${\vec a}=(a_1,a_2,\ldots,a_{h+1})^\text{T}\in\mathbb{N}^{h+1}$ is the architecture vector, $\boldsymbol{\sigma}=(\sigma_1,\ldots,\sigma_{h})^\text{T}$ is the vector of hidden activation functions $\sigma_1,\ldots,\sigma_{h}:\mathbb{R}\to\mathbb{R}$ determining
	\begin{align}
		\begin{split}
			\rho_{i}:\mathbb{R}^{a_{i}} &\to\mathbb{R}^{a_{i}}\\
			\x &\mapsto (\sigma_i(x),\ldots,\sigma_i(x))^\text{T}
		\end{split}
	\end{align}
	for $i\in\{1,\ldots,h\}$ and $\xi:\mathbb{R}^{a_h}\to\mathbb{R}^{a_{h+1}}$ is the output activation function.
	For $i\in\{0,\ldots,h\}$ and $a_0=d$, the linear functions
	\begin{align}
		\begin{split}
			g_i:\mathbb{R}^{a_i} &\to\mathbb{R}^{a_{i+1}}\\
			\x &\mapsto {\vec W}_i \x + {\vec b}_i
		\end{split}
	\end{align}
	are determined by the parameter vector
	\begin{align}
		\boldsymbol{\theta}=(({\vec W}_0,{\vec b}_0),\ldots,({\vec W}_h,{\vec b}_h))\in \bigtimes_{i=0}^h \left(\mathbb{R}^{a_{i+1}\times a_i}\times \mathbb{R}^{a_{i+1}}\right).
	\end{align}
\end{restatable}
Note that Definition~\ref{def:neural_network} is very general in the sense that it models most common examples of neural networks including \textit{restricted Boltzmann machines} and \textit{convolutional} neural networks which are often applied on images.
The number $h$ of hidden layers is called \textit{depth} of the neural network and neural networks with a large depth $h$ are called \textit{deep}~\cite{lecun2015deep}.
\\
\begin{example}
	\label{ex:single_layer_neural_network}
	Consider the single-layer neural network
	\begin{align}
		f_{(a_1,a_2),\sigma,\xi}:\mathbb{R}^d \times \left( (\mathbb{R}^{a_1\times d}\times\mathbb{R}^{a_1})\times (\mathbb{R}^{a_2\times a_1}\times\mathbb{R}^{a_2})\right)&\to \mathbb{R}^{a_2}
	\end{align}
	with
	\begin{align}
		f(\x, (({\vec W}_0,{\vec b}_0),({\vec W}_1,{\vec b}_1)))=\xi\!\left({\vec W}_1 \cdot\rho({\vec W}_0\cdot \x+{\vec b}_0) + {\vec b}_1)\right)
	\end{align}
	where $\rho(\x)=(\sigma(x_1),\ldots,\sigma(x_{a_1}))^\text{T}$ for some activation vector $\x=(x_1,\ldots,x_{a_1})^\text{T}$.
	
	Standard choices for $\sigma$ are the sigmoid function $\mathrm{sigm}(x)=\frac{1}{1+e^{-x}}$, the tangens hyperbolicus $\mathrm{tanh}(x)=1-\frac{2}{e^{2 x}+1}$ and the rectifier linear unit $\mathrm{relu}(x)=\max\{0,x\}$.
	
	Many theoretical results for classification are based on $a_2=1$, $\sigma(x)=\xi(x)=\mathbbm{1}_{\mathbb{R}_+}(x)$ which is one iff $x>0$ and zero otherwise, or the signum function $\sigma(x)=\xi(x)=\sign(x)$.
	
	If the problem is regression, \ie~to approximate some unknown function $l:\mathbb{R}^d\to \mathbb{R}$ based on a given training sample, common choices for the output activation and the output dimension are $\xi(x)=x$ and $a_2=1$, respectively.
	
	If the problem is to discriminate between $c$ classes, a common choice is $a_2=c$ and the softmax function
	\begin{align}
		\label{eq:softmax}
		\xi(\x)=\mathrm{softmax}(\x)=\frac{1}{\sum_{i=1}^{a_1} e^{x_i}}\cdot(e^{x_1},\ldots,e^{x_{a_1}})^\text{T}.
	\end{align}
	In this case, the predicted class $y$ of some input $\x\in\mathbb{R}^{d}$ is given by $y=\argmax(f(\x,\boldsymbol{\theta}))$, where $\argmax({\vec v})=\argmax_{i\in\{1,\ldots,d\}} v_i$ denotes the largest element of some vector ${\vec v}$.
	One advantage of this choice is that the $i$-th elements of the softmax vector in Eq.~\eqref{eq:softmax} can be interpreted as the likelihood of the vector $\x$ belonging to class $i$.
\end{example}

In the following, let us denote the set of neural networks with depth $h$, width $w$, output dimension $c$, activation functions all equal $\sigma:\mathbb{R}\to\mathbb{R}$ and output function $\xi$ by
\begin{align}
	\begin{split}
		\mathcal{N}_{h,w,c,\sigma,\xi}=\Big\{ f:\mathbb{R}^d\to\mathbb{R}^c\,\Big|\, &f(\x)=f_{{\vec a},(\sigma,\ldots,\sigma)^\text{T},\xi}(\x,\boldsymbol{\theta}), {\vec a}=(w,\ldots,w,c)^\text{T}\in\mathbb{N}^{h+1},\\
		&\boldsymbol{\theta}\in \bigtimes_{i=0}^h \left(\mathbb{R}^{a_{i+1}\times a_i}\times \mathbb{R}^{a_{i+1}}\right) \Big\}.
	\end{split}
\end{align}

\subsection{Expressive Power}
\label{subsec:approximation}

In this subsection we provide results showing that neural networks are able to approximate well very general functions if the network size is sufficiently large.

The following result for single-layer neural networks holds.

\Needspace{10\baselineskip}
\begin{restatable}[Universal Approximation Theorem, see \eg~\cite{hassoun1995fundamentals}]{thmrep}{universalapproxthm}%
	\label{thm:universal_approx_thm}%
	Let $\sigma:\mathbb{R}\to\mathbb{R}$ be a non-constant, bounded and continuous activation function.
	Then, for any $d\in\mathbb{N}$, any $\epsilon>0$ and any continuous function $g:[-1,1]^d\to[0,1]$, there exists a single-layer neural network $f\in\mathcal{N}_{1,w,1,\sigma,\mathbbm{1}_{\mathbb{R}_+}}$ such that
	\begin{align*}
		\forall\x\in [-1,1]^d:\,|g(\x)-f(\x)|<\epsilon.
	\end{align*}
\end{restatable}
However, as shown by the following theorem for the sigmoid activation function, the width of the single hidden layer might be very large.
\needspace{10\baselineskip}
\begin{restatable}[Size of Expressive Sigmoid Networks, see \eg~\cite{shalev2014understanding}]{thmrep}{exponentialwidththmforsigmoid}%
	\label{thm:exponential_width_thm_sigmoid}%
	Let $w(d)\in\mathbb{N}$ be the minimal integer such that, for every $\epsilon>0$ and every $1$-Lipschitz continuous function $g:[-1,1]^d\to [0,1]$, there exists some $f\in\mathcal{N}_{1,w,1,\mathrm{sigm},\mathrm{sigm}}$ with the property
	\begin{align*}
		\forall\x\in [-1,1]^d:\,|g(\x)-f(\x)|<\epsilon.
	\end{align*}
	Then $w(n)$ is exponential in $d$.
\end{restatable}
Recall that a function $f:\mathbb{R}^d\to\mathbb{R}$ is $1$-Lipschitz continuous iff $|f(\x_1)-f(\x_2)|\leq \norm{\x_1-\x_2}_2$ for all $\x_1,\x_2\in\mathbb{R}^d$.

Given the success of deep neural networks, the question arises if a neural network that approximates well very general functions always needs to have large width $w$.
It turns out, that also a large depth $h$ can result in a strong expressive power for appropriate activation functions.
\begin{restatable}[Universal Approximation Theorem for ReLU networks, Lu et al.~\cite{lu2017expressive}]{thmrep}{universalapproxthmrelu}%
	\label{thm:universal_approx_thm_relu}%
	For any $d\in\mathbb{N}$, any $\epsilon>0$ and any Lebesgue-integrable function $g:\mathbb{R}^d\to\mathbb{R}$, there exists a neural network $f\in\mathcal{N}_{1,d+4,1,\mathrm{relu},\mathbbm{1}_{\mathbb{R}_+}}$ such that
	\begin{align*}
		\forall\x\in \mathbb{R}^d:\,|g(\x)-f(\x)|<\epsilon.
	\end{align*}
\end{restatable}
The results above show how expressive typical neural networks can be.
A large field of recent theory has sought to explain the broad success of neural networks via such results, see \eg~\cite{pmlr-v97-hanin19a} for further references.
However, it is important to note that the expressive power alone does not guarantee that learning problems can be efficiently solved, see \eg~\cite{shalev2017failures} for recently discovered examples of functions which cannot be efficiently estimated based on finitely many examples.

Despite the computational hardness of learning neural networks, there exists a standard heuristics which performs well in many practical tasks.
This heuristic is described in the next subsection.

\subsection{Stochastic Gradient Descent}
\label{subsec:optimization}

In this subsection we describe a standard heuristic for finding well performing neural networks: the \textit{stochastic gradient descent} algorithm.

Finding a neural network from the class $\mathcal{N}_{h,w,1,\sign,\sign}$ which has minimum empirical risk as described in Section~\ref{sec:stat_learn_th} is NP hard even for networks with a single hidden layer that contains just four neurons~\cite{shalev2014understanding}.
Similar results hold for the aim of close-to-minimal empirical error~\cite{bartlett2002hardness}.
There are also strong indications that the computational hardness is not mitigated by using deep neural networks or activation functions different from the signum function.
One such indication is that, under some cryptographic assumption, the problem of finding, based on finitely many examples, a good approximation of a function composed of intersections of halfspaces, is known to be computationally hard~\cite{klivans2009cryptographic}.

Nevertheless, there is a heuristic which often finds good solutions to practical learning problems: the stochastic gradient descent algorithm.
Gradient descent algorithms are optimization procedures which iteratively improve the solution candidates by making steps towards the negative of the gradient of the function at the current candidate point.
However, in learning problems, only samples are given and the underlying functional dependency is unknown.
Stochastic gradient descent overcomes this problem by allowing to step along a random direction as long as the expected value of the direction is a good approximation of the negative of the gradient.

In particular, stochastic gradient descent algorithms aim at finding a parametric function $f\in\mathcal{F}$ which approximates an unknown functional dependency $l:\mathbb{R}^d\to \mathbb{R}^c$ by minimizing a loss function
\begin{align}
	\label{eq:loss_function}
	\begin{split}
		L:\mathbb{R}^{c}\times\mathbb{R}^{c} &\to\mathbb{R}_+\\
		(f(\x),l(\x)) &\mapsto L(f(\x),l(\x))
	\end{split}
\end{align}
according to the parameter update rule
\begin{align}
	\label{eq:gradient_update}
	\vtheta_{i+1}=\vtheta_i - \alpha\cdot \boldsymbol{\nu}_i\odot \nabla_{\vtheta} \frac{1}{|X_i|}\sum_{\x\in X_i} L(f(\x),l(\x)),
\end{align}
where $X_1,X_2,\ldots\subseteq X_p$ are random submultisets of $X_p$ all having the same size, $\alpha\in\mathbb{R}_+$ is the \textit{learning rate} and $\boldsymbol{\nu}_1,\boldsymbol{\nu}_2,\ldots$ are parameters of the same size as $\vtheta$ realizing some weighting of the learning rate by means of the element-wise multiplication $\odot$ with the gradient.
A pseudo code is given in Algorithm~\ref{alg:sgd}.
There, the computation of the predicted outputs $f(\x)$ for $\x\in X_i$ used to compute the gradient in Step~2 is called \textit{forward pass}.
The subsequent computation of the gradient in $X_i$ is called \textit{backpropagation}.

\SetKwInOut{Stepzero}{Step 0}
\SetKwInOut{Stepone}{Step 1}
\SetKwInOut{Steptwo}{Step 2}
\SetKwInOut{Stepthree}{Step 3}
\SetKwInOut{Stepfour}{Increment}
\SetKwInOut{Init}{Init}
\begin{algorithm}
	\SetAlgoLined
	\KwIn{Sample $X_p=\{\x_1,\ldots,\x_k\}$ with labels $Y=\{l(\x_1,\ldots,l(\x_k)\}$, learning rate $\alpha$ and learning rate weighting $\boldsymbol{\nu}_1,\boldsymbol{\nu}_2,\ldots$
	}
	\KwOut{Parameter vector $\vtheta$}~\\
	\Init{Initialize parameter vector $\vtheta_0$ randomly and set $i=0$}
	\While{\upshape stopping criteria is not met}{
		\Stepone{Find random submultiset $X_i$ from $X_p$}
		\Steptwo{Calculate the gradient ${\vec w}_i=\nabla_{\vtheta} \frac{1}{|X_i|}\sum_{\x\in X_i} L(f(\x),l(\x))$}
		\Stepthree{Update $\vtheta_{i+1}=\vtheta_i - \alpha\cdot \boldsymbol{\nu}_i\odot {\vec w}_i$}
		\Stepfour{$i:=i+1$}
	}
	\caption[Stochastic gradient descent.]{Stochastic gradient descent for minimizing $L(f(\x),l(\x))$}
	\label{alg:sgd}
\end{algorithm}

In the following, we derive the stochastic gradient descent algorithm for the problem of multi-class classification and single-layer neural networks with sigmoid activation function in the hidden layer and the softmax output function as described in Example~\ref{ex:single_layer_neural_network}.
In Subsection~\ref{subsec:domain_adaptation_by_nns} we show how to extend this algorithm for solving domain adaptation problems.

Similarly to Problem~\ref{problem:binary_classification} of binary classification, in multi-class classification we consider some unknown probability density function $p\in\MR$ and a labeling function $l:\mathbb{R}^d\to [0,1]^c$, where the $i$-th coordinate of some vector $l(\x)$ represents the probability that $\x\in\mathbb{R}^d$ belongs to class $i$.
Given a sample $X_p=\{\x_1,\ldots,\x_k\}$ drawn from $p$ with labels $Y=\{l(\x_1),\ldots,l(\x_k)\}$, the problem is to find some $f\in \mathcal{N}_{1,w,\mathrm{softmax}}$ with a small multi-class misclassification risk
\begin{align}
	\label{eq:multiclass_misclassification_risk}
	\int_{\mathbb{R}^d} \sum_{i=1}^c \left|f_i(\x)-l_i(\x)\right|\, p(\x)\diff\x
\end{align}
where $f_i(\x)$ is the $i$-th element of the vector $f(\x)$.

Unfortunately, the function $\x\mapsto \sum_{i=1}^c|f_i(\x)-l_i(\x)|$ is not everywhere differentiable and is consequently not a good choice for a loss.
A standard approach to overcome this problem is to use the cross-entropy loss
\begin{align}
	\label{eq:empirical_cross_entropy_objective}
	L(f(\x),l(\x))= \sum_{i=1}^c - l_i(\x) \log(f_i(\x)).
\end{align}
Consider now the single-layer neural network function
\begin{align}
	\label{eq:single_layer_softmax}
	f_{(a_1,c),\sigma,\mathrm{softmax}}:\mathbb{R}^d \times \left( (\mathbb{R}^{a_1\times d}\times\mathbb{R}^{a_1})\times (\mathbb{R}^{c\times a_1}\times\mathbb{R}^{c})\right) \to \mathbb{R}^{c}
\end{align}
as defined in Example~\ref{ex:single_layer_neural_network} with
\begin{align}
	f(\x, (({\vec W}_0,{\vec b}_0),({\vec W}_1,{\vec b}_1)))=\mathrm{softmax}\!\left({\vec W}_1 \cdot\rho({\vec W}_0\cdot \x+{\vec b}_0) + {\vec b}_1)\right)
\end{align}
where $\rho(\x)=(\sigma(x_1),\ldots,\sigma(x_{a_1}))^\text{T}$ for some activation vector $\x=(x_1,\ldots,x_{a_1})$.

The gradient $\nabla_{\vtheta} \frac{1}{k} \sum_{i=1}^k L(f(\x_i),l(\x_i))$ in $X=\{\x_1,\ldots,\x_k\}$ \wrt~the parameter vector
\begin{align}
	\vtheta=(({\vec W}_0,{\vec b}_0),({\vec W}_1,{\vec b}_1))\in \left( (\mathbb{R}^{a_1\times d}\times\mathbb{R}^{a_1})\times (\mathbb{R}^{c\times a_1}\times\mathbb{R}^{c})\right)
\end{align}
is then given by
\begin{align}
	\label{eq:gradient_for_sgd_algo}
	\begin{split}
		\nabla_{\vtheta} \E[L(f(X),l(X))] = \Big(\big(\nabla_{\vec W_0} \E&[L(f(X),l(X))],\nabla_{\vec b_0} \E[L(f(X),l(X))]\big),\\
		&\big(\nabla_{\vec W_1} \E[L(f(X),l(X))],\nabla_{\vec b_1} \E[L(f(X),l(X))]\big)\Big)
	\end{split}
\end{align}
where $\E[X]=\frac{1}{k}\sum_{i=1}^k \x_i$ is the vector of empirical expectations, $f(X)=\{f(\x_1),\ldots,f(\x_k)\}$ and
\begin{align*}
	\nabla_{\vec b_1} \E[L(f(X),l(X))] &= \E[f(X)- Y]\\
	\nabla_{\vec W_1} \E[L(f(X),l(X))] &= \E[(f(X)- Y)\cdot f(X)^\text{T}] \\
	\nabla_{\vec b_0} \E[L(f(X),l(X))] &= \E[\vec W_1^\text{T} (f(X)- Y) \odot (\vec 1 - f(X))]\\
	\nabla_{\vec W_0} \E[L(f(X),l(X))] &= \E[(\vec V^\text{T} (f(X)- Y) \odot f(X) \odot (\vec 1 - f(X))) \cdot X^\text{T}]
\end{align*}
for $\vec 1=(1,\ldots,1)^\text{T}$.
The above formulas follow from $\nabla_x\mathrm{sigm}(x)=\mathrm{sigm}(x)\odot (1-\mathrm{sigm}(x))$ and standard application of the chain rule, see \eg~\cite{Goodfellow-et-al-2016} for more detailed derivations.

One example for the learning rate weighting sequence $\boldsymbol{\nu}_1,\boldsymbol{\nu}_2,\ldots$ is to choose an exponentially decreasing sequence
$$
\beta_1 e^{-1 \beta_2}\cdot{\vec 1},\beta_1 e^{-2 \beta_2}\cdot{\vec 1},\beta_1 e^{-3 \beta_2}\cdot{\vec 1},\ldots
$$
with constants $\beta_1,\beta_2\in\mathbb{R}_+$ and the vector $\vec 1\in \left( (\mathbb{R}^{a_1\times d}\times\mathbb{R}^{a_1})\times (\mathbb{R}^{c\times a_1}\times\mathbb{R}^{c})\right)$ with all elements being $1$.

Another example is to use
\begin{align}
	\label{eq:adagrad}
	\begin{split}
		\boldsymbol{\nu}_i &= \frac{{\vec 1}}{\sqrt{{\vec z}_i}}\\
		{\vec z}_{i+1} &= {\vec z}_i + \left(\nabla_{\vtheta} \E[L(X,l(X))]\right)^2
	\end{split}
\end{align}
for $i\in\{1,2,\ldots\}$, ${\vec z}_0=1$, element-wise division and element-wise square-root.
The optimization algorithm resulting from applying these weights is called \textit{Adagrad}~\cite{duchi2011adaptive}.
Eq.~\eqref{eq:adagrad} realizes a gradient update according to different update weights for each dimension.
Adagrad can be interpreted as dividing the learning rate $\alpha$ by the $\ell_2$-norm of the historical gradients.
The idea is to give frequently occurring features very low learning rates and infrequent features high learning rates.
The Adagrad algorithm performs well in many practical cases of sparse data as given in the experiment described in Subsection~\ref{subsec:sentiment_analysis_experiment}.

However, in many practical cases of non-sparse data, the Adagrad optimizer can be improved based on the following sequence
\begin{align}
	\label{eq:adadelta}
	\begin{split}
		{\vec z}_{i+1} &= \omega {\vec z}_i + (1-\omega) \left(\nabla_{\vtheta} \E[L(X,l(X))]\right)^2\\
		\boldsymbol{\nu}_{i+1} &= \frac{\sqrt{{\vec v}_i}}{\sqrt{{\vec z}_{i+1}+\boldsymbol{\epsilon}}}\\
		{\vec v}_{i+1} &:= \omega {\vec v}_i - (1-\omega) \left(\frac{\boldsymbol{\nu}_i}{\sqrt{{\vec z}_{i}+\boldsymbol{\epsilon}}}\odot \nabla_{\vtheta} \E[L(X,l(X))]\right)^2,
	\end{split}
\end{align}
for $i\in\{1,2,\ldots\}$, ${\vec z}_0=0$, $\boldsymbol{\epsilon}$ being a vector of small constants for numerical stability and $\omega$ being the so called decay constant often set to $0.95$.
The algorithm resulting from the weighting sequence in Eq.~\eqref{eq:adadelta} is called \textit{Adadelta}~\cite{zeiler2012adadelta}.
Adadelta seeks to reduce the strongly monotonically decreasing learning rate of Adagrad by reducing the effect of historical gradients.
Adadelta requires no manual tuning of a learning rate, i.e. $\alpha=1$ is often a good choice, and appears robust to noisy gradient information, different model architecture choices and various data modalities.

We apply Adadelta on many practical problems in Subsection~\ref{subsec:object_recognition_experiment} and Subsection~\ref{subsec:digit_recognition_experiment}.

Although rarely used, alternatives to the stochastic gradient descent heuristic for training neural networks are \textit{evolutionary algorithms}~\cite{stanley2019designing} which are especially useful for multi-objective optimization problems as \eg~\cite{rossbory2013parallelization,chasparis2016optimization}.

\subsection{Learning Bounds}
\label{subsec:learning_bounds_nns}

This subsection shows the VC-dimension of some neural networks which leads to bounds on the misclassification risk by using the results stated in Subsection~\ref{subsec:learning_bounds_in_stat_learn_th} and Subsection~\ref{subsec:learning_bounds_for_domain_adaptation}.

From Eq.~\eqref{eq:vc_dim__of_affine_functions} we know that the VC-dimension of a neural network without a hidden layer with signum output activation function equals the number of parameters plus one.
This result can be extended to networks with larger depth.

\begin{restatable}[VC-Dimension of Binary Networks, Baum and Haussler~\cite{baum1989size}]{thmrep}{vcdimensionsign}%
	\label{thm:vc_dim_sign_networks}%
	For $w\geq 2$ and $h\geq 1$ it holds that $\mathrm{VC}(\mathcal{N}_{h,w,1,\mathbbm{1}_{\mathbb{R}_+},\mathbbm{1}_{\mathbb{R}_+}})\leq 2 r \log_2\left(e r\right)$, where $r$ is the number of parameters.
\end{restatable}
The number of free parameters $r$ of a neural network equals the total dimension of its parameter vector $\vtheta$. In the case of $f\in \mathcal{N}_{h,w,1,\mathbbm{1}_{\mathbb{R}_+},\mathbbm{1}_{\mathbb{R}_+}}$ it is $r = w d+w+w c+c$ if $h=1$ and $r = w d+w+(h-1) (w^2+w) + w c+c$ if $h\geq 1$.

See Table~\ref{tab:activation_functions} for a summary of results similar to Theorem~\ref{thm:vc_dim_sign_networks} for various activation functions.

\begin{table}[ht]
	\centering
	\begin{tabular}{llr}
		\hline
		Activation Function $\sigma$ & VC-dimension & Reference\\
		\hline
		$\mathrm{sign}, \mathbbm{1}_{\mathbb{R}_+}$ & $O(r\log r)$ & \cite{baum1989size}\\
		& $\Omega(r\log r)$ & \cite{maass1994neural}\\
		\hline
		piecewise linear, incl.~$\mathrm{relu}$ & $O(r h\log r)$ & \cite{harvey2017nearly}\\
		& $\Omega(r h\log(r/h))$ & \cite{harvey2017nearly}\\
		\hline
		Pfaffian, incl.~$\mathrm{sigm}$ & $O(r^2 h^2 w^2)$ & \cite{karpinski1997polynomial}\\
		\hline
	\end{tabular}
	\caption[Complexity of VC-dimension for different classes of neural networks.]{Complexity of VC-dimension for different classes $\mathcal{N}_{h,w,1,\sigma,\mathbbm{1}_{\mathbb{R}_+}}$ of neural networks with $r$ parameters and different activation functions $\sigma$.}
	\label{tab:activation_functions}
\end{table}

\subsection{Domain Adaptation}
\label{subsec:domain_adaptation_by_nns}

In this subsection we show how to extend the stochastic gradient descent algorithm to solve problems of unsupervised domain adaptation.
The described approach follows the principle of learning new data representations by minimizing empirical estimations of integral probability metrics as described in Subsection~\ref{subsec:da_algorithms}.
We also review some recent related works.

Given a source sample $X_p=\{\x_1,\ldots,\x_k\}$ drawn from some unknown $p\in\MR$ with labels $Y=\{l(\x_1),\ldots,l(\x_k)\}$ labeled by some unknown labeling function $l$ and a target sample $X_q=\{\x_1',\ldots,\x_s'\}$ drawn from some unknown $q\in\MR$, the goal of unsupervised domain adaptation is to find some function $f$ from a model class $\mathcal{F}$ which is a good approximation of the unknown dependency $l$.

The representation learning principle aims at minimizing Eq.~\eqref{eq:objective_of_principle_of_representation_learning_for_da}.
Unfortunately, as discussed in Subsection~\ref{subsec:optimization}, this is not a good choice as minimization objective for stochastic gradient descent.
In the case of multi-class classification and single-layer neural networks, an appropriate choice is
\begin{align}
	\label{eq:da_loss_function}
	\frac{1}{k} \sum_{i=1}^k L(f(\x_i),l(\x_i)) + \lambda\cdot \hat{d}\left(h_0(X_p),h_0(X_q)\right),
\end{align}
where $L$ is the cross-entropy loss as defined in Eq.~\eqref{eq:empirical_cross_entropy_objective}, $\hat{d}$ is a distance function between two samples, $\lambda>0$ is a weighting factor and $h_0(X)=\{h_0(\x_1),\ldots,h_0(\x_k)\}$ are the activations of the sample $X_p$ with $h_0(\x)=\rho({\vec W}_0\cdot\x + {\vec b}_0)$ being the output of the hidden layer of the neural network defined in Eq.~\eqref{eq:single_layer_softmax}.
Algorithm~\ref{alg:sgd} can now be used to solve domain adaptation problems by finding, in addition to a random submultisample $X_i$ from $X_p$, a random submultisample $X_i'$ from $X_q$, and, by replacing the gradient in Step~2 with the gradient
\begin{align}
	\label{eq:sgd_da_gradient}
	{\vec w}_i=\nabla_{\vtheta} \left(\frac{1}{|X_i|}\sum_{\x\in X_i} L(f(\x),l(\x)) + \lambda\cdot \hat{d}\left(h_0(X_i),h_0(X_i')\right) \right).
\end{align}
Note that the gradient $\nabla_{\vtheta} \frac{1}{|X_i|}\sum_{\x\in X_i} L(f(\x),l(\x))$ is given in Eq.~\eqref{eq:gradient_for_sgd_algo} and the only missing part is the gradient $\nabla_{\vtheta} \hat{d}\left(h_0(X_i),h_0(X_i')\right)$ with distance $\hat{d}$.
The procedure is summarized in Algorithm~\ref{alg:sgd_for_da} and Figure~\ref{fig:da_by_sgd}.

\begin{algorithm}
	\SetAlgoLined
	\KwIn{Source sample $X_p=\{\x_1,\ldots,\x_k\}$ with labels $Y=\{l(\x_1),\ldots,l(\x_k)\}$, target sample $X_q=\{\x_1',\ldots,\x_s'\}$, learning rate $\alpha$, regularization parameter $\lambda$ and learning rate weighting $\boldsymbol{\nu}_1,\boldsymbol{\nu}_2,\ldots$
	}
	
	\KwOut{Parameter vector $\vtheta=(({\vec W}_0,{\vec b}_0),({\vec W}_1,{\vec b}_1))\in \left( (\mathbb{R}^{w\times d}\times\mathbb{R}^{w})\times (\mathbb{R}^{c\times w}\times\mathbb{R}^{c})\right)$ such that $f(\x, (({\vec W}_0,{\vec b}_0),({\vec W}_1,{\vec b}_1)))=\mathrm{softmax}\left({\vec W}_1 \cdot\rho({\vec W}_0\cdot \x+{\vec b}_0) + {\vec b}_1)\right)$ with $\rho(\x)=(\mathrm{sigm}(x_1),\ldots,\mathrm{sigm}(x_w))^\text{T}$}~\\
	
	\Init{Initialize parameter vector $\vtheta_0$ randomly and set $i=0$}
	\While{\upshape stopping criteria is not met}{
		\Stepone{Find random submultisets $X_i$ from $X_p$ and $X_i'$ from $X_q$}
		\Steptwo{Calculate the gradient ${\vec w}_i=\nabla_{\vtheta} \left(\frac{1}{|X_i|}\sum_{\x\in X_i} L(f(\x),l(\x)) + \lambda\cdot \hat{d}\left(h_0(X_i),h_0(X_i')\right) \right)$ where $h_0(\x)=\rho({\vec W}_0\cdot\x + {\vec b}_0)$}
		\Stepthree{Update $\vtheta_{i+1}=\vtheta_i - \alpha\cdot \boldsymbol{\nu}_i\odot {\vec w}_i$}
		\Stepfour{$i:=i+1$}
	}
	\caption[Unsupervised domain adaptation for single-layer neural networks via stochastic gradient descent.]{
		Unsupervised domain adaptation for finding a single-layer multi-class neural network ${f\in\mathcal{N}_{1,w,c,\mathrm{sigm},\mathrm{softmax}}}$ via stochastic gradient descent
	}
	\label{alg:sgd_for_da}
\end{algorithm}

One good choice for the distance $\hat{d}$ in Algorithm~\ref{alg:sgd_for_da} is the Frobenius norm between the sample covariance matrices of the neural network activations~\cite{sun2016deep}.
This distance function is parameter-free and the resulting algorithm is relatively robust to changes of the regularization parameter $\lambda$.

Similarly, the differences between mean and sample variances in each direction is minimized in~\cite{li2016revisiting,li2018adaptive,wang2019transferable}.
Therefore, a neural network specific method called \textit{batch normalization} is extended for domain adaptation problems.

The Wasserstein distance is applied in~\cite{lee2019sliced} and sampled via a variational formulation.

Another approach is to minimize the empirical $\mathcal{F}$-divergence as described in Eq.~\eqref{eq:h_divergence} where $\mathcal{F}$ is some class of neural networks.
The works proposed in~\cite{ganin2016domain,tzeng2017adversarial,bousmalis2017unsupervised} are based on training a classifier which aims at discriminating source samples from target samples.
For minimizing the distance between source and target data representations, the gradient of the new classifier is reversed during backpropagation.

\begin{figure}[t]
	\includegraphics[width=\linewidth]{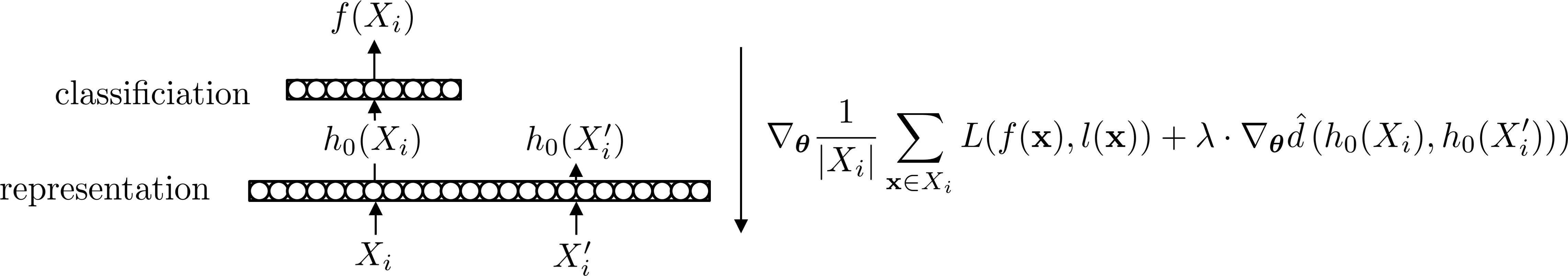}
	\caption{Forward pass and backpropagation in Algorithm~\ref{alg:sgd_for_da}.
	}
	\label{fig:da_by_sgd}
\end{figure}

The maximum mean discrepancy as described in Subsection~\ref{subsec:ten_probability_metrics} is also a good choice.
Different kernel functions lead to different versions of the maximum mean discrepancy and consequently to different behaviours of Algorithm~\ref{alg:sgd_for_da}.
There exist approaches that are based on linear kernels~\cite{tzeng2014deep, csurka2016unsupervised} that can be interpreted as mean feature matching.
A combination of Gaussian kernels is used in~\cite{long2015learning} to tackle the sensitivity of the maximum mean discrepancy \wrt~changes of the Gaussian kernel parameter by means of a combination of different kernels with heuristically selected parameters.
This approach is underpinned by theoretical knowledge from studies about reproducing kernel Hilbert spaces~\cite{fukumizu2009kernel} and a linear-time implementation is proposed.
It is shown that further improvements are possible based on more sophisticated neural network architectures~\cite{bousmalis2016domain,long2016unsupervised,long2016joint}.

In Chapter~\ref{chap:cmd_algorithm} we propose a new moment distance for Algorithm~\ref{alg:sgd_for_da}.

It is important to note that the problem of selecting the parameter $\lambda$ in Eq.~\eqref{eq:sgd_da_gradient} is sophisticated since no target labels are given.
Consequently, classical cross-validation cannot be used as it would suffer from an unbounded bias in the generalization error estimate~\cite{zhong2010cross}.
Consequently, finding good algorithms for selecting the parameter $\lambda$ is an active research area~\cite{you2019towards}.
Many methods rely on a small set of data from the target scenarios~\cite{courty2017joint,long2015learning} or fix their parameters to some default values~\cite{sun2016deep}.

The classical cross-validation algorithm is extended in~\cite{zhong2010cross} and~\cite{ganin2016domain} for problems of domain adaptation.
The approach for calculating an empirical estimate of the true risk is as follows:
The source sample $X_p=\{\x_1,\ldots,\x_k\}$ and the target sample $X_q=\{\x_1',\ldots,\x_s'\}$ are split into training samples $S$ and $T$, respectively, containing $90\%$ of the original samples, and, validation samples $S_\mathrm{val}$ and $T_\mathrm{val}$, respectively, containing $10\%$ of the original samples.
Then, the training sample $S$ with corresponding labels and the training sample $T$ are used to find a classifier $f$ for the unsupervised domain adaptation problem.
Using the same algorithm, an additional \textit{reverse} classifier is trained on the sample $T$ with labels $\{f(\x)\mid \x\in T\}$ as 'source sample' and the sample $S$ as 'target sample'.
For this reverse classifier, the empirical risk is calculated based on the validation sample $S_\mathrm{val}$ and the corresponding given labels.
Finally, the empirical risk is used as an estimate for the target risk.

The above procedure is used in Subsection~\ref{subsec:sentiment_analysis_experiment} to select an appropriate value of the parameter $\lambda$ in benchmark experiments.

\newpage

\chapter{Learning Bounds for Moment-Based Domain Adaptation}
\label{chap:learning_bounds}

Standard approaches for domain adaptation measure the adaptation discrepancy based on empirical estimations of probability metrics.
In this chapter, we derive a theoretical framework for domain adaptation which is based on weak assumptions on the similarity of distributions.
Our weak assumptions are formulated by moment distances.
As a main result, we derive learning bounds under practice-oriented general conditions on the underlying probability distributions.

This chapter is structured as follows:
Section~\ref{sec:motivation_of_learning_bounds} motivates the work done in this chapter.
Section~\ref{sec:learning_bounds_relation_to_state_of_the_art} describes some relations to recent works in domain adaptation, moment-based bounds on distances between distributions and exponential families.
Section~\ref{sec:problem_formulation_binary_classification} formulates the problem considered in this chapter.
Section~\ref{sec:approach_by_convergence_analysis} discusses our approach based on convergence rate analysis.
Section~\ref{sec:main_result_on_learning_bounds} proposes our main result on moment-based learning bounds and Section~\ref{sec:proofs_learning_bounds} gives all proofs.

\section{Motivation and General Idea}
\label{sec:motivation_of_learning_bounds}

Domain adaptation problems are encountered in everyday life of engineering machine learning applications whenever there is a discrepancy between assumptions on the learning and the application setting.
As discussed in Section~\ref{sec:stat_learn_th}, most theoretical and practical results in statistical learning are based on the assumption that the training and test sample are drawn from the same distribution.
However, as outlined in Section~\ref{sec:domain_adaptation}, this assumption may be violated in typical applications such as natural language processing~\cite{blitzer2007biographies,jiang2007instance} and computer vision~\cite{sun2014virtual,ganin2016domain}. 

We relax the classical assumption of identical distributions under training and the application setting by postulating that only a finite number of moments of these distributions are aligned.

This postulate is motivated two-fold.
The first motivation is the current scientific discussion about the choice of an appropriate distance function for domain adaptation~\cite{ben2007analysis,courty2017optimal,long2015learning,long2016unsupervised,zhuang2015supervised,ganin2016domain}.
Standard approaches study domain adaptation based on empirical estimations of strong probability metrics.
The convergence in most common probability metrics of compactly supported distributions implies the convergence of finitely many moments.
In particular, many common probability metrics admit upper bounds on moment distances.
For example consider Figure~\ref{fig:metrics_with_moments} which is based on the following Lemma~\ref{lemma:moment_distance_bound_by_levy}. See Subsection~\ref{subsec:bound_on_moment_distance_by_levy_metric} for its proof.
\begin{restatable}{lemmarep}{momentdistanceboundbylevy}%
	\label{lemma:moment_distance_bound_by_levy}%
	Let $m\in\mathbb{N}, m\geq 2$, $\boldsymbol{\phi}\in\left(\mathbb{R}_m[x]\right)^n$ be a vector of polynomials with maximum total degree $m$ and let $p,q\in\mathcal{M}([0,1])$ with moments denoted by $\boldsymbol{\mu}_p=\int \boldsymbol{\phi} p$ and $\boldsymbol{\mu}_q=\int \boldsymbol{\phi} q$. The there exist some constants $C_\text{L},M_\text{L}\in\mathbb{R}_+$ such that
	\begin{align}
		\label{eq:rachev_inequality}
		d_\text{L}(p,q)\leq M_\text{L}\quad\implies\quad\norm{\boldsymbol{\mu}_p - \boldsymbol{\mu}_q}_1\leq C_\text{L}\cdot d_\text{L}(p,q)^{\frac{1}{2 m+2}}.
	\end{align}
\end{restatable}
The considered postulate of a finite number of aligned moments is therefore weak compared to the assumption of distributions which are similar in typical probability metrics.
One implication is that results under the proposed setting can also give theoretical insights for approaches based on stronger concepts of similarity like the Wasserstein distance~\cite{courty2017optimal,lee2019sliced}, the maximum mean discrepancy~\cite{long2016unsupervised} or $f$-divergences~\cite{zhuang2015supervised}.

The second motivation of our postulate is the methodology to overcome a present difference in distributions by mapping the samples into a latent model space where the resulting corresponding distributions are aligned.
See Subsection~\ref{subsec:da_algorithms} and Figure~\ref{fig:grafical_abstract} for illustration.
Moment-based algorithms perform particularly well in many practical tasks~\cite{duan2012domain,baktashmotlagh2013unsupervised,sun2016deep,li2018adaptive,zhao2017joint,wang2019transferable}.
A domain adaptation algorithm considering moments of higher orders is proposed in Chapter~\ref{chap:cmd_algorithm} of this thesis and further extended in~\cite{peng2018cross,ke2018identity,xing2018adaptive,peng2019weighted,Wei2018GenerativeAG}.

However, distributions with only finitely many moments in common can be very different, see e.g.~\cite{lindsay2000moments}, which implies that classical bounds on the target risk are very loose for general distributions under the proposed setting.
This brings us to our motivating question \textit{under which further conditions can we expect a discriminative model to perform well on a future test sample given that only finitely many moments are aligned with those of a prior training sample.}

We approach this problem by also considering the information encoded in the distributions in addition to the moments.
Following Section~\ref{sec:maximum_entropy_distribution}, this information can be modeled by the deviation of the differential entropy to the entropy of the maximum entropy distribution~\cite{cover2012elements,milev2012moment}, or equivalently, by the error in KL-divergence of approximation by exponential families~\cite{csiszar1975}.
Note that exponential families are the only parametric distributions with fixed compact support having the property that a finite pre-defined vector of moments can serve as sufficient statistic~\cite{koopman1936distributions} and therefore carries all the information about the distribution.
In addition, exponential families are particularly suitable for our analysis as they include truncated Normal distributions arising in many applications.

We analyze the convergence of sequences of probability densities in terms of finite moment convergence by taking the smoothness and the differential entropy of the densities into account.
Based on results about the approximation by maximum entropy distributions and polynomials~\cite{barron1991approximation,cox1988approximation} we provide bounds of the form
\begin{align}
	\label{eq:form_of_bounds}
	\norm{p-q}_{L^1}\leq C\cdot\norm{\boldsymbol{\mu}_p-\boldsymbol{\mu}_q}_1+\varepsilon,
\end{align}
where $\norm{p-q}_{L^1}$ is the $L^1$-difference between the probability densities $p$ and $q$ with respective pre-defined vectors of (sample) moments $\boldsymbol{\mu}_p$ and $\boldsymbol{\mu}_q$,
$C$ is a constant depending on the smoothness of $p$ and $q$ and
$\epsilon$ is the error of approximating $p$ and $q$ by (estimators of) maximum entropy distributions measured in terms of differential entropy (and sample size).
The value $\epsilon^2/2$ can be interpreted as upper bound on the amount of information lost when representing $p$ and $q$ by its moments $\boldsymbol{\mu}_p$ and $\boldsymbol{\mu}_q$, respectively.

To obtain bounds on the expected misclassification risk of a discriminative model tested on a sample with only finitely many moments similar to those of the training sample, we extend the theoretical bounds described in Subsection~\ref{subsec:learning_bounds_for_domain_adaptation} by means of Eq.~\eqref{eq:form_of_bounds}.
The resulting learning bounds do not make assumptions on the structure of the underlying unknown labeling functions.
In the case of two underlying labeling functions, we obtain error bounds that are relative to the performance of some optimal discriminative function and in the case of one underlying labeling function, i.e. in the covariate-shift setting~\cite{sugiyama2012machine,ben2014domain}, we obtain absolute error bounds.

Our bounds show that a small misclassification risk of the discriminative model can be expected if the misclassification risk of the model on the training sample is small, if the samples are large enough and their densities have high entropy in the respective classes of densities sharing the same finite collection of moments.
Our bounds are uniform for a class of smooth distributions and multivariate moments with solely univariate terms.

\section{Related Work}
\label{sec:learning_bounds_relation_to_state_of_the_art}

Our work is partly motivated by the high performance of moment-based unsupervised domain adaptation methods for representation learning models as discussed in Subsection~\ref{subsec:da_algorithms} and Subsection~\ref{subsec:domain_adaptation_by_nns}.
Recent examples can be found in the areas of deep learning~\cite{sun2016deep,koniusz2017domain,li2018adaptive,peng2018cross,ke2018identity,Wei2018GenerativeAG,xing2018adaptive}, kernel methods~\cite{duan2012domain,baktashmotlagh2013unsupervised} and linear regression as described in Chapter~\ref{chap:applications}.
However, none of these works provide theoretical guarantees for a small misclassification risk with exception of~\cite{peng2018moment} who consider general distributions resulting in possibly loose bounds.
Another motivation of our work is that many common probability metrics admit upper bounds on moment-based distance measures as \eg~discussed in~\cite{rachev2013methods}.
Gibbs and Su~\cite{gibbs2002choosing} review different useful relations between probability metrics without considering moment distances.

Our work is based on the observation that bounds on the $L^1$-norm of the difference between densities lead to bounds on the misclassification probability of a discriminative model according to Ben-David et al.~\cite{ben2010theory}.
We refer to Subsection~\ref{subsec:learning_bounds_for_domain_adaptation} for more details of this approach.

Following ideas from Tagliani et al.~\cite{tagliani2003note,tagliani2002entropy,tagliani2001numerical} and properties of maximum entropy distributions~\cite{cover2012elements}, we obtain such bounds for multivariate distributions based on the differential entropy.
We refer to Subsection~\ref{subsec:moment_distances} and Section~\ref{sec:maximum_entropy_distribution} for details on these and related approaches.

Following Barron and Sheu~\cite{barron1991approximation} and Cox~\cite{cox1988approximation}, we present appropriate regularity assumptions on the distributions under which the KL-divergence based bounds are further upper bounded in terms of (sample) moment differences leading to the form of Eq.~\eqref{eq:form_of_bounds}.

Our results supplement the picture of probability metrics proposed by Gibbs and Su~\cite{gibbs2002choosing} by moment distances as shown in Figure~\ref{fig:metrics_with_moments}.
See Section~\ref{sec:probability_metrics} for more details on probability metrics.
In contrast to other works, our main result is a learning bound for domain adaptation that does not depend on the knowledge of a full test sample but only on the knowledge of finitely many of its sample moments.

\section{Problem Formulation}
\label{sec:problem_formulation_binary_classification}

Our formalization is based on Problem~\ref{problem:da_for_binary_classification} of domain adaptation for binary classification.
That is, we assume source and target densities $p, q\in\mathcal{M}\left([0,1]^d\right)$ with corresponding labeling functions $l_p,l_q :[0,1]^d\to[0,1]$.
In addition, we postulate the alignment of finitely many moments, i.e. $\int\boldsymbol{\phi} p\approx\int\boldsymbol{\phi} q$ for some $\boldsymbol{\phi}\in\mathbb{R}_m[x_1,\ldots,x_d]^n$.
As a result, we end up with Problem~\ref{problem:moment_based_da_for_binary_classification} of moment-based domain adaptation for binary classification.

\begin{restatable}[Moment-Based Domain Adaptation for Binary Classification]{problemrep}{momentbaseddaforninaryclassification}%
	\label{problem:moment_based_da_for_binary_classification}%
	Consider two domains, a \textit{source} domain $\left(p,l_p\right)$ and a \textit{target} domain $\left(q,l_q\right)$, such that $\int\boldsymbol{\phi} p\approx\int\boldsymbol{\phi} q$ for some $\boldsymbol{\phi}\in\mathbb{R}_m[x_1,\ldots,x_d]^n$.
	
	Given a source sample $X_p=\{\x_1,\ldots,\x_k\}$ drawn from $p$ with corresponding labels $Y_p=\{l_p(\x_1),\ldots,l_p(\x_k)\}$ and a target sample $X_q=\{\x_1',\ldots,\x_s'\}$ drawn from $q$ with corresponding labels $Y_q\subseteq\{l_q(\x_1'),\ldots,l_q(\x_s')\}$, find some function $f:\mathbb{R}^d\to\{0,1\}$ with a small target misclassification risk
	\begin{align}
		\label{eq:misclassification_risk_moment_based_domain_adaptation}
		\int_{\mathbb{R}^d} \left|f(\x)-l_q(\x)\right| q(\x)\diff\x.
	\end{align}%
\end{restatable}%
Without further conditions on the densities, a solution to Problem~\ref{problem:moment_based_da_for_binary_classification} might not exist.
One of our goals is therefore to determine and describe conditions on the densities $p$ and $q$ such that a solution exists.
In particular, we aim at conditions such that a small target risk in Eq.~\eqref{eq:misclassification_risk_moment_based_domain_adaptation} is induced by a small (sampled) source risk $\int\left|f-l_p\right| p$, a small difference $\norm{\boldsymbol{\mu}_p-\boldsymbol{\mu}_q}_1$ between the (sampled) moments $\boldsymbol{\mu}_p=\int\boldsymbol{\phi} p$ and $\boldsymbol{\mu}_q=\int\boldsymbol{\phi} q$ and a small distance $\lambda^*$ between the labeling functions $l_p$ and $l_q$ as defined in Eq.~\eqref{eq:minimal_combined_error}.

\section{Approach by Convergence Rate Analysis}
\label{sec:approach_by_convergence_analysis}

It will turn out that the assumption of high-entropy distributions satisfying additional smoothness conditions allows us to provide appropriate learning bounds.
Our approach is based on the analysis of the $L^1$-convergence rate of sequences of densities based on the convergence of finitely many of its corresponding moments.

This section is structured as follows:
Subsection~\ref{subsec:from_moment_convergence_to_l1} motivates our approach of bounding the $L^1$-difference between probability densities.
Subsection~\ref{subsec:high_entropy_distr} discusses the convergence of probability densities with high entropy while satisfying certain moment constraints.
Subsection~\ref{subsec:smooth_high_entropy_distr} discusses smoothness constraints for convergence rates that are uniform in certain classes of probability density functions.

\subsection{From Moment Similarity to \texorpdfstring{$L^1$}--Similarity}
\label{subsec:from_moment_convergence_to_l1}

The postulated similarity of finitely many moments as stated in Problem~\ref{problem:moment_based_da_for_binary_classification} does not directly lead to the required error guarantees.
The following Lemma, see Subsection~\ref{subsec:proof_from_moment_similarity_to_l1} for its proof, motivates the consideration of the stronger concept of similarity in $L^1$-difference.

\begin{restatable}{lemmarep}{totalvariation}%
	\label{lemma:motivation}%
	Let $f:[0,1]^d\to\{0,1\}$ be integrable and $p,q\in\mathcal{M}\left([0,1]^d\right)$. Then the following holds:
	\begin{align}
		\max_{l:[0,1]^d\to [0,1]} \left| \int \left|f-l\right| q - \int \left|f-l\right| p \right| = \frac{1}{2} \norm{p-q}_{L^1}.
	\end{align}
\end{restatable}
Lemma~\ref{lemma:motivation} shows that the $L^1$-difference between the densities $p$ and $q$ has to be small to obtain absolute non-probabilistic bounds on the misclassification risk.
Assume the $L^1$-difference is not small, then there exists a labeling function $l_p=l_q=l$ such that the source risk $\int \left|f-l_p\right| p$ is not a good indicator for the target risk $\int \left|f-l_q\right| q$.
Consequently, to achieve our goal, a small difference between the moments has to imply a small $L^1$-difference.

However, two densities with only finitely many moments in common can be far \wrt~the Kolmogorov metric~\cite{lindsay2000moments}, and consequently can have a large $L^1$-difference.

\subsection{Convergence of High-Entropy Distributions}
\label{subsec:high_entropy_distr}

According to Subsection~\ref{subsec:from_moment_convergence_to_l1} additional assumptions on the densities are required for the existence of a solution to Problem~\ref{problem:moment_based_da_for_binary_classification}.
Therefore, we introduce a notion of $\epsilon$-close maximum entropy densities.
\needspace{10\baselineskip}
\begin{restatable}[$\epsilon$-Close Maximum Entropy Density]{defrep}{epsilonclosemaxentdensity}%
	\label{def:epsilon_close_max_ent_density}%
	Let $\epsilon\geq 0$ and $\phim=(\phi_1,\ldots,\phi_{\psi(m,d)})^\text{T}$ be some vector such that $1,\phi_1,\ldots,\phi_{\psi(m,d)}$ is a basis of the space $\mathbb{R}_m[x_1,\ldots,x_d]$ of polynomials with maximum total degree $m$. Then we call $p\in\Ms$ an {\it $\epsilon$-close maximum entropy density} iff
	\begin{align}
		\label{eq:high_entropy_density}
		h_{\phim}(p)-h(p)\leq\epsilon.
	\end{align}
\end{restatable}%
Recall from Section~\ref{sec:maximum_entropy_distribution} that $h_{\phim}(p)=h(p^*)$ with unique maximum entropy density $p^*$ satisfying the moment constraint $\int\phim p^*=\int\phim p$.
Let us also recall the definition of $\psi(m,d)$ of being the number of monomials of maximum total degree $m$ in $d$ variables, excluding the monomial $1$ of degree $0$.
It is given by $\psi(m,d)=\sum_{i=1}^m \zeta(i,d)=\binom{d+m}{m}-1$, where $\zeta(m,d)$ denotes the number of monomials of total degree $m$ in $d$ variables which is equal to the number of weak compositions and therefore $\zeta(m,d)=\binom{d+m-1}{m}$.

For some small $\epsilon$, by Lemma~\ref{lemma:properties_of_maxent} and Eq.~\eqref{eq:relationship:Pinsker}, an $\epsilon$-close maximum entropy density $p$ fulfills $\norm{p-p^*}_{L^1}\leq \sqrt{2\epsilon}$ and can therefore be interpreted as being well approximable by its corresponding maximum entropy density $p^*$.

In the language of Bayesian inference the term $d_\mathrm{KL}(p,p^*)=h_{\phim}(p)-h(p)$ measures the information gained when one revises one's beliefs from the prior probability density $p^*$ to the posterior probability density $p$.
In this sense, the amount of information lost when using the moments $\int\phim p$ instead of the density $p$ is at most $\epsilon$ for $\epsilon$-close maximum entropy densities.

Note that we allow $\epsilon$ to be zero to include maximum entropy densities $p=p^*$ in our discussions.
The following Lemma~\ref{lemma:L1_convergence_in_M}, see Subsection~\ref{subsec:class} for its proof, motivates to consider $\epsilon$-close maximum entropy densities for tackling Problem~\ref{problem:moment_based_da_for_binary_classification}.

\begin{restatable}{lemmarep}{Lone}%
	\label{lemma:L1_convergence_in_M}%
	Let $\epsilon\geq 0$, let $\phim=(\phi_1,\ldots,\phi_{\psi(m,d)})^\text{T}$ be some vector such that $1,\phi_1,\ldots,\phi_{\psi(m,d)}$ is a basis of $\mathbb{R}_m[x_1,\ldots,x_d]$ and let $p_n\in\Ms$ for $n\in\{1,\ldots,\infty\}$ be $\epsilon$-close maximum entropy densities with moments denoted by ${\boldsymbol\mu_n=\int \phim p_n}$.
	Then the following holds: 
	\begin{gather}
		\label{eq:convergence_in_M}
		\lim_{n\to\infty}\norm{\boldsymbol{\mu}_n-\boldsymbol{\mu}_\infty}_1=0
		\quad
		\implies
		\quad
		\limsup_{n\to\infty} \norm{p_n-p_\infty}_{L^1}\leq \sqrt{8 \epsilon}.
	\end{gather}
\end{restatable}
According to Theorem~\ref{thm:lone_domain_adaptation_bound} a small misclassification risk in Eq.~\eqref{eq:misclassification_risk_moment_based_domain_adaptation} is implied by a small source risk $\int\left|f-l_p\right| p$, a small $L^1$-difference between the densities and a small $\lambda^*$.
According to Lemma~\ref{lemma:L1_convergence_in_M} this is the case if $p,q\in\Ms$ are $\epsilon$-close maximum entropy densities and if the moment vectors $\int \phim p$ and $\int\phim q$ are similar.
Unfortunately, the convergence in Eq.~\eqref{eq:convergence_in_M} can be very slow for sequences in $\Ms$ which is shown by the following example.
\\
\begin{example}
	\label{ex:truncated_normal}
	Consider the vector $\boldsymbol{\phi}_2=(x,x^2)^\text{T}\in\mathbb{R}_2[x]$ and two one-dimensional truncated Normal distributions with densities $p,q\in\mathcal{M}([0,1])$ with equal variance but different means.
	These distributions are maximum entropy distributions constrained at the moments $\int \boldsymbol{\phi}_2 p$ and $\int\boldsymbol{\phi}_2 q$ and therefore satisfy Eq.~\eqref{eq:high_entropy_density} with $\epsilon=0$.
	It holds that for every moment difference $\norm{\int \phi_2 p-\int\phi_2 q}_1$ one can always find a small enough variance such that $\norm{p-q}_{L^1}$ is large.
\end{example}

Example~\ref{ex:truncated_normal} shows that additional properties besides Eq.~\eqref{eq:high_entropy_density} are required to obtain fast convergence rates for sequences in $\Ms$.

\subsection{Convergence of Smooth High-Entropy Distributions}
\label{subsec:smooth_high_entropy_distr}

In this subsection we introduce additional smoothness conditions
motivated by approximation results of exponential families~\cite{barron1991approximation} and Legendre polynomials~\cite{cox1988approximation}.
More precisely, we consider the following set of densities.
\Needspace{10\baselineskip}
\begin{restatable}[Smooth High-Entropy Densities]{defrep}{Hmeps}%
	\label{def:H}%
	Let $\epsilon\geq 0$, $m\in\mathbb{N}$, $m\geq 2$ and $\phim=(\phi_1,\ldots,\phi_{m d})^\text{T}$ be a vector of polynomials such that $1,\phi_1,\ldots,\phi_{m d}$ is an orthonormal basis of $\mathrm{Span}(\mathbb{R}_m[x_1]\cup\ldots\cup\mathbb{R}_m[x_d])$.
	We call ${p\in\Ms}$ a \textit{smooth high-entropy density} iff the following conditions are satisfied:
	\begin{enumerate}
		\item[(A1)] $h_{\phim}(p)-h(p)\leq \epsilon$
		\item[(A2)] $\norm{\log p}_\infty \leq \frac{3m-6}{2}$
		\item[(A3)] $\log p_i\in W_2^m\quad\forall i\in\{1,\ldots,d\}$
		\item[(A4)] $\norm{\partial^m_{x_i} \log p_i}_{L^2}\leq 5^{m-4}\quad\forall i\in\{1,\ldots,d\}$
	\end{enumerate}
	where $p_i=\int_0^1\cdots\int_0^1 p(x_1,\ldots,x_d)\diff x_1\cdots \diff x_{i-1} \diff x_{i+1}\cdots \diff x_d$ denote the marginal densities of $p$.
	We denote the set of all smooth high-entropy densities by $\mathcal{H}_{m,\epsilon}$.
\end{restatable}%
The set $\H_{m,\epsilon}$ in Definition~\ref{def:H} contains multivariate probability densities $p$ with loosely coupled marginals.
The reason is the specification of the polynomial vector $\phim$ resulting in maximum entropy densities $p^*$ of densities $p\in\Ms$ with independent marginals as shown by Lemma~\ref{lemma:independence}.
One advantage of this simplification is that no combinatorial explosion has to be taken into account.
We will show in Chapter~\ref{chap:cmd_algorithm} that such moment vectors are sufficient in many practical tasks.
Distributions with loosely coupled marginals are created by many learning algorithms~\cite{comon1994independent,hyvarinen2001topographic,bach2002kernel}.

Note that the present analysis can be extended to general multi-dimensional polynomial vectors by the usual product basis functions for polynomials.
However, the use of such expansions is precluded by an exponential growth of the number of moments with the dimension $d$ and the consideration of additional smoothness constraints, see also~\cite{barron1991approximation}.

The definition of the set $\mathcal{H}_{m,\epsilon}$ is independent of the choice of the orthonormal basis $1,\phi_1,\ldots,\phi_{m d}$.
This follows from properties of the information projection~\cite{barron1991approximation}.

Assumptions (A3) and (A4) restrict the smoothness of the densities.
The upper bound on the $L^2$-norm, and also the one in (A2), can be enlarged at the cost of more complicated dependencies on the shape of the log-density functions as shown in Subsection~\ref{subsec:proof_convergence_of_smooth_high_entropy_densities}.
It is interesting to observe that, when a density is bounded away from zero, assumptions on the log-densities are not too different from the assumptions on derivatives of the densities itself, see \eg~Remark~2 in~\cite{barron1991approximation}.

The set $\H_{m,\epsilon}$ contains densities that are well approximable in KL-divergence by exponential families:
For each $\epsilon>0$ and each density $p\in\Ms$ satisfying (A2) and (A3), there exists a number of moments $m$ such that $\min_{q\in\mathcal{E}_{\phim}} d_\mathrm{KL}(p, q)\leq \epsilon$ for the exponential family $\mathcal{E}_{\phim}$.
This follows from the fact that $h_{\phim}(p)\to h(p)$ for $m\to\infty$ as shown in Lemma~\ref{lemma:properties_of_maxent}.

The following Theorem~\ref{thm:bound_for_smooth_functions} gives an uniform bound for the $L^1$-norm of the difference of densities in $\mathcal{H}_{m,\epsilon}$ in terms of differences of moments.
See Subsection~\ref{subsec:proof_convergence_of_smooth_high_entropy_densities} for its proof.

\begin{restatable}{thmrep}{convergence}%
	\label{thm:bound_for_smooth_functions}%
	Consider some $m$, $ \epsilon$, $\phim$ and $\mathcal{H}_{m,\epsilon}$ as in Definition~\ref{def:H} and let $p,q\in\mathcal{H}_{m,\epsilon}$ with moments denoted by $\boldsymbol{\mu}_p=\int\phim p$ and $\boldsymbol{\mu}_q=\int \phim q$. Then the following holds:
	\begin{gather*}
		\norm{\boldsymbol{\mu}_p-\boldsymbol{\mu}_q}_1 \leq \frac{1}{2 C \left(m+1\right)}
		\quad\implies\quad
		\norm{p-q}_{L^1} \leq \sqrt{2 C}\cdot \norm{\boldsymbol{\mu}_p-\boldsymbol{\mu}_q}_1 + \sqrt{8 \epsilon}
	\end{gather*}
	with the constant $C=2 e^{(3m-1)/2}$.
\end{restatable}
Theorem~\ref{thm:bound_for_smooth_functions} relates the $\ell_1$-distance between moments to other probability metrics as illustrated in Figure~\ref{fig:metrics_with_l1_bond}.
It can be seen that the $\ell_1$-distance implements a weaker convergence than most other commonly applied probability metrics.
However, under the assumptions (A1)--(A4) stated in Definition~\ref{def:H} and small $\epsilon$, stronger convergence properties are implemented.

\begin{figure}[t]
	\centering
	\includegraphics[width=\linewidth]{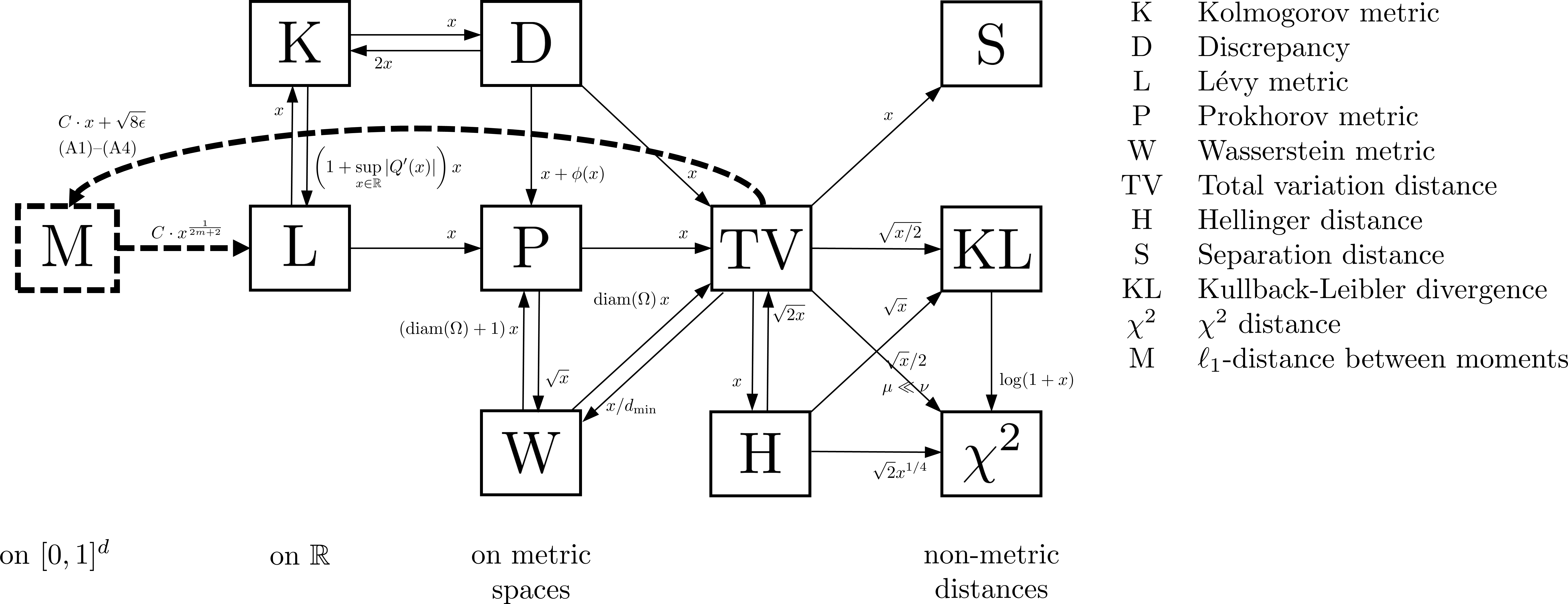}
	\caption[Relationships among probability metrics as illustrated in~\cite{gibbs2002choosing} and supplemented by Lemma~\ref{lemma:moment_distance_bound_by_levy} and Theorem~\ref{thm:bound_for_smooth_functions}.]{Relationships among probability metrics as illustrated in~\cite{gibbs2002choosing} and supplemented by Lemma~\ref{lemma:moment_distance_bound_by_levy} and Theorem~\ref{thm:bound_for_smooth_functions} (dashed).
		A directed arrow from $\text{A}$ to $\text{B}$ annotated by a function $h(x)$ means that $d_\text{A}\leq h(d_\text{B})$.
		For notations, restrictions and applicability see Section~\ref{sec:probability_metrics}.
	}
	\label{fig:metrics_with_l1_bond}
\end{figure}

The more moments we consider in Theorem~\ref{thm:bound_for_smooth_functions}, i.e. the higher $m$ is, the richer is the class $\mathcal{H}_{m,\epsilon}$.
However, with increasing $m$ the constant $C$ also increases.
This constant depends exponentially on $m$ which is induced by the definition of the upper bounds on the norms of the derivatives in the Definition~\ref{def:H}.

However, it is interesting to consider more general upper bounds $c_\infty\geq\norm{\log p}_\infty$ and $c_r\geq \norm{\partial^m_{x_i}\log p_i}_{L^2}$ instead.
This leads to the constant $C$ as in Lemma~\ref{lemma:convergence_in_H}, used to prove Theorem~\ref{thm:bound_for_smooth_functions}, which depends double exponentially on the upper bounds $c_\infty$ and $c_r$.
However, the double exponential dependency weakens when considering higher numbers $r$ of derivatives or numbers $m$ of moments as we discuss in Remark~\ref{remark:dependency}.
Thus, the main influence is an exponential dependency on the upper log-density bound $c_\infty$.

The considered dimension $d$ of the unit cube effects the number of moment differences considered in the $\ell^1$-norms in Theorem~\ref{thm:bound_for_smooth_functions}.
By the specification of the vector $\phim$, this number increases only linearly with the dimension.

Theorem~\ref{thm:bound_for_smooth_functions} together with Theorem~\ref{thm:lone_domain_adaptation_bound} give a first result towards identifying a solution of Problem~\ref{problem:moment_based_da_for_binary_classification}: An upper bound on the misclassification risk of the discriminative model based on differences of moments:
\begin{restatable}{correp}{corabsbound}%
	\label{cor:bound_for_smooth_functions}%
	Let $m$, $\epsilon$ and $\phim$ be as in Definition~\ref{def:H}. Let further $p,q\in\mathcal{H}_{m,\epsilon}$ with moments denoted by ${\boldsymbol{\mu}_p=\int\phim p}, {\boldsymbol{\mu}_q=\int \phim q}$, respectively, let ${l_p,l_q:[0,1]^d\to [0,1]}$ be two integrable labeling functions and $\mathcal{F}\subset\left\{f:[0,1]^d\to\{0,1\}\mid f~\text{integrable}\right\}$ be a set of binary classifiers.
	Then the following holds for all $f\in\mathcal{F}$:
	\begin{gather*}
		\norm{\boldsymbol{\mu}_p-\boldsymbol{\mu}_q}_1 \leq \frac{1}{2 C \left(m+1\right)}\\
		\quad\implies\quad\\
		\int \left|l-l_q\right| q \leq \int \left|l-l_p\right| p + \sqrt{2 C}\cdot \norm{\boldsymbol{\mu}_p-\boldsymbol{\mu}_q}_1 + \sqrt{8 \epsilon} + \lambda^*
	\end{gather*}
	with $C=2 e^{(3m-1)/2}$ and $\lambda^* = \inf_{h\in\mathcal{F}}\big(\int \left|h-l_p\right| p+\int \left|h-l_q\right| q\big)$.
\end{restatable}
Corollary~\ref{cor:bound_for_smooth_functions} gives an error bound on the target error that is relative to the error $\lambda^*$ of some optimal discriminative function.
This is similar to the assumption in \textit{probably approximately correct learning theory} that there exists a perfect discriminative model in the underlying model class~\cite{shalev2014understanding}.
The error $\lambda^*$ can be eliminated in the case of equal labeling functions, i.e. $l_p=l_q$, by using the bound of Theorem~1 in~\cite{ben2010theory} instead of Theorem~\ref{thm:lone_domain_adaptation_bound}.

Further implications of Corollary~\ref{cor:bound_for_smooth_functions} are discussed in more detail in Section~\ref{sec:main_result_on_learning_bounds} together with the sample case.

\needspace{10\baselineskip}
\section{A Learning Bound for Moment-Based Domain Adaptation}
\label{sec:main_result_on_learning_bounds}
\needspace{15\baselineskip}
\begin{restatable}{thmrep}{mainresult}%
	\label{thm:problem_solution}%
	Consider some $m$, $\epsilon$, $\phim$ and $\mathcal{H}_{m,\epsilon}$ as in Definition~\ref{def:H} and a function class $\mathcal{F}$ with finite VC-dimension $\vc$.
	Consider two probability densities $p,q\in\mathcal{H}_{m,\epsilon}$ and two integrable labeling functions ${l_p,l_q:[0,1]^d\to [0,1]}$.
	
	Let $X_p$ and $X_q$ be two $k$-sized samples drawn from $p$ and $q$, respectively, and denote by $\widehat{\boldsymbol{\mu}}_p=\frac{1}{k}\sum_{\x\in X_p}\boldsymbol{\phi}_m(\x)$ and $\widehat{\boldsymbol{\mu}}_q=\frac{1}{k}\sum_{\x\in X_q}\boldsymbol{\phi}_m(\x)$ corresponding sample moments.
	
	Then, for every $\delta\in (0,1)$ and all $f\in\mathcal{F}$, the following holds with probability at least $1-\delta$ over the choice of samples:
	If
	\begin{flalign}
		4 C^2(m+1)^2 m \delta^{-1} \leq k
	\end{flalign}
	and
	\begin{flalign}
		\norm{\widehat{\boldsymbol{\mu}}_p-\widehat{\boldsymbol{\mu}}_q}_1 \leq \left(2 (m+1) e C\right)^{-1}
	\end{flalign}
	then
	\begin{align}
		\label{eq:moment_adapt_result_bound}
		\begin{split}
			\int\left|f-l_q\right| q\leq\, &\frac{1}{k}\sum_{\x\in X_p}\left|f(\x)-l_p(\x)\right| + 
			\sqrt{\frac{4}{k} \left( \vc\log \frac{2 e k}{\vc} + \log\frac{4}{\delta} \right)}+ \lambda^*\\
			&+ \sqrt{2 e C} \norm{\widehat{\boldsymbol{\mu}}_p-\widehat{\boldsymbol{\mu}}_q}_1 + \sqrt{8 C} \sqrt{\frac{d m}{k\delta}} + \sqrt{8\epsilon}
		\end{split}
	\end{align}
	where $C=2 e^{(3m-1)/2}$ and $\lambda^* = \inf_{h\in\mathcal{F}}\big(\int\left|f-l_p\right| p+\int\left|f-l_q\right| q\big)$.
\end{restatable}
Theorem~\ref{thm:problem_solution} provides cases where Problem~\ref{problem:moment_based_da_for_binary_classification} has solutions.
A proof is outlined in Subsection~\ref{subsec:problem_solution_proof}.

Theorem~\ref{thm:problem_solution} directly extends the bound in Theorem~\ref{thm:vc_bound} on the target error in the statistical learning theory proposed by Vapnik and Chervonenkis~\cite{vapnik2015uniform} and the domain adaptation theory in Theorem~\ref{thm:lone_domain_adaptation_bound}.

Note that according to Vapnik and Chervonenkis~\cite{vapnik2015uniform}, a small misclassification risk of a discriminative model is induced by a small training error, if the sample size is large enough.
Due to Ben-David et al.~\cite{ben2010theory}, this statement still holds for a test sample with a distribution different from the training sample, if the $L^1$-difference of the distributions is small and if there exists a model that can perform well on both distributions, \ie~the error $\lambda^*$ in Theorem~\ref{thm:lone_domain_adaptation_bound} is small.

According to Theorem~\ref{thm:problem_solution}, a small misclassification risk of a model on a test sample with moments $\widehat{\boldsymbol{\mu}}_q$ is induced by a small error on a training sample with moments $\widehat{\boldsymbol{\mu}}_p$ being similar to $\widehat{\boldsymbol{\mu}}_q$, if the the following holds: The sample size is large enough, the densities $p$ and $q$ are smooth high-entropy densities with loosely coupled marginals, \ie~$p,q\in\mathcal{H}_{m,\epsilon}$, and there exists a model that can perform well on both densities.

See Lemma~\ref{lemma:sample_convergence_in_H} in Subsection~\ref{subsec:problem_solution_proof} for improved assumptions and an improved constant $C$ with the drawback of some additional and more complicated assumptions on the smoothness of the densities.

It is interesting to investigate in more detail the terms in Eq.~\eqref{eq:moment_adapt_result_bound} that depend on the sample size $k$ which is chosen equally for both samples for better readability:
Let us therefore assume a fixed number of moments $m$ and a given probability ${1-\delta}$.
For model classes with VC-dimension $\vc\geq d$, \ie~supra-linear models, and for a large sample size $k > \vc$, the complexity of the proposed term is bounded by $O(\sqrt{\vc/k})$ which is smaller than the complexity $O(\sqrt{\vc/k \log( 2ek/\vc)})$ of the classical error bound in the first line of Eq.~\eqref{eq:moment_adapt_result_bound} as proposed in~\cite{vapnik2015uniform}.
However, the classical term decreases faster with complexity $O(\sqrt{\log(1/\delta)})$ as the probability $1-\delta$ decreases compared to the proposed term which decreases only with complexity $O(\sqrt{1/\delta})$.

\section{Proofs}
\label{sec:proofs_learning_bounds}

All proofs are summarized in this subsection together with additional remarks and comments.

\subsection{Bound on Moment Distance by L\'evy Metric}
\label{subsec:bound_on_moment_distance_by_levy_metric}

To prove Lemma~\ref{lemma:moment_distance_bound_by_levy}, the following Definition~\ref{def:Zolotarev_metric} and Lemma~\ref{lemma:zolotarev_bound} from~\cite{zolotarev1975two} are helpful.

\Needspace{10\baselineskip}
\begin{restatable}[Zolotarev Metric~\cite{zolotarev1976metric}]{defrep}{zolotarevmetric}%
	\label{def:Zolotarev_metric}%
	The Zolotarev metric between two probability density functions $p,q\in\mathcal{M}(\mathbb{R})$ is defined by
	\begin{align}
		\label{eq:metric:Zolotarev}
		d_\text{Z}(p,q) =\min_{T > 0} \max \left\{ \frac{1}{2} \max_{\left| t\right|\leq T} \left|f(t)-g(t)\right|, \frac{1}{T} \right\},
	\end{align}
	where $f$ and $g$ denote the characteristic functions of $p$ and $q$, respectively.
\end{restatable}%

\begin{restatable}[Zolotarev Metric Bound~\cite{zolotarev1975two}]{lemmarep}{zolbound}%
	\label{lemma:zolotarev_bound}%
	If $p,q\in\mathcal{M}([0,2K])$ then
	\begin{align}
		\label{eq:zolotarev_and_senatov}
		d_\text{Z}(p,q)\leq \sqrt{\left( 2 K + 24\sqrt{d_\text{L}(p,q)}+1/2 \right) d_\text{L}(p,q)}.
	\end{align}
\end{restatable}
\Needspace{10\baselineskip}
\momentdistanceboundbylevy*

\begin{proof}
	Let $\varepsilon=\sqrt{102\, d_L(p,q)}$ and $T_0$ such that
	\begin{align*}
		d_\text{Z}(p,q) =\max \left\{ \frac{1}{2} \max_{\left| t\right|\leq T_0} |f(t)-g(t)|, \frac{1}{T_0} \right\}.
	\end{align*}
	Then it holds that
	\begin{align}
		\label{eq:proof_eq_rachevbound}
		\begin{split}
			\sup_{\left| t\right|\leq T_0} |f(t)-g(t)| &\leq
			2 \max \left\{ \frac{1}{2} \max_{\left| t\right|\leq T_0} |f(t)-g(t)|, \frac{1}{T_0} \right\}\\
			&\leq 2 \sqrt{\left( 2 K + 24\sqrt{d_\text{L}(p,q)}+1/2 \right) d_\text{L}(p,q)}\\
			&\leq \sqrt{102\, d_L(p,q)} =\varepsilon
		\end{split}
	\end{align}
	where the second inequality follows from  Lemma~\ref{lemma:zolotarev_bound} and the last inequality follows from the fact that $d_\text{L}\leq 1$.
	
	Theorem~\ref{thm:rachev} can be applied and it follows that there exists an absolute constant $C_\text{Z}$ such that for all $n\in\mathbb{N}$ with
	\begin{align*}
		n^3 C_\text{Z}^{\frac{1}{n+1}} \varepsilon^{\frac{1}{n+1}}\leq T_0/2
	\end{align*}
	we have that
	\begin{align*}
		\left| \int x^n p\diff x - \int x^n q\diff x\right|\leq C_\text{Z} n^3 \varepsilon^{\frac{1}{n+1}}.
	\end{align*}
	From the definition of $\varepsilon$ and Eq.~\eqref{eq:proof_eq_rachevbound}, in particular using $\frac{2}{T_0}\leq \epsilon$, we obtain for all $n\in\mathbb{N}$ with
	\begin{align}
		\label{eq:proof_zolotarev_assumption}
		n^3 C_\text{Z}^{\frac{1}{n+1}} \left(102\, d_L(p,q)\right)^{\frac{n+2}{2 n+2}}\leq 1
	\end{align}
	the inequality
	\begin{align}
		\label{eq:proof_levy_bound}
		\left| \int x^n p\diff x - \int x^n q\diff x\right|\leq C_\text{Z} n^3 \left(102\, d_L(p,q)\right)^{\frac{1}{2 n+2}}.
	\end{align}
	The vector $\boldsymbol{\phi}$ contains polynomials in $\mathbb{R}_m[x]$ which implies that the value of $\norm{\boldsymbol{\mu}_p - \boldsymbol{\mu}_q}_1$ can be computed as a finite weighted sum of differences of moments as specified by the left-hand side of Eq.~\eqref{eq:proof_levy_bound}.
	As a consequence, the value of $\norm{\boldsymbol{\mu}_p - \boldsymbol{\mu}_q}_1$ can be upper bounded by aggregations of the right-hand side of Eq.~\eqref{eq:proof_levy_bound}.
	Let us define $M_L$ small enough such that Eq.~\eqref{eq:proof_zolotarev_assumption} is fulfilled for all $n\leq m$.
	From $d_L(p,q)^{\frac{1}{2 n +2}}\leq d_L(p,q)^{\frac{1}{2 m +2}}$ for $1\leq n\leq m$ the existence of some $C_L$ as required by Lemma~\ref{lemma:moment_distance_bound_by_levy} follows.
\end{proof}\\

\subsection{Moment Similarity and \texorpdfstring{$L^1$}--Similarity}
\label{subsec:proof_from_moment_similarity_to_l1}

\hspace{1pt}
\totalvariation*

\begin{proof}
	Let us define the labeling function ${l^*:[0,1]^d\to [0,1]}$ by
	\begin{align}
		l^*(\x)=
		\begin{cases}
			1~\text{if}~f(\x)=1~\text{and}~p(\x) < q(\x)\\
			1~\text{if}~f(\x)=0~\text{and}~p(\x)\geq q(\x)\\
			0~\text{if}~f(\x)=1~\text{and}~p(\x)\geq q(\x)\\
			0~\text{if}~f(\x)=0~\text{and}~p(\x) < q(\x)
		\end{cases}
	\end{align}
	By this construction the following holds:
	\begin{align}
		\label{eq:lstar_identity}
		|f-l^*|=\mathbbm{1}_A
	\end{align} where $\mathbbm{1}_A(\x)=\begin{cases}1:\x\in A\\0:\text{else}\end{cases}$ and $A=\{\x\in [0,1]^d\mid p(\x)\geq q(\x)\}$.
	From Eq.~\eqref{eq:lstar_identity} we obtain
	\begin{align}
		\label{eq:lstar_identity_integral}
		\begin{split}
			\int_{[0,1]^d} |f-l^*|\, (p-q) &=
			\int_{[0,1]^d} \mathbbm{1}_A (p-q)\\
			&= \int_{[0,1]^d} \mathbbm{1}_A\, p - \int_{[0,1]^d} \mathbbm{1}_A\, q\\
			&= 1 - \int_{[0,1]^d} \mathbbm{1}_{A^c}\, p - 1 + \int_{[0,1]^d} \mathbbm{1}_{A^c}\, q\\
			&= \int_{[0,1]^d} \mathbbm{1}_{A^c} (q-p)
		\end{split}
	\end{align}
	where $A^c=[0,1]^d\setminus A$ denotes the complement of $A$.
	
	For all $l:[0,1]^d\to [0,1]$, it holds that
	\begin{align*}
		\left| \int_{[0,1]^d} \left|f-l\right| q - \int_{[0,1]^d} \left|f-l\right| p \right| &= \left| \int_{[0,1]^d} |f-l|\, (p-q) \right|\\
		&\leq \left|\sup_{\x\in[0,1]^d} \big\{|f(\x)-l(\x)|\big\} \int_{[0,1]^d} (p-q) \right|\\
		&\leq \left| \int_{[0,1]^d} (p-q) \right|\\
		&\leq \max\left\{ \int_{[0,1]^d} (p-q), \int_{[0,1]^d} (q-p) \right\}\\
		&\leq \max\left\{ \int_{[0,1]^d}\mathbbm{1}_A\, (p-q), \int_{[0,1]^d}\mathbbm{1}_{A^c}\, (q-p) \right\}\\
		&= \int_{[0,1]^d} |f-l^*|\, (p-q)
	\end{align*}
	where the last line is obtained from Eq.~\eqref{eq:lstar_identity_integral}.
	It follows that
	\begin{align*}
		\sup_{l:[0,1]^d\to [0,1]} \left| \int_{[0,1]^d} \left|f-l\right| q - \int_{[0,1]^d} \left|f-l\right| p \right|\leq \int_{[0,1]^d} |f-l^*|\, (p-q).
	\end{align*}
	Since $l^*:[0,1]^d\to [0,1]$, it also holds that
	\begin{align*}
		\int_{[0,1]^d} |f-l^*|\, (p-q) &\leq \sup_{l:[0,1]^d\to [0,1]} \left| \int_{[0,1]^d} \left|f-l\right| q - \int_{[0,1]^d} \left|f-l\right| p \right|
	\end{align*}
	and therefore
	\begin{align}
		\max_{l:[0,1]^d\to [0,1]} \left| \int_{[0,1]^d} \left|f-l\right| q - \int_{[0,1]^d} \left|f-l\right| p \right| = \int_{[0,1]^d} |f-l^*|\, (p-q).
	\end{align}
	Using Eq.~\eqref{eq:lstar_identity} and Eq.~\eqref{eq:lstar_identity_integral} yields
	\begin{align*}
		2 \int_{[0,1]^d} |f-l^*|\, (p-q) &= 2 \int_{[0,1]^d}\mathbbm{1}_A\, (p-q)\\
		&= \int_{[0,1]^d}\mathbbm{1}_A\, (p-q) + \int_{[0,1]^d}\mathbbm{1}_{A^c}\, (q-p)\\
		&= \int_{[0,1]^d}|p-q|
	\end{align*}
	which finalizes the proof.
\end{proof}\\

\subsection{Convergence of High-Entropy Distributions}
\label{subsec:class}

For this subsection let $\phim=(\phi_1,\ldots,\phi_{\psi(m,d)})^\text{T}$ be such that $1,\phi_1,\ldots,\phi_{\psi(m,d)}$ is a basis of $\mathbb{R}_m[x_1,\ldots,x_d]$.

The following Lemma~\ref{lemma:basic_inequality} provides a key relationship allowing to focus on differences of distributions in exponential families.

\begin{restatable}{lemmarep}{basicinequ}%
	\label{lemma:basic_inequality}%
	Let $\epsilon\geq 0$ and $p, q\in\Ms$ be two $\epsilon$-close maximum entropy densities.
	Then the following holds:
	\begin{align}
		\label{eq:basic_epsilon_bound}
		\norm{p-q}_{L^1} \leq \sqrt{2 d_\mathrm{KL}(p^*,q^*)} + \sqrt{8 \epsilon}.
	\end{align}
\end{restatable}

\begin{proof}
	Applying the Triangle Inequality and Eq.~\eqref{eq:relationship:Pinsker} yields
	\begin{align*}
		\norm{p-q}_{L^1} &\leq \norm{p^*-q^*}_{L^1}
		+ \norm{p^*-p}_{L^1} + \norm{q^*-q}_{L^1}\\
		&\leq \sqrt{2 d_\mathrm{KL}(p^*,q^*)}
		+ \sqrt{2 d_\mathrm{KL}(p,p^*)} + \sqrt{2 d_\mathrm{KL}(q,q^*)}.
	\end{align*}
	The proof now follows from Property~2 in Lemma~\ref{lemma:properties_of_maxent} and the definition of $\epsilon$-close maximum entropy densities.
\end{proof}\\\\
The following Lemma~\ref{lemma:convergence_in_M} analyzes the convergence in KL-divergence of sequences of distributions in exponential families in terms of the convergence of respective moment vectors.
\begin{restatable}{lemmarep}{convinM}%
	\label{lemma:convergence_in_M}%
	Let $(p_n)_{n\in\mathbb{N}}\subset \Ms$ and $p_\infty\in\Ms$ be such that $p_n$ is an $\epsilon$-close maximum entropy density for all $n\in\{1,\ldots,\infty\}$ and denote its respective moments by ${\boldsymbol\mu_n=\int \phim p_n}$.
	Then the following holds: 
	\begin{gather*}
		\lim_{n\to\infty}\norm{\boldsymbol{\mu}_n-\boldsymbol{\mu}_\infty}_1=0
		\quad
		\implies
		\quad
		\lim_{n\to\infty}d_\mathrm{KL}(p_n^*, p_\infty^*)=0.
	\end{gather*}
\end{restatable}
\begin{proof}
	As shown by Theorem~\ref{thm:uniqueness_and_existence_of_maxent}, the maximum entropy density $p_n^*$ of $p_n$ is independent of the choice of the basis $1,\phi_1,\ldots,\phi_{\psi(m,d)}$.
	Therefore, we may assume without loss of generality that the elements of $\phim$ are solely positive monomials.
	
	According to Eq.~\eqref{eq:maxent_distr_formula}, the maximum entropy distributions $p_n^*$ are of the form $p^*_n=c(\boldsymbol\lambda_n) \exp\left(-\langle\boldsymbol\lambda_n,\phim\rangle\right)$ with parameter vectors $\boldsymbol\lambda_n\in\mathbb{R}^{\psi(m,d)}$.
	Since $p_n^*\in\Ms$ it holds that
	\begin{align*}
		d_\mathrm{KL}(p^*_n, p^*_\infty) &=\int p_n^* \log \frac{p_n^*}{p_\infty^*}\\
		&=\int p_n^* \log\frac{c(\boldsymbol{\lambda}_n) \exp(-\langle\boldsymbol{\lambda}_n,\phim\rangle)}{c(\boldsymbol{\lambda}_\infty) \exp(-\langle\boldsymbol{\lambda}_\infty,\phim\rangle)}\\
		&=\int p_n^* (\log c(\boldsymbol{\lambda}_n) - \log c(\boldsymbol{\lambda}_\infty)) + \int p_n^* (-\langle\boldsymbol{\lambda}_n,\phim\rangle + \langle\boldsymbol{\lambda}_\infty,\phim\rangle)\\
		&= \left(\log c(\boldsymbol{\lambda}_n) - \log c(\boldsymbol{\lambda}_\infty)\right) + (-\langle\boldsymbol{\lambda}_n,\int p_n^* \phim\rangle + \langle\boldsymbol{\lambda}_\infty,\int p_n^*\phim\rangle)\\
		&= \left(\log c(\boldsymbol{\lambda}_n) - \log c(\boldsymbol{\lambda}_\infty)\right) + (-\langle\boldsymbol{\lambda}_n,\boldsymbol{\mu}_n\rangle + \langle\boldsymbol{\lambda}_\infty,\boldsymbol{\mu}_n\rangle)\\
		&= \left(\log c(\boldsymbol{\lambda}_n) - \log c(\boldsymbol{\lambda}_\infty)\right) + \langle\boldsymbol{\mu}_n,\boldsymbol{\lambda}_\infty-\boldsymbol{\lambda}_n\rangle\\
		&\leq \left|\log c(\boldsymbol\lambda_n)-\log c(\boldsymbol\lambda_\infty)\right| + \left\langle\boldsymbol\mu_n,|\boldsymbol\lambda_n-\boldsymbol\lambda_\infty|\right\rangle\\
		&\leq \left|\log c(\boldsymbol\lambda_n)-\log c(\boldsymbol\lambda_\infty)\right| + \norm{\boldsymbol\lambda_n-\boldsymbol\lambda_\infty}_1.
	\end{align*}
	where the last inequality follows from the choice of the basis $1,\phi_1,\ldots,\phi_{\psi(m,d)}$.
	
	In the following we show that $\log c(\boldsymbol\lambda_n)\to \log c(\boldsymbol\lambda_\infty)$ and $\boldsymbol\lambda_n\to\boldsymbol\lambda_\infty$ as $\boldsymbol\mu_n\to\boldsymbol\mu_\infty$:
	As shown in Lemma~\ref{lemma:properties_of_maxent}, the elements of the parameter vector $\boldsymbol\lambda_*$ of the maximum entropy distribution $p^*=c(\boldsymbol\lambda_*) \exp\left(-\langle\boldsymbol\lambda_*,\phim\rangle\right)$ in Eq.~\eqref{eq:maxent_distr_formula} correspond to the Lagrange multipliers solving the optimization problem $\min_{\boldsymbol\lambda\in\mathbb{R}^{\psi(m,d)}}\Gamma(\boldsymbol\lambda) $ where $\Gamma(\boldsymbol\lambda)=\langle \boldsymbol\lambda, \boldsymbol\mu_*\rangle-\log(c(\boldsymbol\lambda))$
	and $\boldsymbol\mu_p=\int\phim p$.
	Let $q=c(\boldsymbol\lambda_q)\exp\left(-\langle\boldsymbol\lambda_q,\phim\rangle\right)$ be a probability density of an exponential family with moments $\boldsymbol\mu_q=\int\phim q$ and parameter vector $\boldsymbol{\lambda}_q=(\lambda_1,\ldots,\lambda_{\psi(m,d)})^\text{T}$.
	Then the partial derivative of the function $\boldsymbol\lambda_q\mapsto\Gamma(\boldsymbol\lambda_q)$ \wrt~the variable $\lambda_i$ is given by
	\begin{align*}
		\partial_{\lambda_i} \Gamma(\boldsymbol{\lambda}_q) &= \int \phi_i p - \partial_{\lambda_i} \log c(\boldsymbol{\lambda}_q)\\
		&= \int \phi_i p - \frac{1}{c(\boldsymbol{\lambda}_q)}\partial_{\lambda_i} c(\boldsymbol{\lambda}_q)\\
		&= \int \phi_i p + \frac{\partial_{\lambda_i} \int \exp\left(-\langle \boldsymbol{\lambda}_q,\phim\rangle\right) }{c(\boldsymbol{\lambda}_q) \left(\int \exp\left(-\langle \boldsymbol{\lambda}_q,\phim\rangle\right)\right)^{2}} \\
		&= \int \phi_i p + c(\boldsymbol{\lambda}_q) \int \exp\left(-\langle \boldsymbol{\lambda}_q,\phim\rangle\right) ( - \partial_{\lambda_i} \langle \boldsymbol{\lambda}_q,\phim\rangle)\\
		&= \int \phi_i p - \int c(\boldsymbol{\lambda}_q) \exp\left(-\langle \boldsymbol{\lambda}_q,\phim\rangle\right) \phi_i\\
		&= \int \phi_i p - \int \phi_i q
	\end{align*}
	and the gradient vector $\nabla \Gamma(\boldsymbol\lambda_q)$ can therefore be computed by
	\begin{align*}
		\nabla \Gamma(\boldsymbol\lambda_q) &= \boldsymbol\mu_p - \int \phim q.
	\end{align*}
	Consequently, the second partial derivative \wrt~the variables $\lambda_i$ and $\lambda_j$ is given by
	\begin{align*}
		\partial_{\lambda_i,\lambda_j}^2 \Gamma(\boldsymbol{\lambda}_q) &= \partial_{\lambda_j} (\int \phi_i p - \int \phi_i q)\\
		&= \int c(\boldsymbol{\lambda}_q) \exp\left(-\langle \boldsymbol{\lambda}_q,\phim\rangle\right) \phi_i \phi_j - \int \exp\left(-\langle \boldsymbol{\lambda}_q,\phim\rangle\right) \phi_i\, \partial_{\lambda_j} c(\boldsymbol{\lambda}_q)\\
		&= \int \phi_i\phi_j q - \int \exp\left(-\langle \boldsymbol{\lambda}_q,\phim\rangle\right) \phi_i\, c(\boldsymbol{\lambda}_q)^2 \int \exp\left(-\langle \boldsymbol{\lambda}_q,\phim\rangle\right)\phi_j\\
		&= \int \phi_i\phi_j q - \int q\phi_i (\int q \phi_j)
	\end{align*}
	and the Hessian matrix $H_\Gamma(\boldsymbol\lambda_q)$ can be computed by
	\begin{align*}
		H_\Gamma(\boldsymbol\lambda_q) &= \int (\phim\cdot \phim^\text{T}) q - \int \phim q\cdot(\int\phim q)^\text{T}.
	\end{align*}
	The Hessian matrix $H_\Gamma$ equals the covariance matrix of a random variable with density $q$. It is assumed that the elements of $\phim$ are independent. $H_\Gamma$ is therefore positive definite and the function $\boldsymbol\lambda_q\mapsto \Gamma(\boldsymbol\lambda_q)$ reaches its minimum at a vector $\boldsymbol\lambda_*$ with $\nabla\Gamma(\boldsymbol\lambda_*)=0$.
	The Implicit Function Theorem can be applied to the function
	$$
	I:(\boldsymbol\mu,\boldsymbol\lambda)\mapsto\boldsymbol\mu - \int \phim c(\boldsymbol\lambda)\exp\left(-\langle\boldsymbol\lambda,\phim\rangle\right)$$
	guaranteeing the existence of an open set $U\subset\mathbb{R}^{\psi(m,d)}$ (containing $\boldsymbol{\mu}_p$) and a unique continuous function $g:\boldsymbol\mu\mapsto\boldsymbol\lambda$ with $I(\boldsymbol\mu,g(\boldsymbol\mu))=0$ for all $\boldsymbol\mu\in U$.
	Consequently the convergence of the moment vector $\boldsymbol\mu_n\to\boldsymbol\mu_\infty$ implies the convergence of the corresponding parameter vectors $\boldsymbol\lambda_n\to\boldsymbol\lambda_\infty$ as $n$ tends to infinity.
	
	The convergence of $\log c(\boldsymbol\lambda_n)$ to $\log c(\boldsymbol\lambda_\infty)$ follows from the continuity of the \textit{cumulant function} $\boldsymbol\lambda \mapsto -\log c(\boldsymbol\lambda) = \log \left(\int \exp\left(-\langle \boldsymbol{\lambda},\phim\rangle\right)\right)$ as described in Subsection~\ref{subsec:existence_of_maxent}.
\end{proof}\\\\
Lemma~\ref{lemma:basic_inequality} together with Lemma~\ref{lemma:convergence_in_M} motivate to focus on densities with $\epsilon$-close maximum entropy and together prove Lemma~\ref{lemma:L1_convergence_in_M}.

\subsection{Convergence of Smooth High-Entropy Distributions}
\label{subsec:proof_convergence_of_smooth_high_entropy_densities}

In this Subsection, we propose a uniform upper bound on the $L^1$-difference between two densities in the set $\mathcal{H}_{m,\epsilon}$ as defined in Definition~\ref{def:H} that is linear in terms of the $\ell^1$-norm of the difference of finite moment vectors.
Let us start with the following helpful statement.
\Needspace{40\baselineskip}
\begin{restatable}{lemmarep}{lemmaApbound}%
	\label{lemma:Ap_bound}%
	Let $f_m\in\mathbb{R}_m[x]$ be a polynomial of degree less than or equal to $m$ on $[0,1]$ and $p\in\mathcal{M}([0,1])$ such that $\norm{\log p}_\infty= c_\infty$ for some $c_\infty\in\mathbb{R}$. Then the following holds:
	\begin{align*}
		\norm{f_m}_\infty \leq (m+1) e^{c_\infty/2} \norm{f_m}_{L^2(p)}.
	\end{align*}
\end{restatable}
\begin{proof}
	For all $f_m\in\mathbb{R}_m[x]$ the following holds by Lemma~\ref{lemma:barron_A_bound}:
	\begin{align*}
		\norm{f_m}_\infty &\leq (m+1)\norm{f_m}_{L^2}=(m+1)\sqrt{\int_0^1 |f_m|^2 \frac{p}{p}} \leq (m+1)\sqrt{\sup\frac{1}{|p|} \int_0^1 |f_m|^2 p}.
	\end{align*}
	Since ${c_\infty=\norm{\log p}_\infty}$, it holds that $-c_\infty\leq \log p\leq c_\infty$ and therefore also ${e^{-c_\infty}\leq 1/|p|\leq e^{c_\infty}}$ which yields the required result.
\end{proof}\\\\
The following Lemma~\ref{lemma:convergence_in_H} serves as our anchor in the approximation theory summarized in Subsection~\ref{subsec:approximation_theory}.
\newcommand{\ga}{\gamma}
\newcommand{\de}{\xi}

\Needspace{10\baselineskip}
\begin{restatable}{lemmarep}{lemmaconinH}%
	\label{lemma:convergence_in_H}%
	Consider some $m\geq r\geq 2$ and some $\phim=(\phi_1,\ldots,\phi_m)^\text{T}$ such that $1,\phi_1,\ldots,\phi_m$ is an orthonormal basis of $\mathbb{R}_m[x]$.
	
	Let $p,q\in\mathcal{M}([0,1])$ such that $\log p,\log q\in W_2^r$ and denote by $p^*$ and $q^*$ corresponding maximum entropy densities satisfying $\int\phim p^*=\int\phim p$ and $\int\phim q^*=\int\phim q$, respectively.
	
	If $4 e^{4\ga + 1} e^{c_\infty/2} (m+1) \de\leq 1$ then the following holds:
	\begin{gather}
		\norm{\boldsymbol{\mu}_p-\boldsymbol{\mu}_q}_2 \leq \frac{1}{2 C \left(m+1\right)}
		\quad\implies\quad
		d_\mathrm{KL}(p^*, q^*) \leq C\cdot \norm{\boldsymbol{\mu}_p-\boldsymbol{\mu}_q}_2^2
	\end{gather}
	where
	\begin{align}
		\label{eq:ct}
		C &= 2 e^{1+c_\infty+2 \gamma + 4 e^{4 \gamma + 1} \xi e^{\nicefrac{c_\infty}{2}}(m+1)}
	\end{align}
	and
	\begin{align}
		\label{eq:gamma}
		\ga &= \frac{e^r}{\sqrt{r-1} (m+r)^{r-1}} \left(\frac{1}{2}\right)^r c_r\\
		\label{eq:xi}
		\de^2 &= \frac{e^{c_\infty}}{(m+r+1)\cdots (m-r+2)}\left( \frac{1}{4} \right)^r c_r^2\\
		c_r &= \norm{\partial^r_x \log p}_{L^2}\\
		c_\infty &= \norm{\log p}_\infty.
	\end{align}%
\end{restatable}%

\begin{proof}
	Let $m,r$ be such that $m\geq r\geq 2$.
	Consider some $\phim=(\phi_1,\ldots,\phi_m)^\text{T}$ with $1,\phi_1,\ldots,\phi_m$ forming an orthonormal basis of $\mathbb{R}_m[x]$, i.e. forming an orthonormal basis of $\mathbb{R}_m[x]$ \wrt~the uniform weight function $\tilde q$ which is $1$ if $x\in[0,1]$ and $0$ otherwise.
	
	For $\tilde q$ it holds that $\norm{\log\tilde q}_\infty < \infty$ and with $A_{\tilde q}=m+1$, due to Lemma~\ref{lemma:barron_A_bound}, it also holds that
	\begin{align*}
		\norm{f_m}_\infty \leq A_{\tilde q} \norm{f_m}_{L^2(\tilde q)}
	\end{align*}
	for all  $f_m\in\mathbb{R}_m[x]$.
	Let $p,q\in\mathcal{M}([0,1])$ such that $\log p,\log q\in W_2^r$ and denote its moments by $\boldsymbol{\mu}_p=\int\phim p$ and $\boldsymbol{\mu}_q=\int \phim q$.
	Choose $\tilde{\boldsymbol{\mu}}=\boldsymbol{\mu}_q$, $\tilde p_0=p, \tilde b=e^{\norm{\log \tilde q/\tilde p_0^*}_\infty}$ and note that
	\begin{align*}
		\tilde b=e^{\norm{\log \tilde q/\tilde p_0^*}_\infty}
		= e^{\norm{\log \tilde q/p^*}_\infty}
		= e^{\norm{\log \tilde q-\log p*}_\infty}
		= e^{\norm{\log p^*}_\infty}.
	\end{align*}
	If
	\begin{align*}
		\norm{\boldsymbol{\mu}_p-\boldsymbol{\mu}_q}_2 \leq \frac{1}{4 (m+1) e^{1+\norm{\log p^*}_\infty}}
	\end{align*}
	then, due to Lemma~\ref{lemma:barron_and_sheu_lemma5}, the maximum entropy density $q^*\in\mathcal{M}([0,1])$ satisfies
	\begin{align*}
		d_\mathrm{KL}(p^*, q^*)\leq 2 e^{\tilde t} e^{\norm{\log p^*}_\infty} \norm{\boldsymbol{\mu}_p-\boldsymbol{\mu}_q}_2^2
	\end{align*}
	for $\tilde t$ satisfying $4 (m+1) e^{1+\norm{\log p^*}_\infty} \norm{\boldsymbol{\mu}_p-\boldsymbol{\mu}_q}_2\leq \tilde t\leq 1$, in particular for $\tilde t=1$, such that
	\begin{align*}
		d_\mathrm{KL}(p^*, q^*)\leq 2 e^{1+\norm{\log p^*}_\infty}\cdot \norm{\boldsymbol{\mu}_p-\boldsymbol{\mu}_q}_2^2.
	\end{align*}
	In the following, we aim at an upper bound on $\norm{\log p^*}_\infty$. It holds that
	\begin{align}
		\label{eq:log_p_in_proof}
		\norm{\log p^*}_\infty=\norm{\log p^* p/p}_\infty = \norm{\log p - \log p/p^*}_\infty \leq \norm{\log p}_\infty + \norm{\log p/p^*}_\infty
	\end{align}
	where the last inequality is due to the Triangle Inequality.
	Lemma~\ref{lemma:Ap_bound} yields
	\begin{align}
		\norm{f_m}_\infty &\leq (m+1) e^{c_\infty/2} \norm{f_m}_{L^2(p)}.
	\end{align}
	Denote by $\tilde p=p$, $\tilde f=\log p$ and $A_{\tilde p}=(m+1) e^{c_\infty/2}$.
	Let us further denote by $\tilde \gamma=\min_{f_m\in\mathbb{R}_m[x]} \norm{f-f_m}_\infty$ and ${\tilde \xi=\min_{f_m\in\mathbb{R}_m[x]} \norm{f-f_m}_{L^2(p)}}$ minimal errors of approximating $f$ by polynomials $f_m\in \mathbb{R}_m[x]$.
	From Corollary~\ref{lemma:barron_proof_thm3}, we obtain
	\begin{align*}
		4 e^{4\tilde\gamma + 1} A_{\tilde p} \tilde\xi \leq 1\quad\implies\quad
		\norm{\log{p/p^*}}_\infty \leq 2 \tilde\gamma + 4 e^{4 \tilde\gamma + 1} \tilde\xi A_{\tilde p}.
	\end{align*}
	Consider $\gamma$ and $\xi$ as defined in Eq.~\eqref{eq:gamma} and Eq.~\eqref{eq:xi}, respectively.
	Lemma~\ref{lemma:cox} yields $\tilde \gamma\leq \gamma, \tilde \xi\leq \xi$ and therefore also
	\begin{align*}
		4 e^{4\tilde\gamma + 1} A_{\tilde p} \tilde\xi\leq 4 e^{4\ga + 1} e^{c_\infty/2} (m+1) \de
	\end{align*}
	Consequently, if $4 e^{4\ga + 1} e^{c_\infty/2} (m+1) \de\leq 1$ then
	\begin{align}
		\label{eq:proof_bridge}
		\norm{\log{p/p^*}}_\infty \leq  2 \gamma + 4 e^{4 \gamma + 1} \xi (m+1) e^{c_\infty/2}
	\end{align}
	and together with Eq.~\eqref{eq:log_p_in_proof} we obtain
	\begin{align*}
		\norm{\log p^*}_\infty \leq c_\infty + 2 \gamma + 4 e^{4 \gamma + 1} \xi (m+1) e^{c_\infty/2}.
	\end{align*}
\end{proof}\\
\begin{remark}
	\label{remark:dependency}
	For $\gamma$ and $\xi$ as defined in Lemma~\ref{lemma:convergence_in_H} it holds that $\gamma,\xi\in O\left(\frac{1}{m^{r-1}}\right)$ and therefore $C\to 2 e^{1+c_\infty}$ as ${m,r\to\infty}$.
\end{remark}
To obtain simpler statements and useful bounds for small moment orders, we consider specific upper bounds on the norms of the log-derivatives as defined in Definition~\ref{def:H} of the set $\in \mathcal{H}_{m,\epsilon}$.
\begin{restatable}{lemmarep}{lemmasimplify}%
	\label{lemma:simplification}%
	Consider some $\epsilon\geq 0$, some $m=r\geq 2$ and let $p\in\mathcal{H}_{m,\epsilon}$. Then the following holds:
	\begin{align*}
		4 e^{4\ga + 1} e^{c_\infty/2} (m+1) \de\leq 1\quad\text{and}\quad C\leq 2 e^{(3m-1)/2}
	\end{align*}
	with $\gamma$, $\xi$, $c_r$, $c_\infty$, $C$ as defined in Lemma~\ref{lemma:convergence_in_H}.
\end{restatable}
\begin{proof}
	We start by proving the following inequalities inductively for $m\geq 2, m\in\mathbb{N}$:
	\begin{align}
		\label{eq:induction1}
		5^{m-4} &\leq\frac{\sqrt{(2 m+1)!\, (m-1)}}{2 e^{m+2}}\\
		\label{eq:induction2}
		\frac{3m-6}{2} &\leq \log\left( \frac{e^m 2^{m-1}}{(m+1) \sqrt{m-1}} \right)\\
		\label{eq:induction3}
		\frac{\sqrt{(2 m+1)!}}{4^m m^{m-1} e^{2}}&\leq \frac{1}{4}
	\end{align}
	For ${m=2,\ldots,7}$ all inequalities are fulfilled.
	Note that for any $m\geq 8$ the non-negativeness of later considered terms is ensured. To continue our proof by induction we may therefore assume that Eqs.~\eqref{eq:induction1}--\eqref{eq:induction3} are fulfilled for some arbitrary but fixed $m\in\mathbb{N}$ with $m\geq 8$.
	
	Since
	$$
	\frac{(2 m+3)(2m +2) m}{(m-1)} - 25 e^2
	$$ is a positive and monotonic increasing sequence for $m\geq 8$, as can be proven with any computer algebra system, it follows that
	\begin{align*}
		5\leq \sqrt{\frac{(2\cdot 8+3)(2\cdot 8 +2) 8}{(8-1) e^2}}\leq 
		\sqrt{\frac{(2 m+3)(2m +2) m}{(m-1) e^2}}
	\end{align*}
	such that
	\begin{align*}
		5^{m+1-4} &=5^{m-4}\cdot 5\\
		&\leq 5^{m-4}\sqrt{\frac{(2 m+3)(2m +2) m}{(m-1) e^2}}\\
		&\leq \frac{\sqrt{(2 m+1)!\, (m-1)}}{2 e^{m+2}} \sqrt{\frac{(2 m+3)(2m +2) m}{(m-1) e^2}}\\
		&= \frac{\sqrt{(2 (m+1)+1)!\, ((m+1)-1)}}{2 e^{(m+1)+2}}.
	\end{align*}
	
	Since
	\begin{align*}
		\log\left(\frac{e\, 2\, (m+1)\sqrt{m-1}}{(m+2)\sqrt{m}}\right) - \frac{3}{2}
	\end{align*}
	is a positive and monotonic increasing sequence for $m\geq 8$, as can be proven with any computer algebra system, it follows that
	\begin{align*}
		\frac{3}{2}\leq \log\left(\frac{e\, 2\, (8+1)\sqrt{8-1}}{(8+2)\sqrt{8}}\right)
		\leq \log\left(\frac{e\, 2\, (m+1)\sqrt{m-1}}{(m+2)\sqrt{m}}\right)
	\end{align*}
	such that
	\begin{align*}
		\frac{3 (m+1)-6}{2} &= \frac{3m-6}{2} + \frac{3}{2}\\
		&\leq \frac{3m-6}{2} + \log\left(\frac{e\, 2\, (m+1)\sqrt{m-1}}{(m+2)\sqrt{m}}\right)\\
		&\leq \log\left( \frac{e^m 2^{m-1}}{(m+1) \sqrt{m-1}} \right) + \log\left(\frac{e\, 2\, (m+1)\sqrt{m-1}}{(m+2)\sqrt{m}}\right)\\
		&=\log\left( \frac{e^{(m+1)} 2^{(m+1)-1}}{((m+1)+1) \sqrt{(m+1)-1}} \right).
	\end{align*}
	
	Since
	\begin{align*}
		\frac{\sqrt{(2 m+3)(2 m+2)}}{4 m} - 1
	\end{align*}
	is a negative and monotonic decreasing sequence for $m\geq 8$, as can be proven with any computer algebra system, it follows that
	\begin{align*}
		\frac{\sqrt{(2 m+3)(2 m+2)}}{4 m}\leq \frac{\sqrt{(2\cdot 8+3)(2\cdot 8+2)}}{4\cdot 8}\leq 1
	\end{align*}
	such that
	\begin{align*}
		\frac{\sqrt{(2 (m+1)+1)!}}{4^{(m+1)} (m+1)^{(m+1)-1} e^{2}} &\leq
		\frac{\sqrt{(2 m+1)!}}{4^m m^{m-1} e^{2}} \frac{\sqrt{(2 m+3)(2 m+2)}}{4 m}\\
		&\leq \frac{\sqrt{(2 m+1)!}}{4^m m^{m-1} e^{2}}\cdot 1\\
		&\leq \frac{1}{4}.
	\end{align*}
	
	According to Definition~\ref{def:H} and the verified Eq.~\eqref{eq:induction1} it holds that
	\begin{align}
		\label{eq:cr_proof}
		c_r &\leq 5^{m-4} \leq \frac{\sqrt{(2 m+1)!\, (m-1)}}{2 e^{m+2}}
	\end{align}
	which, together with Eq.~\eqref{eq:induction3}, implies that
	\begin{align}
		\label{eq:gamma_proof}
		\ga &\leq \frac{e^m}{\sqrt{m-1} (2 m)^{m-1}} \left(\frac{1}{2}\right)^m \frac{\sqrt{(2 m+1)!\, (m-1)}}{2 e^{m+2}}=
		\frac{\sqrt{(2 m+1)!}}{4^m m^{m-1} e^{2}}\leq \frac{1}{4}.
	\end{align}
	Applying Eq.~\eqref{eq:gamma_proof}, the definition of $\xi$ and Eq.~\eqref{eq:cr_proof}, we obtain
	\begin{align*}
		4 e^{4\ga + 1} &e^{c_\infty/2} (m+1) \de\\
		&\leq
		4 e e (m+1) e^{c_\infty/2} \frac{e^{c_\infty/2}}{\sqrt{(2 m+1)!}}\left( \frac{1}{2} \right)^m \frac{\sqrt{(2 m+1)!\, (m-1)}}{2 e^{m+2}}\\
		&=
		e\frac{4 e (m+1) \sqrt{(m-1)}}{2 e^{m+2} 2^m} e^{c_\infty}
	\end{align*}
	From Definition~\ref{def:H} and Eq.~\eqref{eq:induction2} we know that
	\begin{align}
		\label{eq:cinfty_proof}
		c_\infty &\leq \frac{3 m -6}{2} \leq \log\left( \frac{e^m 2^{m-1}}{(m+1) \sqrt{m-1}} \right)
	\end{align}
	which further gives
	\begin{align*}
		e\frac{4 e (m+1) \sqrt{(m-1)}}{2 e^{m+2} 2^m} e^{c_\infty} \leq e\frac{4 e (m+1) \sqrt{(m-1)}}{2 e^{m+2} 2^m} \left( \frac{e^m 2^{m-1}}{(m+1) \sqrt{m-1}} \right)=1
	\end{align*}
	and therefore
	\begin{align*}
		4 e^{4\ga + 1} &e^{c_\infty/2} (m+1) \de\leq 1.
	\end{align*}
	From Eq.~\eqref{eq:ct} we obtain
	\begin{align*}
		C &= 2 e^{1+c_\infty+2 \gamma + 4 e^{4 \gamma + 1} \xi e^{\nicefrac{c_\infty}{2}}(m+1)}
		\leq 2 e^{2 + 2\ga + c_\infty}
	\end{align*}
	and by applying Eq.~\eqref{eq:cinfty_proof} and Eq.~\eqref{eq:gamma_proof} it holds that
	\begin{align*}
		C \leq 2 e^{\frac{5}{2} + \frac{3 m -6}{2}}
		\leq 2 e^{\frac{3 m -1}{2}}.
	\end{align*}
\end{proof}\\\\
The following lemma allows to focus on distributions from exponential families with independent marginals by considering specific vectors of polynomials.
\needspace{25\baselineskip}
\begin{restatable}{lemmarep}{lemmaindependence}%
	\label{lemma:independence}%
	Consider some polynomial vector $\phim=(\phi_1,\ldots,\phi_{m d})^\text{T}$ such that $1,\phi_1,\ldots,\phi_{m d}$ is an orthonormal basis of $\mathrm{Span}(\mathbb{R}_m[x_1]\cup\ldots\cup\mathbb{R}_m[x_d])$.
	
	Let $p^*,q^*$ be two maximum entropy densities satisfying $\int\phim p^*=\int\phim p, \int\phim q^*=\int\phim q$ for some $p,q\in\Ms$. 
	Then the following holds:
	\begin{align}
		d_\mathrm{KL}(p^*,q^*)=\sum_{i=1}^d d_\mathrm{KL}(p_i^*,q_i^*)
	\end{align}
	where $p_i^*$ denotes the maximum entropy density of $p$
	satisfying $\int\boldsymbol{\phi}_m^{(i)} p^*=\int \boldsymbol{\phi}_m^{(i)} p$ for some vector $\boldsymbol{\phi}_m^{(i)}=(\phi_{i1},\ldots,\phi_{im})$ such that $1,\phi_{i1},\ldots,\phi_{im}$ is an orthonormal basis of $\mathbb{R}_m[x_i]$.
\end{restatable}
\begin{proof}
	According to Eq.~\eqref{eq:maxent_distr_formula} it holds that $p_i^*$ is of the form
	\begin{align*}
		p_i^*(x_i) = c_i(\boldsymbol\lambda_i) \exp\left(-\langle \boldsymbol\lambda_i,\boldsymbol{\phi}_m^{(i)}(x_i)\rangle\right)
	\end{align*}
	where $c_i(\boldsymbol\lambda_i)=\left(\int_0^1 \exp\left(-\langle \boldsymbol\lambda_i,\boldsymbol{\phi}_m^{(i)}(x_i)\rangle\right)dx_i\right)^{-1}$ is the constant of normalization and $\boldsymbol\lambda_i\in\mathbb{R}^{m}$ is a parameter vector.
	It follows that
	\begin{align*}
		\tilde p^* &= p_1^*\cdots p_d^*\\
		&= c_1(\boldsymbol\lambda_1) \exp\left(-\langle \boldsymbol\lambda_1,\boldsymbol{\phi}_m^{(1)}(x_1)\rangle\right) \cdots c_d(\boldsymbol\lambda_d) \exp\left(-\langle \boldsymbol\lambda_d,\boldsymbol{\phi}_m^{(d)}(x_d)\rangle\right)\\
		&= \left(\int_{[0,1]^d} \exp\left(-\langle \boldsymbol{\tilde\lambda},\boldsymbol{\tilde\phi}_m(\x)\rangle\right)d\x\right)^{-1} \exp\left(-\langle \boldsymbol{\tilde\lambda},\boldsymbol{\tilde\phi}_m(\x)\rangle\right)
	\end{align*}
	where $\boldsymbol{\tilde\lambda}\in\mathbb{R}^{m d}$ is the concatenation of the vectors $\boldsymbol\lambda_1,\ldots,\boldsymbol\lambda_d$ and $\boldsymbol{\tilde\phi}_m\in\mathbb{R}_m[x_1,\ldots,x_d]$ is the vector of polynomials obtained by the concatenation of $\boldsymbol{\phi}_m^{(1)},\ldots,\boldsymbol{\phi}_m^{(d)}$.
	It holds that $\tilde p^*$ is a probability density of exponential form with sufficient statistic $\boldsymbol{\tilde\phi}_m$.
	The elements of $\boldsymbol{\tilde\phi}_m$, together with the unit $1$, form an orthonormal basis of $\mathrm{Span}(\mathbb{R}_m[x_1]\cup\ldots\cup\mathbb{R}_m[x_d])$.
	The uniqueness and the exponential form of the maximum entropy density $p^*$ implies that $\tilde p^*=p^*$ and the following holds:
	\begin{align*}
		d_\mathrm{KL}(p^*, q^*)&=\int_{[0,1]^d} p^*\log\frac{p^*}{q^*}d\x\\
		&=\int_0^1\ldots\int_0^1 p_1^*\cdots p_d^*\log\frac{p_1^*\cdots p_d^*}{q_1^*\cdots q_d^*} dx_1\ldots dx_d\\
		&= \sum_{i=1}^d \int_0^1\ldots\int_0^1 p_1^*\cdots p_d^* \log\frac{p_i^*}{q_i^*}dx_1\ldots dx_d\\
		&= \sum_{i=1}^d\left( \int_0^1 p_i^*\log\frac{p_i^*}{q_i^*} dx_i\prod_{j\neq i} \int_0^1 p_j^* dx_j\right)\\
		&= \sum_{i=1}^d \int_0^1 p_i^*\log\frac{p_i^*}{q_i^*} dx_i\\
		&= \sum_{i=1}^d d_\mathrm{KL}(p_i^*, q_i^*).
	\end{align*}
\end{proof}\\\\
We are now ready to prove Theorem~\ref{thm:bound_for_smooth_functions}.
\convergence*
\begin{proof}
	Consider some $m, \epsilon,\phim$ and $\mathcal{H}_{m,\epsilon}$ as in Definition~\ref{def:H} and some $p, q\in \mathcal{H}_{m,\epsilon}$.
	Then $p, q\in \mathcal{M}([0,1]^d)$ and have $\epsilon$-close maximum entropy.
	Applying Lemma~\ref{lemma:basic_inequality} yields
	\begin{align*}
		\norm{p-q}_{L^1} \leq \sqrt{2 d_\mathrm{KL}(p^*, q^*)} + \sqrt{8 \epsilon}
	\end{align*}
	for $p^*,q^*$ being the maximum entropy densities satisfying
	$\int \phim p^*=\int\phim p, \int\phim q^*=\int\phim q$.
	The vector $\phim=(\phi_1,\ldots,\phi_{m d})^\text{T}$ is a polynomial vector such that $1,\phi_1,\ldots,\phi_{m d}$ is an orthonormal basis of ${\mathrm{Span}(\mathbb{R}_m[x_1]\cup\ldots\cup\mathbb{R}_m[x_d])}$. 
	Therefore, by applying Lemma~\ref{lemma:independence}, we obtain
	\begin{align}
		\label{eq:inequ_impl_proof_thm1}
		\norm{p-q}_{L^1} \leq \sqrt{2\sum_{i=1}^d d_\mathrm{KL}(p_i^*, q_i^*)} + \sqrt{8 \epsilon},
	\end{align}
	where $p_i^*$ denotes the maximum entropy density of $p$
	satisfying $\int\boldsymbol{\phi}_m^{(i)} p_i^* =\int \boldsymbol{\phi}_m^{(i)} p$ for some vector $\boldsymbol{\phi}_m^{(i)}=(\phi_{i1},\ldots,\phi_{im})$ such that $1,\phi_{i1},\ldots,\phi_{im}$ is an orthonormal basis of $\mathbb{R}_m[x_i]$.
	
	The densities $p_i^*$ can also be seen as maximum entropy densities
	satisfying $\int \boldsymbol{\phi}_m^{(i)} p_i^*=\boldsymbol{\mu}_{p_i}=\int \boldsymbol{\phi}_m^{(i)} p_i$ for the marginal densities $p_i$ of $p$ defined by
	\begin{align*}
		p_i(x_i)=\int_0^1\cdots\int_0^1 p(x_1,\ldots,x_d)\, d x_1\cdots d x_{i-1} d x_{i+1}\cdots d x_d.
	\end{align*}
	From Definition~\ref{def:H} it follows that ${\log p_i\in W_2^m}$ with Sobolev space $W_2^m$.
	If it holds that $4 e^{4\ga + 1} e^{c_\infty/2} (m+1) \de\leq 1$ then the following holds by Lemma~\ref{lemma:convergence_in_H}:
	\begin{gather}
		\label{eq:inequ_impl_proof_thm1_2}
		\norm{\boldsymbol{\mu}_{p_i}-\boldsymbol{\mu}_{q_i}}_2 \leq \frac{1}{2 C \left(m+1\right)}
		\quad\implies\quad
		d_\mathrm{KL}(p_i^*, q_i^*) \leq C\cdot \norm{\boldsymbol{\mu}_{p_i}-\boldsymbol{\mu}_{q_i}}_2^2
	\end{gather}
	with $C,\ga,c_\infty,\de$ as defined in Lemma~\ref{lemma:convergence_in_H} with $r=m$.
	Since $p\in\H_{m,\epsilon}$, Lemma~\ref{lemma:simplification} implies that $4 e^{4\ga + 1} e^{c_\infty/2} (m+1) \de\leq 1$ and $C\leq e^{(3m-6)/2}$.
	Since
	\begin{align*}
		\norm{\boldsymbol{\mu}_{p_i}-\boldsymbol{\mu}_{q_i}}_2 \leq 
		\norm{\boldsymbol{\mu}_{p}-\boldsymbol{\mu}_{q}}_2 \leq\norm{\boldsymbol{\mu}_{p}-\boldsymbol{\mu}_{q}}_1
	\end{align*}
	it follows that
	\begin{align}
		\label{eq:inequ_impl_proof_thm1_3}
		\norm{\boldsymbol{\mu}_{p}-\boldsymbol{\mu}_{q}}_1 \leq \frac{1}{2 C \left(m+1\right)}
		\quad\implies\quad
		d_\mathrm{KL}(p_i^*, q_i^*) \leq C\cdot \norm{\boldsymbol{\mu}_{p_i}-\boldsymbol{\mu}_{q_i}}_2^2.
	\end{align}
	Therefore, if
	\begin{align*}
		\norm{\boldsymbol{\mu}_{p}-\boldsymbol{\mu}_{q}}_1 \leq \frac{1}{2 C \left(m+1\right)}
	\end{align*}
	then Eq.~\eqref{eq:inequ_impl_proof_thm1} can be further extended by
	\begin{align*}
		\norm{p-q}_{L^1} &\leq \sqrt{2\sum_{i=1}^d d_\mathrm{KL}(p_i^*, q_i^*)} + \sqrt{8 \epsilon}
		\leq \sqrt{2\sum_{i=1}^d C\cdot \norm{\boldsymbol{\mu}_{p_i}-\boldsymbol{\mu}_{q_i}}_2^2} + \sqrt{8 \epsilon}\\
		&\leq \sqrt{2 C}\cdot \sum_{i=1}^d \norm{\boldsymbol{\mu}_{p_i}-\boldsymbol{\mu}_{q_i}}_2 + \sqrt{8 \epsilon}
		\leq \sqrt{2 C}\cdot \sum_{i=1}^d \norm{\boldsymbol{\mu}_{p_i}-\boldsymbol{\mu}_{q_i}}_1 + \sqrt{8 \epsilon}\\
		&= \sqrt{2 C}\cdot \norm{\boldsymbol{\mu}_{p}-\boldsymbol{\mu}_{q}}_1 + \sqrt{8 \epsilon}.
	\end{align*}
\end{proof}\\

\subsection{Learning Bound for Moment-Based Domain Adaptation}
\label{subsec:problem_solution_proof}

In the following, we consider the sample case.

\Needspace{25\baselineskip}
\begin{restatable}{lemmarep}{sampleconvinH}%
	\label{lemma:sample_convergence_in_H}%
	Consider some $m\geq r\geq 2$ and some $\phim=(\phi_1,\ldots,\phi_m)^\text{T}$ such that $1,\phi_1,\ldots,\phi_m$ is an orthonormal basis of $\mathbb{R}_m[x]$.
	
	Let $p,q\in\mathcal{M}([0,1])$ such that $\log p,\log q\in W_2^r$ and denote by $\widehat{\boldsymbol{\mu}}_p=\frac{1}{k}\sum_{\x\in X_p}\phim(\x)$ and $\widehat{\boldsymbol{\mu}}_q=\frac{1}{k}\sum_{\x\in X_q}\phim(\x)$ the moments of two $k$-sized samples $X_p$ and $X_q$ drawn from $p$ and $q$, respectively.
	
	If $4 e^{4\gamma + 1} e^{c_\infty/2} (m+1) \xi \leq 1$ then for all $\delta\in (0,1)$ such that
	\begin{align}
		\label{eq:assumption_on_sample_size}
		4 C^2 (m+1)^2 m e^{-c_\infty} \leq \delta k
	\end{align}
	with probability at least $1-\delta$, the maximum entropy densities $\widehat{p}$ and $\widehat{q}$ satisfying $\int\phim \widehat{p}=\widehat{\boldsymbol{\mu}}_p$ and $\int\phim \widehat{q}=\widehat{\boldsymbol{\mu}}_q$, respectively, exist and the following holds:	
	\begin{gather}
		\label{eq:sample_size_p}
		d_\mathrm{KL}(p^*, \widehat{p})\leq C e^{-c_\infty} \frac{m}{k\delta}\\
		\label{eq:sample_size_q}
		d_\mathrm{KL}(q^*, \widehat{q})\leq C e^{-c_\infty} \frac{m}{k\delta}\\
		\label{eq:sample_moment_diff}
		\norm{\widehat{\boldsymbol{\mu}}_p-\widehat{\boldsymbol{\mu}}_q}_2 \leq \frac{1}{2 (m+1) e C}\quad \implies\quad d_\mathrm{KL}(\widehat{p},\widehat{q})\leq e C \norm{\widehat{\boldsymbol{\mu}}_p-\widehat{\boldsymbol{\mu}}_q}_2^2
	\end{gather}
	where
	\begin{align}
		c_\infty &=\max\left\{\norm{\log p}_\infty,\norm{\log q}_\infty\right\}\\
		c_r &=\max\{\norm{\partial_x^r \log p}_{L^2}\norm{\partial_x^r \log q}_{L^2}\}
	\end{align}
	and $\gamma,\xi$ and $C$ are defined as in Lemma~\ref{lemma:convergence_in_H}.
\end{restatable}
\begin{proof}	
	Let $m,r$ be such that $m\geq r\geq 2$ and $\phim=(\phi_1,\ldots,\phi_m)^\text{T}$ be such that $1,\phi_1,\ldots,\phi_m$ is an orthonormal basis of $\mathbb{R}_m[x]$.
	
	Let $p,q\in\mathcal{M}([0,1])$ such that $\log p,\log q\in W_2^r$ with Sobolev space $W_2^r$.
	Let further $\widehat{\boldsymbol{\mu}}_p=\frac{1}{k}\sum_{\x\in X_p}\phim(\x)$ and $\widehat{\boldsymbol{\mu}}_p=\frac{1}{k}\sum_{\x\in X_q}\phim(\x)$ be the moments of two $k$-sized samples $X_p$ and $X_q$ drawn from $p$ and $q$, respectively.
	
	From Lemma~\ref{lemma:Ap_bound} we obtain ${A_p=e^{\norm{\log p}_\infty/2} (m+1)}$ and ${A_q=e^{\norm{\log q}_\infty/2} (m+1)}$ such that ${\norm{f_m}_\infty \leq A_p \norm{f_m}_{L^2(p)}}$ and ${\norm{f_m}_\infty \leq A_q \norm{f_m}_{L^2(q)}}$ for all $f_m\in\mathbb{R}_m[x]$.
	
	Denote by $\tilde A=\max\{A_p,A_q\}$ and by $f_p=\log p, f_q=\log q$.
	Further denote by
	\begin{align*}
		\tilde \gamma=\max \left\{\min_{f_m\in\mathbb{R}_m[x]} \norm{f_p-f_m}_\infty, \min_{f_m\in\mathbb{R}_m[x]} \norm{f_q-f_m}_\infty\right\}
	\end{align*}
	and
	\begin{align*}
		\tilde \xi=\max\left\{ \min_{f_m\in\mathbb{R}_m[x]} \norm{f_p-f_m}_{L^2(p)}, \min_{f_m\in\mathbb{R}_m[x]} \norm{f_q-f_m}_{L^2(p)} \right\}
	\end{align*}
	minimal errors of approximating $f_p$ and $f_q$ by polynomials $f_m\in \mathbb{R}_m[x]$.
	Denote by ${\tilde b=e^{2 \tilde\gamma + 4 e^{4 \tilde\gamma + 1} \tilde\xi \tilde A}}$.
	
	If $4 e^{4\tilde\gamma+1} \tilde A \tilde\xi\leq 1$, then Corollary~\ref{lemma:barron_proof_thm3} implies that
	\begin{align}
		\label{eq:barron_log_bound}
		\norm{\log{p/p^*}}_\infty \leq 2 \tilde\gamma + 4 e^{4 \tilde\gamma + 1} \tilde\xi \tilde A
	\end{align}
	and for all $\delta\in (0,1)$ such that $(4 e \tilde b \tilde A)^2 m \leq \delta k$. Corollary~\ref{cor:barron_sample} implies the existence of the maximum entropy densities $\widehat{p}$ and $\widehat q$ with probability at least $1-\delta$ and it holds that
	\begin{align}
		\label{eq:barron_sample_appl1}
		d_\mathrm{KL}(p^*, \widehat{p})\leq 2 e \tilde b \frac{m}{k\delta}\\
		\label{eq:barron_sample_appl2}
		d_\mathrm{KL}(q^*, \widehat{q})\leq 2 e \tilde b \frac{m}{k\delta}\\
		\label{eq:barron_sample_appl3}
		\norm{\log p^*/\widehat{p}}_\infty\leq 1.
	\end{align}
	Consider $\gamma$ and $\xi$ as defined in Eq.~\eqref{eq:gamma} and Eq.~\eqref{eq:xi}, respectively.
	Note that
	\begin{align*}
		\min_{f_m\in\mathbb{R}_m[x]} \norm{f_p-f_m}_{L^2(p)}^2 &= \min_{f_m\in\mathbb{R}_m[x]} \int |f_p-f_m|^2 p\\
		&\leq \sup \left|p\right| \min_{f_m\in\mathbb{R}_m[x]} \int |f_p-f_m|^2\\
		&\leq e^{c_\infty} \min_{f_m\in\mathbb{R}_m[x]} \norm{f_p-f_m}_{2}^2.
	\end{align*}
	Lemma~\ref{lemma:cox} yields $\tilde \gamma\leq \gamma, \tilde \xi\leq \xi$.
	It also holds that $\tilde A=\max\{A_p,A_q\}\leq e^{c_\infty/2} (m+1)$ which implies
	\begin{align*}
		(4 e \tilde b \tilde A)^2 m &\leq (4 e e^{2 \gamma + 4 e^{4 \gamma + 1} \xi \tilde A} \tilde A)^2 m\\
		&\leq (4 e e^{2 \gamma + 4 e^{4 \gamma + 1} \xi e^{c_\infty/2} (m+1)} e^{c_\infty/2} (m+1))^2 m\\
		&= 4 C^2 (m+1)^2 m e^{-c_\infty}.
	\end{align*}
	Therefore, if $4 e^{4\gamma + 1} e^{c_\infty/2} (m+1) \xi \leq 1$ then for all $\delta\in (0,1)$ such that
	$$
	4 C^2 (m+1)^2 m e^{-c_\infty} \leq \delta k
	$$
	with probability at least $1-\delta$ the maximum entropy densities $\widehat{p}$ and $\widehat{q}$ satisfying $\int \phim \widehat{p}= \widehat{\boldsymbol{\mu}}_p$ and $\int \phim \widehat{q}= \widehat{\boldsymbol{\mu}}_q$, respectively, exist and the following inequalities hold:
	\begin{align}
		d_\mathrm{KL}(p^*, \widehat{p}) &\leq 2 e \tilde b \frac{m}{k\delta}\leq 2 e e^{2 \tilde\gamma + 4 e^{4 \tilde\gamma + 1} \tilde\xi \tilde A} \frac{m}{k\delta}\leq C e^{-c_\infty} \frac{m}{k\delta}\\
		d_\mathrm{KL}(q^*, \widehat{q}) &\leq C e^{-c_\infty} \frac{m}{k\delta}\\
		\label{eq:proof_sample1}
		\norm{\log p^*/\widehat{p}}_\infty &\leq 1\\
		\label{eq:proof_sample2}
		\norm{\log{p/p^*}}_\infty &\leq 2 \tilde\gamma + 4 e^{4 \tilde\gamma + 1} \tilde\xi \tilde A
		\leq 2 \gamma + 4 e^{4 \gamma + 1} \xi e^{c_\infty/2} (m+1)
	\end{align}
	where the last inequality follows from Eq.~\eqref{eq:barron_log_bound}.
	
	Let us now prove the upper bound on $d_\mathrm{KL}(\widehat{p},\widehat{q})$.
	To do this, note that $1,\phi_1,\ldots,\phi_m$ form an orthonormal basis of $\mathbb{R}_m[x]$, i.e. they form an orthonormal basis of $\mathbb{R}_m[x]$ \wrt~the uniform weight function $\tilde q$ on $[0,1]$.
	For $\tilde q$ it holds that $\norm{\log\tilde q}_\infty < \infty$ and with $A_{\tilde q}=m+1$, due to Lemma~\ref{lemma:barron_A_bound}, it also holds that
	\begin{align*}
		\norm{f_m}_\infty \leq A_{\tilde q} \norm{f_m}_{L^2(\tilde q)}
	\end{align*}
	for all  $f_m\in\mathbb{R}_m[x]$.
	Consider the vector of moments $\boldsymbol{\tilde\mu}=\widehat{\boldsymbol{\mu}}_q\in [0,1]^m$.
	Let $\tilde p_0=\widehat p\in\mathcal{M}([0,1])$ and note that its moments are given by $\int \phim \tilde p_0=\widehat{\boldsymbol{\mu}}_p$.
	Let $\tilde b =e^{\norm{\log \tilde q/\widehat{p}}_\infty}$.
	If the maximum entropy densities $\widehat{p}$ and $\widehat{q}$ satisfying $\int\phim \widehat{p}=\widehat{\boldsymbol{\mu}}_p$ and $\int\phim \widehat{q}=\widehat{\boldsymbol{\mu}}_q$, respectively, exist, then by Lemma~\ref{lemma:barron_and_sheu_lemma5} it holds that
	\begin{align*}
		\norm{\widehat{\boldsymbol{\mu}}_p-\widehat{\boldsymbol{\mu}}_q}_2\leq\frac{1}{4 A_{\tilde q} e \tilde b}\quad\implies\quad d_\mathrm{KL}(\widehat p,\widehat q)\leq 2 e^t \tilde b \norm{\widehat{\boldsymbol{\mu}}_p-\widehat{\boldsymbol{\mu}}_q}_2^2
	\end{align*}
	especially for $t$ such that $4 e \tilde b A_{\tilde q} \norm{\widehat{\boldsymbol{\mu}}_p-\widehat{\boldsymbol{\mu}}_q}_2\leq t\leq 1$.
	If Eq.~\eqref{eq:proof_sample1} and Eq.~\eqref{eq:proof_sample2} hold, then
	\begin{align*}
		\tilde b &\leq e^{\norm{\log \tilde (\tilde q p p^*) / (p p^* \widehat{p})}_\infty}\\
		&\leq e^{\norm{\log p}_\infty +\norm{\log p/p^*}_\infty + \norm{\log p^*/\widehat{p}}_\infty}\\
		&\leq e^{c_\infty + 2 \gamma + 4 e^{4 \gamma + 1} \xi e^{c_\infty/2} (m+1) + 1}\\
		&= \frac{1}{2} C
	\end{align*}
	for $C$ as defined in Lemma~\ref{lemma:convergence_in_H}.
	Therefore, if $4 e^{4\gamma + 1} e^{c_\infty/2} (m+1) \xi \leq 1$ then for all $\delta\in [1,0)$ such that
	$$
	4 C^2 (m+1)^2 m e^{-c_\infty} \leq \delta k
	$$
	with probability at least $1-\delta$ the maximum entropy densities $\widehat{p}$ and $\widehat{q}$ satisfying $\int\phim \widehat{p}=\widehat{\boldsymbol{\mu}}_p$ and $\int\phim \widehat{q}=\widehat{\boldsymbol{\mu}}_q$, respectively, exist, and, since Eq.~\eqref{eq:proof_sample1} and Eq.~\eqref{eq:proof_sample2} hold, the following also holds:
	\begin{align*}
		\norm{\widehat{\boldsymbol{\mu}}_p-\widehat{\boldsymbol{\mu}}_q}_2\leq\frac{1}{2 (m+1) e C}\quad\implies\quad d_\mathrm{KL}(\widehat p,\widehat q)\leq e C \norm{\widehat{\boldsymbol{\mu}}_p-\widehat{\boldsymbol{\mu}}_q}_2^2.
	\end{align*}
\end{proof}\\
\begin{remark}
	\label{remark:simplification_min_sample_size}
	If the densities $p,q\in\mathcal{H}_{m,\epsilon}$ then Lemma~\ref{lemma:simplification} allows to replace the assumption in Eq.~\eqref{eq:assumption_on_sample_size} of Lemma~\ref{lemma:sample_convergence_in_H} by the assumption
	\begin{align}
		\frac{1}{\delta} 16 e^{3m-1} (m+1)^2 m\leq k.
	\end{align}
	However, smaller lower bounds on the sample size are obtained by using the definition of $C$ as in Lemma~\ref{lemma:convergence_in_H}.
\end{remark}
We are now able to prove our main result.
\Needspace{15\baselineskip}
\mainresult*

\begin{proof}
	Consider some $m, \epsilon,\phim$ and $\mathcal{H}_{m,\epsilon}$ as in Definition~\ref{def:H} and a function class $\mathcal{F}$ with finite VC-dimension $\vc$.
	Let $p,q\in\mathcal{H}_{m,\epsilon}$ and ${l_p,l_q:[0,1]^d\to [0,1]}$.
	Let $X_p$ and $X_q$ be two arbitrary $k$-sized samples drawn from $p$ and $q$, respectively.
	
	Eq.~\eqref{eq:simple_da_equation} (proven by Ben-David et al.~\cite{ben2010theory}) implies that
	\begin{align}
		\int \left|f-l_q\right| q\leq \int \left|f-l_p\right| p + \norm{p-q}_{L^1} + \lambda^*
	\end{align}
	where $\lambda^* = \inf_{h\in\mathcal{F}}\big(\int \left|f-l_p\right| p+\int \left|f-l_q\right| q\big)$.
	Combining Theorem~\ref{thm:vc_bound} with Theorem~\ref{thm:lone_domain_adaptation_bound} the following holds with probability at least $1-\delta$ (over the choice of $k$-sized samples $X_q$ drawn from $q$):
	\begin{align}
		\begin{split}
			\int \left|f-l_q\right| q &\leq \frac{1}{k}\sum_{\x\in X_p}|f(\x)-l(\x)|
			+ \sqrt{\frac{4}{k} \left( \vc\log \frac{2 e k}{\vc} + \log\frac{4}{\delta} \right)}\\
			&\phantom{\leq} + \lambda^* + \norm{p-q}_{L^1}
		\end{split}
	\end{align}
	
	In the following, we bound the term $\norm{p-q}_{L^1}$ from above to obtain the second line of Eq.~\eqref{eq:moment_adapt_result_bound}:
	If the maximum entropy densities $\widehat{p}$ and $\widehat{q}$
	satisfying $\int \phim \widehat{p}=\widehat{\boldsymbol{\mu}}_p=\frac{1}{k}\sum_{\x\in X_p}\boldsymbol{\phi}_m(\x)$ and $\int \phim \widehat{q}=\widehat{\boldsymbol{\mu}}_q=\frac{1}{k}\sum_{\x\in X_q}\boldsymbol{\phi}_m(\x)$, respectively, exist, then the Triangle inequality and Eq.~\eqref{eq:relationship:Pinsker} imply
	\begin{align*}
		\norm{p-q}_{L^1} &\leq \norm{\widehat{p}-\widehat{q}}_{L^1} +\norm{\widehat{p}-p^*}_{L^1} + \norm{\widehat{q}-q^*}_{L^1} + \norm{p^*-p}_{L^1} + \norm{q^*-q}_{L^1}\\
		&\leq \norm{\widehat{p}-\widehat{q}}_{L^1} +\norm{\widehat{p}-p^*}_{L^1} + \norm{\widehat{q}-q^*}_{L^1} + \sqrt{2 d_\mathrm{KL}(p, p^*)} + \sqrt{2 d_\mathrm{KL}(q, q^*)}
	\end{align*}
	which, by the $\epsilon$-closeness of $p,q\in\mathcal{H}_{m,\epsilon}$, further implies that
	\begin{align}
		\label{eq:sample_proof_l1_inequ}
		\norm{p-q}_{L^1} 
		&\leq \norm{\widehat{p}-\widehat{q}}_{L^1} +\norm{\widehat{p}-p^*}_{L^1} + \norm{\widehat{q}-q^*}_{L^1} + \sqrt{8\epsilon}\\
		\nonumber
		&\leq \sqrt{d_\mathrm{KL}(\widehat{p}, \widehat{q})} +\sqrt{d_\mathrm{KL}(\widehat{p}, p^*)} + \sqrt{d_\mathrm{KL}(\widehat{q}, q^*)} + \sqrt{8\epsilon}.
	\end{align}
	The vector $\phim=(\phi_1,\ldots,\phi_{m d})^\text{T}$ is a polynomial vector such that $1,\phi_1,\ldots,\phi_{m d}$ is an orthonormal basis of ${\mathrm{Span}(\mathbb{R}_m[x_1]\cup\ldots\cup\mathbb{R}_m[x_d])}$. 
	Therefore, by applying Lemma~\ref{lemma:independence}, we obtain
	\begin{align}
		\label{eq:inequ_impl_proof_main_sample_thm}
		\norm{p-q}_{L^1} \leq \sqrt{2\sum_{i=1}^d d_\mathrm{KL}(\widehat{p}_i, \widehat{q}_i)} + \sqrt{2\sum_{i=1}^d d_\mathrm{KL}(\widehat{p}_i, p^*_i)} + \sqrt{2\sum_{i=1}^d d_\mathrm{KL}(\widehat{q}_i, q^*_i)} + \sqrt{8 \epsilon}
	\end{align}
	where $p_i^*$ and $\widehat{p}_i$ denote the maximum entropy densities of $p$ and $\widehat{p}$ satisfying $\int \boldsymbol{\phi}_m^{(i)} p_i^*=\int \boldsymbol{\phi}_m^{(i)} p$ and $\int \boldsymbol{\phi}_m^{(i)} \widehat{p}_i=\int \boldsymbol{\phi}_m^{(i)} \widehat{p}$, respectively, for some vector $\boldsymbol{\phi}_m^{(i)}=(\phi_{i1},\ldots,\phi_{im})$ such that $1,\phi_{i1},\ldots,\phi_{im}$ is an orthonormal basis of $\mathbb{R}_m[x_i]$.
	
	The density $p_i^*$ is the maximum entropy density satisfying $\int \boldsymbol{\phi}_m^{(i)} p_i^*=\boldsymbol{\mu}_{p_i}=\int \boldsymbol{\phi}_m^{(i)} p_i$ for the marginal density $p_i$ of $p$ defined by
	\begin{align*}
		p_i(x_i)=\int_0^1\cdots\int_0^1 p(x_1,\ldots,x_d)\, d x_1\cdots d x_{i-1} d x_{i+1}\cdots d x_d.
	\end{align*}
	Denote by $X_{p_i}$ the $k$-sized sample (multiset) consisting of the $i$-th coordinates of the vectors stored in the sample $X$.
	It holds that the sample $X_{p_i}$ is drawn from the probability density $p_i$ and the density $\widehat{p}_i$ can be seen to be the maximum entropy density satisfying $\int \boldsymbol{\phi}_m^{(i)} \widehat{p}_i=\widehat{\boldsymbol{\mu}}_{p_i}=\frac{1}{k}\sum_{\x\in X_{p_i}}\boldsymbol{\phi}_m^{(i)}(\x)$.
	From Definition~\ref{def:H} it follows that $\norm{\partial_{x_i}^m \log p_i}_{L^2}\leq 5^{m-4}$ and therefore ${\log p_i\in W_2^r}$ with Sobolev space $W_2^r$.
	All assumptions from Lemma~\ref{lemma:sample_convergence_in_H} are fulfilled and therefore the following holds:
	If $4 e^{4\gamma + 1} e^{c_\infty/2} (m+1) \xi \leq 1$ then for all $\delta\in (0,1)$ such that
	\begin{align*}
		4 C^2 (m+1)^2 m e^{-c_\infty} \leq \delta k
	\end{align*}
	with probability at least $1-\delta$ the maximum entropy densities $\widehat{p}_i$ and $\widehat{q}_i$ exist and the following holds:
	\begin{gather}
		\label{eq:proof_sample_size_p}
		d_\mathrm{KL}(p_i^*, \widehat{p}_i)\leq C e^{-c_\infty} \frac{m}{k\delta}\\
		\label{eq:proof_sample_size_q}
		d_\mathrm{KL}(q_i^*, \widehat{q}_i)\leq C e^{-c_\infty} \frac{m}{k\delta}\\
		\label{eq:proof_sample_moment_diff}
		\norm{\widehat{\boldsymbol{\mu}}_{p_i}-\widehat{\boldsymbol{\mu}}_{q_i}}_2 \leq \frac{1}{2 (m+1) e C}\quad \implies\quad d_\mathrm{KL}(\widehat{p}_i,\widehat{q}_i)\leq e C \norm{\widehat{\boldsymbol{\mu}}_{p_i}-\widehat{\boldsymbol{\mu}}_{q_i}}_2^2
	\end{gather}
	with
	\begin{align*}
		c_\infty &=\max\left\{\norm{\log p_i}_\infty,\norm{\log q_i}_\infty\right\}\\
		c_r &=\max\{\norm{\partial_x^r \log p_i}_{L^2}\norm{\partial_x^r \log q_i}_{L^2}\}
	\end{align*}
	and $\gamma,\xi$ and $C$ are defined as in Lemma~\ref{lemma:convergence_in_H}.
	Since $p,q\in\mathcal{H}_{m,\epsilon}$, Lemma~\ref{lemma:simplification} implies that
	\begin{align*}
		4 e^{4\ga + 1} e^{c_\infty/2} (m+1) \de\leq 1\quad\text{and}\quad C\leq 2 e^{(3m-1)/2}
	\end{align*}
	and by Remark~\ref{remark:simplification_min_sample_size} we may simplify the assumption in Eq.~\eqref{eq:assumption_on_sample_size} and obtain
	\begin{align*}
		4 C^2(m+1)^2 m \delta^{-1} \leq k
	\end{align*}
	as alternative.
	
	Combining the bounds in Eq.~\eqref{eq:proof_sample_size_p}, Eq.~\eqref{eq:proof_sample_size_q} and Eq.~\eqref{eq:proof_sample_moment_diff} with the bound on the $L^1$-difference in Eq.~\eqref{eq:inequ_impl_proof_main_sample_thm}, yields the following statement.
	For every $\delta\in (0,1)$ and all $f\in\mathcal{F}$ the following holds with probability at least $1-\delta$ (over the choice of samples):
	If
	\begin{flalign*}
		4 C^2(m+1)^2 m \delta^{-1} \leq k\\
		\norm{\widehat{\boldsymbol{\mu}}_p-\widehat{\boldsymbol{\mu}}_q}_1 \leq \left(2 (m+1) e C\right)^{-1}
	\end{flalign*}
	then the maximum entropy densities $\widehat{p}$ and $\widehat{q}$ exist and it holds that
	\begin{align}
		\norm{p-q}_{L^1} &\leq \sqrt{2\sum_{i=1}^d d_\mathrm{KL}(\widehat{p}_i, \widehat{q}_i)} + \sqrt{2\sum_{i=1}^d d_\mathrm{KL}(\widehat{p}_i, p^*_i)} + \sqrt{2\sum_{i=1}^d d_\mathrm{KL}(\widehat{q}_i, q^*_i)} + \sqrt{8 \epsilon}\nonumber\\
		&\leq \sqrt{2\sum_{i=1}^d e C \norm{\widehat{\boldsymbol{\mu}}_{p_i}-\widehat{\boldsymbol{\mu}}_{q_i}}_2^2} + 2 \sqrt{2\sum_{i=1}^d C e^{-c_\infty} \frac{m}{k\delta}} + \sqrt{8 \epsilon}\nonumber\\
		&\leq \sqrt{2 e C} \sum_{i=1}^d { \norm{\widehat{\boldsymbol{\mu}}_{p_i}-\widehat{\boldsymbol{\mu}}_{q_i}}_2} + \sqrt{8 C \frac{d m}{\delta k}} e^{-c_\infty/2} + \sqrt{8 \epsilon}\nonumber\\
		\label{eq:sample_bound_improved}
		&\leq \sqrt{2 e C} \norm{\widehat{\boldsymbol{\mu}}_p-\widehat{\boldsymbol{\mu}}_q}_1 + \sqrt{8 C \frac{d m}{k\delta}} + \sqrt{8\epsilon}
	\end{align}
	where the last inequality is due to the fact that $e^{-c_\infty/2}\leq 1$ and the inequality $\norm{\x}_2\leq\norm{\x}_1$.
\end{proof}\\

\section{Discussion}
\label{sec:conclusion}

In this chapter, we formalize the problem of domain adaptation for binary classification under the assumption that finitely many moments of the source and the target distribution are similar.
We show that additional conditions are needed to guarantee a small misclassification risk of discriminative models trained only on source data.
Appropriate conditions on the underlying distributions are presented based on the sample size, the number of moments, the smoothness of the underlying probability densities and the entropy of the densities.
For smooth densities with weakly coupled marginals, our conditions can be made as precise as required by increasing the number of moments or the smoothness of the distributions.
Explicit upper bounds on the misclassification risk are provided.

Our analysis formalizes the following intuition: The more information the similar moments store about the source and the target distribution, the higher is the expected success of training a model only on data from the source distribution.
Moreover, the smoother the distributions are, the less moments are needed.

Although additional conditions on the distributions are needed, the weakness of our moment-based assumptions on the similarity between distributions implies that our results give immediate consequences for most other concepts of similarity.

\newpage

\chapter{Moment-Based Regularization for Domain Adaptation}
\label{chap:cmd_algorithm}

In this thesis, we study domain adaptation problems under weak assumptions on the similarity of distributions.
In Chapter~\ref{chap:learning_bounds}, we formalize this problem based on finitely many differences of moments and propose conditions for the existence of solutions.
In this chapter of the thesis, we study domain adaptation problems beyond these conditions.
We propose a new metric-based regularization strategy which aims at learning new domain-specific data representations that have finitely many moments in common.

As discussed in Section~\ref{sec:original_contribution}, parts of this chapter have already been published.
As a result, extensions of our approach from independent research groups have been developed \eg~for semi-supervised text classification~\cite{peng2018cross}, person re-identification~\cite{ke2018identity}, word segmentation~\cite{xing2018adaptive}, more general problems strongly violating the covariate shift assumption~\cite{peng2019weighted} and it has been combined with other distance measures for higher performance~\cite{Wei2018GenerativeAG}.

This chapter is structured as follows:
Section~\ref{sec:motivation_cmd_algo} motivates our approach and Section~\ref{sec:relation_to_related_works_cmd_algorithm} discusses some relations to the state-of-the-art.
Section~\ref{sec:moment_distance_for_da} proposes our new moment distance which is appropriate for domain adaptation.
Section~\ref{sec:cmd_regularization} shows how to use our metric for regularization of neural networks.
Section~\ref{sec:experimental_evaluations} gives empirical results on large scale datasets together with its discussion.
Section~\ref{sec:proofs_cmd_algo} gives all the proofs of the stated claims and Section~\ref{sec:conclusion_cmd_algo} concludes this chapter.

\section{Motivation and General Idea}
\label{sec:motivation_cmd_algo}

Domain adaptation problems arise in many practical fields.
One important example is sentiment analysis of product reviews~\cite{glorot2011domain} where a model is trained on data of a source product category, \eg~kitchen appliances, and it is tested on data of a related category, \eg~books.
A second example is the training of image classifiers on unlabeled real images by means of nearly-synthetic images that are fully labeled but have a distribution different from the one of the real images~\cite{ganin2016domain}.

In this chapter of the thesis, we approach these problems by following the principle of learning new data transformations as described in Subsection~\ref{subsec:da_algorithms}.
That is, we transform the data in a new space where the domain-specific distributions are similar and learn a classifier on the source transformations.
See Eq.~\eqref{eq:objective_of_principle_of_representation_learning_for_da} and Figure~\ref{fig:grafical_abstract} for illustration.
As motivated in Chapter~\ref{chap:introduction} and Chapter~\ref{chap:learning_bounds}, we model weak assumptions on the similarity of distributions by focusing on moment distances as described in Subsection~\ref{subsec:moment_distances}.
For the model class, we rely on neural networks as described in Section~\ref{sec:neural_networks}.

In addition to the general goal of finding a model with a low misclassification risk, we aim at a robust learning behaviour.
That is, the final models' performance should be insensitive to changes of the regularization parameter needed in the objective in Eq.~\eqref{eq:objective_of_principle_of_representation_learning_for_da} of the principle of learning new feature representations.
The robustness is especially important as the selection of the regularization parameter has to be performed without target labels.

Our idea is to approach both properties, \ie~high performance and robustness, by applying a combination of integral probability metrics~\cite{muller1997integral} on polynomial function spaces as regularizer in the objective of stochastic gradient descent.
See Figure~\ref{fig:da_by_sgd} and Algorithm~\ref{alg:sgd_for_da} for illustration.
Although, the alignment of first and second order polynomial statistics performs well in domain adaptation~\mbox{\cite{tzeng2014deep,sun2016deep}} and generative modeling~\cite{mroueh2017mcgan}, higher order polynomials have not been considered before.
Possible reasons are instability issues that arise in the application of higher order polynomials.
We approach these issues by modifying an integral probability metric such that it becomes less translation-sensitive on a polynomial function space.
We call the new probability metric the \textit{central moment discrepancy} (CMD).

The CMD is a moment distance as described in Subsection~\ref{subsec:moment_distances}.
It has an intuitive representation in the dual space as the sum of differences of higher order central moments of the corresponding distributions.
We provide a strictly decreasing upper bound for its moment terms.
We give upper and lower bounds of the CMD in terms of other probability metrics.
The upper bounds are in terms of the L\'evy metric and the lower bounds in terms of the total variation distance.
The relation of the CMD to various other probability metrics is derived by supplementing Figure~\ref{fig:metrics} of the relations between probability metrics.
From the theory proposed in Chapter~\ref{chap:learning_bounds}, we derive a bound on the misclassification risk of our approach.

In addition, the classification performance is analyzed on artificial data as well as on benchmark datasets for sentiment analysis of product reviews~\cite{chen2012marginalized}, object recognition~\cite{saenko2010adapting} and digit recognition~\cite{lecun1998mnist,netzer2011reading,ganin2016domain}.
In order to increase the visibility of the effects of the proposed method we refrain from excessive parameter tuning but carry out our experiments with fixed regularization weighting parameter, fixed parameters of the metric, and without tuning the learning rate.
A post-hoc analysis is used to test the sensitivity of our approach to changes of the number-of-moments parameter and changes of the number of hidden nodes.

The experiments indicate that (a) our approach often outperforms related approaches which are based on stronger concepts of similarity and (b) it is not very sensitive to parameter changes.

\section{Related Work}
\label{sec:relation_to_related_works_cmd_algorithm}

The metric-based regularization of neural networks for domain adaptation has been approached by many methods.
Subsection~\ref{subsec:domain_adaptation_by_nns} outlines three main principles which are based on applying the $\mathcal{F}$-divergence, the maximum mean discrepancy based on Gaussian kernels or combining specific neural network architectures with the maximum mean discrepancy.

In contrast to our approach, methods which apply the $\mathcal{F}$-divergence normally train an additional classifier which includes the need for new parameters, additional computation times and validation procedures.
In addition, the reversal of the gradient can cause several theoretical problems~\cite{arjovsky2017towards,eghbal2019mixture} that contribute to instability and saturation during training.
Our approach achieves higher or comparable accuracy on several domain adaptation tasks on benchmark datasets.

Compared to approaches applying the maximum mean discrepancy with Gaussian kernel as regularizer, our approach is sometimes less sensitive to changes of the regularization parameter as discussed in Subsection~\ref{subsec:parameter_sensitivity_experiment}.

Compared to approaches which combine specific neural network architectures with the application of the maximum mean discrepancy with Gaussian kernel, our approach is not restricted to multiple layers or network architectures.
Actually, it can be combined with these ideas.

\section{A Moment Distance for Domain Adaptation}
\label{sec:moment_distance_for_da}

In this section, we describe a new moment distance with a low translation sensitivity that is appropriate as a regularizer for the principle of learning new data representations.

\subsection{Integral Probability Metrics on Polynomial Function Spaces}
\label{subsec:integral_prob_metric_on_polynomial_function_space}

As discussed in Subsection~\ref{subsec:ten_probability_metrics}, one important class of probability metrics are integral probability metrics defined by Eq.~\eqref{eq:integral_probability_metrics}.
Depending on the choice of the function set $\mathcal{F}$ in Eq.~\eqref{eq:integral_probability_metrics}, one might obtain the Wasserstein distance, the total variation distance, or the maximum mean discrepancy.
In our approach, we focus on polynomial function spaces.
The expectations of polynomials are sums of moments.
The resulting metrics are therefore moment distances.

In the following let us denote the vector
\begin{align}
	\label{eq:monomial_vector_nuk_for_cmd}
	\boldsymbol{\nu}_j(\x)=(\nu_1(\x),\ldots,\nu_{\zeta(j,d)}(\x))^\text{T}
\end{align}
of all $\zeta(j,d)=\binom{d+j-1}{j}$ monomials $\nu_1,\ldots,\nu_{\zeta(j,d)}\in\mathbb{R}[x_1,\ldots,x_d]$ of total degree $j$ in $d$ variables, \eg
\begin{align}
	\boldsymbol{\nu}_3((x_1,x_2)^\text{T})=(x_1^3, x_1^2 x_2, x_1 x_2^2, x_2^3)^\text{T}.
\end{align}

Further, let us denote by $\mathcal{P}^m$ the class of homogeneous polynomials ${g:\mathbb{R}^d\rightarrow\mathbb{R}}$ of degree $m$ with normalized coefficient vector, \ie
\begin{align}
	\label{eq:polyspace}
	g(\vec x) = \langle \vec{w}, \boldsymbol{\nu}_j(\vec x)\rangle
\end{align}
with $\norm{\vec w}_2 \leq 1$ for a real vector $\vec{w}\in\mathbb{R}^{\zeta(m,d)}$.
For example, the expectations of polynomials $g\in\mathcal{P}^3$ \wrt~a probability density function $p\in\MR$ are linear combinations of the third raw moments of $p$, \ie
\begin{align}
	\begin{split}
		\int_{\mathbb{R}^d} g(\x) p(\x) \diff\x = w_1 \int_{\mathbb{R}^d} x_1^3\, &p(\x) \diff\x+ w_2 \int_{\mathbb{R}^d} x_1^2 x_2\, p(\x) \diff\x\\
		&+ w_3 \int_{\mathbb{R}^d} x_1 x_2^2\, p(\x) \diff\x+ w_4 \int_{\mathbb{R}^d} x_2^3\, p(\x) \diff\x,
	\end{split}
\end{align}
with $\sqrt{w_1^2+w_2^2+w_3^2+w_4^2}\leq 1$.

It is interesting to point out that the space of polynomials $\mathcal{P}^m$ in Eq.~\eqref{eq:polyspace} is the unit ball of a reproducing kernel Hilbert space as described in Subsection~\ref{subsec:ten_probability_metrics}.

\subsection{Problem of Mean Over-Penalization}
\label{subsec:mean_over_penalization_problem}

{Unfortunately, an integral probability metric in Eq.~\eqref{eq:integral_probability_metrics} based on the function space $\Pk^m$ in Eq.~\eqref{eq:polyspace} and different other metrics~\mbox{\cite{mroueh2017mcgan,li2017mmd}} suffer from the drawback of mean over-penalization which becomes worse with increasing polynomial order.}

For the sake of illustration, let us consider two probability density functions $p,q\in\mathcal{M}(\mathbb{R})$.
For $m=1$ we obtain
\begin{align}
	\label{eq:muDiff}
	d_{\mathcal{P}^1}(p,q)  = \sup_{|w|\leq 1} \left|\int w x\,p\diff x - \int w x\, q\diff x\right| = |\mu_p - \mu_q|,
\end{align}
where $\mu_p = \int xp$ and $\mu_q = \int x q$. 
Now, let us consider higher orders $m \in \mathbb{N}$.
Assume that the densities $p$ and $q$ have identical central moments
$c_j(p) := \int (x - \mu_p)^j p$ for $j\geq 2$ but different means $\mu_p  \neq \mu_q$.
By expressing the raw moment $\int x^m\,p$ by central moments 
$c_j(p)$, we obtain, by means of the Binomial theorem,
\begin{align}
	\label{eq:meansensitive}
	d_{\mathcal{P}^m}(p,q) =  \left|\int x^m\,p(x)\diff x-\int x^m\,q(x)\diff x\right|= \left|\sum_{j=0}^m\binom{m}{j} c_j(p) \left(\mu_p^{m-j} - {\mu_q}^{m-j}\right)\right|.
\end{align}
Since the mean values contribute to the sum of Eq.~\eqref{eq:meansensitive} by its powers, the metric in Eq.~\eqref{eq:integral_probability_metrics} with polynomials as function set is not translation invariant.
Much worse, consider for example $\mu_p = 1 +\varepsilon/2$ and $\mu_q = 1 -\varepsilon/2$.
Then small changes of the mean values can lead to large deviations in the resulting metric, i.e.~can cause instability in the learning process.

For another example consider Figure~\ref{fig:problem}.
Different raw moment based metrics consider the source Beta distribution (dashed) to be more similar to the Normal distribution on the left (solid) than to the slightly shifted Beta distribution on the right (solid).
This is especially the case for the integral probability metrics in Eq.~\eqref{eq:integral_probability_metrics} with the polynomial spaces $\Pk^1$, $\Pk^2$ and $\Pk^4$, the maximum mean discrepancy with the standard polynomial kernel ${\kappa(x,y):=(1+\langle x,y\rangle)^2}$ and the quartic kernel ${\kappa(x,y):=(1+\langle x,y\rangle)^4}$~\cite{gretton2006kernel,li2017mmd}, and the integral probability metrics in~\cite{mroueh2017mcgan}.
See Subsection~\ref{subsec:proof_mean_overpenalization_problem} for the proof.

In this work, we propose a metric that considers the distributions on the right to be more similar.

\begin{figure}[ht]
	\makebox[\linewidth][c]{%
		\begin{subfigure}[b]{.40\textwidth}
			\centering
			\includegraphics[width=.95\textwidth]{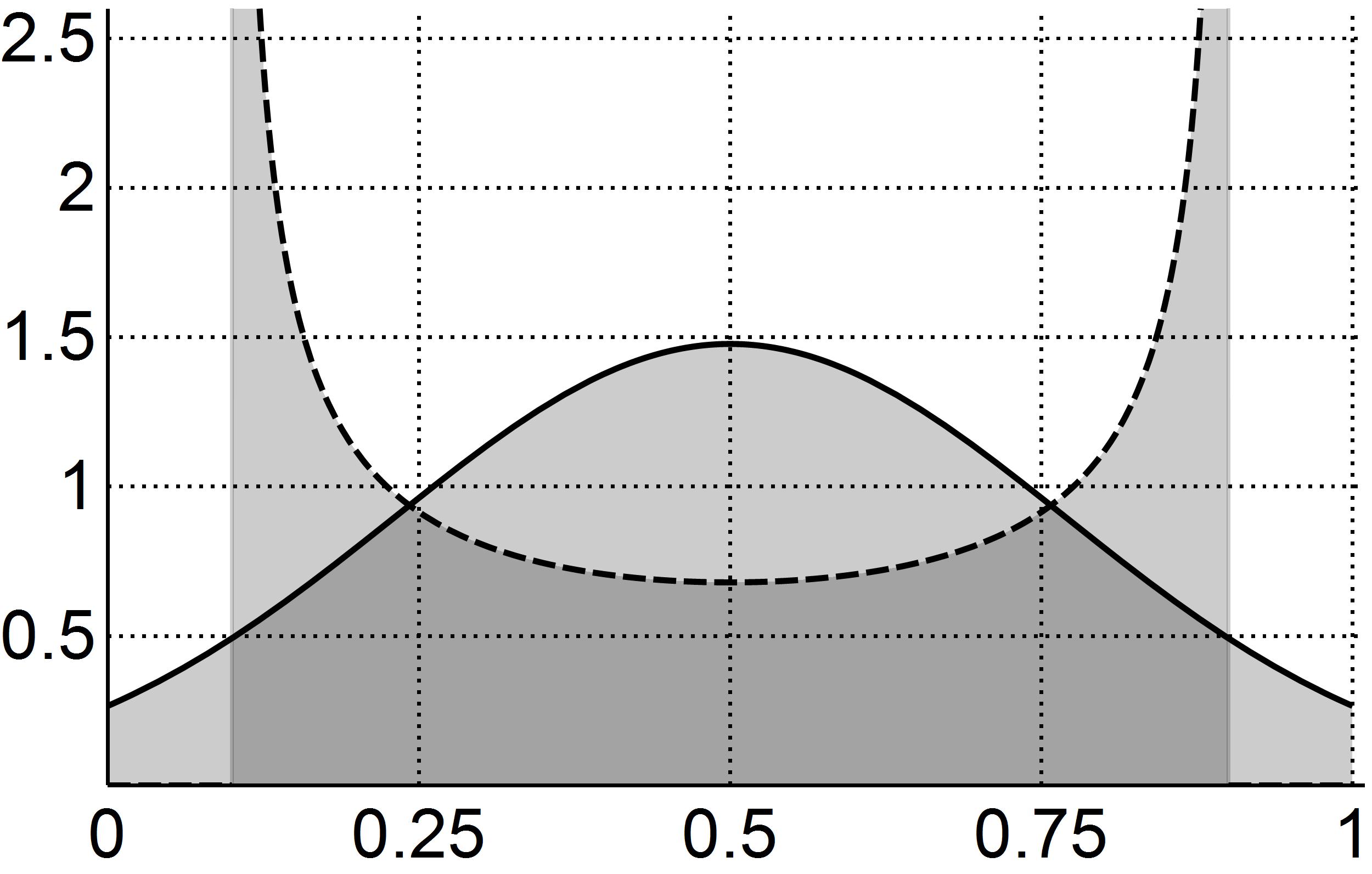}
		\end{subfigure}%
		~~~~~~~~
		\begin{subfigure}[b]{.40\textwidth}
			\centering
			\includegraphics[width=0.95\textwidth]{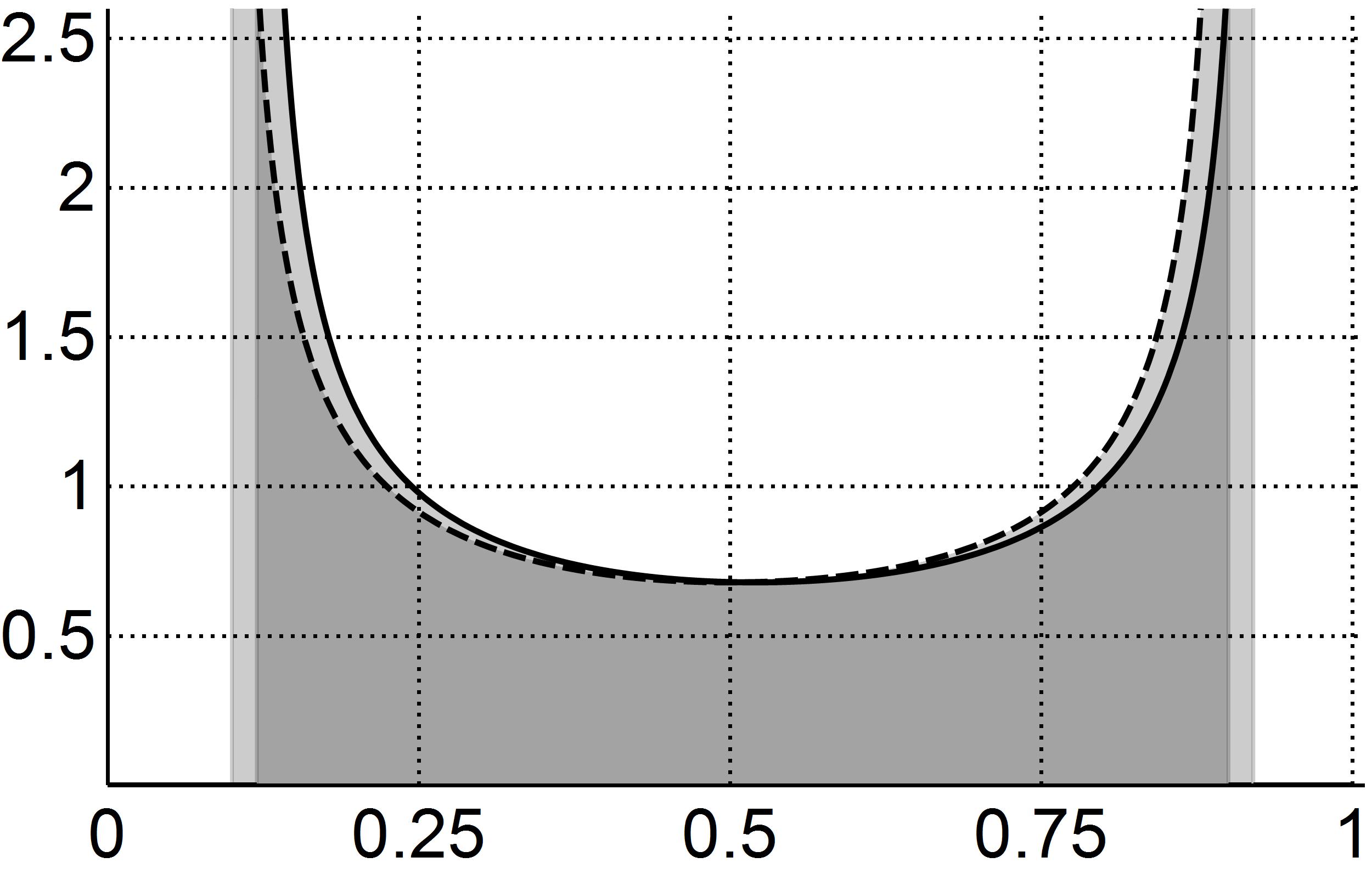}
		\end{subfigure}%
	}
	\caption[Illustrative example of the problem of mean over-penalization.]{
		Illustrative example of the problem of mean over-penalization.
		The maximum mean discrepancy with standard polynomial kernel~\cite{gretton2006kernel} and different other raw moment based metrics~\cite{mroueh2017mcgan,li2017mmd} lead to counter-intuitive distance measurement as they consider the source Beta distribution (dashed) to be more similar to the Normal distribution on the left (solid) than to the slightly shifted Beta distribution on the right (solid).
		The proposed metric considers the distributions on the right to be more similar.
	}
	\label{fig:problem}
\end{figure}

\subsection{The Central Moment Discrepancy}
\label{subsec:cmd}

Eq.~\eqref{eq:meansensitive} motivates us to look for a modified version of integral probability metrics based on polynomial function spaces that are less sensitive to translation.
Therefore, we  propose the following centralized and translation-invariant versions of integral probability metrics on polynomial function spaces.
\Needspace{10\baselineskip}
\begin{restatable}[Centralized Integral Probability Metric]{defrep}{centralizedipm}%
	\label{def:centralized_integral_probability_metric_ipm}%
	We define the polynomial centralized integral probability metric of order $m$ between two probability density functions $p,q\in\MR$ with finite central moments of order $m$ by
	{\small
		\begin{align}
			\label{eq:centralized_integral_probability_metric_ipm}
			d^\text{c}_{\Pk^m}(p,q)=\sup_{g\in\Pk^m}\left| \int_{\mathbb{R}^d} g\!\left(\x-\int_{\mathbb{R}^d} \x p(\x)\diff\x\right) p(\x)\diff\x - \int_{\mathbb{R}^d} g\!\left(\x-\int_{\mathbb{R}^d} \x q(\x)\diff\x\right) q(\x)\diff\x \right|.
		\end{align}}
	\end{restatable}
	
	We now introduce a ``refined'' metric as the weighted sum of polynomial centralized integral probability metrics in Eq.~\eqref{eq:centralized_integral_probability_metric_ipm}.
	\Needspace{20\baselineskip}
	\begin{restatable}[Central Moment Discrepancy]{defrep}{cmd}%
		\label{def:cmd}%
		We define the central moment discrepancy (CMD) of order $m$ between two probability density functions $p,q\in\MR$ with finite central moments up to order $m$ by
		\begin{equation}
			\label{eq:cmd}
			\mathrm{cmd}_m(p,q) = a_1\, d_{\mathcal{P}^1}(p,q) + \sum_{j=2}^m  a_j\, d_{\Pk^j}^\text{c}(p,q),
		\end{equation}
		where $a_j \geq 0$ are weighting factors.
	\end{restatable}
	Note that in Eq.~\eqref{eq:cmd} for $m=1$ we take $d_{\mathcal{P}^1}(p,q) = |\mu_p - \mu_q|$ which still behaves smoothly \wrt~changes of the mean values and is more informative than $d_{\mathcal{P}^1}^\text{c}(p,q) = 0$.
	The lower the value of $a_1$, the less translation sensitive is the CMD.
	
	Probability density functions with compact support are completely determined by their infinite sequence of moments.
	The CMD is therefore a metric on the set of compactly supported distributions for $m=\infty$.
	However, as a moment distance, the CMD is only a pseudo-metric for $m<\infty$.
	
	The questions of how to compute the metric efficiently, how to appropriately set the weighting values $a_j$ and how the CMD relates to other probability metrics, are discussed in the next Subsection~\ref{subsec:properties_of_cmd}.
	
	\subsection{Properties of The Central Moment Discrepancy}
	\label{subsec:properties_of_cmd}
	
	So far, our approach of defining an appropriate metric, \ie~Eq.~\eqref{eq:cmd}, has been motivated by theoretical considerations starting from Eq.~\eqref{eq:integral_probability_metrics} and our analysis in Subsection~\ref{subsec:mean_over_penalization_problem}.
	However, for practical applications we need to compute our metric in a computationally efficient way.
	The following theorem provides a key.
	See Subsection~\ref{subsec:proof_dual_cmd} for its proof.
	\begin{restatable}[Dual Representation of Central Moment Discrepancy]{thmrep}{thmdualcmd}%
		\label{thm:dual_cmd}%
		By setting $c_1(p)=\int \x p(\x)\diff\x$ and $c_j(p)=\int \boldsymbol{\nu}_j(\x-c_1(\x))\,p(\x)\diff\x$ for $j\geq2$ with the vector $\boldsymbol{\nu}_j$ of monomials as in Eq.~\eqref{eq:monomial_vector_nuk_for_cmd}, 
		we obtain as equivalent representation for the central moment discrepancy:
		\begin{equation}
			\label{eq:dual_cmd}
			\mathrm{cmd}_m(p,q) = \sum_{j=1}^m a_j \norm{c_j(p)-c_j(q)}_2.
		\end{equation}
	\end{restatable}
	In the special case of $m = 2$, the CMD is the weighted sum between the maximum mean discrepancy with linear kernel and the Frobenius norm of the difference between the covariance matrices which allows to interpret the CMD as an extension to correlation alignment approaches~\cite{sun2016deep,sun2016return} and linear kernel based maximum mean discrepancy approaches~\mbox{\cite{tzeng2014deep,csurka2016unsupervised}}.
	
	So far, our analysis has been mainly theoretically motivated.
	In practice, not all cross-moments are always needed.
	Our experiments in Section~\ref{sec:experimental_evaluations} show that reducing the monomial vector in Eq.~\eqref{eq:monomial_vector_nuk_for_cmd} to
	\begin{equation}
		\label{eq:monomial_vector_marginal}
		\boldsymbol{\nu}_j(\vec x) = \left(x_1^j,\ldots,x_d^j\right)^\text{T}
	\end{equation}
	can lead already to better results compared to related approaches while computational efficiency is improved.
	Focusing on monomial vectors as in Eq.~\eqref{eq:monomial_vector_marginal} is consistent with the theoretical results proposed in Chapter~\ref{chap:learning_bounds} which are based on similar assumptions to overcome a number of polynomial terms which increases exponentially with dimension.
	
	The next practical aspect we must address is how to set the weighting factors $a_j$ in Eq.~\eqref{eq:dual_cmd} such that the terms of the sum do not increase too much.
	For distributions with compact support $[a,b]^d$, the following Lemma~\ref{lemma:decreasing_cmd_terms} provides us with suitable weighting factors, namely $a_j:=1/{|b-a|^j}$.
	See Subsection~\ref{subsec:proof_decreasing_terms} for a proof.
	\Needspace{10\baselineskip}
	\begin{restatable}[Decreasing Upper Bound]{lemmarep}{lemmadecreasingterms}%
		\label{lemma:decreasing_cmd_terms}%
		Let $p,q\in\mathcal{M}\left([a,b]^d\right)$ with finite mean vector ${c_1(p)=\int \x p(\x)\diff\x}$, central moment vector $c_j(p)=\int \boldsymbol{\nu}_j(\x-c_1(\x))\,p(\x)\diff\x$ for $j\geq2$ and the vector $\boldsymbol{\nu}_j$ of monomials as in Eq.~\eqref{eq:monomial_vector_marginal}.
		Then the following holds:
		\begin{align}
			\begin{split}
				\frac{1}{|b-a|^j}&\norm{c_j(p)-c_j(q)}_2\leq 2 \sqrt{d} \left(\frac{1}{j+1}\left(\frac{j}{j+1}\right)^j+\frac{1}{2^{1+j}}\right).
			\end{split}
		\end{align}
	\end{restatable}
	
	\subsection{Relation to Other Probability Metrics}
	\label{subsec:relation_of_cmd_to_other_probability_metrics}
	
	In the one dimensional case, the CMD can be upper bounded by the L\'evy metric.
	\needspace{10\baselineskip}
	\begin{restatable}[Upper Bound by L\'evy Metric]{correp}{cmdlevybound}%
		\label{cor:cmd_levy_bound}%
		Let $p,q\in\mathcal{M}([0,1])$, $m\in\mathbb{N}$ and the CMD as in Definition~\ref{def:cmd} with $a_i=\ldots=a_m=1$. Then there exist constants $C_\mathrm{L},M_\mathrm{L}\in\mathbb{R}_+$ such that
		\begin{align}
			\label{eq:cmd_levy_inequality}
			d_\text{L}(p,q)\leq M_L\quad\implies\quad\mathrm{cmd}_m(p,q)\leq C_\text{L}\cdot d_\text{L}(p,q)^{\frac{1}{2 m+2}}.
		\end{align}
	\end{restatable}
	Under the assumptions of Chapter~\ref{chap:learning_bounds}, the total variation distance can be upper bounded in terms of the CMD as follows.
	\needspace{10\baselineskip}
	\begin{restatable}[Lower Bound by Total Variation Distance]{correp}{cmdTVbound}%
		\label{cor:cmd_total_variation_bound}%
		Let $m\geq 2$, $\epsilon\geq 0$ and $\mathcal{H}_{m,\epsilon}$ as in Definition~\ref{def:H}. Let further the CMD be as in Defintion~\ref{def:cmd} with $a_1=\ldots=a_m=1$ and central moment vectors as defined in Lemma~\ref{lemma:decreasing_cmd_terms}.
		Then there exists some constant $C_\mathrm{cmd}\in\mathbb{R}_+$ such that for all $p,q\in\mathcal{H}_{m,\epsilon}$ the following holds:
		\begin{gather*}
			\mathrm{cmd}_m(p,q) \leq \frac{1}{4 C_\mathrm{cmd} e^{(3m-1)/2} \left(m+1\right)}\\
			\quad\implies\quad\\
			d_\mathrm{TV}(p,q) \leq e^{(3m-1)/4}\cdot C_\mathrm{cmd}\cdot \mathrm{cmd}_m(p,q) + \sqrt{8 \epsilon}.
		\end{gather*}
	\end{restatable}
	Figure~\ref{fig:metrics_with_cmd} illustrates Corollary~\ref{cor:cmd_levy_bound} and Corollary~\ref{cor:cmd_total_variation_bound} showing some relations between the CMD and other probability metrics.
	It can be seen that the CMD implements a weaker convergence than most other commonly applied probability metrics.
	However, under the assumptions (A1)--(A4) stated in Definition~\ref{def:H}, and sufficiently small $\epsilon$, stronger convergence properties are implemented.
	
	\begin{figure}[t]
		\centering
		\includegraphics[width=\linewidth]{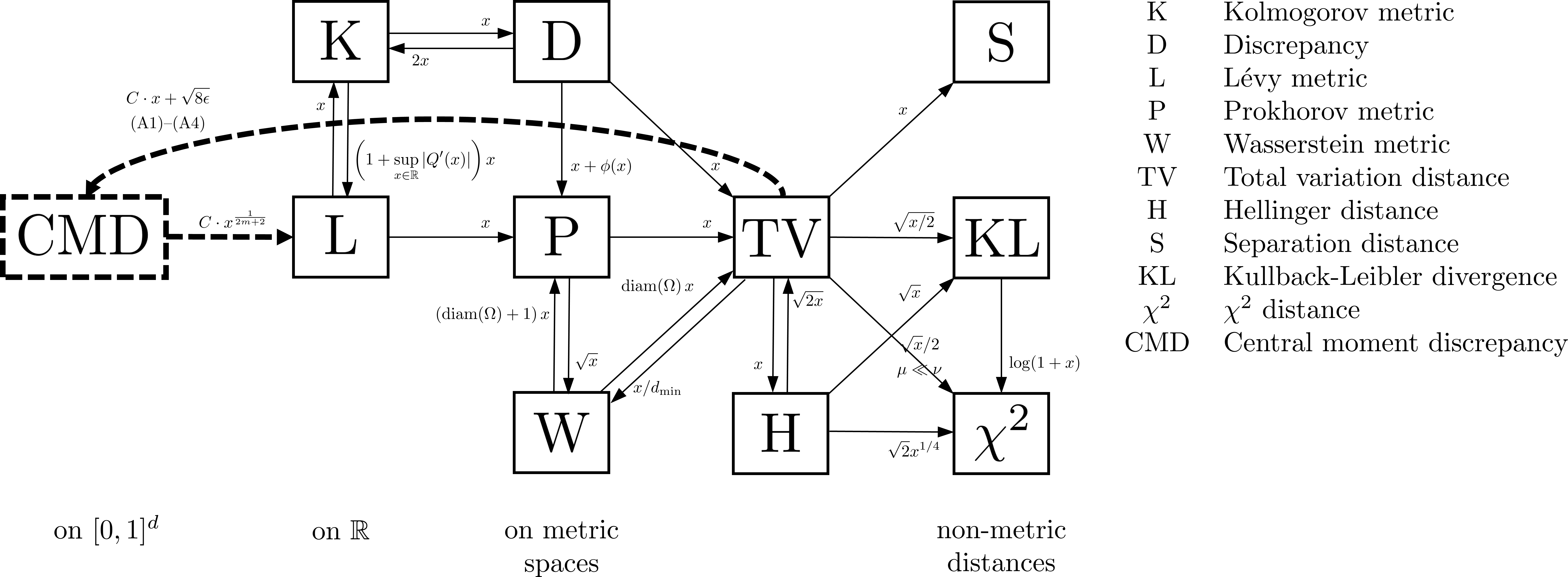}
		\caption[Relationships among probability metrics as illustrated in~\cite{gibbs2002choosing} and supplemented by the Central Moment Discrepancy.]{Relationships among probability metrics as illustrated in~\cite{gibbs2002choosing} and supplemented by Corollary~\ref{cor:cmd_levy_bound} and Corollary~\ref{cor:cmd_total_variation_bound} (dashed).
			A directed arrow from $\text{A}$ to $\text{B}$ annotated by a function $h(x)$ means that $d_\text{A}\leq h(d_\text{B})$.
			For notations, restrictions and applicability see Section~\ref{sec:probability_metrics}.
		}
		\label{fig:metrics_with_cmd}
	\end{figure}
	
	\section{Regularization for Neural Networks}
	\label{sec:cmd_regularization}
	
	In the following we show how to implement the principle of learning new data representations based on an empirical estimation of the CMD.
	This principle is described in more detail in Subsection~\ref{subsec:da_algorithms} and illustrated in Figure~\ref{fig:grafical_abstract}.
	
	In the following setting of unsupervised domain adaptation under covariate-shift we consider two domains $(p,l)$ and $(q,l)$ with $p,q\in\MR$ and $l:\Omega\to [0,1]$ with $\Omega$ being an open subset of $\mathbb{R}^d$.
	Given a \mbox{$k$-sized} source sample $X_p$ drawn from $p$ with labels $l(X_p)$ and a \mbox{$k$-sized} target sample $X_q$ drawn from $q$, the goal is to find two functions $g:\Omega\to [0,1]^s$ and $f:[0,1]^s\to\{0,1\}$ from two classes $\mathcal{G}$ and $\mathcal{F}$, respectively, such that the target risk
	\begin{align}
		\int_{[0,1]^s}\left|f(g(\x))-l(\x)\right| q(\x)\diff\x 
	\end{align}
	is small.
	In the principle of learning domain-invariant data representations this is often done by minimizing an approximation of the following objective function:
	\begin{align}
		\label{eq:objective_of_principle_of_representation_learning_for_da2}
		\frac{1}{k}\sum_{\x\in X_p}\left|f(g(\x))-l(\x)\right| + \lambda\cdot \hat{d}\left(g(X_p),g(X_q)\right)
	\end{align}
	where $\lambda>0$ is a parameter and $\hat d$ is a distance between the source and target sample representation $g(X_p)$ and $g(X_q)$.
	
	We propose to use the following estimation of the CMD for the distance $\hat d$:
	\begin{align}
		\label{eq:cmd_estimate}
		\mathrm{cmd}_m(X_p,X_q)=\sum_{j=1}^m a_j \norm{\widehat{c}_j(p)-\widehat{c}_j(q)}_2,
	\end{align}
	where $\widehat{c}_1(p)=\frac{1}{k}\sum_{\x\in X_p} \boldsymbol{\nu}_j(\x)$ is the mean of $X_p$ and $\widehat{c}_j(p)=\frac{1}{k}\sum_{\x\in X_p} \boldsymbol{\nu}_j(\x-c_1(p))$ for $j\in\{1,\ldots,m\}$ are the sampled central moments of $p$ with $\boldsymbol{\nu}_j$ as in Eq.~\eqref{eq:monomial_vector_marginal}.
	
	Note that the sampled moment $\widehat{c}_j(p)$ converges to the true central moment $c_j(p)$ as defined in Theorem~\ref{thm:dual_cmd} for $k\to\infty$.
	It follows from the continuous mapping theorem~\cite{mann1943stochastic} that the CMD estimate in Eq.~\eqref{eq:cmd_estimate} is a consistent estimator of the CMD.
	However, it is a biased estimate.
	To obtain an unbiased estimate of a moment distance with similar properties as the CMD, one can apply the sample central moments as unbiased estimates of the central moments and use the squared Euclidean norm instead of the Euclidean norm in Eq.~\eqref{eq:dual_cmd} as similarly proposed for the maximum mean discrepancy in~\cite{gretton2012kernel}.
	
	\subsection{Learning Bound}
	\label{subsec:learning_bound_cmd_application}
	
	In the following we give an example application of the learning bound proposed in Theorem~\ref{thm:problem_solution} to the method described above.
	Our example is based on a function class $\mathcal{F}$ with finite VC-dimension $\vc$ and the function class
	\begin{align}
		\label{eq:G}
		\mathcal{G} &=\{g\in C^r(\Omega,[0,1]^s)\mid r\geq d-s+1, \mathrm{rank}\, \mathbf{J}_g=d\,\text{a.e.}\}
	\end{align}
	where $\Omega\subseteq\mathbb{R}^d$ is an open set, $C^r(\Omega,[0,1]^s)$ refers to the set of functions $g:\Omega\to[0,1]^s$ with continuous derivatives up to order $r$, $\mathrm{rank}\, \mathbf{J}_g$ refers to the rank of the Jacobian matrix $\vec J_g$ of the function $g$ and $a.\,e.$~abbreviates {\it almost everywhere}.
	This definition of $\mathcal{G}$ together with the openness of $\Omega$ ensures that the pushforward measures $\mu\circ g^{-1}$ and $\nu\circ g^{-1}$ of two Borel probability measures $\mu$ and $\nu$ with densities $p$ and $q$, respectively, have probability densities $\tilde{p}$ and $\tilde{q}$, respectively~\cite{ponomarev1987submersions}.
	
	Consider some $\epsilon\geq 0$ and the maximum order of moments be $m=5$.
	The moment order $m=5$ is appropriate in many practical tasks as shown in~\cite{peng2018cross,ke2018identity,xing2018adaptive,peng2019weighted,Wei2018GenerativeAG} and Section~\ref{sec:experimental_evaluations}.
	Let us further denote by
	\begin{align*}
		\phim=\left(\eta_1(x_1),\ldots,\eta_5(x_1),\eta_1(x_2),\ldots,\eta_5(x_2),\ldots,\eta_1(x_s),\ldots,\eta_1(x_s),\ldots,\eta_5(x_s)\right)^\text{T}
	\end{align*}
	the vector of polynomials such that
	\begin{align*}
		&\eta_1(x)=\sqrt{3} (2 x-1)\\
		&\eta_2(x)=\sqrt{5} \left(6 x^2-6 x+1\right)\\
		&\eta_3(x)=\sqrt{7} \left(20 x^3-30 x^2+12 x-1\right)\\
		&\eta_4(x)=3 \left(70 x^4-140 x^3+90 x^2-20 x+1\right)\\ &\eta_5(x)=\sqrt{11} \left(252 x^5-630 x^4+560 x^3-210 x^2+30 x-1\right)
	\end{align*}
	are the orthonormal Legendre polynomials in the variable $x$ up to order $5$.
	Let $g\in\mathcal{G}$ be such that the latent densities fulfill
	\begin{align*}
		h_{\phim}(\tilde p)-h(\tilde p)\leq \epsilon\quad\text{and}\quad h_{\phim}(\tilde q)-h(\tilde q)\leq \epsilon
	\end{align*}
	and have log-density functions $\log \tilde p, \log \tilde q\in W_2^5$ such that
	\begin{align*}
		\norm{\log\tilde p}_\infty\leq 5, \norm{\log \tilde q}_\infty\leq 5\quad\text{and}\quad\norm{\partial^5_{x_i}\log \tilde p_i}\leq 10, \norm{\partial^5_{x_i}\log \tilde q}\leq 10
	\end{align*}
	for all $i\in\{1,\ldots,s\}$.
	
	From Eq.~\eqref{eq:lone_cmd_bound} in the proof of Corollary~\ref{cor:cmd_total_variation_bound} it follows that
	\begin{align*}
		\norm{\widehat{\boldsymbol{\mu}}_p-\widehat{\boldsymbol{\mu}}_q}_1 \leq C_\mathrm{cmd} \cdot \mathrm{cmd}_5(X_p,X_q)
	\end{align*}
	with $C_\mathrm{cmd}=C_5\cdot 5^2\cdot (5+1)\cdot \max_{t\in\{0,1,\ldots,5\}} \left\{\binom{5}{t}\right\} \cdot\sqrt{s}$ and $C_5=\max_{i\in\{1,\ldots,s\}} r_i$, where $r_i=\sum_{t=1}^5 |l_t|$ is the sum of the absolute values of the coefficients $l_t$ of all terms in the orthonormal Legendre polynomials $\eta_1(x_j),\ldots,\eta_5(x_j)$ which contain the monomial $x_j^i$, \ie~$C_5\leq 2331$ and $C_\mathrm{cmd}\leq 3.5\cdot 10^6\cdot \sqrt{s}$.
	
	Following~\cite{ben2007analysis}, we define the labeling functions ${l_p:\mathbb{R}^s\to [0,1]}$ by
	\begin{align*}
		l_p(\vec a)=\frac{\int_{\{\x\mid g(\x) =\vec a\}} l(\x) p(\x) \diff\x }{\int_{\{\x\mid g(\x) =\vec a\}} p(\x)\diff\x}
	\end{align*}
	and $l_q$ analogously.
	Let the sample size $k\geq 6.3\cdot 10^9$ and $\norm{\widehat{\boldsymbol{\mu}}_p-\widehat{\boldsymbol{\mu}}_q}_1\leq 2.3\cdot 10^{-5}$ (or $\mathrm{cmd}_m(X_p,X_q)\leq 6.7\cdot 10^{-12}$).
	Then, by applying Theorem~\ref{thm:problem_solution} on the domains $(\tilde p,l_p)$ and $(\tilde q,l_q)$ with the improved assumptions and constants of Lemma~\ref{lemma:sample_convergence_in_H}, the following holds with probability at least $0.8$:
	\begin{align}
		\label{eq:ex_wo_cmd}
		\begin{split}
			\int \left|f-l_q\right|\tilde q \leq\, &\frac{1}{k}\sum_{\x_\in X_p}\left|f(g(\x))-l(\x)\right| + 
			\sqrt{\frac{4}{k} \left( \vc\log \frac{2 e k}{\vc} + 3 \right)}+ \lambda^*\\
			&+ 84.6 \norm{\widehat{\boldsymbol{\mu}}_p-\widehat{\boldsymbol{\mu}}_q}_1 + 513 \sqrt{\frac{s}{k}} + \sqrt{8\epsilon}\\
			\leq\, &\frac{1}{k}\sum_{\x_\in X_p}\left|f(g(\x))-l(\x)\right| + 
			\sqrt{\frac{4}{k} \left( \vc\log \frac{2 e k}{\vc} + 3 \right)}+ \lambda^*\\
			&+ 2.96\cdot 10^8\cdot \mathrm{cmd}_5({X_p,X_q}) + 513 \sqrt{\frac{s}{k}} + \sqrt{8\epsilon}.
		\end{split}
	\end{align}
	From the ``change of variables'' Theorem~4.1.11 in~\cite{dudley2002real} we obtain
	\begin{align*}
		\int \left| f-l_q \right| \tilde q
		=\int \left| f-l_q \right| \diff(Q\circ g^{-1})
		= \int \left|f-l_q\right|\circ g\diff Q
		= \int \left| f\circ g-l\right| q.
	\end{align*}
	
	In particular, if the dimension of the latent space is taken to be $s=5$, the sample size $k= 6.3\cdot 10^9$ and if the function class $\mathcal{F}$ is the class of neural networks with one layer and activation functions $\mathbbm{1}_{\mathbb{R}_+}$, \ie~$\vc=6$, then the following holds with probability at least $0.8$:
	\begin{align*}
		\begin{split}
			\int \left|f\circ g-l\right| q \leq\, &\frac{1}{k}\sum_{\x_\in X_p}\left|f(g(\x))-l(\x)\right| + 2.96\cdot 10^8\cdot \mathrm{cmd}_5({X_p,X_q}) + 0.0148 + \sqrt{8\epsilon} + \lambda^*,
		\end{split}
	\end{align*}
	where the sampling error originating from the application of statistical learning theory is approximately $2.95\cdot 10^{-4}$ and the sampling error originating from our analysis is approximately $1.44\cdot 10^{-2}$.
	
	\subsection{Algorithm}
	\label{subsec:algo_for_nn_with_cmd}
	
	A concrete implementation of the principle of learning new data representations for unsupervised domain adaptation for multi-class classification based on the CMD and neural networks is given by Algorithm~\ref{alg:sgd_for_da_with_cmd}.
	See Subsection~\ref{subsec:domain_adaptation_by_nns} for descriptions of the neural network and further notations.
	
	\begin{algorithm}
		\SetAlgoLined
		\KwIn{Source sample $X_p=\{\x_1,\ldots,\x_k\}$ with labels $Y=\{l(\x_1),\ldots,l(\x_k)\}$, target sample $X_q=\{\x_1',\ldots,\x_s'\}$, learning rate $\alpha$, regularization parameter $\lambda$, learning rate weighting $\boldsymbol{\nu}_1,\boldsymbol{\nu}_2,\ldots$ and number of moments $m$
		}
		
		\KwOut{Parameter vector $\vtheta=(({\vec W}_0,{\vec b}_0),({\vec W}_1,{\vec b}_1))\in \left( (\mathbb{R}^{w\times d}\times\mathbb{R}^{w})\times (\mathbb{R}^{c\times w}\times\mathbb{R}^{c})\right)$ such that $f(\x, (({\vec W}_0,{\vec b}_0),({\vec W}_1,{\vec b}_1)))=\mathrm{softmax}\left({\vec W}_1 \cdot\rho({\vec W}_0\cdot \x+{\vec b}_0) + {\vec b}_1)\right)$ with $\rho(\x)=(\mathrm{sigm}(x_1),\ldots,\mathrm{sigm}(x_w))^\text{T}$}~\\
		
		\Init{Initialize parameter vector $\vtheta_0$ randomly and set $i=0$}
		\While{\upshape stopping criteria is not met}{
			\Stepone{Find random submultisets $X_i$ from $X_p$ and $X_i'$ from $X_q$}
			\Steptwo{Calculate the gradient ${\vec w}_i=\nabla_{\vtheta} \left(\frac{1}{|X_i|}\sum_{\x\in X_i} L(f(\x),l(\x)) + \lambda\cdot \mathrm{cmd}_m\left(h_0(X_i),h_0(X_i')\right) \right)$ where $h_0(\x)=\rho({\vec W}_0\cdot\x + {\vec b}_0)$}
			\Stepthree{Update $\vtheta_{i+1}=\vtheta_i - \alpha\cdot \boldsymbol{\nu}_i\odot {\vec w}_i$}
			\Stepfour{$i:=i+1$}
		}
		\caption[Moment-based unsupervised domain adaptation for single-layer neural networks via stochastic gradient descent.]{Moment-based unsupervised domain adaptation for finding a single-layer neural network ${f\in\mathcal{N}_{1,w,c,\mathrm{sigm},\mathrm{softmax}}}$ via stochastic gradient descent.
		}
		\label{alg:sgd_for_da_with_cmd}
	\end{algorithm}
	
	Note that the gradient of the term $\nabla_{\vtheta} \frac{1}{|X_i|} \sum_{\x\in X_i} L(f(\x),l(\x))$ needed in Step~2 of Algorithm~\ref{alg:sgd_for_da_with_cmd} is given in Eq.~\eqref{eq:gradient_for_sgd_algo}.
	The gradient $\nabla_{\vtheta} \mathrm{cmd}_m(h_0(X),h_0(X'))$ \wrt~the parameter vector
	\begin{align}
		\vtheta=(({\vec W}_0,{\vec b}_0),({\vec W}_1,{\vec b}_1))\in \left( (\mathbb{R}^{a_1\times d}\times\mathbb{R}^{a_1})\times (\mathbb{R}^{c\times a_1}\times\mathbb{R}^{c})\right)
	\end{align}
	is given by
	\begin{align}
		\label{eq:gradient_for_sgd_algo_cmd}
		\begin{split}
			\nabla_{\vtheta} &\mathrm{cmd}_m(h_0(X),h_0(X')) =\\
			&= \Big(\big(\nabla_{\vec W_0} \mathrm{cmd}_m(h_0(X),h_0(X')),\nabla_{\vec b_0} \mathrm{cmd}_m(h_0(X),h_0(X'))\big),\big(\boldsymbol{0},\boldsymbol{0}\big)\Big)
		\end{split}
	\end{align}
	with the matrix $\boldsymbol{0}$ having all elements zero which is assumed to have appropriate dimensions and the notation $h_0(X)=\{h_0(\x)\mid \x\in X\}$.
	Let us denote the mean of a sample $X$ by $\E[X] = \frac{1}{|X|}\sum_{\x\in X} \x$ and the sampled central moments by ${\E[\boldsymbol{\nu}_j(X-\E[X])]}$ with the set notations $X-\E[X] = \{\x-\E[X] \mid \x\in X\}$, $\boldsymbol{\nu}_j(X) = \{\boldsymbol{\nu}_j(\x)|\x\in X\}$ and the vector $\boldsymbol{\nu}_j(\x)$ as defined in Eq.~\eqref{eq:monomial_vector_marginal}.
	
	Let $\odot$ be the coordinate-wise multiplication. Then, by setting
	\begin{align*}
		\boldsymbol{\Gamma}_{j}(X) &=  \boldsymbol{\nu}^{(j)}(h_0(X)-\E[h_0(X)])\\
		\boldsymbol{\Delta}({X,X'}) &= h_0(X)- h_0(X') \\
		\boldsymbol{q}(X) &= h_0(X)\odot (\vec 1 - h_0(X)),
	\end{align*}
	the application of the chain rule gives
	\begin{align*}
		\nabla_{{\vec b}_0}&\mathrm{cmd}_m(h_0(X),h_0(X'))\\
		&= \nabla_{{\vec b}_0} \norm{\E[\boldsymbol{\Delta}(X,X')]}_2
		+\sum_{j=2}^{m} \nabla_{{\vec b}_0} \| \E[\boldsymbol{\Gamma}_{j}(X)] -  \E[\boldsymbol{\Gamma}_{j}(X')]\|_2 \\
		&= \frac{\E[\boldsymbol{\Delta}(X,X')]\odot 
			(\E[\boldsymbol{q}(X)] - \E[\boldsymbol{q}(X')])}{\norm{\E[\boldsymbol{\Delta}(X,X')]}_2}\\
		&\phantom{=} \quad+\sum_{j=2}^{k} \frac{{\E} [\boldsymbol{\Gamma}_{j}(X)] -{\E}[\boldsymbol{\Gamma}_{j}(X')]}
		{\norm{\E[\boldsymbol{\Gamma}_{j}(X)] -\E[\boldsymbol{\Gamma}_{j}(X')]}_2} \odot \left(\E[\nabla_{{\vec b}_0} 
		\boldsymbol{\Gamma}_{j}(X)]  - \E[\nabla_{{\vec b}_0}
		\boldsymbol{\Gamma}_{j}(X')]\right)
	\end{align*}
	and
	$
	\nabla_{{\vec b}_0} \boldsymbol{\Gamma}_{j}(X) = j\cdot \boldsymbol{\Gamma}_{j-1}(X) 
	\odot 
	(\boldsymbol{q}(X) - \E[\boldsymbol{q}(X)])
	$
	which follows from the form of the gradient of the sigmoid function ${\nabla_{\x}\mathrm{sigm}(\x)=\mathrm{sigm}(\x)\odot \left(1-\mathrm{sigm}(\x)\right)}$.
	Analogously, we obtain $\nabla_{{\vec W}_0}\mathrm{cmd}(X,X')$.
	
	\section{Empirical Evaluations}
	\label{sec:experimental_evaluations}
	
	Our experimental evaluations are based on seven datasets, one toy dataset, two benchmark datasets for domain adaptation, {\it Amazon reviews} and {\it Office} and four digit recognition datasets, {\it MNIST}, {\it SVHN}, {\it MNIST-M} and {\it SynthDigits}, described in Subsection~\ref{subsec:datasets}.
	
	Our experiments aim at providing evidence regarding the following aspects: Subsection~\ref{subsec:toy_dataset} on the usefulness of our algorithm for adapting neural networks to synthetically shifted and rotated data, Subsection~\ref{subsec:sentiment_analysis_experiment} on the classification accuracy of the proposed algorithm on the sentiment analysis of product reviews based on the learning of neural networks with a single hidden-layer, Subsection~\ref{subsec:object_recognition_experiment} on the classification accuracy on object recognition tasks based on the learning of pre-trained convolutional neural networks, Subsection~\ref{subsec:digit_recognition_experiment} on the classification accuracy of deep convolutional neural networks trained on raw image data, and, Subsection~\ref{subsec:parameter_sensitivity_experiment} on the accuracy sensitivity regarding changes in the number-of-moments parameter $m$ and changes in the number of hidden nodes.
	
	\subsection{Datasets}
	\label{subsec:datasets}
	
	The following datasets are summarized in Table~\ref{tab:datasets}.
	
	\begin{table}[t]
		\renewcommand{\arraystretch}{1.1}
		\centering
		\begin{tabular}{|c|c||c|c|c|c|}
			\hline
			\bfseries Task & Domain/Dataset & \bfseries Samples & \bfseries Classes & \bfseries Features\\
			\hline\hline
			\multirow{2}{*}{\begin{tabular}{@{}c@{}}Artificial\\example\end{tabular}} & Source & $639$ & $3$ & $2$\\
			\cline{2-5}
			& Target & $639$ & $3$ & $2$\\
			\hline
			\hline
			\multirow{4}{*}{\begin{tabular}{@{}c@{}}Sentiment\\analysis\end{tabular}} & Books (B) & $6465$ & $2$ & $5000$\\
			\cline{2-5}
			& DVDs (D) & $5586$ & $2$ & $5000$\\
			\cline{2-5}
			& Electronics (E) & $7231$ & $2$ & $5000$\\
			\cline{2-5}
			& Kitchen appliances (K) & $7945$ & $2$ & $5000$\\
			\hline
			\hline
			\multirow{3}{*}{\begin{tabular}{@{}c@{}}Object\\recognition\end{tabular}} & Amazon (A) & $2817$ & $31$ & $227\times 227$\\
			\cline{2-5}
			& Webcam (W) & $795$ & $31$ & $227\times 227$\\
			\cline{2-5}
			& DSLR (D) & $498$ & $31$ & $227\times 227$\\
			\hline
			\hline
			\multirow{4}{*}{\begin{tabular}{@{}c@{}}Digit\\recognition\end{tabular}} & SVHN & $99289$ & $10$ & $32\times 32$\\
			\cline{2-5}
			& MNIST & $70000$ & $10$ & $32\times 32$\\
			\cline{2-5}
			& MNIST-M & $59001$ &  $10$ & $32\times 32$\\
			\cline{2-5}
			& SynthDigits & $500000$ &  $10$ & $32\times 32$\\
			\hline
		\end{tabular}
		\caption{Datasets}
		\label{tab:datasets}
	\end{table}
	
	{\bf Toy dataset:}
	In order to analyze the applicability of our algorithm for adapting neural networks to rotated and shifted data, we created a toy dataset illustrated in Figure~\ref{fig:artificial_problem}.
	The source data consists of three classes that are arranged in two-dimensional space.
	Different transformations such as shifts and rotations are applied on all classes to create unlabeled target data.
	
	{\bf Sentiment analysis:}
	To analyze the accuracy of the proposed approach on sentiment analysis of product reviews, we rely on the {\it Amazon reviews} benchmark dataset with the same preprocessing as used by others~\cite{chen2012marginalized,ganin2016domain,louizos2015variational}.
	The dataset contains product reviews of four categories: {\it books} (B), {\it DVDs} (D), {\it electronics} (E) and {\it kitchen appliances} (K). Reviews are encoded in 5000 dimensional feature vectors of bag-of-words unigrams and bigrams with binary labels: $0$ if the product is ranked by $1-3$ stars and $1$ if the product is ranked by $4$ or $5$ stars. From the four categories we obtain twelve domain adaptation tasks where each category serves once as source domain and once as target domain.
	
	{\bf Object recognition:}
	In order to analyze the accuracy of our algorithm on an object recognition task, we perform experiments based on the {\it Office} dataset~\cite{saenko2010adapting}, which contains images from three distinct domains: {\it amazon} (A), {\it webcam} (W) and {\it DSLR} (D).
	This dataset is a standard benchmark dataset for domain adaptation algorithms in computer vision.
	According to the standard protocol~\cite{ganin2016domain,long2015learning}, we downsample and crop the images such that all are of the same size $({227\times 227})$.
	We assess the performance of our method across all six possible transfer tasks.
	
	{\bf Digit recognition:}
	To analyze the accuracy of our algorithm on digit recognition tasks, we rely on domain adaptation between the three digit recognition datasets {\it MNIST}~\cite{lecun1998mnist}, {\it SVHN}~\cite{netzer2011reading}, {\it MNIST-M}~\cite{ganin2016domain} and {\it SynthDigits}~\cite{ganin2016domain}.
	MNIST contains $70000$ black and white digit images, SVHN contains $99289$ images of real world house numbers extracted from Google Street View and MNIST-M contains $59001$ digit images created by using the MNIST images as a binary mask and inverting the images with the colors of a background image. The background images are random crops uniformly sampled from the Berkeley Segmentation Data Set~\cite{arbelaez2011contour}.
	SynthDigits contains $500000$ digit images generated by varying the text, positioning, orientation, background, stroke colors and blur of Windows$^\text{TM}$ fonts.
	According to the standard protocol~\cite{tzeng2017adversarial}, we resize the images $({32\times 32})$.
	We compare our method based on the standard benchmark experiments SVHN$\rightarrow$MNIST and MNIST$\rightarrow$MNIST-M (source$\rightarrow$target).
	The datasets are summarized in Table~\ref{tab:datasets}.
	
	\subsection{Toy Example}
	\label{subsec:toy_dataset}
	
	The toy dataset is described in Section~\ref{subsec:datasets} and visualized in Figure~\ref{fig:artificial_problem}.
	We study the adaptation capability of our algorithm by comparing it to a standard neural network described in Subsection~\ref{subsec:optimization} with $15$ hidden neurons.
	That is, we apply Algorithm~\ref{alg:sgd_for_da_with_cmd} twice, once with $\lambda=0$ and once with $\lambda>0$.
	We refer to the two versions as shallow neural network (shallow NN) and moment alignment neural network (MANN) respectively.
	To start from a similar initial situation, we use the weights of the shallow NN after $\nicefrac{2}{3}$ of the training time as initial weights for the MANN and train the MANN for $\nicefrac{1}{3}$ of the training time of the shallow NN.
	
	The classification accuracy of the shallow NN in the target domain is $86.7\%$ and the accuracy of the MANN is $99.7\%$.
	The decision boundaries of the algorithms are shown in Figure~\ref{fig:artificial_problem}, shallow NN on the left and MANN on the right.
	The shallow NN misclassifies some data points of the "$+$"-class and of the star-class in the target domain (points).
	The MANN clearly adapts the decision boundaries to the target domain and only a small number $0.3\%$ of the points is misclassified.
	We recall that this is the founding idea of the principle of learning new data representations for domain adaptation.
	
	\begin{figure}[ht]
		\makebox[\linewidth][c]{%
			\begin{subfigure}[b]{.45\textwidth}
				\centering
				\includegraphics[width=.95\textwidth]{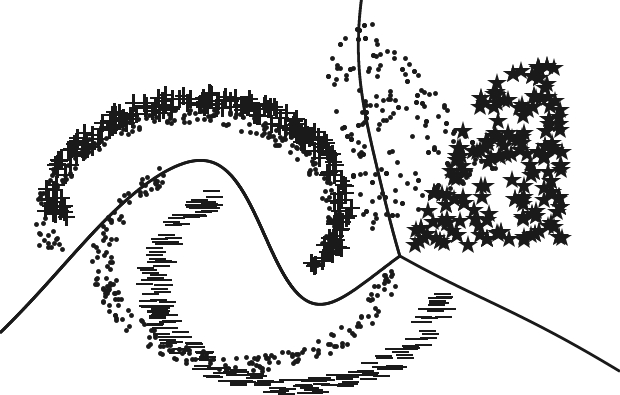}
			\end{subfigure}%
			\begin{subfigure}[b]{.45\textwidth}
				\centering
				\includegraphics[width=0.95\textwidth]{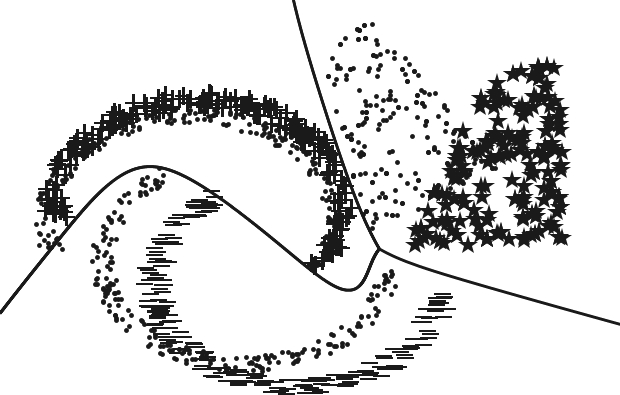}
			\end{subfigure}%
		}
		\caption[Toy example for unsupervised domain adaptation with Algorithm~\ref{alg:sgd_for_da_with_cmd}.]{Toy example for classification with three classes \mbox{("$+$", "$-$" and stars)} in the source domain and unlabeled data in the target domain (points) solved by Algorithm~\ref{alg:sgd_for_da_with_cmd}. Left: without domain adaptation, i.e. without the central moment discrepancy in Step~2; Right: with the proposed approach.}
		\label{fig:artificial_problem}
	\end{figure}
	
	Let us now test the hypothesis that the CMD helps to align the activation distributions of the hidden nodes.
	We measure the significance of a distribution difference by means of the p-value of a two-sided Kolmogorov-Smirnov test for goodness of fit.
	For the shallow NN, $13$ out of $15$ hidden nodes show significantly different distributions, whereas for the MANN only five distribution pairs are considered as being significantly different with p-value lower than $10^{-2}$.
	Kernel density estimates~\cite{fan1994fast} of these five distribution pairs are visualized in Figure~\ref{fig:artificial_activations} (bottom).
	Figure~\ref{fig:artificial_activations} (top) shows kernel density estimates of the distribution pairs corresponding to the five smallest p-values of the shallow NN.
	As the only difference between the two algorithms is the CMD, we conclude that the CMD successfully helps to align the activation distributions in this example.
	
	\begin{figure}[ht]
		\makebox[\linewidth][c]{%
			\begin{subfigure}[b]{.9\textwidth}
				\centering
				\includegraphics[width=.95\textwidth]{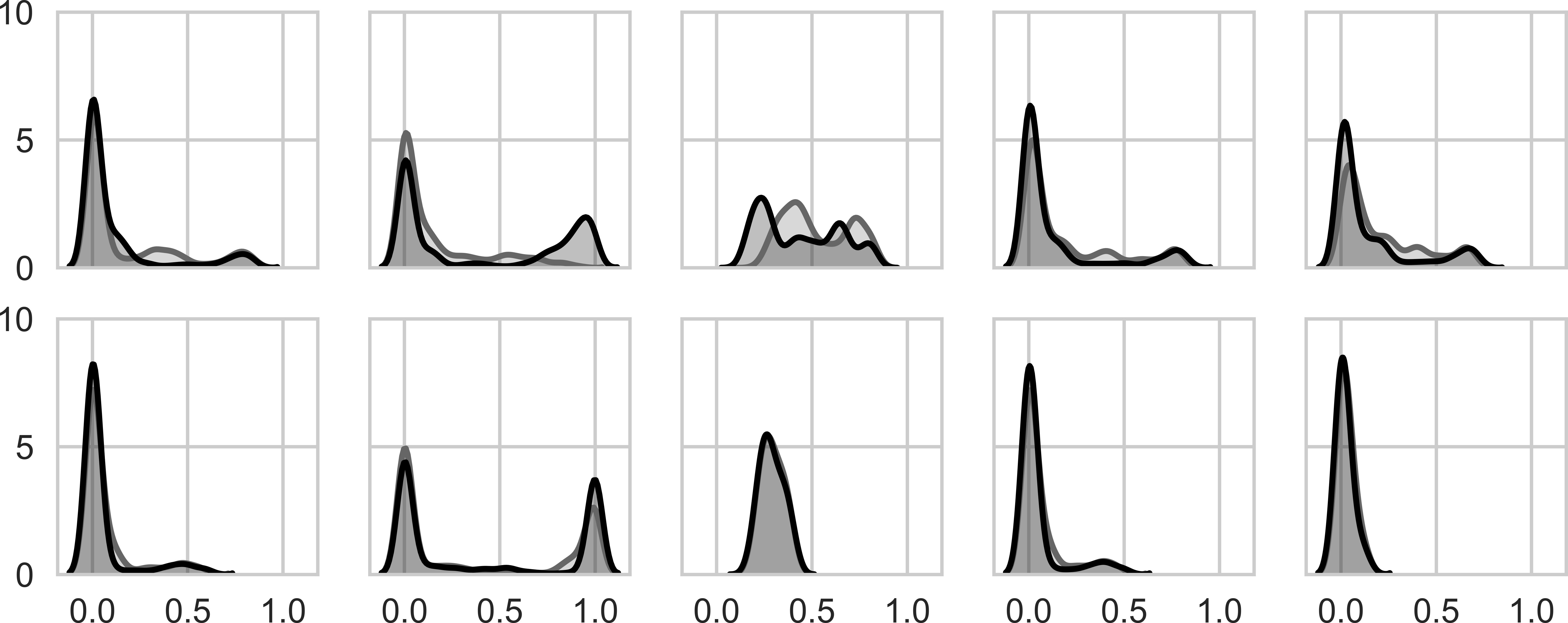}
			\end{subfigure}
		}
		\caption[Illustration of activation distributions before and after the application of Algorithm~\ref{alg:sgd_for_da_with_cmd}.]{Five most different source (dark gray) and target (light gray) activation distributions of the hidden nodes of the neural networks trained by Algorithm~\ref{alg:sgd_for_da_with_cmd} on the toy dataset illustrated in Figure~\ref{fig:artificial_problem} without domain adaptation (top) and with the proposed approach (bottom).}
		\label{fig:artificial_activations}
	\end{figure}
	
	\subsection{Sentiment Analysis of Product Reviews}
	\label{subsec:sentiment_analysis_experiment}
	
	In the following experiment, we compare our method to related approaches based on the single-layer neural network architecture proposed in Subsection~\ref{subsec:optimization}.
	
	We use the Amazon reviews dataset with the same data splits as previous works for every task~\cite{chen2012marginalized,louizos2015variational,ganin2016domain}.
	Thus, we have a labeled source sample of size $2000$ and an unlabeled target sample of size $2000$ for training, and sample sizes between $3000$ and $6000$ for testing.
	
	Since no target labels are available in the unsupervised domain adaptation setting, we cannot select parameters via standard cross-validation procedures.
	Therefore, we apply a variant of the {\it reverse validation} approach~\cite{zhong2010cross} as refined for neural networks~\cite{ganin2016domain}.
	See Subsection~\ref{subsec:domain_adaptation_by_nns} for details.
	
	We report results for the following methods:
	\begin{itemize}
		\item {\it Shallow Neural Network (NN)}: Trained by Algorithm~\ref{alg:sgd_for_da_with_cmd} without domain adaptation, \ie~$\lambda=0$, based on a neural network with $50$ hidden nodes~\cite{ganin2016domain}.
		
		\item {\it Transfer Component Analysis (TCA)}~\cite{pan2011domain}: This kernel learning algorithm tries to learn some transfer components across domains in an reproducing kernel Hilbert space using the maximum mean discrepancy. For competitive classification accuracies, we report results~\cite{li2017prediction} that search the model architecture in a supervised manner by also considering target labels instead of using unsupervised parameter selection.
		The trade-off parameter of the TCA is set to $\mu=0.1$ and the optimal dimension of the subspace is searched for $k\in\{10,20,\ldots,100,500\}$.
		
		\item {\it Domain-Adversarial Neural Networks (DANN)}~\cite{ganin2016domain}: This algorithm is summarized in Subsection~\ref{subsec:domain_adaptation_by_nns}.
		We report the results of the original paper~\cite{ganin2016domain}, where the adaptation weighting parameter $\lambda$ is chosen among $9$ values between $10^{-2}$ and 1 on a logarithmic scale.
		The hidden layer size is either $50$ or $100$ and the learning rate is set to $10^{-3}$.
		
		\item {\it Deep Correlation Alignment (Coral)}~\cite{sun2016deep}: We apply Algorithm~\ref{alg:sgd_for_da} with the CORAL distance function as regularizer $\hat{d}$.
		We use the default parameter $\lambda=1$ as suggested the original paper~\cite{sun2016deep}.
		
		\item {\it Maximum Mean Discrepancy (MMD)}~\cite{gretton2006kernel}: We apply Algorithm~\ref{alg:sgd_for_da} with the maximum mean discrepancy with Gaussian kernel as regularizer $\hat{d}$.
		Parameter $\lambda$ is chosen among $10$ values between $0.1$ and $500$ on a logarithmic scale.
		The Gaussian kernel parameter is chosen among $10$ values between $0.01$ and $10$ on a logarithmic scale.
		
		\item {\it Central Moment Discrepancy (CMD)}:
		In order to increase the visibility of the effects of the proposed method we refrain from hyper parameter tuning but carry out our experiments with the same fixed parameter values of $\lambda$ and $m$ for all experiments.
		The number-of-moments parameter $m$ of the CMD in Eq.\eqref{eq:cmd_estimate} is heuristically set to five, as the first five moments capture rich geometric information about the shape of a distribution and $k=5$ is small enough to be computationally efficient.
		Note that the experiments in Section~\ref{subsec:parameter_sensitivity_experiment} show that similar results are obtained for all $k\in\{4,\ldots,7\}$.
		We use the default parameter $\lambda=1$ to articulates our preference that domain adaptation is equally important as the classification accuracy in the source domain.
	\end{itemize}
	Since we must deal with sparse data, we rely on Adagrad~\cite{duchi2011adaptive} optimization technique described in Subsection~\ref{subsec:optimization}.
	For all evaluations, the default parametrization is used as implemented in the software framework {\it Keras}~\cite{chollet2015keras}.
	We repeat our experiments ten times with different random initializations.
	
	The mean values and average ranks over all tasks are shown in Table~\ref{tab:am_rev_results}.
	Our method outperforms others in average accuracy as well as in average rank in all except one task.
	
	\begin{table}[ht]
		\small
		\renewcommand{\arraystretch}{1.1}
		\centering
		\begin{tabular}{|c||c|c|c|c|c||c|}
			\hline
			Method 	&			NN 	&			DANN~\cite{ganin2016domain} 	&			CORAL~\cite{sun2016deep} 	&			TCA~\cite{pan2011domain} 	&			MMD~\cite{gretton2006kernel} 	&				CMD (ours) 		\\
			\hline
			\hline
			B$\shortrightarrow$D 	&	 $78.7$ 	&	 $78.4$ 	&	 $79.2$ 	&	 $78.9$ 	&	 $\mathit{79.6}$ 	&		 $\mathbf{80.5}$ 		\\
			\hline
			B$\shortrightarrow$E 	&	 $71.4$ 	&	 $73.3$ 	&	 $73.1$ 	&	 $74.2$ 	&	 $75.8$ 	&		 $\mathbf{78.7}$ 		\\
			\hline
			B$\shortrightarrow$K 	&	 $74.5$ 	&	 $77.9$ 	&	 $75.0$ 	&	 $73.9$ 	&	 $\mathit{78.7}$ 	&		 $\mathbf{81.3}$ 		\\
			\hline
			D$\shortrightarrow$B	&	 $74.6$ 	&	 $72.3$ 	&	 $77.6$ 	&	 $77.5$ 	&	 $\mathit{78.0}$ 	&		 $\mathbf{79.5}$ 		\\
			\hline
			D$\shortrightarrow$E	&	 $72.4$ 	&	 $75.4$ 	&	 $74.9$ 	&	 $\mathit{77.5}$ 	&	 $76.6$ 	&		 $\mathbf{79.7}$ 		\\
			\hline
			D$\shortrightarrow$K 	&	 $76.5$ 	&	 $78.3$ 	&	 $79.2$ 	&	 $\mathit{79.6}$ 	&	 $\mathit{79.6}$ 	&		 $\mathbf{83.0}$ 		\\
			\hline
			E$\shortrightarrow$B	&	 $71.1$ 	&	 $71.3$ 	&	 $71.6$ 	&	 $72.7$ 	&	 $\mathit{73.3}$ 	&		 $\mathbf{74.4}$ 		\\
			\hline
			E$\shortrightarrow$D	&	 $71.9$ 	&	 $73.8$ 	&	 $72.4$ 	&	 $\mathit{75.7}$ 	&	 $74.8$ 	&		 $\mathbf{76.3}$ 		\\
			\hline
			E$\shortrightarrow$K 	&	 $84.4$ 	&	 $85.4$ 	&	 $84.5$ 	&	 $\mathbf{86.6}$ 	&	 $85.7$ 	&		 $\mathit{86.0}$ 		\\
			\hline
			K$\shortrightarrow$B	&	 $69.9$ 	&	 $70.9$ 	&	 $73.0$ 	&	 $71.7$ 	&	 $\mathit{74.0}$ 	&		 $\mathbf{75.6}$ 		\\
			\hline
			K$\shortrightarrow$D	&	 $73.4$ 	&	 $74.0$ 	&	 $75.3$ 	&	 $74.1$ 	&	 $\mathit{76.3}$ 	&		 $\mathbf{77.5}$ 		\\
			\hline
			K$\shortrightarrow$E 	&	 $83.3$ 	&	 $84.3$ 	&	 $84.0$ 	&	 $83.5$ 	&	 $\mathit{84.4}$ 	&		 $\mathbf{85.4}$ 		\\
			\hline
			\hline
			Average 	&	 $75.2$ 	&	 $76.3$ 	&	 $76.7$ 	&	 $77.2$ 	&	 $\mathit{78.1}$ 	&		 $\mathbf{79.8}$ 		\\
			\hline
			Average rank	&	 $5.8$	&	 $4.5$	&	 $4.0$	&	 $3.3$	&	 $\mathit{2.3}$	&		 $\mathbf{1.1}$\\
			\hline
		\end{tabular}
		\caption[Classification accuracy on the Amazon reviews dataset for twelve domain adaptation scenarios.]{Classification accuracy on Amazon reviews dataset for twelve domain adaptation scenarios (source$\shortrightarrow$target).}
		\label{tab:am_rev_results}
	\end{table}
	
	\subsection{Object Recognition}
	\label{subsec:object_recognition_experiment}
	
	In the following experiments we investigate our approach based on the learning of deep features which are created as an intermediate layer output of a convolutional neural network that is pre-trained on a larger related dataset.
	We aim at a robust approach, i.e. we try to find a balance between a low number of parameters and a high accuracy.
	
	Since the Office dataset is rather small (with only $2817$ images in its largest domain), we employ the pre-trained convolutional neural network {\it AlexNet}~\cite{krizhevsky2012imagenet}.
	We follow the standard training protocol for this dataset and use the fully labeled source sample and the unlabeled target sample for training~\cite{long2015learning,ganin2016domain,sun2016deep,long2016unsupervised,long2016joint} and the target labels for testing.
	Using this "fully-transductive" protocol, we compare the proposed approach to the most related distribution alignment methods as described in Section~\ref{subsec:sentiment_analysis_experiment}.
	For a fair comparison we report original results of works that only align the distributions of a single neural network layer of the AlexNet.
	
	We compare our algorithm to the following approaches:
	\begin{itemize}
		\item {\it Convolutional Neural Network (CNN)}~\cite{krizhevsky2012imagenet}: We apply Algorithm~\ref{alg:sgd} without domain adaptation to the network architecture of Subsection~\ref{subsec:optimization} on top of the output of the layer called {\it fc7} of AlexNet.
		We use a hidden layer size of $256$~\cite{tzeng2014deep,ganin2016domain}.
		Following~\cite{sun2016deep,ganin2016domain,long2015learning}, we randomly crop and mirror the images, ensure a balanced source batch and optimize via stochastic gradient descent with a momentum term of $0.9$ and learning rate decay.
		In order to increase the visibility of the effects of the proposed method we refrain from hyper parameter tuning but carry out our experiments with the Keras~\cite{chollet2015keras} default learning rate and default learning rate decay.
		
		\item {\it Transfer Component Analysis (TCA)}~\cite{pan2011domain}: We report results~\cite{long2016joint} that are based on the output of the {\it fc7} layer of AlexNet with parameters tuned via reverse validation as described in Subsection~\ref{subsec:domain_adaptation_by_nns}.
		
		\item {\it Domain-Adversarial Neural Networks (DANN)}~\cite{ganin2016domain}: The original paper~\cite{ganin2016domain} reports results for the adaptation tasks A$\shortrightarrow$W, D$\shortrightarrow$W and W$\shortrightarrow$D.
		For the rest of the scenarios, we report the results of~\cite{long2016joint}.
		The distribution alignment is based on a $256$-sized layer on top of the $fc7$ layer.
		The images are randomly cropped and mirrored and stochastic gradient descent is applied with a momentum term of $0.9$.
		The learning rate is decreased polynomially and divided by ten for the lower layers.
		It is proposed to decrease the regularization parameter $\lambda$ with exponential order according to a specifically designed $\lambda$-schedule~\cite{ganin2016domain}.
		
		\item {\it Deep Correlation Alignment (CORAL)}~\cite{sun2016deep}: We report the results and parameters of the original paper in which they perform domain adaptation on a $31$-sized layer on top of the $fc7$-layer.
		Stochastic gradient descent is applied with a learning rate of $10^{-3}$, weight decay of $5\cdot 10^{-4}$ and momentum of $0.9$.
		The domain adaptation weighting parameter $\lambda$ is chosen in such a way that "at the end of training the classification loss and the CORAL loss are roughly the same"~\cite{sun2016deep}.
		
		\item {\it Maximum Mean Discrepancy (MMD)}~\cite{gretton2006kernel}: We report the results of Long et al.~\cite{long2015learning} in which the maximum mean discrepancy is applied on top of the $31$-dimensional layer after the $fc7$-layer.
		The domain adaptation weighting parameter $\lambda$ is chosen based on assessing the error of a two-sample classifier according to~\cite{fukumizu2009kernel}.
		A multi-kernel version of the maximum mean discrepancy is used with varying bandwidth of the Gaussian kernel between $2^{-8}\gamma$ and $2^{8}\gamma$ with multiplicative step-size of $\sqrt{2}$.
		Parameter $\gamma$ is chosen as the median pairwise distance on the training data, i.e. the {\it median heuristic}~\cite{gretton2012optimal}.
		The network is trained via stochastic gradient descent with momentum of $0.9$ and polynomial learning rate decay and cross-validated initial learning rate between $10^{-5}$ and $10^{-2}$ with multiplicative step size of $\sqrt{10}$.
		The learning rate is set to zero for the first three layers and for the lower layers it is divided by $10$.
		The images are randomly cropped and mirrored in this approach to stabilize the learning process.
		
		\item {\it Central Moment Discrepancy (CMD)}: The approach of this paper with the same optimization strategy as for CNN, with the number-of-moments parameter $k=5$ and the domain adaptation weight $\lambda=1$ as described in Subsection~\ref{subsec:sentiment_analysis_experiment}.
		
		\item {\it Few Parameter Central Moment Discrepancy (FP-CMD)}:
		This approach aims at a low number of parameters.
		The Adadelta gradient weighting scheme as described in Subsection~\ref{subsec:optimization} is used instead of the momentum in the method above.
		In addition, no data augmentation is applied.
	\end{itemize}
	
	The parameter settings of the neural network based approaches are summarized in Table~\ref{tab:param_settings}.
	\begin{table}[ht]
		\renewcommand{\arraystretch}{1.1}
		\footnotesize
		\centering
		\begin{tabular}{|c||c|c|c||c|c|}
			\hline
			Method 		&			CORAL~\cite{sun2016deep} 	&			DANN~\cite{ganin2016domain} 	&			MMD~\cite{long2015learning} 		&			CMD (ours) 	&			FP-CMD (ours) 	\\
			\hline
			\hline
			\begin{tabular}{@{}c@{}}Adaptation \\ nodes\end{tabular} 		&	 $\mathbf{31}$ 	&	 $256$ 	&	 $\mathbf{31}$ 		&	 $256$ 	&	 $256$ 	\\
			\hline
			\begin{tabular}{@{}c@{}}Adaptation \\ weight $\lambda$\end{tabular} 		&	 \begin{tabular}{@{}c@{}}manually \\ tuned\end{tabular} 	&	 exp. decay 	&	 class. strategy 		&	 $\mathbf{1.0}$ 	&	 $\mathbf{1.0}$ 	\\
			\hline
			\begin{tabular}{@{}c@{}}Additional \\ hyper-parameters\end{tabular} 		&	 \textbf{no} 	&	 \begin{tabular}{@{}c@{}}additional \\ classifier\end{tabular}	&	 \begin{tabular}{@{}c@{}}range of \\ kernel params\end{tabular} 		&	 $k=5$ 	&	 $k=5$ 	\\
			\hline
			\begin{tabular}{@{}c@{}}Gradient \\ weighting $\eta$\end{tabular} 		&	 momentum 	&	 momentum 	&	 momentum 		&	 momentum 	&	 adadelta 	\\
			\hline
			\begin{tabular}{@{}c@{}}Learn.  rate\end{tabular} 		&	 $10^{-3}$ 	&	 $10^{-3}$ 	&	 cv 		&	 \textbf{default} 	&	 \textbf{no} 	\\
			\hline
			\begin{tabular}{@{}c@{}}Learn. rate \\ decay parameter\end{tabular} 		&	 \textbf{no} 	&	 yes 	&	 yes 		&	 \textbf{default} 	&	 \textbf{default} 	\\
			\hline
			\begin{tabular}{@{}c@{}}Data \\ augmentation\end{tabular} 		&	 yes 	&	 yes 	&	 yes 		&	 yes 	&	 \textbf{no} 	\\
			\hline
			\begin{tabular}{@{}c@{}}Weight  decay\end{tabular}		&	 yes 	&	 \textbf{no} 	&	 \textbf{no} 		&	 \textbf{no}  	&	 \textbf{no} 	\\
			\hline
		\end{tabular}
		\caption[Parameter settings of several state-of-the-art neural network approaches as applied on the Office dataset.]{Summary of parameter settings of state-of-the-art neural network approaches as applied on the Office dataset. Bold numbers indicate preferable settings.}
		\label{tab:param_settings}
	\end{table}
	We repeated all evaluation five times with different random initializations and report the average accuracies and average ranks over all tasks in Table~\ref{tab:office_results}.
	\begin{table}[!ht]
		\renewcommand{\arraystretch}{1.1}
		\small
		\centering
		\begin{tabular}{|c||c|c|c|c|c|c||c|c|}
			\hline
			Method & A$\shortrightarrow$W & D$\shortrightarrow$W & W$\shortrightarrow$D & A$\shortrightarrow$D& D$\shortrightarrow$A& W$\shortrightarrow$A & Average & Average rank\\
			\hline\hline
			CNN~\cite{krizhevsky2012imagenet} & $52.9$ & $94.7$ & $99.0$ & $62.5$ & $50.2$ & $48.1$ & $67.9$ & $6.3$\\
			\hline
			TCA~\cite{pan2011domain} & $61.0$ & $95.4$ & $95.2$ & $60.8$ & $51.6$ & $50.9$ & $69.2$ & $6.0$\\
			\hline
			MMD~\cite{gretton2006kernel,long2015learning} & $63.8$ & $94.6$ & $98.8$ & $65.8$ & $52.8$ & $\mathit{51.9}$ & $71.3$ & $4.7$\\
			\hline
			CORAL~\cite{sun2016deep} & $\mathit{66.4}$ & $95.7$ & $99.2$ & $66.8$ & $52.8$ & $51.5$ & $72.1$ & $3.2$\\
			\hline
			DANN~\cite{ganin2016domain} & $\mathbf{73.0}$ & $\mathit{96.4}$ & $99.2$ & $\mathbf{72.3}$ & $53.4$ & $51.2$ & $\mathbf{74.3}$ & $\mathit{2.5}$\\
			\hline
			\hline
			CMD (ours) & $62.8$ & $\mathbf{96.7}$ & $\mathit{99.3}$ & $66.0$ & $\mathit{53.6}$ & $\mathit{51.9}$ & $71.7$ & $2.7$\\
			\hline
			FP-CMD (ours) & $64.8$ & $95.4$ & $\mathbf{99.4}$ & $\mathit{67.0}$ & $\mathbf{55.1}$ & $\mathbf{53.5}$ & $\mathit{72.5}$ & $\mathbf{2.0}$\\
			\hline
		\end{tabular}
		\caption[Classification accuracy on the Office dataset for six domain adaptation scenarios.]{Classification accuracy on Office dataset for six domain adaptation scenarios (source$\shortrightarrow$target)}
		\label{tab:office_results}
	\end{table}
	
	Without considering the FP-CMD implementation, the CMD implementation shows the highest accuracy in four of six domain adaptation tasks on this dataset.
	In the last two tasks, the DANN algorithm shows the highest accuracy and also has the highest average accuracy due to these two scenarios.
	
	The FP-CMD implementation shows the highest accuracy in three of six tasks over all approaches and achieves the best average rank.
	In contrast to the other approaches, FP-CMD does so without data mirroring or rotation, no tuned, manually decreasing or cross-validated learning rates, no different learning rates for different layers and no tuning of the domain adaptation weighting parameter $\lambda$.
	
	\subsection{Digit Recognition}
	\label{subsec:digit_recognition_experiment}
	
	In the following domain adaptation experiments {SVHN$\rightarrow$MNIST}, {SynthDigits$\rightarrow$SVHN} and MNIST$\rightarrow$MNIST-M, we analyze the accuracy of our method based on the learning of deep convolutional neural networks on raw image data without using any additional knowledge.
	We use the provided training and test splits of the datasets described in Section~\ref{subsec:datasets}.
	
	In semi-supervised learning research it is often the case that the parameters of deep neural network architectures are specifically tuned for certain datasets~\cite{odena2018realistic} which can cause problems when applying these methods to real-world applications.
	Since our goal is to propose a robust method, we rely on the one architecture for all three digit recognition task.
	The architecture is not specifically developed for high performance of our method but rather independently developed in~\cite{haeusser2017associative}.
	In addition, we fix the learning rate, set the domain adaptation parameters to our default setting and change the activation function of the last layer to be the $\tanh$ function such that the output of the layer is bounded.
	
	We compare our algorithm to the following approaches:
	\begin{itemize}
		\item {\it Deep Convolutional Neural Network (CNN)}: The architecture of~\cite{haeusser2017associative} used by other methods~\cite{bousmalis2016domain,sener2016learning,tzeng2017adversarial,sun2016deep}. Data augmentation is applied.
		
		\item {\it Deep Correlation Alignment (CORAL)}~\cite{sun2016deep}: The same optimization procedure and architecture as of CNN is used.
		The domain adaptation weighting parameter $\lambda$ is chosen in such a way that "at the end of training the classification loss and the CORAL loss are roughly the same"~\cite{sun2016deep}, i.e. $\lambda=1$ as in the original work.
		
		\item {\it Maximum Mean Discrepancy (MMD)}~\cite{gretton2006kernel}: We report the results of Bousmalis et al.~\cite{bousmalis2016domain} in which two separate architectures for each of the two tasks are trained by the Adam optimizer.
		The parameters are tuned according to the procedure reported in~\cite{long2015learning}.
		
		\item {\it Adversarial Discriminative Domain Adaptation (ADDA)~\cite{tzeng2017adversarial}}: We report results of the original paper for the SVHN$\rightarrow$MNIST task.
		
		\item {\it Domain Adversarial Neural Networks (DANN)~\cite{ganin2016domain}}: The results of the original paper are reported. They used stochastic gradient descent with a polynomial decay rate, a momentum term and an exponential learning rate schedule.
		
		\item {\it Domain Separation Networks (DSN)~\cite{bousmalis2016domain}}: We report the results of the original work in which they used the adversarial approach as distance function for the similarity loss. Different architectures are used for both tasks. The hyper-parameters are tuned using a small labeled set from the target domain.
		
		\item {\it Central Moment Discrepancy (CMD)}: The approach of this paper with the same optimization strategy as of CNN, the number-of-moments parameter $k=5$ and the domain adaptation weight $\lambda=1$ as described in Subsection~\ref{subsec:sentiment_analysis_experiment}.
		
		\item {\it Cross-Variance Central Moment Discrepancy (CV-CMD)}: The approach of this paper including the alignment of all cross-variances, i.e. all monomials of order $2$ in Eq.~\eqref{eq:monomial_vector_nuk_for_cmd}.
		The alignment term in the sum of the CMD is divided by $\sqrt{2}$ to compensate for the higher number of second order terms. The parameters $k=5$ and $\lambda=1$ are used as in all other experiments.
	\end{itemize}
	
	\begin{table}[ht]
		\renewcommand{\arraystretch}{1.3}
		\small
		\centering
		\begin{tabular}{|c||c|c|c||c|c|}
			\hline
			Method & $\xrightarrow[\text{MNIST}]{\text{SVHN}}$ & $\xrightarrow[\text{MNIST-M}]{\text{MNIST}}$ & $\xrightarrow[\text{SVHN}]{\text{SynthDigits}}$ & Average & Average rank\\
			\hline\hline
			CNN &  $66.74$ &	$70.85$ &	$80.94$ &	$72.84$ &	$7.3$\\
			\hline
			CORAL~\cite{sun2016deep} & $69.39$ & $77.34$ & $83.58$ & $76.77$ & $5.3$\\
			\hline
			ADDA~\cite{tzeng2017adversarial} & $76.00$ & $-$ &	$-$ &  	$76.00$ & $5.0$\\
			\hline
			MMD~\cite{gretton2006kernel,long2015learning} & $76.90$ & $71.10$ & $88.00$ & $78.67$ & $4.7$\\
			\hline
			DANN~\cite{ganin2016domain} & $76.66$ & $73.85$ & $\it 91.09$ & $80.53$ & $4.3$\\
			\hline
			DSN~\cite{bousmalis2016domain} &  $83.20$ & $82.70$ & $\bf 91.20$ & $\it 85.70$ & $\bf 2.3$\\
			\hline
			\hline
			CMD (ours) &  $\it 84.52$ & $\it 85.04$ & $85.52$ & $85.03$ & $\it 2.7$\\
			\hline
			CV-CMD (ours) &  $\bf 86.34$ & $\bf 88.03$ & $85.42$ & $\bf 86.60$ & $\bf2.3$\\
			\hline
		\end{tabular}
		\caption[Classification accuracy for three domain adaptation scenarios based on four large scale digit datasets.]{Classification accuracy for three domain adaptation scenarios (source$\shortrightarrow$target) based on four large scale digit datasets~\cite{lecun1998mnist,netzer2011reading,ganin2016domain}.}
		\label{tab:digit_results}
	\end{table}
	
	The results are shown in Table~\ref{tab:digit_results}.
	Our method outperforms others in average accuracy as well as in average rank in the tasks SVHN$\rightarrow$MNIST and MNIST$\rightarrow$MNIST-M and performs worse on SynthDigits$\rightarrow$SVHN.
	
	At the SynthDigits$\rightarrow$SVHN task, the $\mathcal{F}$-divergence based approaches, \ie~DANN and DSN, perform better than distance based approaches without adversarial-based implementation, \ie~MMD, CORAL and CMD.
	Note that the performance gain, \ie~the percentage over the baseline, of the best method on the SynthDigits$\rightarrow$SVHN task is rather low with $12.68\%$ compared to the other tasks which show $29.37\%$ and $24.25\%$.
	That is, the methods perform more similar on this task than on the others.
	
	The next section analyzes the accuracy sensitivity \wrt~changes of the hidden layer size and the number-of-moments parameter.
	
	\subsection{Accuracy Sensitivity \texorpdfstring{w.\,r.\,t}.~Parameter Changes}
	\label{subsec:parameter_sensitivity_experiment}
	
	The first sensitivity experiment aims at providing evidence regarding the accuracy sensitivity of the CMD regularizer \wrt~parameter changes of the number-of-moments parameter $m$.
	That is, the contribution of higher terms in the CMD are analyzed.
	The claim is that the accuracy of CMD-based networks does not depend strongly on the choice of $m$ in a range around its default value $5$.
	
	In Figure~\ref{fig:psens} we analyze the classification accuracy of a CMD-based network trained on all tasks of the Amazon reviews experiment.
	We perform a grid search for the number-of-moments parameter $m$ and the regularization parameter $\lambda$.
	We empirically choose a representative stable region for each parameter, $[0.3,3]$ for $\lambda$ and $\{1,\ldots,7\}$ for $m$.
	Since we want to analyze the sensitivity \wrt~$m$, we averaged over the $\lambda$-dimension, resulting in one accuracy value per $m$ for each of the $12$ tasks.
	Each accuracy is transformed into an accuracy ratio value by dividing it by the accuracy of $m=5$.
	Thus, for each $m$ and each task, we get one value representing the ratio between the obtained accuracy and the accuracy of $m=5$.
	The results are shown in Figure~\ref{fig:psens} at the upper left.
	The accuracy ratios between $m=5$ and $m\in\{3,4,6,7\}$ are lower than $0.5\%$, which underpins the claim that the accuracy of CMD-based networks does not depend strongly on the choice of $m$ in a range around its default value $5$.
	For $m=1$ and $m=2$ higher ratio values are obtained. In addition, for these two values many tasks show worse accuracy than obtained by $m\in\{3,4,5,6,7\}$. From this we additionally conclude that higher values of $m$ are preferable to $m=1$ and $m=2$.
	
	The same experimental procedure is performed with maximum mean discrepancy regularization weighted by $\lambda\in[5,45]$ and Gaussian kernel parameter $\beta\in[0.3,1.7]$.
	We calculate the ratio values \wrt~the accuracy of $\beta=1.2$, since this value of $\beta$ shows the highest mean accuracy of all tasks.
	Figure~\ref{fig:psens} on the upper right shows the results.
	The accuracy of the maximum mean discrepancy network is more sensitive to parameter changes than the CMD optimized version.
	Note that the problem of finding the best settings for parameter $\beta$ of the Gaussian kernel is a well known problem~\cite{hsu2003practical}.
	
	The default number of hidden nodes in the sentiment analysis experiments in Subsection~\ref{subsec:sentiment_analysis_experiment} is $50$ to be comparable with other state-of-the-art approaches~\cite{ganin2016domain}.
	The question arises whether the accuracy improvement of the CMD-regularization is robust to changes of the number of hidden nodes.
	
	In order to answer this question we calculate the accuracy ratio between the CMD-based network and the non-regularized network for each task of the Amazon reviews dataset for different numbers of hidden nodes in $\{128,256,384,\ldots,1664\}$.
	For higher numbers of hidden nodes our NN models do not converge with the optimization settings under consideration.
	For the parameters $\lambda$ and $m$ we use our default setting $\lambda=1$ and $m=5$.
	Figure~\ref{fig:psens} on the lower left shows the ratio values on the vertical axis for every number of hidden nodes shown on the horizontal axis and every task represented by different colors.
	The accuracy improvement of the CMD domain regularizer varies between $4\%$ and $6\%$.
	However, no significant accuracy ratio decrease can be observed.
	
	Figure~\ref{fig:psens} shows that our default setting ($\lambda=1, m=5$) can be used independently of the number of hidden nodes for the sentiment analysis task.
	
	The same procedure is performed with the maximum mean discrepancy weighted by parameter $\lambda=9$ and $\beta=1.2$ as these values show the highest classification accuracy for $50$ hidden nodes.
	Figure~\ref{fig:psens} at the lower right shows that the accuracy improvement using the maximum mean discrepancy decreases with increasing number of hidden nodes for this parameter setting.
	That is, for accurate performance of the maximum mean discrepancy, additional parameter tuning procedures for $\lambda$ and $\beta$ need to be performed.
	
	\begin{figure}[ht]
		\centering
		\includegraphics[width=.75\textwidth]{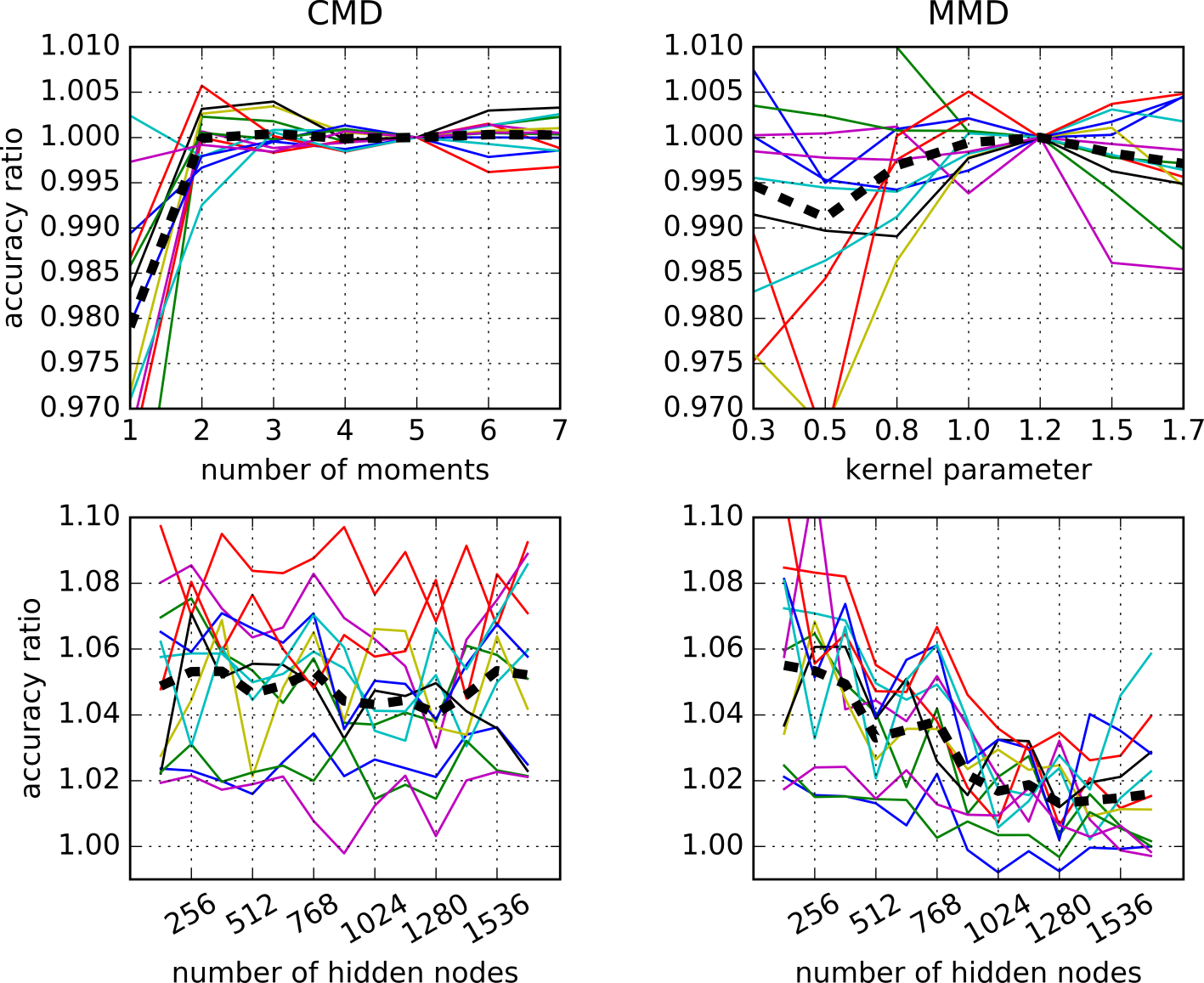}
		\caption[Accuracy sensitivity of Central Moment Discrepancy on the Amazon reviews dataset.]{Sensitivity of classification accuracy \wrt~different parameters of CMD (left) and maximum mean discrepancy (right) on the Amazon reviews dataset.
			The horizontal axes show parameter values and the vertical axes show accuracy ratio values.
			Each line in the plots represents accuracy ratio values for one specific task.
			The ratio values on the upper left are computed \wrt~the default accuracy for CMD ($m=5$) and on the right \wrt~the best obtainable accuracy for maximum mean discrepancy ($\beta=1.2$).
			The ratio values in the lower column are computed \wrt~the accuracies of the networks with the same hidden layer but without domain adaptation.}
		\label{fig:psens}
	\end{figure}
	
	\section{Proofs}
	\label{sec:proofs_cmd_algo}
	
	\subsection{Example of Mean Over-Penalization}
	\label{subsec:proof_mean_overpenalization_problem}
	
	Let the source probability density function $p$ be defined as the density of the random variable ${X_S=0.8\, Y+0.1}$ with $Y$ following a Beta distribution with shape parameters ${\alpha=\beta=0.4}$ (Figure~\ref{fig:problem} dashed). Let the left target distribution $q^{(L)}$ be a Normal distribution with mean $0.5$ and variance $0.27^2$ (Figure~\ref{fig:problem} left) and let the right target distribution $q^{(R)}$ be defined by the random variable $X_T=0.8\cdot Y+0.12$ (Figure~\ref{fig:problem} right). Then
	\begin{align*}
		d_{\Pk^1}(p,q^{(L)}) = \left|\int x\, p(x)\diff x-\int x\, q^{(L)}(x)\diff x\right| = 0 < 0.02 < d_{\Pk^1}(p,q^{(R)}),
	\end{align*}
	and for $\Pk^2$ and $\Pk^4$ it follows
	\begin{align*}
		&d_{\Pk^2}(p,q^{(L)})<0.016<0.02 < d_{\Pk^2}(p,q^{(R)})\\
		&d_{\Pk^4}(p,q^{(L)})<0.02<0.021 < d_{\Pk^4}(p,q^{(R)}).
	\end{align*}
	Let us now consider the maximum mean discrepancy~\cite{gretton2006kernel,li2017mmd} with standard polynomial kernel $\kappa_2(x,y)=(1+x y)^2$.
	According to Lemma~4 in~\cite{gretton2012kernel}, the squared population maximum mean discrepancy $\mathrm{MMD}^2$ is given by
	\begin{align*}
		\text{MMD}^2_{\kappa_2}(p,q^{(L)})
		&= \norm{\left(\int x^2 p, \sqrt{2}\int x p, 1\right)^\text{T}-\left(\int x^2 q, \sqrt{2}\int x q, 1\right)^\text{T}}^2_{\mathcal{H}}\\
		&= 2\,\left|\int x\,p(x)\diff x-\int x\,q^{(L)}(x)\diff x\right|^2+\left|\int x^2\,p(x)\diff x-\int x^2\,q^{(L)}(x)\diff x\right|^2\\
		&< 0.00025 < 0.0012 < \text{MMD}^2_{\kappa_2}(p,q^{(R)}).
	\end{align*}
	Similarly it follows for the quartic kernel $\kappa_4(x,y)=(1+x y)^4$ that
	\begin{align*}
		\text{MMD}^2_{\kappa_4}(p,q^{(L)})<0.004<0.006<\text{MMD}^2_{\kappa_4}(p,q^{(R)}).
	\end{align*}
	The mean and covariance feature matching integral probability metrics in~\cite{mroueh2017mcgan} coincide in our example with the integral probability metrics based on $\Pk^1$ and $\Pk^2$.
	
	Finally, for the CMD in Eq.~\eqref{eq:cmd} with $a_1=a_2=a_3=a_4=1$, we obtain
	\begin{align*}
		\mathrm{cmd}_4(p,q^{(L)})>0.0207>0.02>\mathrm{cmd}_4(p,q^{(R)}).
	\end{align*}
	
	\subsection{Dual Representation}
	\label{subsec:proof_dual_cmd}
	
	\thmdualcmd*
	
	\begin{proof}
		The proof follows from the linearity of the expectation for finite sums and the self-duality of the Euclidean norm.
		It holds that
		\begin{align*}
			&\mathrm{cmd}_k(p,q)\\
			&= a_1\, d_{\mathcal{P}^1}(p,q) + \sum_{j=2}^k  a_j\, d_{\Pk^j}^\text{c}(p,q)\\
			&= a_1\, \sup_{g\in\Pk^1} \left| \int_{\mathbb{R}^d} g(\x)\, p(\x)\diff\x - \int_{\mathbb{R}^d} g(\x)\, q(\x)\diff\x \right| +\\
			&\phantom{=} \quad+\sum_{j=2}^k  a_j\, \sup_{g\in\Pk^k}\left| \int_{\mathbb{R}^d} g\!\left(\x-c_1(p)\right) p(\x)\diff\x - \int_{\mathbb{R}^d} g\!\left(\x-c_1(q)\right) q(\x)\diff\x \right|\\
			&= a_1 \sup_{\norm{\w}\leq 1}\left|\int_{\mathbb{R}^d}\langle\w,\x\rangle\, p(\x)\diff\x -\int_{\mathbb{R}^d}\langle\w,\x\rangle\, q(\x)\diff\x\right|+\\
			&\phantom{=} \quad+ \sum_{j=2}^k a_j \sup_{\norm{\w}\leq 1}\left|\int_{\mathbb{R}^d} \langle\w,\boldsymbol{\nu}_j(\x-c_1(p))\rangle\, p(\x)\diff\x - \int_{\mathbb{R}^d} \langle\w,\boldsymbol{\nu}_j(\x-c_1(q))\rangle\, q(\x)\diff\x\right|\\
			&= a_1 \sup_{\norm{\w}\leq 1}\left|\langle\w,\int_{\mathbb{R}^d} \x p(\x)\diff\x -\int_{\mathbb{R}^d}\x q(\x)\diff\x \rangle\right|+\\
			&\phantom{=} \quad+ \sum_{j=2}^k a_j \sup_{\norm{\w}\leq 1}\left| \langle\w,\int_{\mathbb{R}^d} \boldsymbol{\nu}_j(\x-c_1(p))\, p(\x)\diff\x - \int_{\mathbb{R}^d} \boldsymbol{\nu}_j(\x-c_1(q))\, q(\x)\diff\x \rangle\right|
		\end{align*}
		and finally, by duality, $\mathrm{cmd}_k(p,q)  = \sum_{j=1}^k a_j \norm{c_j(p)-c_j(q)}$.
	\end{proof}\\\\
	
	\subsection{Decreasing Upper Bound}
	\label{subsec:proof_decreasing_terms}
	
	\lemmadecreasingterms*
	
	\begin{proof}
		Let $c_1(p)=\int_{\mathbb{R}^d}\x p(\x)\diff\x$ and $c_j(p)=\int_{\mathbb{R}^d} \boldsymbol{\nu}_j(\x-c_1(p)) p(\x)\diff\x$ for $j\geq 2$ be the central moment vectors of $p$ with $\boldsymbol{\nu}_j$ as defined in Eq.~\eqref{eq:monomial_vector_marginal}.
		Then
		\begin{align*}
			\frac{1}{|b-a|^j}&\|c_j(p)-c_j(q)\|_2 \leq 2 \sqrt{d} \sup_{p\in\mathcal{M}\left([a,b]\right)} \left|\frac{c_j(p)}{(b-a)^j}\right| \leq  2 \sqrt{d} \sup_{p\in\mathcal{M}\left([a,b]\right)} \int \left|\frac{x-\int x p(x)\diff x}{b-a}\right|^j\diff x.
		\end{align*}
		It is shown in~\cite{madansky1959bounds} that if $f:[a,b]\to\mathbb{R}$ is a convex function and $p\in\mathcal{M}\left([a,b]\right)$, then
		\begin{align}
			\label{eq:edmundson_mandasky}
			\int f(x) p(x)\diff x\leq \frac{b-\int x p(x)\diff x}{b-a} f(a) + \frac{\int x p(x)\diff x-a}{b-a} f(b).
		\end{align}
		For the rest of this proof we follow~\cite{egozcue2012smallest}.
		Therefore, we apply Eq.~\eqref{eq:edmundson_mandasky} to the convex function ${x\mapsto |(x-\int x p(x)\diff x)/(b-a)|^j}$ and obtain
		\begin{align*}
			&\int \left|\frac{x-\int x p(x)\diff x}{b-a}\right|^j p(x) \diff x\\
			&\quad\leq \frac{b-\int x p(x)\diff x}{b-a} \cdot\left|\frac{a-\int x p(x)\diff x}{b-a}\right|^j + \frac{\int x p(x)\diff x-a}{b-a} \cdot \left|\frac{b-\int x p(x)\diff x}{b-a}\right|^j
		\end{align*}
		Let us denote by $v=(\int x p(x)\diff x - a)/(b-a)$ and note that $p(x)\in\mathcal{M}([a,b])$ implies $v\in [0,1]$.
		It follows that
		\begin{align*}
			\frac{b-\int x p(x)\diff x}{b-a} &\cdot\left|\frac{a-\int x p(x)\diff x}{b-a}\right|^j + \frac{\int x p(x)\diff x-a}{b-a} \cdot \left|\frac{b-\int x p(x)\diff x}{b-a}\right|^j\\
			&\quad\leq \left( (1-v) v^j + (1-v)^j v\right)\\
			&\quad\leq \max_{x\in[0,1]}\left( (1-x) x^j + (1-x)^j x\right)\\
			&\quad= \max_{x\in[0,1/2]}\left( (1-x) x^j + (1-x)^j x\right)\\
			&\quad\leq \max_{x\in[0,1/2]} (1-x) x^j + \max_{x\in[0,1/2]} (1-x)^j x.
		\end{align*}
		Since $(1-x) x^j$ is increasing in $[0,\frac{1}{2}]$, it holds that
		\begin{align*}
			\max_{x\in[0,1/2]} (1-x) x^j \leq \frac{1}{2^{1+j}},
		\end{align*}
		and since the maximum of $(1-x)^j x$ in the interval $[0,\frac{1}{2}]$ is obtained at $\frac{1}{j+1}$, it follows that
		\begin{align*}
			\max_{x\in[0,1/2]} (1-x)^j x \leq \frac{1}{j+1}\left(\frac{j}{j+1}\right)^j.
		\end{align*}
		Finally, we obtain
		\begin{align*}
			\left|\frac{x-\int x p(x)\diff x}{b-a}\right|^j p(x) \diff x\leq \frac{1}{2^{1+j}} + \frac{1}{j+1}\left(\frac{j}{j+1}\right)^j.
		\end{align*}
	\end{proof}
	
	\needspace{10\baselineskip}
	\subsection{Relation to Other Probability Metrics}
	\label{subsec:proof_relation_to_other_probability metrics}
	
	\Needspace{10\baselineskip}
	\cmdlevybound*
	\begin{proof}
		Let the central moments $c_j(p)$ of $p$ for $j=2,\ldots,m$ as defined in Theorem~\ref{thm:dual_cmd} and denote by $\rho_j=\int x^j p$ and $\nu_j=\int x^j q$ the $j$-th raw moment of $p$ and $q$.
		It follows that
		\begin{align*}
			\mathrm{cmd}_m(p,q) &= \left|\rho_1-\nu_1\right| + \sum_{j=2}^m\left|c_j(p)-c_j(q)\right|\\
			&\leq \left|\rho_1-\nu_1\right| + \sum_{j=2}^m \sum_{i=0}^j \binom{j}{i} \left|\rho_j \rho_1^{j-i}-\nu_j\nu_1^{j-i}\right|\\
			&\leq \left|\rho_1-\nu_1\right| + \sum_{j=2}^m \sum_{i=0}^j \binom{j}{i} \left(\left|\rho_j-\nu_j\right|+\left|\rho_1^{j-i}-\nu_1^{j-i}\right|\right)\\
			&\leq \left|\rho_1-\nu_1\right| + \sum_{j=2}^m \sum_{i=0}^j \binom{j}{i} \left(\left|\rho_j-\nu_j\right|+(j-i) \left|\rho_1-\nu_1\right|\right)\\
			&\leq \left(1+\sum_{j=2}^m\sum_{i=0}^j \binom{j}{i} (j-i)\right) \norm{\int\phim p-\int\phim q}_1
		\end{align*}
		for $\phim=(1,x,x^2,\ldots,x^m)^\text{T}\in\left(\mathbb{R}_m\right)^{m+1}$,
		where the first inequality follows from the Binomial theorem, the second inequality follows from the fact that
		\begin{align*}
			|x_1 y_1-x_2 y_2|\leq |x_1-x_2|+|y_1-y_2|\quad\forall x_1,x_2,y_1,y_2\in[-1,1]
		\end{align*}
		and the third inequality follows from
		\begin{align*}
			|x_1^k-x_2^k|\leq k\cdot |x_1-x_2|\quad\forall x_1,x_2\in[-1,1], k\in\mathbb{N}.
		\end{align*}
		The statement now follows from Lemma~\ref{lemma:moment_distance_bound_by_levy}.
	\end{proof}\\
	\Needspace{15\baselineskip}
	\cmdTVbound*
	\begin{proof}
		Let us define the vector
		\begin{align*}
			\phim=\left(\eta_1(x_1),\ldots,\eta_m(x_1),\eta_1(x_2),\ldots,\eta_m(x_2),\ldots,\eta_1(x_d),\ldots,\eta_1(x_d),\ldots,\eta_m(x_d)\right)^\text{T}
		\end{align*}
		of polynomials such that $1,\eta_1(x_i),\ldots,\eta_m(x_i)$ are the orthonormal Legendre polynomials in the variable $x_i$ up to order $m$.
		
		Denote by $\rho_{i j}=\int x_j^i p$ and by $\nu_{i j}=\int x_j^i q$ the $i$-th raw moments of $p$ and $q$ in the variable $x_j$.
		It follows that
		\begin{align}
			\norm{\boldsymbol{\mu}_p-\boldsymbol{\mu}_q}_1 &= \sum_{j=1}^{d} \sum_{i=1}^{m} \left\lvert \int \eta_i(x_j) p-\int \eta_i(x_j) q\right\rvert\nonumber\\
			&\leq C_m \cdot \sum_{j=1}^{d} \sum_{i=1}^{m}\left\lvert \rho_{i j}-\nu_{i j}\right\rvert\nonumber\\
			&\leq C_m\cdot \sum_{j=1}^{d} \sum_{i=1}^{m} \sum_{t=0}^{i} \binom{i}{t}  \left\lvert \rho_{t j}' \rho_{1 j}^{i-t}-\nu_{t j}' \nu_{1 j}^{i-t}\right\rvert\nonumber\\
			&\leq C_m\cdot \sum_{j=1}^{d} \sum_{i=1}^{m} \sum_{t=0}^{i} \binom{i}{t}
			\left(\left\lvert \rho_{1 j}^{i-t}-\nu_{1 j}^{i-t}\right\rvert+\left\lvert \rho_{t j}'-\nu_{t j}' \right\rvert\right)\nonumber\\
			\nonumber
			&\leq C_m\cdot \sum_{j=1}^{d} \sum_{i=1}^{m} \sum_{t=0}^{i} \binom{i}{t}
			\left(\left(i-t\right) \left\lvert \rho_{1 j}-\nu_{1 j}\right\rvert+\left\lvert \rho_{t j}'-\nu_{t j}' \right\rvert\right)
		\end{align}
		where $C_m=\max_{i\in\{1,\ldots,d\}} r_i$ and $r_i=\sum_{t=1}^m |l_t|$ is the sum of the absolute values of the coefficients $l_t$ of all terms in the orthonormal Legendre polynomials $\eta_1(x_j),\ldots,\eta_m(x_j)$ which contain the monomial $x_j^i$, see \eg~Subsection~\ref{subsec:learning_bound_cmd_application}.
		The term $\rho_{i j}'=\int (x_j-\int x_j p\diff x_j)^i\diff x_j, i\in\mathbb{N}$ denotes the $i$-th central moment of the marginal density $p_j$, especially $\rho_{0 j}'=1$ and $\rho_{1 j}'=0$.
		The terms $\nu_{i j}'$ analogously denote the central moments of the marginals of $q$.
		The second inequality follows from the Binomial theorem, the third inequality follows from the fact that
		\begin{align*}
			|x_1 y_1-x_2 y_2|\leq |x_1-x_2|+|y_1-y_2|\quad\forall x_1,x_2,y_1,y_2\in[-1,1]
		\end{align*}
		and the fourth inequality follows from
		\begin{align*}
			|x_1^k-x_2^k|\leq k\cdot |x_1-x_2|\quad\forall x_1,x_2\in[-1,1], k\in\mathbb{N}.
		\end{align*}
		It further holds that
		\begin{align}
			\norm{\boldsymbol{\mu}_p-\boldsymbol{\mu}_q}_1 &\leq C_m\cdot \sum_{j=1}^{d} \sum_{i=1}^{m} \sum_{t=0}^{i} \binom{i}{t}
			\left(\left(i-t\right) \left\lvert \rho_{1 j}-\nu_{1 j}\right\rvert+\left\lvert \rho_{t j}'-\nu_{t j}' \right\rvert\right)\nonumber\\
			&\leq C_m\cdot \sum_{j=1}^{d} \sum_{i=1}^{m} \sum_{t=0}^{i} i \binom{i}{t}
			\left(\left\lvert \rho_{1 j}-\nu_{1 j}\right\rvert+\left\lvert \rho_{t j}'-\nu_{t j}' \right\rvert\right)\nonumber\\
			&\leq C_m\cdot \sum_{j=1}^{d} \sum_{i=1}^{m} \sum_{t=0}^{m} m \binom{m}{t}
			\left(\left\lvert \rho_{1 j}-\nu_{1 j}\right\rvert+\left\lvert \rho_{t j}'-\nu_{t j}' \right\rvert\right)\nonumber\\
			&\leq C_m\cdot m^2\cdot \max_{t\in\{0,1,\ldots,m\}} \left\{\binom{m}{t}\right\}\cdot \sum_{j=1}^{d} \sum_{t=0}^{m} \left(\left\lvert \rho_{1 j}-\nu_{1 j}\right\rvert+\left\lvert \rho_{t j}'-\nu_{t j}' \right\rvert\right)\nonumber\\
			&\leq C_m\cdot m^2\cdot (m+1)\cdot \max_{t\in\{0,1,\ldots,m\}} \left\{\binom{m}{t}\right\}\cdot \sum_{t=2}^{m} \norm{c_t(p)-c_t(q)}_1\nonumber\\
			\label{eq:lone_cmd_bound}
			&\leq C_m\cdot m^2\cdot (m+1)\cdot \max_{t\in\{0,1,\ldots,m\}} \left\{\binom{m}{t}\right\} \cdot\sqrt{d}\cdot \mathrm{cmd}_m(p,q)
		\end{align}
		where $c_t(p)$ and $c_t(q)$ for $t\in\{1,\ldots,m\}$ are defined as in Theorem~\ref{thm:dual_cmd}.
		
		Note that the elements of $\phim$, together with one, form an orthonormal basis of the linear space $\mathrm{Span}(\mathbb{R}_m[x_1]\cup\ldots\cup\mathbb{R}_m[x_d])$.
		
		Denote $C_\mathrm{cmd}=C_m\cdot m^2\cdot (m+1)\cdot \max_{t\in\{0,1,\ldots,m\}} \left\{\binom{m}{t}\right\} \cdot\sqrt{d}$ and $\tilde{C}=2 e^{(3m-1)/2}$.
		If
		\begin{align*}
			\mathrm{cmd}_m(p,q)\leq \frac{1}{C_\mathrm{cmd} 2 \tilde{C} (m+1)}
		\end{align*}
		then, by applying the inequality proven above, we obtain
		\begin{align*}
			\norm{\boldsymbol{\mu}_p-\boldsymbol{\mu}_q}_1 \leq \frac{1}{2 \tilde{C} (m+1)}.
		\end{align*}
		Theorem~\ref{thm:bound_for_smooth_functions} can be applied and it follows that
		\begin{align*}
			\norm{p-q}_{L^1} &\leq \sqrt{2\tilde{C}}\cdot \norm{\boldsymbol{\mu}_p-\boldsymbol{\mu}_q}_1 + \sqrt{8 \epsilon}\\
			&\leq \sqrt{2\tilde{C}}\cdot C_\mathrm{cmd} \cdot \mathrm{cmd}_m(p,q) + \sqrt{8 \epsilon}.
		\end{align*}
		The statement now follows from Theorem~\ref{thm:TV_distance}.
	\end{proof}\\
	
	\section{Discussion}
	\label{sec:conclusion_cmd_algo}
	
	This chapter proposes a novel approach for unsupervised domain adaptation with neural networks that relies on a metric-based regularization of the learning process.
	The regularization aims at implementing the principle of learning new data representations such that finitely many moments of the domain-specific representations are similar.
	The proposed metric is motivated by instability issues that can arise in the application of integral probability metrics on polynomial function spaces.
	Some relations of the new moment distance to other probability metrics are provided and a bound on the misclassification error of our method is derived.
	To underpin the relevance of our ideas beyond the conditions studied in Chapter~\ref{chap:learning_bounds}, we test our approach on an artificial dataset and $21$ standard benchmark tasks for domain adaptation based on $6$ large scale datasets.
	
	Compared to related approaches, additional assumptions on the distributions are needed to theoretically prove the success of our method.
	However, it turns our that often a lower misclassification error can be achieved compared to related approaches that are based on stronger concepts of similarity.
	In addition, the accuracy of our method is often not very sensitive to changes of the regularization parameter.
	The time complexity of our approach is linear in the number of hidden nodes of the network.
	
	We found that, due to its conceptual simplicity, its solid performance, its low sensitivity \wrt~parameter changes and its low time complexity, our approach serves as a good starting point for further application-specific improvements in domain adaptation applications.

\newpage

\chapter{Industrial Applications}
\label{chap:applications}

In this chapter, we show how our ideas can be applied to construct new algorithms for industrial regression problems.
We discuss two problems arising in two different fields: industrial manufacturing, which we discuss in Section~\ref{sec:cyclical_manufacturing}, and analytical chemistry, which we discuss in Section~\ref{sec:analytical_chemistry}.

The first problem is a special case of multi-source domain generalization for regression as described in Subsection~\ref{subsec:domain_adaptation_for_binary_classification}.
Motivated by our learning bound proposed in Chapter~\ref{chap:learning_bounds}, we propose a new algorithm that is based on the similarity of the first moments of multiple distributions.
Our method outperforms classical algorithms in several domain adaptation experiments with real-world data.
Moreover, it finds well-performing regression models for previously unseen domains which is not possible with classical regression methods.

The second problem from the area of analytical chemistry is unsupervised domain adaptation for regression.
Motivated by our metric-based regularization proposed in Chapter~\ref{chap:cmd_algorithm}, we propose a new regularization strategy for the domain adaptation of linear regression models.
We adapt the partial least squares regression algorithm for the calibration of chemical measurement systems.
In contrast to standard approaches in this field which are based on so-called \textit{transfer samples}, our algorithm only uses unlabeled data from the application measurement system.
Theoretical properties of the algorithm are discussed and it is tested on three real-world datasets.

\section{Industrial Manufacturing}
\label{sec:cyclical_manufacturing}

\begin{figure}[ht]
	\centering
	\includegraphics[width=\linewidth]{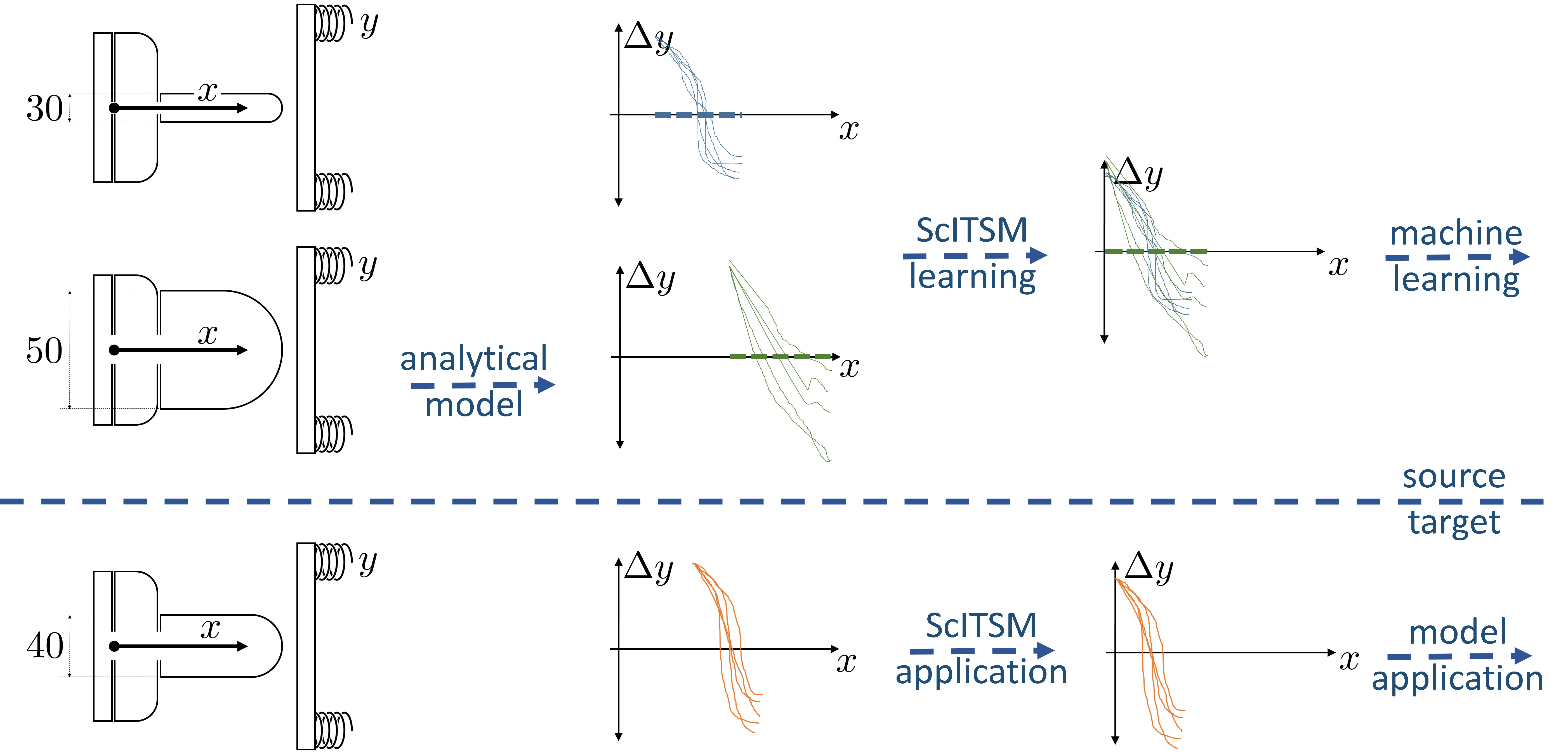}
	\caption[Schematic sketch of the proposed Scenario-Invariant Time Series Mapping algorithm.]{Schematic sketch of ScITSM for two-source domain generalization with feature $x$ and target $\Delta y$. Left column: Differently parametrized tools acting with feature $x$ on a workpiece causing target feature $y$. Four basic training steps are performed:
		(a) Collection of training data from source domains (representing tools parametrized by $30$ and $50$);
		(b) pre-processing, \eg~analytic modeling, normalization and subsampling;
		(c) ScITSM for aligning source data distributions (lines in the right column) based on parametric domain-dependent corrected and smoothed mean curves (dashed lines);
		(d) training of a single machine learning model based on the aligned data of all source domains.
		The prediction for an unseen target domain (parametrized by $40$) is based on three steps: (a) Collection of target domain data; (b) application of ScITSM; (c) prediction of $\Delta y$ using the trained machine learning model.}
	\label{fig:grafical_abstract_scitsm}
\end{figure}

In industrial manu\-facturing processes, data is often collected from different operating conditions and environments leading to different distributions.
One example is the drilling of steel components~\cite{pena2005monitoring,ferreiro2012bayesian} where different machine settings can lead to different torque curves during time.
Other examples can be found in the optical inspection of textures or surfaces~\cite{malaca2016online,stubl2012discrepancy,zuavoianu2017multi}, where different lightening conditions and texture classes can lead to variations in measurements.

Many of these problems are multi-source domain generalization problems as described in Section~\ref{sec:domain_adaptation}.
Given labeled data from multiple source domains, the goal is to find a model that performs well on some application data with a distribution different from the source distributions.
Note that in contrast to unsupervised domain adaptation, in domain generalization no target data, neither labeled nor unlabeled, is given.

We aim at predicting time series from target domains arising in problems of industrial manu\-facturing, \eg~torque curves.

We propose a new domain generalization method called {\it scenario-invariant time series mapping} (ScITSM) that leverages available information in multiple similar domains and applies it to the prediction of previously unseen domains.
ScITSM follows the principle of learning new data representations and maps the data in a new space where the first moments of the domain-specific data distributions are aligned.
Our method is illustrated in Figure~\ref{fig:grafical_abstract_scitsm}.

The performance of ScITSM is demonstrated by experiments on a real-world problem of industrial manufacturing.
Details of the application must be kept confidential, so it is introduced here in an abstracted way.
In particular, a schematic sketch of the application is shown in Figure~\ref{fig:grafical_abstract_scitsm}, the results of the experiments are presented and parts of the collected and preprocessed data are shown.
The results indicate that prediction accuracy can be significantly improved by ScITSM.

This section is organized as follows:
Subsection~\ref{subsec:relatedwork_scitsm} gives relations to the state-of-the-art.
Subsection~\ref{subsec:problem_scitsm} describes the problem.
Subsection~\ref{subsec:approach_scitsm} details our algorithm.
Finally, Subsection~\ref{subsec:use_case} gives our experiments and results on industrial data.

\subsection{Related Work}
\label{subsec:relatedwork_scitsm}

Published domain adaptation algorithms in manu\-facturing applications are rather scarce.
Successful application in chemistry-oriented manu\-facturing processes with the usage of che\-mo\-me\-tric modeling techniques are presented in~\cite{malli2017standard}.
Another successful application of domain adaptation in intelligent manu\-facturing for improving product quality was presented in~\cite{luis2010inductive}.

The presented method corresponds to the domain adaptation subtask of domain generalization~\cite{muandet2013domain}.
As such, our problem setting is similar to the one of some domain generalization algorithms in the area of kernel methods~\cite{muandet2013domain,grubinger2015domain,grubinger2017multi,deshmukhmulticlass,gan2016learning,erfani2016robust} and neural networks~\cite{ghifary2015,li2017deeper,li2017learning}.

However, to the best of our knowledge there is no domain generalization method that accounts for multiple source domains and temporal information in time series data in related fields.

\subsection{Problem Formulation}
\label{subsec:problem_scitsm}

For simplicity, we formulate the problem of multi-source domain generalization for time series of equal length $t$.
Such time series are obtained as results of subsampling procedures as it is the case in our application in Section~\ref{subsec:use_case}.
In addition to the standard assumptions in domain generalization~\cite{muandet2013domain,sugiyama2012machine,ben2014domain}, we assume for each domain a given parameter vector identifying some properties of the underlying real-world setting, \eg~corresponding tool dimensions or material properties.
\needspace{15\baselineskip}
\begin{restatable}[Multi-Source Domain Generalization for Regression]{problemrep}{domain_generalization_for_regression}%
	\label{problem:domain_generalization_for_regression_scitsm}%
	Consider $s$ source domains $(p_1,l),\ldots,(p_s,l)$ and a target domain $(q,l)$ with labeling function $l:\mathbb{R}^{d\times t}\to \mathbb{R}^t$, probability density functions $p_1,\ldots,p_s,q\in\mathcal{M}\left(\mathbb{R}^{d\times t}\right)$ and $s+1$ corresponding parameter vectors $\boldsymbol{\rho}_1,\ldots,\boldsymbol{\rho}_s,\boldsymbol{\rho}_q\in\mathbb{R}^z$.
	
	Given $s$ source samples $X_1,\ldots, X_s$ drawn from $p_1,\ldots,p_s$, respectively, with corresponding labels $Y_1=l(X_1),\ldots, Y_s=l(X_s)$ and parameters $\boldsymbol{\rho}_1,\ldots,\boldsymbol{\rho}_s,\boldsymbol{\rho}_q$, find some $f:\mathbb{R}^{d\times t}\to\mathbb{R}^t$ with a small target error
	\begin{equation}
		\label{eq:error_domain_generalization}
		\int_{\mathbb{R}^{d\times t}} \norm{f(\x)-l(\x)}_2 \diff\x.
	\end{equation}
\end{restatable}%
Note that, except for the parameter vector $\boldsymbol{\rho}_q$, no information is given about data in the target domain.

\subsection{Motivating Learning Bound}
\label{subsec:learning_bound}

Intuitively the error in Eq.~\eqref{eq:error_domain_generalization} cannot be small if the target domain is too different from the source domains.
However, if the data distributions of the domains are similar, this error can be small as shown by the following theorem.
The proof is obtained as extension of Theorem~1 in~\cite{ben2010theory} to multiple sources and time series.

\begin{restatable}{thmrep}{mainthmscitsm}%
	\label{thm:scitsm}%
	Consider some $p_1,\ldots,p_s, q\in \mathcal{M}\left(\mathbb{R}^{d\times t}\right)$ and a labeling function $l:\mathbb{R}^{d\times t}\to [0,1]^t$. Then the following holds for all integrable functions $f:\mathbb{R}^{d\times t}\to [0,1]^t$:
	\begin{align}
		\label{eq:error_bound_scitsm}
		\begin{split}
			\int \norm{f-l}_2 q \leq \frac{1}{s} \sum_{i=1}^{s} \int \norm{f-l}_2 p_i + \frac{2\sqrt{t}}{s} \sum_{j=1}^s d_\mathrm{TV} (p_j,q).
		\end{split}
	\end{align}
\end{restatable}

\begin{proof}
	The following holds:
	\begingroup
	\allowdisplaybreaks
	\begin{align*}
		\int\norm{f-l}_2 q &= \int \norm{f-l}_2 q + \frac{1}{s}\sum_{i=1}^s \int\norm{f-l}_2 p_i
		- \frac{1}{s}\sum_{i=1}^s \int\norm{f-l}_2 p_i\\
		&= \frac{1}{s}\sum_{i=1}^s \int\norm{f-l}_2 p_i + \int \norm{f-l}_2
		\left(q-\frac{1}{s}\sum_{i=1}^s p_i\right)\\
		&\leq \frac{1}{s}\sum_{i=1}^s \int\norm{f-l}_2 p_i + \int \norm{f-l}_2
		\left|q-\frac{1}{s}\sum_{i=1}^s p_i\right|\\
		&= \frac{1}{s}\sum_{i=1}^s \int\norm{f-l}_2 p_i + \int \norm{f-l}_2
		\left|\frac{1}{s}\sum_{i=1}^s q-\frac{1}{s}\sum_{i=1}^s p_i\right|\\
		&\leq \frac{1}{s}\sum_{i=1}^s \int\norm{f-l}_2 p_i + \int \norm{f-l}_2
		\frac{1}{s}\sum_{i=1}^s \left|q- p_i\right|\\
		&\leq \frac{1}{s}\sum_{i=1}^s \int\norm{f-l}_2 p_i
		+ \sup_{\x\in \mathbb{R}^{N\times t}}\norm{f(\x)-l(\x)} \int \frac{1}{s}\sum_{i=1}^s \left|q- p_i\right|\\
		&\leq \frac{1}{s}\sum_{i=1}^s \int\norm{f-l}_2 p_i
		+ \sup_{\y,\y'\in [0,1]^t}\norm{\y-\y'} \frac{1}{s}\sum_{i=1}^s \int \left|q- p_i\right|\\
		&=\frac{1}{s}\sum_{i=1}^s \int\norm{f-l}_2 p_i
		+ \sup_{\x \in [-1,1]^t}\norm{\x} \frac{1}{s}\sum_{i=1}^s \int \left|q- p_i\right|\\
		&=\frac{1}{s}\sum_{i=1}^s \int\norm{f-l}_2 p_i
		+  \frac{\sqrt{t}}{s}\sum_{i=1}^s \int \left|q- p_i\right|\\
		&=\frac{1}{s}\sum_{i=1}^s \int\norm{f-l}_2 p_i
		+ \frac{2 \sqrt{t}}{s}\sum_{i=1}^s d_\mathrm{TV}(p_i,q)
	\end{align*}
	\endgroup
	where the last equality follows from Lemma~\ref{thm:TV_distance}.
\end{proof}\\

Theorem~\ref{thm:scitsm} shows that the error in the target domain can be expected to be small if the mean over all errors in the source domains is small and the mean distance of the target distribution to the source distributions is small.
For simplicity, Theorem~\ref{thm:scitsm} assumes a target feature in the unit cube which can be realized in practice \eg~by additional normalization procedures.

Our method tries to minimize the target error on the left-hand side of Eq.~\eqref{eq:error_bound_scitsm} by mapping the data in a new space where an approximation of the right-hand side is minimized.
The minimization of the second term on the right-hand side is tackled by aligning all source distributions in the new space.
For example consider the right column in Figure~\ref{fig:grafical_abstract_scitsm}.
The minimization of the first term is tackled by subsequent regression.

It is important to note that the alignment of only the source distributions does not minimize the second term on the right-hand side, if the target density $q$ is too different from all the source densities~\cite{ben2010theory}.
As there is no data given from $q$ in Problem~\ref{problem:domain_generalization_for_regression_scitsm}, we cannot identify such cases based on samples.
As one possible solution to this problem, we propose to consider only domains with parameter vectors $\boldsymbol{\rho}_q$ representing physical dimensions of tool settings that are similar to related tool settings with parameters $\boldsymbol{\rho}_1,\ldots,\boldsymbol{\rho}_s$.
See Figure~\ref{fig:scenarios} for an example.

\begin{figure}[t]
	\center
	\includegraphics[scale=0.55]{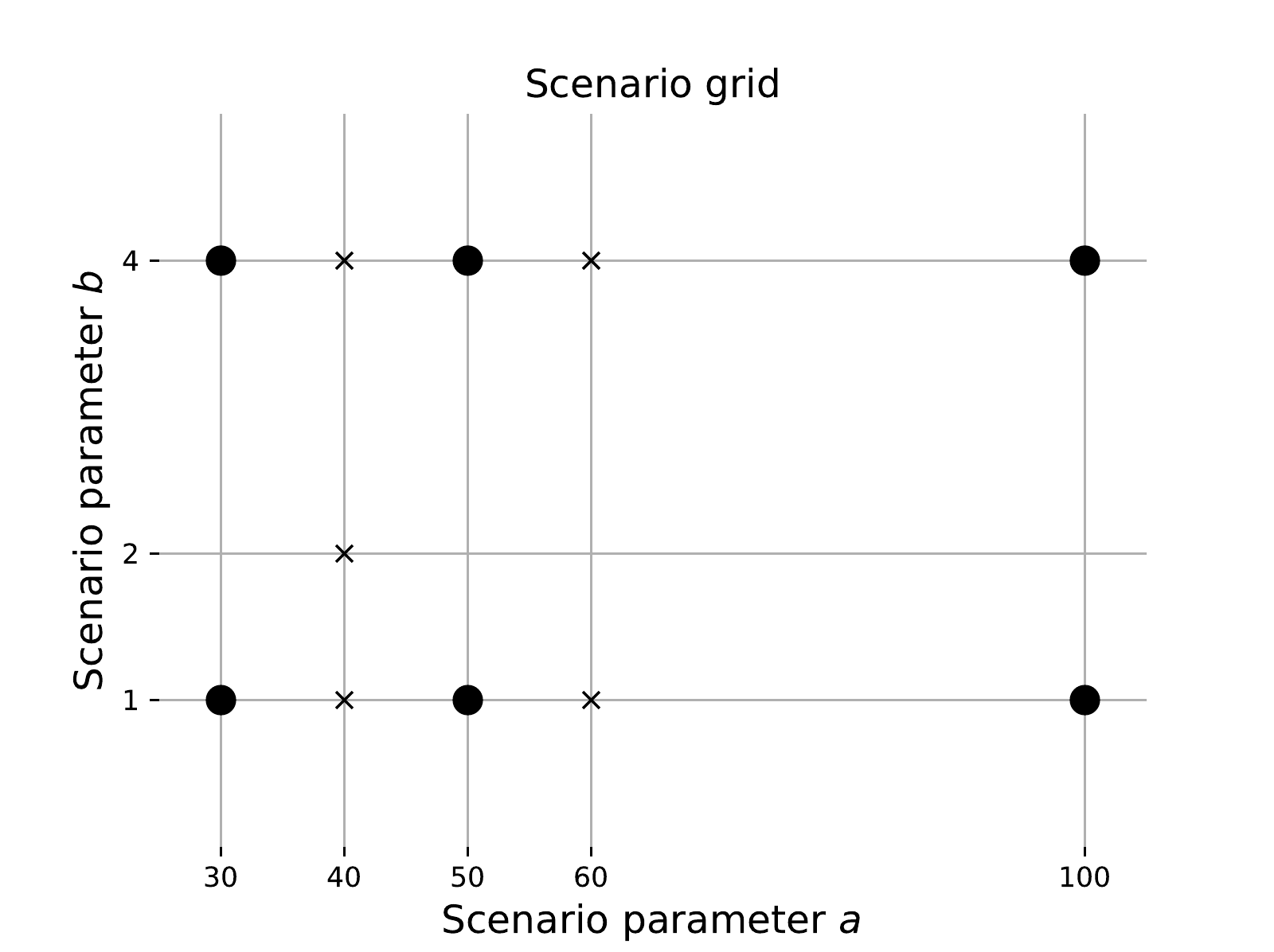}
	\caption[Physical parameters of different domains.]{Use case domains with parameters $a$ (horizontal axis) and $b$ (vertical axis). Source domains are marked by dots and target domains by crosses.}
	\label{fig:scenarios}
\end{figure}

\subsection{Scenario-Invariant Time Series Mapping}
\label{subsec:approach_scitsm}

Let us consider some source samples $X_1,\ldots,X_s\in\mathbb{R}^{k\times d\times t}$ with label feature vectors $Y_1,\ldots,Y_s\in\mathbb{R}^{k\times t}$ and parameter vectors $\boldsymbol{\rho}_1,\ldots,\boldsymbol{\rho}_s\in\mathbb{N}^z$, \eg~parameters $30$ and $50$ in Figure~\ref{fig:grafical_abstract_scitsm}.
For simplicity of the subsequent description, the number of samples $k$ is assumed to be equal for each domain.

The goal of ScITSM is to compute a mapping
\begin{gather}
	\label{eq:final_transformation}
	\begin{split}
		\Psi : \mathbb{R}^{d\times t}\times \mathbb{R}^z &\to \mathbb{R}^{t}\\
		(\x,\boldsymbol{\rho})&\mapsto \Psi(\x,\boldsymbol{\rho})
	\end{split}
\end{gather}
which transforms a time series $\x$ of a domain parametrized by $\boldsymbol{\rho}$ to a new time series $\Psi(\x,\vec p)$ such that the latent samples $\Psi(X_1,\boldsymbol{\rho}_1),\ldots,\Psi(X_s,\boldsymbol{\rho}_s)$ are similar and such that a subsequently learned regression model $f:\mathbb{R}^{t}\to\mathbb{R}^t$ performs well on each domain, where $\Psi(X,\boldsymbol{\rho})=\{\Psi(\x,\boldsymbol{\rho})\mid \x\in X\}$.

The computation of the function $\Psi$ in ScITSM involves three processing steps: Step~1: Calculation of a mean curve for each source domain, Step~2: Learning of correction functions at equidistant fixed time steps, and, Step~3: Smooth connection of correction functions.

{\bf Step~1 (Calculation of Mean Curves):}
In a first step a smooth curve called {\it mean curve} is fitted for each source domain, see \eg~dashed lines in middle column of Figure~\ref{fig:grafical_abstract_scitsm}.
Therefore, for each of the domain samples $X_1,\ldots,X_s$, the mean value for each of the $d$ features and $t$ time steps is computed and a spline curve is fitted subsequently by means of the algorithm proposed in~\cite{dierckx1982fast}.
This process results in a multiset $\widehat X\in\mathbb{R}^{s\times d\times t}$ storing the mean curves, \ie~the rows, for each of the $s$ source domains.

{\bf Step~2 (Learning of Equidistant Corrections):}
After the mean curves are computed, $b$ equidistant points $t_1,\ldots,t_b$ are fixed and $b$ corresponding {\it correction functions}
\begin{align}
	\Phi_1,\ldots,\Phi_b:\mathbb{R}^z\to\mathbb{R}^d
\end{align}
are learned which map a parameter vector $\boldsymbol{\rho}_i$ corresponding to the $i$-th domain close to the corresponding points $\widehat x_{t_1},\ldots,\widehat x_{t_b}$ of the $i$-th mean curve $\widehat{\vec x}_i=(\widehat x_1,\ldots,\widehat x_t)$, \ie~the $i$-th row of $\widehat X$.
This is done under the constraint of similar predictions $\Phi_{t'}(\vec p_i),\Phi_{t''}(\vec p_i)$ of nearby time steps $t',t''$ of two points $\widehat x_{t'},\widehat x_{t''}$ on the mean curve.
We apply ideas from the multi-task learning approach proposed in~\cite{evgeniou2004regularized} that aims at similar predictions by means of similar parameters $\theta_1,\ldots,\theta_b$ of the learning functions $\Phi_1,\ldots,\Phi_b$.
More precisely, we propose the following objective function:
\begin{align}
	\label{eq:objective_scitsm}
	\min_{\Phi_1,\ldots,\Phi_b} \sum_{j=1}^{b} \left( \sum_{i=1}^{s} \norm{\widehat{X}_{i,:,t_j} - \Phi_j\left(\boldsymbol{\rho}_i\right)}_2 + \alpha \sum_{r=\max(1,j-u)}^{\min(j+u,b)} \frac{\norm{\theta_j-\theta_r}_2^2}{\delta^{|j-r|-1}} + \beta\norm{\theta_j}_1\right),
\end{align}
where $X_{i,:,j}$ is the vector of features corresponding to the $i$-th domain and the $j$-th timestep and $\theta_j\in\mathbb{R}^z$ refers to the parameter vector of $\Phi_j$, \eg~$\Phi_j(\boldsymbol{\rho})=\langle \theta_j, \boldsymbol{\rho}\rangle +c$ is a linear model with parameter vector $\theta_j\in\mathbb{R}^z$ and bias $c\in\mathbb{R}$.
The first term of Eq.~\eqref{eq:objective_scitsm} ensures that the prediction of the correction functions applied on the mean curves are not far away from the mean curves itself.
The second term of Eq.~\eqref{eq:objective_scitsm} ensures similar parameter vectors of $2 u$ nearby correction functions, where $u\in\mathbb{N}$ and $\alpha,\delta\in\mathbb{R}$ are hyper-parameters.
The last term ensures sparse parameter vectors by means of $L^1$-regularization~\cite{andrew2007scalable} with hyper-parameter $\beta\in\mathbb{R}$.

{\bf Step 3 (Smooth Connection):}
To obtain a time series of length $t$, we aim at a smooth connection of the functions $\Phi_1,\ldots,\Phi_b$ between the points $t_1,\ldots,t_b$.
This is done by applying ideas from moving average filtering~\cite{makridakis1977adaptive}.
For a new time step $v\leq t$, we denote by
\begin{align}
	\begin{split}
		R(v)=\Big\{&\big(\floor{v}-u+1,\ceil{v}+u-1\big),\big(\floor{v}-u+2,\ceil{v}+u-2\big),\ldots,\big(\floor{v},\ceil{v}\big)\Big\}
	\end{split}
\end{align}
a set of pairs constructed from the equidistant timesteps $t_1,\ldots,t_b$ in a nested order, where $\floor{v}$ respectively $\ceil{v}$ denote the largest respectively smallest number in $\{t_1,\ldots,t_b\}$ being smaller respectively larger than $t$.
The coordinates of the final transformation vector ${\Psi(\vec x,\boldsymbol{\rho})=(\Psi_1(\vec x,\boldsymbol{\rho}),\ldots,\Psi_T(\vec x,\boldsymbol{\rho}))^\text{T}}$ in Eq.~\eqref{eq:final_transformation} are obtained by
\begin{align}
	\label{eq:smoothing}
	\begin{split}
		\Psi_v(\vec x,\boldsymbol{\rho})= \vec x_v -\sum_{(i,j)\in R(v)} \frac{\gamma^{\frac{|R(v)|-2i+2}{2}}\left(\Phi_{i}(\boldsymbol{\rho}) + (v-i) \frac{\Phi_{j}(\boldsymbol{\rho}) - \Phi_{i}(\boldsymbol{\rho})}{j - i}\right)}{\sum_{(i,j)\in R(v)} \gamma^{\frac{|R(v)|-2i+2}{2}}}
	\end{split}
\end{align}
where $|R(v)|$ is the cardinality of $R(v)$ and $\gamma\in(0,1]$ is the smoothing hyper-parameter.
That is, for each vector element $\vec x_v$ of the time series $\vec x$, a sum is subtracted which describes a weighted average of linear interpolations between the points $\Phi_i$ and $\Phi_j$ for each time step pair $(i,j)\in R(v)$.
ScITSM is summarized by Algorithm~\ref{alg:ScITSM}.

\begin{algorithm}
	\SetAlgoLined
	\KwIn{Samples $X_1,\ldots,X_s \in \mathbb{R}^{k\times d\times t}$, scenario parameters $\boldsymbol{\rho}_1,\ldots,\boldsymbol{\rho}_s\in\mathbb{R}^z$ and hyper-parameters $\alpha,\beta,\gamma\in\mathbb{R}$, $b,u\in\mathbb{N}$ and $\delta\in (0,1]$
	}
	\KwOut{Mapping $\Psi:\mathbb{R}^{d\times t}\times \mathbb{R}^z\to\mathbb{R}^{t}$}~\\
	
	\Stepone{Calculation of mean curve tensor $\widehat{X}\in\mathbb{R}^{s\times d\times t}$}
	\Steptwo{Computation of correction functions according to Eq.~\eqref{eq:objective_scitsm}}
	\Stepthree{Computation of transformation $\Psi$ using Eq.~\eqref{eq:smoothing}.}
	\caption[Scenario-invariant time series mapping.]{Scenario-invariant time series mapping (ScITSM)}
	\label{alg:ScITSM}
\end{algorithm}

{\bf Subsequent Regression:}
Consider a transformation function $\Psi : \mathbb{R}^{d\times t}\times \mathbb{R}^z \to \mathbb{R}^{t}$ as computed by ScITSM, a previously unseen target scenario sample $X_q=(\vec x_1,\ldots,\vec x_k)$ of size $k$ drawn from $q\in\mathcal{M}\left(\mathbb{R}^{d\times t}\right)$ and a corresponding parameter vector $\boldsymbol{\rho}_q\in\mathbb{R}^z$, \eg~parameter $40$ in Figure~\ref{fig:grafical_abstract}.
As motivated in Subsection~\ref{subsec:learning_bound}, the distribution of the transformed sample $\Psi(X_p,\boldsymbol{\rho}_q)$ is assumed to be similar to the distributions of the samples $\Psi(X_1,\boldsymbol{\rho}_1),\ldots, \Psi(X_s,\boldsymbol{\rho}_s)$ which is induced by the selection of an appropriate corresponding parameter space, see \eg~Figure~\ref{fig:grafical_abstract} and Figure~\ref{fig:scenarios}.
Subsequently to ScITSM, a regression function
\begin{align}
	f:\mathbb{R}^{t}\to\mathbb{R}^t
\end{align}
is trained using the concatenated input sample $(\Psi(X_1,\boldsymbol{\rho}_1);\ldots; \Psi(X_s,\boldsymbol{\rho}_s))$ and its corresponding concatenated label values $(Y_1;\ldots;Y_s)$.
Finally, the target features of $X_q$ can be computed by $f(\Psi(X_q,\boldsymbol{\rho}_q))$.

Theorem~\ref{thm:scitsm} indicates that the empirical error
\begin{align}
	\frac{1}{d}\sum_{i=1}^d\norm{ f(\Psi(\vec x_i,\boldsymbol{\rho}_q))-l(\vec x_i)}_2
\end{align}
of the function $f\circ \Psi$ on a new unseen target sample is small if the empirical error is small on the source samples.

\subsection{Empirical Evaluations}
\label{subsec:use_case}

We integrated our approach described in Section~\ref{subsec:approach_scitsm} into the data-flow of an industrial machine learning pipeline used to implement a virtual sensor~\cite{wang2009sensor} in an intelligent manufacturing setting similar to the one described in Figure~\ref{fig:grafical_abstract_scitsm}.

\paragraph{Dataset}
Our use case consists of $11$ domains based on physical tool settings with parameters describing physical tool dimensions as illustrated in Figure~\ref{fig:scenarios}.
For each domain, we collected around $50$ time series.
We applied some application-specific normalization and transformation steps to each time series including its subtraction from a finite element simulation of the mechanical tool process.
Some representative resulting time series from the source domains are illustrated in Figure~\ref{fig:1figs} on the left.
For our experiments we choose $6$ out of $11$ domains as source domains and $5$ domains as target domains.
The target domains are chosen such that its parametrization is well captured by the parametrization of the source domains as shown in Figure~\ref{fig:scenarios}.

\begin{figure}[t]
	\makebox[\linewidth][c]{%
		\begin{subfigure}[b]{.5\textwidth}
			\centering
			\includegraphics[width=.95\textwidth]{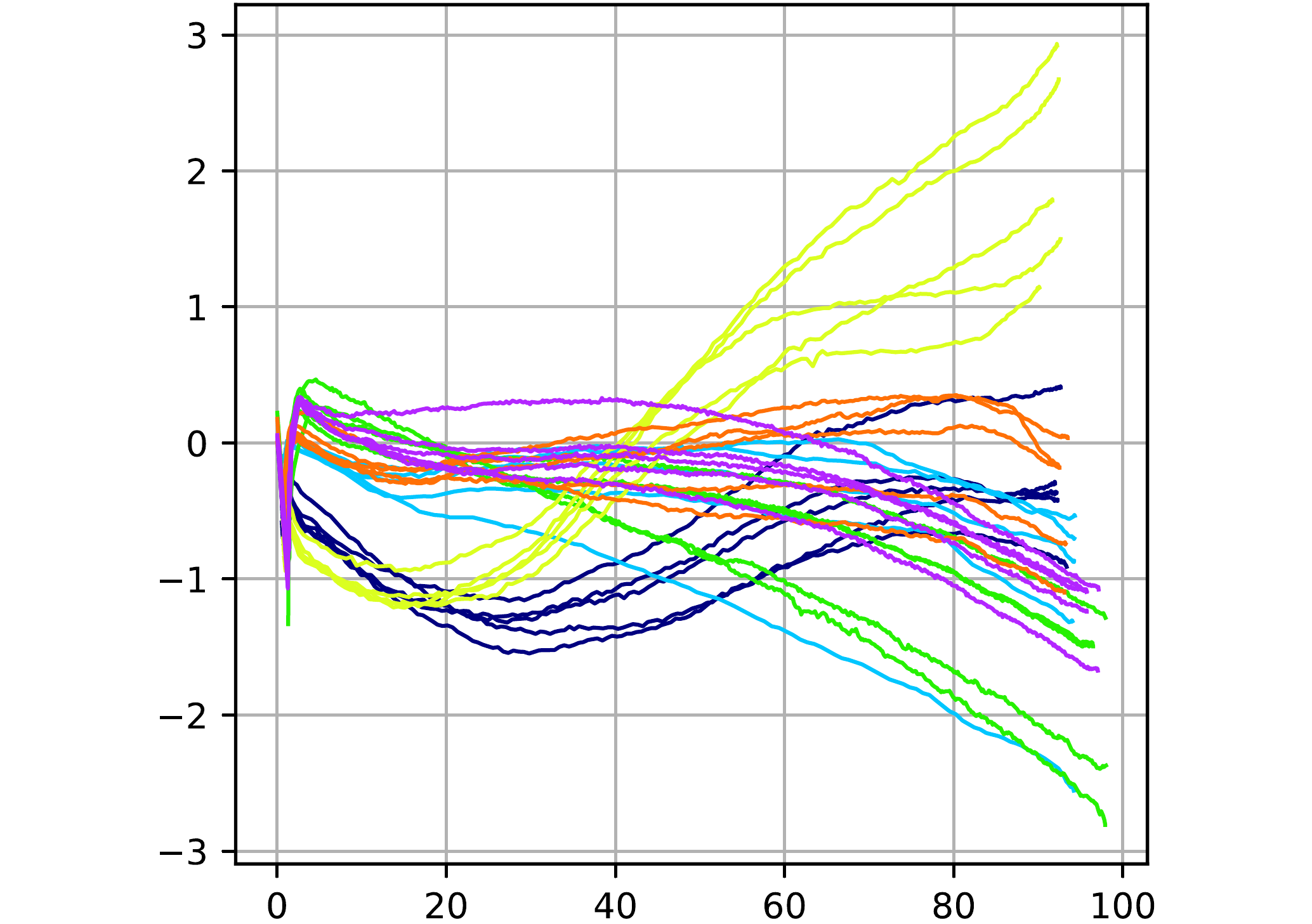}
		\end{subfigure}%
		\begin{subfigure}[b]{.5\textwidth}
			\centering
			\includegraphics[width=0.95\textwidth]{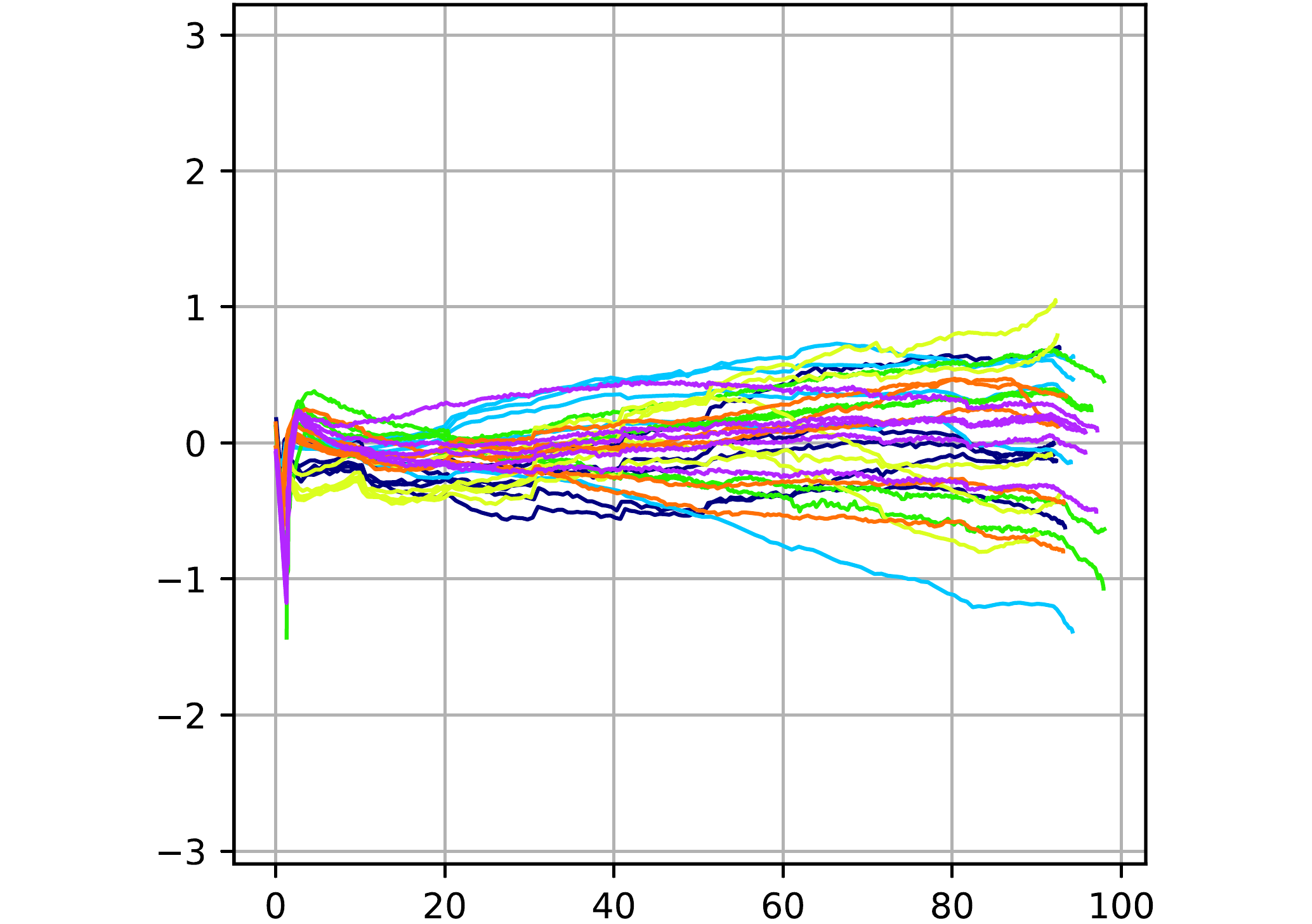}
		\end{subfigure}%
	}
	\caption[Some selected pre-processed time series of source domains before and after the application of Scenario-Invariant Time Series Mapping.]{Some selected pre-processed time series of source domains (different colors) before (left) and after (right) the application of ScITSM.}
	\label{fig:1figs}
\end{figure}

\paragraph{Validation Procedure}
To estimate the performance of the proposed ScITSM on previously unseen domains, we evaluate different regression models based on an unsupervised transductive training protocol~\cite{ganin2016domain,gong2013connecting,chopra2013dlid,long2016joint} combined with cross-validation on source domains.
In a first step, we select appropriate hyper-parameters in a semi-automatic way.
That is, the parameters are fixed by a method expert based only on the unsupervised data from the source domains without considering any labels, \ie~output values, or target samples.
The decision is based on visual quantification of the distribution alignment in the representation space.
As a result, the hyper-parameters are the same for all subsequently trained regression models.
The result of some representative time series is illustrated in Figure~\ref{fig:1figs}.

For evaluating the performance of regression models trained subsequently to ScITSM we use 10-fold cross-validation~\cite{varma2006bias}.
That is, in each of 10 steps, $90\%$ of the data points, \ie~$90\%$ of each source domain, are chosen as training data and $10\%$ as validation data.
Since no data of the target domains is used for training, the models are evaluated on the whole data of the target domains in each fold.
Using this protocol, 10 different root-mean squared errors for each model and each domain are computed, properly aggregated and, together with its standard deviation, reported in Table~\ref{tab:multi-model}.

To show the advantage of using more than one source domain, we additionally optimize each regression model using the training data of only a single source domain as shown in Table~\ref{tab:single_model}.

We compare the following regression models and we use the following parameter sets for selection:
\begin{itemize}
	\item \textit{Bayesian Ridge Regression}~\cite{mackay1992bayesian}: The four gamma priors are searched in the set $\{10^{-3},10^{-4}, 10^{-5}, 10^{-6}\}$ and the iterative algorithm is stopped when a selected error in the set $\{10^{-2}, 10^{-3},10^{-4}, 10^{-5}\}$ is reached.
	\item \textit{Random Forest}~\cite{breiman2001random}: We used $100$ estimators, the maximum depth is searched in the set $\{1,2,4, 8,\ldots,\infty\}$ where $\infty$ refers to a pure expansion of the leaves and the minimum number of splits is selected in the set $\{2,4,8,\ldots,1024\}$.
	\item \textit{Support Vector Regression}~\cite{smola2004tutorial} (SVR) with sigmoid kernel: The epsilon parameter is selected from the set $\{10^{-1},10^{-2},10^{-3}\}$, the parameter $C$ is selected in $\{10^{-5},5\cdot 10^{-4},10^{-4},5\cdot 10^{-3},10^{-3}\}$ and the algorithm is stopped when a selected error in the set $\{10^{-3}, 10^{-5}\}$ is reached.
	\item \textit{Support Vector Regression} with Gaussian kernel: The epsilon parameter is selected from the set $\{10^{-1},10^{-2},10^{-3}\}$, the parameter $C$ is selected in $\{10, 25, 30\}$, the bandwidth parameter is selected in the set $\{10^{-5},10^{-4},10^{-3},10^{-2},10^{-1},1\}$ and the algorithm is stopped when a selected error in the set $\{10^{-3}, 10^{-2}\}$ is reached.
\end{itemize}

\begin{figure}[t]
	\centering 
	\includegraphics[width=0.6\linewidth]{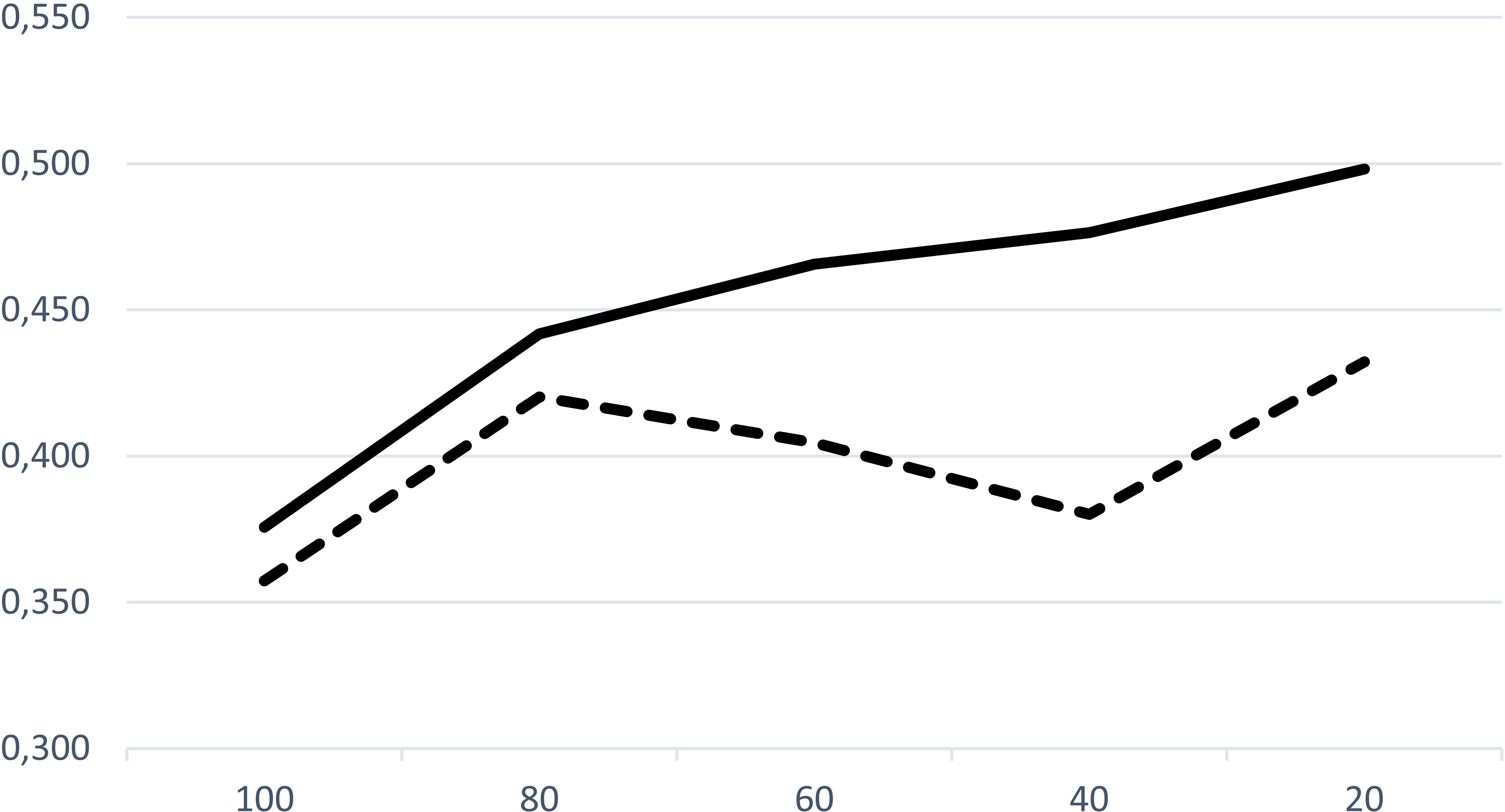}
	\caption[Performance dependency on sample size of support vector regression with and without Scenario-Invariant Time Series Mapping.]{Performance dependency on sample size of support vector regression with Gaussian kernel without applying ScITSM (solid) and with the proposed ScITSM (dashed). Horizontal axis: Percentage of training data; Vertical axis: Average root  mean squared error over all unseen target domains except negative transfer domain $(2,40)$.}
	\label{fig:svr}
\end{figure}

\paragraph{Results}
Figure~\ref{fig:1figs} illustrates some selected time series pre-processed by ScITSM.
It can be seen that the diversity caused by different source domains is reduced resulting in more homogeneous time series for subsequent regression.
Table~\ref{tab:multi-model} shows the results of applying ScITSM to multiple source domains.
The application of ScISTM improves all regression models in average root mean squared error except the support vector regression model based on Gaussian kernel.

The domain $(2,40)$ is the only domain where the application of ScITSM reduces the performance of support vector regression models by a large margin.
From Figure~\ref{fig:scenarios} it can be seen that both tool dimensions $2$ and $40$ are not considered in the source domains.
We conclude that at least one dimension should be considered in the source domains in our use case, otherwise the domain distributions are too different.
This well known phenomenon is often called {\it negative transfer}~\cite{pan2010survey}.

It is interesting to observe that the random forest models overfit the source domains.
This can be seen by a low average root mean squared error on the source domains compared to the target domains.
Consequently, it is hard for ScITSM to improve the performance on the source domains (average error decreased only to $97.59\%$ of that of the raw models) where the target domains errors are improved by a large margin.
The target domain improvement is without considering domain $(4,60)$ where the random forest model performed best over all models.
This improvement is not unexpected, as the overfitting of source domains can imply performance improvements in some very similar target domains.
However, our goal is an improvement in many domains, not in single ones.

In general ScITSM improves the results of regression models in $9$ out of $11$ domains, where the remaining two results have explainable reasons of negative transfer and overfitting.

In principle it is possible that a high root mean squared error of the models without ScITSM is caused by mixing data from different domains, \ie~negative transfer happens.
To exclude this possibility, we train one model for each domain and computed the root mean squared error for all other domains.
In a first step, we observe that no model is able to generalize to domains other than the single training one.
The resulting root mean squared errors of the single domain trained models are excessively high and give no further information. 
One possible reason is that the domains are too different.
For example, consider a model trained on the yellow time series in Figure~\ref{fig:1figs}.
Obviously this model will not perform well on the green time series.
This experiment underpins that generalization is not possible for models trained only on single domains, \ie~the standard regression case, and that the considered problem of domain generalization is important in our use case.

It is interesting to observe that even models trained on single domains can be improved by considering data from different domains.
To see this, consider Table~\ref{tab:single_model}.
Each column denoted by 'without ScITSM' shows the performance of different models trained on data from a single domain only.
This is in contrast to Table~\ref{tab:multi-model} where each column shows errors of the same model on different domains.
The application of ScITSM almost always improves the performance of classical regression models.
This is interesting as one may expect that models trained on data from a specific domain cannot be improved by data from different domains. 
However, this positive effect of transfer learning can happen \eg~when a high number of domains is considered with a comparably low sample sizes.

Another interesting question is about the effect of ScITSM when the amount of source domain samples decreases.
Therefore, we consider the average root mean squared errors over all target domains of the best regression models, \ie~SVR with Gaussian kernel, for a varying number of source samples.
The result is shown in Figure~\ref{fig:svr}.
In our example the positive effect of ScITSM gets stronger when the sample size of all domains decreases by a certain percentage value.

Our procedure of choosing appropriate parameters for ScITSM requires expert knowledge about our method.
In our use case, long-term knowledge from several years resulted in a well-performing default setting.
It is interesting to observe that this default setting gives a high performance independently of the data size as indicated by Figure~\ref{fig:svr}.
It is important to note that the selection of appropriate parameters is sophisticated in the considered problem of domain generalization, as no data of the target domains is given.
We refer to Subsection~\ref{subsec:domain_adaptation_by_nns} for a discussion of this problem.
By using our expert knowledge based method for parameter tuning, the resulting performance of the regression models in the source domains cannot be directly interpreted as estimating the generalization error.
However, in this work, we are more interested in the generalization error of the unseen target domains, which are not effected.

We finally conclude that our method successfully enables the improvement of the performance of regression models in previously unseen domains by using information from multiple similar source domains.
The result is obtained by a single regression model, which is conceptually and computationally simpler than the application of multiple single models for separate domains.

\begin{table}[ht]
	\tiny
	\centering
	\resizebox{\linewidth}{!}{%
		\begin{tabular}{ |l|r|r|r|r|r|r|r|r|r| }
			\hline
			& \multicolumn{3}{|c|}{Bayesian Ridge} & \multicolumn{3}{|c|}{Random Forest}\\
			\hline
			Scenario & without ScITSM & with ScITSM & perc. & without ScITSM & with ScITSM & perc. \\
			\hline
			(1, 30) & 0.443 (0.082) & 0.239 (0.056) & \textit{53.93} & 0.259 (0.109) & 0.262 (0.082) & 101.13\\
			(1, 50) & 0.645 (0.070) & 0.359 (0.103) & \textit{55.69} & 0.322 (0.140) & 0.311 (0.111) & \textit{96.62}\\
			(1, 100) & 0.431 (0.140) & 0.299 (0.070) & \textit{69.34} & 0.308 (0.090) & 0.267 (0.064) & \textit{86.48}\\
			(4, 30) & 0.690 (0.117) & 0.334 (0.077) & \textit{48.47} & 0.346 (0.095) & 0.372 (0.064) & 107.31\\
			(4, 50) & 0.431 (0.052) & 0.243 (0.090) & \textit{56.44} & 0.317 (0.098) & 0.238 (0.051) & \textit{75.11}\\
			(4, 100) & 0.488 (0.105) & 0.235 (0.064) & \textit{48.05} & 0.197 (0.077) & 0.234 (0.101) & 118.87\\
			\hline
			Average & 0.521	(0.094) &	0.285	(0.077) &	\textit{55.32} &	0.292	(0.102) &	0.281	(0.079) &	\textit{97.59}\\
			\hline
			(1, 40) & 0.523 (0.078) & 0.403 (0.125) & \textit{77.12} & 0.707 (0.243) & 0.418 (0.163) & \textit{59.12}\\
			(1, 60) & 0.709 (0.058) & 0.394 (0.087) & \textit{55.54} & 0.461 (0.148) & 0.381 (0.099) & \textit{82.72}\\
			(2, 40) & 0.576 (0.092) & 0.426 (0.117) & \textit{73.90} & 0.949 (0.236) & 0.440 (0.108) & \textit{46.34}\\
			(4, 40) & 0.426 (0.031) & \textbf{0.342 (0.076)} & \textit{80.30} & 1.062 (0.238) & 0.399 (0.114) & \textit{37.57}\\
			(4, 60) & 0.519 (0.110) & 0.371 (0.142) & \textit{71.58} & \textbf{0.291 (0.060)} & 0.395 (0.165) & 135.76 \\
			\hline
			Average & 0.551	(0.074) &	0.387	(0.109) &	\textit{71.69} &	0.694	(0.185) &	0.407	(0.130) &	\textit{72.30}\\
			\hline
			\hline
			& \multicolumn{3}{|c|}{SVR (sigmoid)} & \multicolumn{3}{|c|}{SVR (RBF)}\\
			\hline
			Scenario & without ScITSM & with ScITSM & perc. & without ScITSM & with ScITSM & perc.\\
			\hline
			(1, 30) & 0.586 (0.114) & 0.253 (0.081) & \textit{43.17} & 0.243 (0.072) & \textbf{0.238 (0.068)} & \textit{97,64} \\
			(1, 50) & 0.519 (0.221) & 0.364 (0.170) & \textit{70.15} & 0.229 (0.092) & \textbf{0.226 (0.078)} & \textit{98.46} \\
			(1, 100) & 0.694 (0.202) & 0.379 (0.159) & \textit{54.63} & 0.249 (0.064) & \textbf{0.242 (0.070)} & \textit{97.26} \\
			(4, 30) & 1.697 (0.341) & 0.407 (0.067) & \textit{23.97} & 0.342 (0.122) & \textbf{0.294 (0.098)} & \textit{85.95} \\
			(4, 50) & 0.363 (0.154) & 0.325 (0.141) & \textit{89.66} & 0.201 (0.060) & \textbf{0.192 (0.042)} & \textit{95.71} \\
			(4, 100) & 0.682 (0.199) & 0.341 (0.090) & \textit{49.93} & 0.186 (0.059) & \textbf{0.166 (0.032)} & \textit{89.00} \\
			\hline
			Average & 0.757 (0.205) & 0.345 (0.118) & \textit{55.25} &  0.242 (0.078) & \textbf{0.226 (0.065)} & \textit{93.28} \\
			\hline
			(1, 40) & 0.491 (0.142) & 0.483 (0.134) & \textit{98.34} & 0.445 (0.151) & \textbf{0.387 (0.129)} & \textit{87.13} \\
			(1, 60) & 0.637 (0.208) & 0.450 (0.134) & \textit{70.70} & 0.337 (0.079) & \textbf{0.321 (0.064)} & \textit{95.24} \\
			(2, 40) & 0.518 (0.085) & 0.570 (0.158) & 109.95 & \textbf{0.314 (0.055)} & 0.385 (0.096) & 122.72 \\
			(4, 40) & 0.684 (0.189) & 0.452 (0.153) & \textit{66.08} & 0.382 (0.156) & 0.378 (0.156) & \textit{98.66} \\
			(4, 60) & 0.507 (0.202) & 0.487 (0.196) & \textit{96.08} & 0.334 (0.056) & 0.339 (0.134) & 101.45 \\
			\hline
			Average & 0.567	(0.165) & 0.488 (0.155) &	\textit{88.23} & \textbf{0.362 (0.099)} & 0.363 (0.116)& 101.04\\
			\hline 
		\end{tabular}
	}
	\caption[Root mean squared error and standard deviation of different regression models with and without Scenario-Invariant Time Series Mapping.]{Root mean squared error (and standard deviation) of regression models evaluated using $10$-fold cross-validation. Best values of domains are shown in boldface, improvements of ScITSM are shown by italic numbers.}
	\label{tab:multi-model}
\end{table}

\begin{table}[h!]
	\tiny
	\centering
	\resizebox{\linewidth}{!}{%
		\begin{tabular}{ |l|r|r|r|r|r|r|r|r|r| }
			\hline
			& \multicolumn{3}{|c|}{Bayesian Ridge} & \multicolumn{3}{|c|}{Random Forest}\\
			\hline
			Scenario & without ScITSM & with ScITSM & perc. & without ScITSM & with ScITSM & perc. \\
			\hline
			(1,30)  & 0.215 (0.069)& 0.210 (0.065) & \textit{97.66}  & 0.255 (0.079) & 0.261 (0.078) & 102.15\\
			(1,50)  & 0.202 (0.047)& 0.202 (0.048) & 100.00  & 0.370 (0.172) & 0.352 (0.151) & \textit{95.05}\\
			(1,100) & 0.342 (0.112)& 0.341 (0.109) & \textit{99.67}  & 0.325 (0.100) & 0.330 (0.127) & 101.55\\
			(4,30)  & 0.275 (0.072)& 0.275 (0.074) & 100.09 &      0.351 (0.090) & 0.334 (0.094) & \textit{95.00}\\
			(4,50)  & 0.217 (0.069)& 0.217 (0.070) & 100.00  & 0.301 (0.091) & 0.292 (0.081) & \textit{96.84}\\
			(4,100) & 0.197 (0.057)& 0.196 (0.058) & \textit{99.42}  & 0.240 (0.058) & 0.269 (0.095) & 111.70\\
			\hline
			\hline
			& \multicolumn{3}{|c|}{SVR (sigmoid)} & \multicolumn{3}{|c|}{SVR (RBF)}\\
			\hline
			Scenario & without ScITSM & with ScITSM & perc. & without ScITSM & with ScITSM & perc.\\
			\hline
			(1,30) & 0.404 (0.096) & 0.273 (0.099) & \textit{67.54} & 0.390 (0.157) & 0.380 (0.161) & \textit{97.42}\\
			(1,50) & 0.486 (0.223) & 0.394 (0.222) & \textit{81.01} & 0.364 (0.173) & 0.357 (0.159) & \textit{98.12}\\
			(1,100) & 0.656 (0.229) & 0.405 (0.167) & \textit{61.72} & 0.360 (0.201) & 0.369 (0.194) & 102.24\\
			(4,30) & 1.130 (0.149) & 0.440 (0.071) & \textit{38.97} & 0.502 (0.298) & 0.438 (0.244) & \textit{87.17}\\
			(4,50) & 0.382 (0.176) & 0.354 (0.174) & \textit{92.76} & 0.323 (0.108) & 0.322 (0.110) & \textit{99.67}\\
			(4,100) & 0.580 (0.181) & 0.364 (0.094) & \textit{62.80} & 0.215 (0.080) & 0.234 (0.102) & 108.64\\
			\hline
		\end{tabular}
	}
	\caption[Root mean squared error and standard deviation of different regression models trained on a single source domain with and without Scenario-Invariant Time Series Mapping.]{Root mean squared error (and standard deviation) of regression models trained and evaluated on a single source domain, \ie~one model per domain, using $10$-fold cross-validation.}
	\label{tab:single_model}
\end{table}

\section{Analytical Chemistry}
\label{sec:analytical_chemistry}

Recently, domain adaptation techniques attracted considerable attention in analytical chemistry since adaptation of calibration models, model maintenance and calibration transfer between similar analytical devices are recurring tasks~\cite{malli2017standard,NikzadLughoferCernudaReischerKantnerPawliczekBrandstetter18,luoma2018additive,workman2018review,andries2017penalized}.

Yet the success of domain adaptation techniques on the type of data typically derived from chemical measurement systems has been limited.
One reason might be that the assumptions of the underlying models do not comply with the properties of the data.
Primarily, most of the domain adaptation techniques developed over the past decade involve non-linear hypotheses, which is the natural choice for applications in \eg~computer vision, text mining or natural language processing.

This prompted us to revisit three typical phenomenons often observed in spectroscopic applications: 1.~A linear input-output relationship, 2.~approximately normally distributed data and 3.~multicollinearity among input dimensions.
In particular, a linear input-output relationship is often motivated by \textit{Beer-Lambert's} law.
This physical law describes a linear relationship between absorbance of electromagnetic radiation and analyte concentration \cite{Swinehart1962}, \ie 
\begin{equation}
	\label{eqn:Beer-Lampert}
	A = -\log{\frac{I_0}{I}} = \epsilon\cdot c\cdot o,
\end{equation}
where $A$ denotes absorbance, $\epsilon$ is the characteristic substance specific absorptivity of the analyte, $c$ the concentration in solution and $o$ the optical path length. $I_0$ is the raw intensity for $c=0$, \ie~the background signal, and $I$ is the attenuated signal.

Note that the linear dependence of the measured signal on concentration might be violated due to \eg~light scattering, non-linear interactions between different analytes or sample inhomogeneities.
However, Beer-Lambert's law holds surprisingly well for a wide array of analytical techniques~\cite{mark2010chemometrics}.

In this section, we propose a new algorithm for regression that combines the principle of learning new data representations with an old technique that strongly influenced the field of chemometrics: The non-linear iterative partial least squares algorithm~\cite{wold1975soft}.
Our algorithm aims at mapping the input data on a low-dimensional subspace explaining a high amount of information of the output variable and at the same time a small difference between first and second moments of the domain-specific samples.
The directions of this subspace are computed consecutively as closed-form solution of a convex optimization problem.
Each iteration of our algorithm is followed by matrix deflation yielding orthogonal, domain-invariant latent variables with high predictive power \wrt~the output variable in the source domain.
Our method is called \textit{domain-invariant iterative partial least squares} (DIPALS).

The rest of this section is organized as follows:
Subsection~\ref{subsec:related_work_dipls} gives a brief overview of related works.
Subsection~\ref{subsec:problem_dipls} describes the problem.
Subsection~\ref{subsec:learning_bound_dipls} gives a motivating learning bound.
Subsection~\ref{subsec:algo_dipls} describes our algorithm.
Subsection~\ref{subsec:parameter_setting} proposes a parameter heuristic for the regularization parameter.
Finally, Subsection~\ref{subsec:experiments_dipls} compares our algorithm to different domain adaptation techniques on two benchmark datasets from analytical chemistry.

\subsection{Related Work}
\label{subsec:related_work_dipls}

State-of-the-art domain adaptation algorithms are summarized in Subsection~\ref{subsec:da_algorithms}.

In the present contribution we introduce an algorithm for regression that takes into account three observations from analytical chemistry: A linear input-output relationship, approximately normally distributed data and multicollinearity among input dimensions.

In contrast to non-linear kernel based approaches, we aim at a linear projection of the data motivated by Beer Lambert's law.
In contrast to linear kernel based approaches, we aim at distribution similarity by considering also second moments which we motivate by approximately normally distributed data.
In contrast to neural network based approaches, we compute an orthogonal projection leading to a small number of latent variables.
In each iteration, we obtain one coordinate of the projection as closed-form solution which is motivated by a small number of needed iterations induced by the orthogonality and high input collinearity.

\subsection{Problem Formulation}
\label{subsec:problem_dipls}

In this section, we consider the problem of unsupervised domain adaptation for regression under the covariate-shift assumption.
Our formulation follows Problem~\ref{problem:da_for_binary_classification}.
For simplicity we assume equal sample sizes for the source and the target domain.

\Needspace{15\baselineskip}
\begin{restatable}[Unsupervised Domain Adaptation for Regression]{problemrep}{domain_adaptation_for_regression}%
	\label{problem:da_for_regression}%
	Consider a source domain $\left(p,l\right)$ and a target domain $\left(q,l\right)$ with $p,q\in\MR$ and some labeling function $l:\mathbb{R}^d\to [0,1]$.
	
	Given a source sample $X_p=\{\x_1,\ldots,\x_k\}$ drawn from $p$ with corresponding labels $Y_p=\{l(\x_1),\ldots,l(\x_k)\}$ and a target sample $X_q=\{\x_1',\ldots,\x_k'\}$ drawn from $q$ without labels, find some function $f:\mathbb{R}^d\to [0,1]$ with a small target risk
	\begin{align}
		\label{eq:target_risk_domain_adaptation_for_regression}
		\int_{\mathbb{R}^d} \left|f(\x)-l(\x)\right| q(\x)\diff\x.
	\end{align}%
\end{restatable}

\subsection{Motivating Learning Bound}
\label{subsec:learning_bound_dipls}

In the following, we motivate our algorithm by means of a new learning bound under three typical characteristics often observed in chemical data: Linear dependency between input and output, multicollinearity of input signals, and, approximately normally distributed data.

Therefore, let us consider two integrable functions $g:\mathbb{R}^d\to\mathbb{R}^s\in\mathcal{G}$ and $f:\mathbb{R}^s\to[0,1]\in\mathcal{F}$.
Assume that the probability density functions $\tilde p$ and $\tilde q$ of the pushforward measures ${\mu\circ g^{-1}}$ and ${\nu\circ g^{-1}}$, respectively, exist, where $\mu$ and $\nu$ are the probability measures corresponding to $p$ and $q$, respectively.
Let $d_\mathrm{KL}(\mathcal{N}_{{\tilde p}},\mathcal{N}_{{\tilde q}})$ denote the KL-divergence between the two probability density functions $\mathcal{N}_{\tilde p}$ and $\mathcal{N}_{\tilde q}$ of the Normal distributions with equal mean and covariance as $\tilde p$ and $\tilde q$, respectively.
Based on these notations, we obtain the following statement.
\Needspace{10\baselineskip}
\begin{restatable}[]{thmrep}{dipls}%
	\label{thm:dipls}%
	Consider two domains $(p,l), (q,l)$ and the function $f\circ g$ inducing the latent distributions $\tilde p, \tilde q$ and the Normal distributions $\mathcal{N}_{{\tilde p}}, \mathcal{N}_{{\tilde q}}$ as defined above.
	Then the following holds:
	\begin{equation}
		\label{eq:err_inequality}
		\int \left|f\circ g-l\right| q \leq \int \left|f\circ g-l\right| p + \sqrt{2 d_\mathrm{KL}(\mathcal{N}_{{\tilde p}},\mathcal{N}_{{\tilde q}})} + \lambda^* + \sqrt{ 8\epsilon}
	\end{equation}
	where
	\begin{equation}
		\epsilon=\max\left\{d_\mathrm{KL}(\mathcal{N}_{\tilde p}, {\tilde p}),d_\mathrm{KL}(\mathcal{N}_{\tilde q}, {\tilde q})\right\}
	\end{equation}
	and
	\begin{equation}
		\lambda^* = \inf_{f\in\mathcal{F}} \left(\int \left|f\circ g-l\right| q+\int \left|f\circ g-l\right| p\right).
	\end{equation}
\end{restatable}
\begin{proof}
	Following~\cite{ben2007analysis}, we define the labeling functions ${l_p:\mathbb{R}^s\to [0,1]}$ by
	\begin{align*}
		l_p(\vec a)=\frac{\int_{\{\x\mid g(\x) =\vec a\}} l(\x) p(\x) \diff\x }{\int_{\{\x\mid g(\x) =\vec a\}} p(\x)\diff\x}
	\end{align*}
	and $l_q$ analogously.
	Applying Theorem~\ref{thm:lone_domain_adaptation_bound} together with Theorem~\ref{thm:TV_distance} to the two domains $(p,l_p)$ and $(q,l_q)$ yields
	\begin{align*}
		\int \left|f-l_q\right| \tilde q \leq \int \left|f-l_p\right| \tilde p + 2 d_\text{TV}(\tilde p,\tilde q)+\inf_{f\in\mathcal{F}} \left(\int \left|f-l_q\right| \tilde q+\int \left|f-l_p\right| \tilde p\right)
	\end{align*}
	where $d_\text{TV}$ refers to the total variation distance.
	From the ``change of variables'' Theorem~4.1.11 in~\cite{dudley2002real} we obtain
	\begin{align*}
		\int \left| f-l_p \right| \tilde p
		=\int \left| f-l_p \right| \diff(P\circ g^{-1})
		= \int \left|f-l_p\right|\circ g\diff P
		= \int \left| f\circ g-l\right| p
	\end{align*}
	which, together with the application of the Triangle inequality for $d_\text{TV}$, implies that
	\begin{align*}
		\int \left| f\circ g-l\right| q \leq \int \left| f\circ g-l\right| p  +\lambda^*+ 2 d_\text{TV}(\mathcal{N}_{\tilde p},\mathcal{N}_{\tilde q})+ 2 d_\text{TV}(\tilde p,\mathcal{N}_{\tilde p}) + 2 d_\text{TV}(\tilde q,\mathcal{N}_{\tilde q}).
	\end{align*}
	Eq.~\eqref{eq:err_inequality} then follows from Theorem~\ref{thm:relationships_measurable} and the definition of $\epsilon$.
\end{proof}\\\\
Theorem~\ref{thm:dipls} shows that the error in the target domain can be bounded in terms of the error in the source domain, the KL-divergence between Normal approximations of the latent distributions, a corresponding approximation error $\epsilon$ and the domain adaptation error $\lambda^*$.
Sample-based upper bounds can be obtained by means of Theorem~\ref{thm:problem_solution}.

Theorem~\ref{thm:dipls} suggests a small target error if the terms on the right-hand side of Eq.~\eqref{eq:err_inequality} are small.
In the following, we motivate different algorithmic properties under which, in combination with the three observations from chemical measurements, each of these terms can be expected to be small.

{\bf Domain Adaptation Error $\lambda^*$:} Beer Lambert's law states a linear relationship between output variables and inputs.
Therefore, we assume a target function $l:\mathbb{R}^d\to [0,1]$ that is well approximable by a linear function, \ie~$l(\vec x)\approx \vec x^\text{T} \vec d$ for some $\vec d\in\mathbb{R}^d$.
For such a target function $l$ and each linear function $g:\mathbb{R}^d\to\mathbb{R}^s$ with $g(\vec x)=(\vec x\T\vec A)\T$ and orthogonal matrix $\vec A\in\mathbb{R}^{d\times s}$, it always exits a linear function $f\in\mathcal{F}$, \eg~$f(\vec x)=x\T \vec A\T\vec d$, such that $l\approx f\circ g$ and $\lambda^*\approx 0$.
We therefore aim at finding a function $f\circ g$ with orthogonal projection $g:\mathbb{R}^d\to\mathbb{R}^s$ and linear function $f:\mathbb{R}^s\to [0,1]$.

{\bf Source Error $\int\left|f\circ g-l\right|$:} To overcome numerical instabilities caused by the observed high multicollinearity of the input data, the non-linear iterative partial least squares algorithm has been proposed to find a linear latent variable model $f\circ g$ as defined above with a small source error.
This algorithm serves as a starting point for our method.

{\bf Approximation Error $\epsilon$:} One implication of the assumption of approximately normally distributed input data is that the application of the linear transformation $g$ leads to latent densities ${\tilde p}$ and ${\tilde q}$ that are well approximable by Normal densities.
It is therefore reasonable to assume a small $\epsilon$ in Theorem~\ref{thm:dipls}.
Similarly to the error in our analysis in Chapter~\ref{chap:learning_bounds}, the term $\epsilon$ can be interpreted as an upper bound on the information stored in the densities $p$ and $q$ in addition to the first two moments~\cite{cover2012elements}.

{\bf Distribution Divergence $d_\mathrm{KL}(\mathcal{N}_{\tilde p},\mathcal{N}_{\tilde q})$:} It follows from \eg~Theorem~30.2 in~\cite{billingsley2008probability} that the convergence $d_\mathrm{KL}(\mathcal{N}_{p_n}, \mathcal{N}_{p_\infty})\to 0$ for $n\to\infty$ of some zero mean centered distributions ${p_n,n\in\mathbb{N}}$ and $p_\infty$ is implied by the convergence of the respective covariances ${\sigma_n\to\sigma_{\infty}}$.
This motivates us to aim at zero means and similar covariance matrices of $\tilde p$ and $\tilde q$.

\subsection{Domain-Invariant Iterative Partial Least Squares}
\label{subsec:algo_dipls}

Let $\vec X_p\in\mathbb{R}^{k\times d}$ and $\vec X_q\in\mathbb{R}^{k\times d}$ be the two matrices consisting of all input signals $\x$ as rows $\x\T$ and let $\y\in \mathbb{R}^{k}$ be the vector of corresponding outputs.

As motivated in Subsection~\ref{subsec:learning_bound_dipls}, we aim at computing linear functions $f:\mathbb{R}^s\to \mathbb{R}$ with $f(\vec t)=\vec t\T\vec c$ for $\vec c\in\mathbb{R}^{s}$ and $g:\mathbb{R}^d\to\mathbb{R}^s$ with $g(\vec x)=(\vec x\T\vec A)\T$ for orthogonal $\vec A\in\mathbb{R}^{d\times s}$ such that the source error is minimized and the sample covariance matrices of the latent samples $\vec X_p \vec A$ and $\vec X_q \vec A$ are similar.
To handle collinearity in the inputs, we rely on a regularized version of the non-linear iterative partial least squares algorithm.

{\bf Step 0 (Initialization):} The initial step of our algorithm consists of zero mean centering of the inputs and outputs such that $\E[\vec X_p]=\E[\vec X_q]={\E[\vec y]=0}$ where $\E[\vec X]$ refers to the column-wise empirical mean of the matrix $\vec X$.

Then, we follow the basic ideas of the non-linear iterative partial least squares algorithm by iterating over the following steps to compute one direction of the latent mapping and a corresponding regression coefficient after another.

{\bf Step 1 (Domain-Invariant Projection):}
The following objective function is considered:
\begin{align}
	\label{eq:objective_function}
	\min_{\vec w\T\vec w=1} \norm{\vec X_p - \vec y \vec w\T}^2_\text{F} + \gamma \vec w\T \boldsymbol\Lambda \vec w
\end{align}
where $\norm{.}_\text{F}$ refers to the Frobenius norm, $\gamma$ is the domain-regularization parameter and 
\begin{align}
	\label{eqn:BigL}
	\boldsymbol\Lambda=\vec K \mathrm{diag}(|\lambda_1|,\ldots,|\lambda_d|) \vec K\T
\end{align}
is the matrix obtained by taking the absolute value of all eigenvalues $\lambda_1,\ldots,\lambda_d$ in the eigendecomposition
\begin{align}
	\label{eq:diff_matrix}
	\begin{split}
		\vec K \mathrm{diag}&(\lambda_1,\ldots,\lambda_d) \vec K\T = \frac{1}{k-1}\vec X_p^\text{T} \vec X_p-\frac{1}{k-1}\vec X_q^\text{T} \vec X_q
	\end{split}
\end{align}
with corresponding eigenvector matrix $\vec K$ of the difference of the domain-specific covariance matrices.
The first term in Eq.~\eqref{eq:objective_function} corresponds to the ordinary non-linear iterative partial least squares objective and its minimum is obtained by the direction $\vec w$ where $\vec X_p$ has maximum sample covariance with $\vec y$~\cite{wold1975soft}.
The second term in Eq.~\eqref{eq:objective_function} is our contribution and represents an upper bound on the absolute difference between the source sample variance and the target sample variance in the direction $\vec w$, see Subsection~\ref{subsec:properties_dipls} for its discussion.
The unique solution of Eq.~\eqref{eq:objective_function} is achieved by the vector
\begin{align}
	\label{eq:optimal_w}
	\vec w\T = \frac{\vec y\T\vec X_p}{\vec y\T\vec y}\left(\vec I + \frac{\gamma}{\vec y\T\vec y}\boldsymbol\Lambda\right)^{-1}
\end{align}
divided by its length $\vec w\T \vec w$.
The coordinates $\vec t_p$ and $\vec t_q$ of the projections corresponding to the direction $\vec w$ can be computed by
\begin{align}
	\vec t_p=\vec X_p \vec w\quad\text{and}\quad \vec t_q=\vec X_q \vec w.
\end{align}

{\bf Step 2 (Regression):}
Classical ordinary least squares regression of $\vec y$ on $\vec t_p$ yields 
\begin{align}
	c=(\vec t_p\T \vec t_p)^{-1} \vec t_p\T \vec y.
\end{align}

{\bf Step 3 (Deflation):}
Following the Gram-Schmidt process, our algorithm removes the variation in $\vec X_p$ explained by the current latent variable by subtracting the projection of $\vec X_p$ along $\vec t_p$, \ie~the following update is performed
\begin{align}
	\vec X_p = \vec X_p - \vec t_p(\vec t_p\T \vec t_p)^{-1}\vec t_p\T \vec X_p.
\end{align}
The matrix $\vec X_q$ is updated analogously by means of $\vec t_q$.
After each iteration, the coordinates of the vectors are properly aggregated to obtain the final regression vector $\vec b$ such that $f(g(\vec x))=\vec x\T\vec b$.
See Algorithm~\ref{alg:dipls} for the formulas and~\cite{geladi1986partial} for its derivations.
The projection matrix $\vec A$ such that $g(\vec x)=(\vec x\T\vec A)\T$ can be computed by the relationship \cite{MANNE1987187}:
\begin{align}
	\vec A = \vec W(\vec P\T \vec W)^{-1}.
\end{align}

\SetKwInOut{StepCombine}{Combine}
\begin{algorithm}
	\SetAlgoLined
	\KwIn{Source sample $\vec X_p$, labels $\vec y$ , target sample $\vec X_q$, number of latent variables $s$ and regularization weighting $\gamma$
	}
	\KwOut{Regression vector $\vec b \in\mathbb{R}^d$ such that $f(g(\vec x))=\vec x\T\vec b$}~\\
	
	\Init{Set $\vec W=(\vec w_i,\ldots,\vec w_s)$, $\vec P=(\vec p_i,\ldots,\vec p_s)$, $\vec c=(c_1,\ldots,c_s)$, $\vec y_0 = \vec y - \E[\vec y]$, $\vec S_0 = \vec X_p - \E[\vec X_p]$ and $\vec T_0 = \vec X_q - \E[\vec X_q]$}
	\For{$i\in\{1,\ldots,s\}$}
	{
		\Stepone{
			Compute eigenvalues $\lambda_1^{(i)},\ldots,\lambda_d^{(i)}$ and
			eigenvector matrix $\vec K_i$ of
			$\frac{1}{k-1}\vec S_{i-1}^\text{T} \vec S_{i-1}-\frac{1}{k-1}\vec T_{i-1}^\text{T} \vec T_{i-1}$
			and define
			$\vec v_i\T = \frac{\vec y_i\T\vec S_{i-1}}{\vec y_i\T\vec y_i}\left(\vec I + \frac{\gamma}{\vec y_i\T\vec y_i}\boldsymbol\Lambda_i\right)^{-1}$
			such that $\vec w_i=\vec v_i/\norm{\vec v_i}_2$
			with $\boldsymbol\Lambda_i=\vec K_i \mathrm{diag}(|\lambda_1^{(i)}|,\ldots,|\lambda_d^{(i)}|) \vec K_i\T$
		}
		\Steptwo{
			Set $c_i = (\vec t_p\T \vec t_p)^{-1} \vec t_p\T \vec y_i$ with $\vec t_p  = \vec S_{i-1} \vec w_{i}$, $\vec t_q  = \vec T_{i} \vec w_i$
		}
		\Stepthree{
			Set $\vec p\T_i  = (\vec t_p\T \vec t_p)^{-1}\vec t_p \T \vec S_{i-1} $, $\vec q\T_i  = (\vec t_p \T \vec t_q )^{-1}\vec t_q \T \vec T_{i-1}$,
			$\vec S_i  = \vec S_{i-1}  -\vec t_p \vec p_i \T $,
			$\vec T_i  = \vec T_{i-1}  -\vec t_q \vec q_i \T$ and
			$\vec y_i = \vec y_{i-1} - c_{i}\vec t_p$
		}
	}
	\StepCombine{$\vec b=\vec W(\vec P\T \vec W)^{-1}\vec c$}
	\caption[Domain-invariant iterative partial least squares.]{Domain-invariant iterative partial least squares (DIPALS)}
	\label{alg:dipls}
\end{algorithm}

\subsection{Parameter Heuristic}
\label{subsec:parameter_setting}

Consider the matrix $\boldsymbol{\Lambda}$ from Eq.~\eqref{eqn:BigL} and the vector $\vec w_0$ corresponding to the unconstrained objective function of the non-linear iterative partial least squares algorithm, \ie~$\gamma=0$ in Eq.~\eqref{eq:objective_function}.
We propose to use the value
\begin{align}
	\label{eqn:param_heuristic}
	\gamma_i=\frac{\lVert\vec S_i - \vec y_i \vec w_0\T\rVert^2_\text{F}}{\vec w_0\T\boldsymbol\Lambda\vec w_0}
\end{align}
differently in each iteration of Step~1 in Algorithm~\ref{alg:dipls}.
This setting leads to equal weighting of the terms in the objective Eq.~\eqref{eq:objective_function} in the direction $\vec w_0$.

\subsection{Properties of Algorithm}
\label{subsec:properties_dipls}

The optimum of the first term in the objective function in Eq.~\eqref{eq:objective_function} is achieved by the direction $\vec w_0$ where the sample covariance $\frac{1}{k-1}\vec w_0\T\vec X_p\T\vec y$ between $\vec X_p$ and the output vector $\vec y$ is maximal~\cite{geladi1986partial}.
As a result, the classical non-linear iterative partial least squares algorithm well handles multicollinearity of the input sample.

The value of our regularizer $\vec w^\text{T} \boldsymbol\Lambda \vec w$ with $\vec w\T \vec w=1$ is nothing but the value of the Rayleigh quotient of the positive semi-definite matrix $\boldsymbol \Lambda$.
It is therefore convex and its summation preserves the convexity of the original non-linear iterative partial least squares objective, \ie~the first term in Eq.~\eqref{eq:objective_function}.
As a result, the unique solution of the objective function can be obtained as the root of its derivative and has the form of Eq.~\eqref{eq:optimal_w}.

Our regularizer is an upper bound on the absolute difference
\begin{align}
	\label{eq:non_conv_obj}
	\left| \frac{1}{k-1}\vec w\T\vec X_p^\text{T} \vec X_p\vec w-\frac{1}{k-1}\vec w\T\vec X_q\T \vec X_q\vec w \right|
\end{align}
between the domain-specific sample variances in the direction $\vec w$.
To see this, consider the eigenvector matrix $\vec K$ and the eigenvalues $\lambda_1,\ldots,\lambda_d$ as in Eq.~\eqref{eq:diff_matrix}.
Then, by letting $\vec v=(v_1,\ldots,v_d)=\vec K\T\vec w$, Eq.~\eqref{eq:non_conv_obj} is equal to
\begin{align*}
	|\vec w\T\vec K \mathrm{diag}(\lambda_1,\ldots,\lambda_d)&\vec K\T\vec w|
	= |v_1^2\lambda_1+\ldots+v_d^2\lambda_d|
	\leq |v_1^2\lambda_1|+\ldots+|v_d^2\lambda_d|\\
	&= v_1^2|\lambda_1|+\ldots+v_d^2|\lambda_d|
	= \vec v\T \mathrm{diag}(|\lambda_1|,\ldots,|\lambda_d|)\vec v
	= \vec w\T \boldsymbol{\Lambda}\vec w.
\end{align*}
This shows that the proposed regularizer corresponds to an upper bound on the difference between the source sample variance and the target sample variance in the direction $\vec w$.
It can therefore be interpreted as biasing the non-linear iterative partial least squares solutions towards directions with a low variance difference between the domains in the projection space.

The derivations above allow to interpret the regularization strength $\gamma$ as trade-off between high input-output covariance in the source domain and low variance difference between the domains.
This leads to intuitive heuristics for default values of $\gamma$ as proposed in Subsection~\ref{subsec:parameter_setting}.
With the value of $\gamma$ as in Eq.~\eqref{eqn:param_heuristic} we articulate our preference of treating regression and domain alignment as equal important in a range around the optimal non-linear iterative partial least squares solution.

\subsection{Empirical Evaluations}
\label{subsec:experiments_dipls}

In this subsection we compare our method with several state-of-the-art domain adaptation techniques on two benchmark datasets from analytical chemistry.

\paragraph{Datasets}
We consider two benchmark datasets from analytical chemistry: The \textit{Corn} dataset and the \textit{Tablets} dataset.
The Corn dataset is a well established dataset used to benchmark instrument standardization algorithms in analytical chemistry and comprises near-infrared spectra from a set of $80$ corn samples measured on $3$ similar spectrometers (m5,mp5 and mp6)\footnote{\url{http://www.eigenvector.com/data/Corn/} (accessed April 11, 2018)}.
The goal is to predict oil, water, starch and protein contents from the corresponding spectra.
The Tablets dataset was originally published by the \textit{international diffuse reflectance conference} in 2002 and consists of near-infrared spectra of $654$ pharmaceutical tablets recorded on two spectrometers at $650$ individual wavelengths\footnote{\url{http://www.eigenvector.com/data/tablets/} (accessed January 14, 2019)}.
The goal is to predict the active pharmaceutical ingredient concentration from the near-infrared spectra.

\paragraph{Validation Procedure}
For the Corn dataset we consider domain adaptation between the different instruments by defining the source domain as the first $40$ samples and the target domain as the following $40$ samples of the dataset.
Given the four output variables, this translates into $24$ domain adaptation scenarios.
We split the target domain data randomly into an unlabelled training and a test set comprising $24$ and $16$ samples, respectively.
For the Tablets dataset, we split the data into a calibration, a validation and a test set comprising $155$, $40$ and $460$ samples measured on both instruments.
We consider domain adaptation between calibration and test sets from the two instruments including the wavelength range $600$-$1600$ nanometres and proceed in analogy with the experiments on the Corn dataset. We compare the following approaches:
\begin{itemize}
	\item \textit{Partial Least Squares} (PLS)~\cite{wold1975soft}: The number of latent variables is searched in the set $\{1,\ldots,12\}$ using $10$-fold cross-validation in the source domain.
	\item \textit{Correlation Alignment}~\cite{sun2016return}: CORAL applied to the projections of the training set of the PLS model followed by ordinary least squares regression.
	\item \textit{Transfer Component Analysis}~\cite{pan2011domain}: Motivated by Beer Lambert's law, we use a linear kernel. We search the best number of latent variables in the set $\{1,\dots,30\}$ and $\mu \in \{10^{-10},10^{-9},\dots,1\}$ by using all labels in the target domain.
	\item \textit{Joint Distribution Optimal Transport} (JDOT)~\cite{courty2017joint}: We vary $\alpha \in \{10^{-10},10^{-9},\dots,1\}$ as proposed in the original paper.
	In addition, we vary the linear kernel ridge regression parameter $\lambda$ in the range $\{10^{-10},10^{-9},\dots,1\}$. Both parameters are tuned using target labels.
	\item \textit{DIPALS} with heuristic: The parameter $\gamma$ is set for each latent variable using the parameter heuristic described in Subsection~\ref{subsec:parameter_setting}.
	\item \textit{DIPALS} with source training: The parameter $\gamma$ is trained by means of $10$-fold cross-validation on the source data in the set $\{0.1,1,\dots,10^{10}\}$ for $i \in \{1,\dots,s\}$ latent variables.
\end{itemize}
Note that we apply the target test set for searching the best parameters for TCA and JDOT.
Without using target labels, we were not able to get competitive results on our datasets.

\paragraph{Results}
The domain differences observed in the Corn datasets occur mainly due to changes in the instruments' response and are mostly manifested in offsets between the corresponding spectra.
All in all, we found similar performance of TCA and DIPALS with slightly better results with the former for prediction of oil content and with the latter when predicting moisture as shown in Table~\ref{tab:Corn}.
Although JDOT could improve the accuracy on the target task for determination of protein and starch compared to the PLS, accuracy was significantly lower in most scenarios compared to DIPALS despite tuning of the hyper parameters using target labels.
Finally, no improvement of the PLS model could be achieved with CORAL.

\begin{table}[ht]
	\centering
	\resizebox{\linewidth}{!}{%
		\begin{tabular}{ |l|l|c|c|c|c|c|c| }
			\hline
			Response & Scenario & NIPALS & CORAL & TCA (Sup) & JDOT (Sup) & DIPALS (Heur) & DIPALS (Source) \\
			\hline
			\multirow{6}{*}{Protein}     & m5$\rightarrow$mp5  & 0.68$\pm$0.13 & 0.65$\pm$0.14 & 0.40$\pm$0.03 & 0.61$\pm$0.06 & \bf 0.38$\pm$0.08        & 0.39$\pm$0.09 \\
			& m5$\rightarrow$mp6  & 0.70$\pm$0.12 & 0.77$\pm$0.11 & 0.41$\pm$0.05 & 0.57$\pm$0.04 & 0.\bf 41$\pm$0.03        & 0.42$\pm$0.10 \\
			& mp5$\rightarrow$m5  & 0.70$\pm$0.11 & 0.71$\pm$0.09 & \bf 0.43$\pm$0.09 & 0.57$\pm$0.05 & 0.44$\pm$0.07        & 0.44$\pm$0.08 \\
			& mp5$\rightarrow$mp6 & 0.65$\pm$0.12 & 0.66$\pm$0.12 & \bf 0.36$\pm$0.04 & 0.59$\pm$0.07 & 0.50$\pm$0.08        & 0.49$\pm$0.04 \\
			& mp6$\rightarrow$m5  & 0.69$\pm$0.10 & 0.69$\pm$0.08 & 0.43$\pm$0.09 & 0.58$\pm$0.07 & \bf 0.43$\pm$0.07        & 0.43$\pm$0.12 \\
			& mp6$\rightarrow$mp5 & 0.66$\pm$0.11 & 0.61$\pm$0.15 & 0.41$\pm$0.05 & 0.44$\pm$0.05 & 0.40$\pm$0.06        & \bf 0.36$\pm$0.03 \\
			\hline
			\multirow{6}{*}{Starch}      & m5$\rightarrow$mp5  & 1.19$\pm$0.17 & 1.11$\pm$0.19 & 0.70$\pm$0.08 & 0.76$\pm$0.08 & 0.69$\pm$0.18        & \bf 0.66$\pm$0.10 \\
			& m5$\rightarrow$mp6  & 1.08$\pm$0.17 & 1.11$\pm$0.21 & 0.67$\pm$0.10 & 0.75$\pm$0.08 & \bf 0.64$\pm$0.15        & 0.68$\pm$0.12 \\
			& mp5$\rightarrow$m5  & 1.38$\pm$0.26 & 1.30$\pm$0.17 & 0.68$\pm$0.11 & 0.83$\pm$0.06 & 0.71$\pm$0.08        & \bf 0.67$\pm$0.15 \\
			& mp5$\rightarrow$mp6 & 1.20$\pm$0.16 & 1.27$\pm$0.13 & \bf 0.68$\pm$0.09 & 0.82$\pm$0.08 & 0.80$\pm$0.14        & 0.72$\pm$0.15 \\
			& mp6$\rightarrow$m5  & 1.38$\pm$0.17 & 1.48$\pm$0.23 & \bf 0.64$\pm$0.10 & 0.82$\pm$0.07 & 0.76$\pm$0.14        & 0.69$\pm$0.15 \\
			& mp6$\rightarrow$mp5 & 1.21$\pm$0.14 & 1.30$\pm$0.11 & \bf 0.55$\pm$0.08 & 0.81$\pm$0.06 & 0.72$\pm$0.18        & 0.89$\pm$0.19 \\
			\multirow{6}{*}{Oil}         & m5$\rightarrow$mp5  & 0.27$\pm$0.03 & 0.27$\pm$0.05 & \bf 0.15$\pm$0.02 & 0.20$\pm$0.03 & 0.19$\pm$0.03        & 0.19$\pm$0.01 \\
			\hline
			& m5$\rightarrow$mp6  & 0.24$\pm$0.05 & 0.30$\pm$0.03 & \bf 0.14$\pm$0.01 & 0.22$\pm$0.03 & 0.17$\pm$0.02        & 0.18$\pm$0.04 \\
			& mp5$\rightarrow$m5  & 0.21$\pm$0.02 & 0.24$\pm$0.03 & \bf 0.16$\pm$0.02 & 0.22$\pm$0.02 & 0.26$\pm$0.04        & 0.21$\pm$0.02 \\
			& mp5$\rightarrow$mp6 & 0.21$\pm$0.03 & 0.21$\pm$0.03 & \bf 0.15$\pm$0.02 & 0.20$\pm$0.03 & 0.20$\pm$0.02        & 0.21$\pm$0.03 \\
			& mp6$\rightarrow$m5  & 0.22$\pm$0.02 & 0.23$\pm$0.03 & \bf 0.17$\pm$0.01 & 0.22$\pm$0.03 & 0.23$\pm$0.05        & 0.19$\pm$0.02 \\
			& mp6$\rightarrow$mp5 & 0.20$\pm$0.03 & 0.22$\pm$0.03 & \bf 0.16$\pm$0.02 & 0.21$\pm$0.02 & 0.21$\pm$0.04        & 0.18$\pm$0.01 \\
			\multirow{6}{*}{Moisture}    & m5$\rightarrow$mp5  & 0.24$\pm$0.03 & 0.28$\pm$0.04 & 0.24$\pm$0.02 & 0.30$\pm$0.04 & \bf 0.22$\pm$0.04        & 0.22$\pm$0.08 \\
			\hline
			& m5$\rightarrow$mp6  & 0.27$\pm$0.03 & 0.27$\pm$0.04 & 0.27$\pm$0.02 & 0.31$\pm$0.04 & 0.25$\pm$0.03        & \bf 0.20$\pm$0.02 \\
			& mp5$\rightarrow$m5  & 0.26$\pm$0.03 & 0.27$\pm$0.03 & 0.29$\pm$0.03 & 0.27$\pm$0.05 & \bf 0.23$\pm$0.06        & 0.25$\pm$0.06 \\
			& mp5$\rightarrow$mp6 & 0.27$\pm$0.02 & 0.26$\pm$0.04 & 0.28$\pm$0.02 & 0.31$\pm$0.03 & 0.23$\pm$0.04        & \bf 0.20$\pm$0.03 \\
			& mp6$\rightarrow$m5  & 0.28$\pm$0.03 & 0.28$\pm$0.03 & 0.31$\pm$0.03 & 0.24$\pm$0.01 & \bf 0.22$\pm$0.03        &  0.22$\pm$0.05 \\
			& mp6$\rightarrow$mp5 & 0.27$\pm$0.03 & 0.26$\pm$0.02 & 0.26$\pm$0.03 & 0.31$\pm$0.04 & \bf 0.17$\pm$0.02        & 0.20$\pm$0.03 \\
			\hline 
		\end{tabular}
	}
	\caption[Average root mean squared error and standard deviation of the proposed Domain-Invariant Partial Least Squares algorithm on the Corn dataset.]{Average root mean squared errors and standard deviations for $10$-fold cross-validation on Corn dataset. The best value for each scenario is indicated in bold. (Sup) indicates supervised hyper-parameter selection using target domain labels, (Heur) indicates parameter setting using Eq.~\eqref{eqn:param_heuristic} and (Source) indicates parameter setting using cross-validation on source.}
	\label{tab:Corn}
\end{table}

Similar to the Corn datasets, the Tablets dataset involves domain adaptation between similar near-infrared spectrometers.
Accordingly, we found similar overall performance of DIPALS and TCA on the target tasks as shown by Table~\ref{tab:tablets}.
In contrast, JDOT could not surpass the performance of the PLS model, which can be explained with the fact that the Tablets dataset contains several $\vec y$-direction outliers, \ie~spectra with wrongly assigned values of active pharmaceutical ingredients, that apparently lead to erroneous transport of the joint distribution.

\begin{table}[ht]
	\centering
	\resizebox{\linewidth}{!}{%
		\begin{tabular}{ |l|c|c|c|c|c|c|c| }
			\hline
			Scenario & NIPALS & CORAL & TCA (Sup) & JDOT (Sup) & DIPALS (Heur) & DIPALS (Source) \\
			\hline
			cal1$\rightarrow$test2 & 9.48$\pm$0.64 & 9.56$\pm$0.58 & 8.50$\pm$0.59 & 12.86$\pm$1.07 & \bf 7.69$\pm$0.47        & 8.04$\pm$0.53 \\
			test2$\rightarrow$cal1 & 8.18$\pm$0.94 & 7.89$\pm$1.01 & 7.23$\pm$1.04 & 10.26$\pm$0.67 & \bf 6.54$\pm$1.25        & 7.58$\pm$1.59 \\
			cal2$\rightarrow$test1 & 8.35$\pm$0.58 & 7.51$\pm$0.73 & \bf 6.63$\pm$0.92 & 13.46$\pm$0.50 & 7.12$\pm$0.68        & 6.75$\pm$0.48 \\
			test1$\rightarrow$cal2 & 8.12$\pm$1.55 & 8.39$\pm$1.33 & \bf 7.26$\pm$1.83 & 10.24$\pm$0.55 & 7.66$\pm$1.66        & 8.43$\pm$1.25\\
			\hline 
		\end{tabular}
	}
	\caption[Average root mean squared error and standard deviation of Domain-Invariant Partial Least Squares on the Tablets dataset.]{Average root mean squared errors and standard deviations for $10$-fold cross-validation on Tablets dataset. The best value for each scenario is indicated in bold. (Sup) indicates supervised hyperparameter selection using target domain labels, (Heur) indicates parameetr setting using Eq.~\eqref{eqn:param_heuristic} and (Source) indicates parameter setting using cross-validation on source.}
	\label{tab:tablets}
\end{table}

\section{Discussion}
\label{sec:conclusion_applications}

In this section, we applied our ideas on two industrial domain adaptation problems.
To extend the scope of problems in this thesis and to show the general applicability of the proposed ideas, we choose two regression problems.

The first problem is in the area of industrial manufacturing.
We propose to transform the data in a new space such that the first moments of samples produced by multiple different tool settings are similar.
In a real world application of industrial manu\-facturing, the proposed methods significantly reduce the prediction error on data originating from already seen tool settings.
The biggest benefit of the proposed method is that it can be applied to unseen data from new unseen tool settings without the need of time and cost intensive collection of training data using these settings.

The second problem is in the area of analytical chemistry.
We consider unsupervised domain adaptation for multivariate regression under linear input-output relationship, multicollinearity and approximately normally distributed domains -- a situation frequently encountered in analytical chemistry.
Motivated by our ideas from Chapter~\ref{chap:cmd_algorithm}, we propose a novel metric-based regularization that performs domain adaptation under the non-iterative partial least squares framework.
Our approach outperforms different state-of-the-art domain adaptation techniques for linear regression on two benchmark datasets from analytical chemistry.
In contrast to state-of-the-art calibration methods in this field, our method does not require labeled calibration samples in the target domain.

Unfortunately, parameter selection becomes an important issue without labels in the target domain.
In the first application this issue is mitigated, but not completely solved, by focusing on domains originating from similar physical tool settings.
In the second problem, our ideas from Chapter~\ref{chap:cmd_algorithm} are applied to utilize unlabeled data from the target domain.

Although better results might be achieved by using target labels, the lack of information in our industrial problem settings is due to high costs in money and resources.
The proposed methods therefore underpin the usefulness of our ideas for resource and money restricted applications.

\newpage

\chapter{Conclusion}
\label{chap:conclusion}

In this thesis we study domain adaptation under weak assumptions on the similarity of source and target distribution.
Our assumptions are based on moment distances which realize weaker similarity concepts than most other common probability metrics.
Under this new setting, we provide new insights to main components of statistical learning:
\begin{itemize}
	\item In Chapter~\ref{chap:learning_bounds} we formalize the novel problem setting, give conditions for the convergence of a discriminative model under this setting and derive bounds describing its generalization ability.
	For smooth densities with weakly coupled marginals, our conditions can be made as precise as required based on the number of moments and the smoothness of the distributions.
	\item In Chapter~\ref{chap:cmd_algorithm} we implement the domain adaptation principle of learning new data representations such that our moment assumptions are satisfied.
	We provide a new moment distance for the regularization of the stochastic optimization of neural networks and provide several properties including some relations to other probability metrics, a dual form, a computationally efficient estimation and a learning bound for the regularization.
	To underpin the relevance of our ideas beyond the conditions studied in Chapter~\ref{chap:learning_bounds}, we perform empirical experiments on a new artificial dataset and $21$ standard benchmark tasks for domain adaptation which are based on $6$ large scale datasets.
	Results show that our method often outperforms related alternatives which are based on stronger assumptions on the similarity of distributions.
	\item In Chapter~\ref{chap:applications} we apply our ideas on two industrial regression problems.
	In contrast to classical approaches in these fields, our new moment-based methods achieve low errors on new domains with missing target labels.
\end{itemize}
Our focus on studying weak assumptions on the similarity of distributions enables straight forward extensions using stronger assumptions, \eg~new learning bounds and algorithms.
All in all, we consider that throughout this thesis we introduce some theoretical and computational novelties that can benefit the field of statistical learning and, in particular, domain adaptation.

With regard to the work discussed in this thesis, we would primarily like to extend the proposed bounds on the difference between distributions by further upper bounding the entropy-based terms in terms of smoothness of log-densities as it is done \eg~in~\cite{barron1991approximation}.
Such bounds can lead to estimates of the number of moments needed such that an underlying smooth distribution is defined up to arbitrary accuracy which is, to the best of our knowledge, an open problem~\cite{schmudgen2017moment,tardella2001note}.
Concerning improved algorithms for domain adaptation, future plans are centered around entropy minimization as suggested by our learning bounds.
Generally in industrial applications with low sample sizes we consider a significant potential for moment distance based domain adaptation as a starting point for developing more problem-specific distance concepts.

\newpage
\bibliographystyle{plain}
\bibliography{thesis.bib}

\chapter*{}

\AddToShipoutPicture*{\BackgroundPic}
\blfootnote{The quote above is attributed to Albert Einstein and written in English with some letters exchanged by numbers, \eg~the letter \textit{E} is exchanged by the number \textit{3} and the letter \textit{A} is exchanged by the number \textit{4}. The text without exchanged numbers is: \textit{The measure of intelligence is the ability to change.} However, the example above shows that humans can adapt from the distribution of english texts to the different distribution underlying the text above, and therefore, are able to solve problems of domain adaptation.}

\end{document}